%% file: uncert_quant.tex
\documentclass{article}
\usepackage[utf8]{inputenc}
\usepackage{custom_tex}
\usepackage{dsfont}
\usepackage{pgf}
\usepackage{enumitem}
\usepackage{multirow}
\usepackage{booktabs}

\title{T-Cal: An optimal test for the calibration of predictive models}

\author{
Donghwan Lee,\footnote{Equal Contribution.} \footnote{Graduate Group in Applied Mathematics and Computational Science, Univ. of Pennsylvania. \texttt{dh7401@sas.upenn.edu}.}
\,
Xinmeng Huang,\footnotemark[1]\,\,\footnote{Graduate Group in Applied Mathematics and Computational Science, Univ. of Pennsylvania. \texttt{xinmengh@sas.upenn.edu}.}
\,
Hamed Hassani,\footnote{Department of Electrical and Systems Engineering, Univ. of Pennsylvania. \texttt{hassani@seas.upenn.edu}.}
\,
and Edgar Dobriban\footnote{Department of Statistics and Data Science, Univ. of Pennsylvania. \texttt{dobriban@wharton.upenn.edu}.}
}
\date{\today}

\newcommand{\cH}{\mathcal{H}}

\newtheorem{assumption}{Assumption}
\numberwithin{assumption}{section}
\newtheorem{remark}{Remark}
\usepackage[margin=1in]{geometry}

\definecolor{ed}{RGB}{225,0,0}

\begin{document}
\maketitle

\sloppy
\begin{abstract}
The prediction accuracy of machine learning methods is steadily increasing, but the calibration of their uncertainty predictions poses a significant challenge.  
Numerous works focus on obtaining well-calibrated predictive models, but less is known about reliably assessing model calibration. This limits our ability to know when algorithms for improving calibration have a real effect, and when their improvements are merely artifacts due to random noise in finite datasets. In this work, we consider detecting mis-calibration of predictive models using a finite validation dataset as a hypothesis testing problem.
The null hypothesis is that the predictive model is calibrated, while the alternative hypothesis is that the deviation from calibration is sufficiently large.

We find that detecting mis-calibration is only possible when the conditional probabilities of the classes are sufficiently smooth functions of the predictions.  When the conditional class probabilities are H\"older continuous, we propose \emph{T-Cal}, a minimax optimal test for calibration based on a debiased plug-in estimator of the $\ell_2$-Expected Calibration Error (ECE).  We further propose \emph{adaptive T-Cal}, a version that is adaptive to unknown smoothness.  We verify our theoretical findings with a broad range of experiments, including with several popular deep neural net architectures and several standard post-hoc calibration methods. T-Cal is a practical general-purpose tool, which---combined with classical tests for discrete-valued predictors---can be used to test the calibration of virtually any probabilistic classification method.
T-Cal is available at \url{https://github.com/dh7401/T-Cal}.
\end{abstract}
\fussy

\tableofcontents
\medskip

\section{Introduction}

\sloppy
The prediction accuracy of contemporary machine learning methods such as deep neural networks is steadily increasing, leading to adoption in more and more safety-critical fields such as medical diagnosis \citep{Esteva2017DermatologistlevelCO}, self-driving vehicles \citep{bojarski2016end}, and recidivism forecasting \citep{berk2017impact}.
In these applications and beyond, machine learning models are required not only to be accurate but also to be well-calibrated: giving precise probability estimates for the correctness of their predictions.

To be concrete, consider a classification problem where the goal is to classify features $\mathbf{x}$ (such as images) into one of several classes $\mathbf{y}$ (such as a building, vehicle, etc.). A probabilistic classifier (or, probability predictor) $f$ assigns to each input $\mathbf{x}$ a probability distribution $f(\mathbf{x})$ over the classes. 
For a given input $\mathbf{x}$, the entries of $f(\mathbf{x})$ represent the probabilities assigned by the classifier to the event that the outcome belongs to the $k$-th class, for any $ k=1,\ldots,K$.
This classifier is \emph{calibrated} if for any value $\mathbf{z}$ taken by $f(\mathbf{x})$, and for all classes $k$, the probability that the outcome belongs to the $k$-th class, i.e., $[\mathbf{y}]_k=1$, equals the predicted probability, i.e., the $k$-th coordinate $[\mathbf{z}]_k$ of $\mathbf{z}$: 
$$P([\mathbf{y}]_k=1|f(\mathbf{x})=\mathbf{z})=[\mathbf{z}]_k.$$

This form of calibration is an important part of uncertainty quantification, decision science, analytics, and forecasting
\citep[see e.g.,][etc]{hilden1978measurement,miller1991validation,miller1993validation,steyerberg2010assessing,hand1997construction,jolliffe2012forecast,van2015calibration,harrell2015regression,tetlock2016superforecasting,shah2018big,steyerberg2019clinical}.
Unfortunately, however, recent works starting from at least \cite{guo2017calibration} have reported that modern machine learning methods are often poorly calibrated despite their high accuracy; which can lead to harmful consequences \cite[e.g.,][]{van2015calibration,steyerberg2019clinical}.

To address this problem, there has been a surge of works aimed at improving the calibration of machine learning models.
These methods seek to achieve calibration either by modifying the training procedure \citep{harrell2015regression,Lakshminarayanan2017SimpleAS, kumar2018trainable, Thulasidasan2019OnMT, zhang2020mix, mukhoti2020calibrating} or by learning a re-calibration function that transforms, in a post-hoc way, the predictions to well-calibrated ones \citep{cox1958two,mincer1969evaluation,steyerberg2010assessing,Platt1999ProbabilisticOF, Zadrozny2001ObtainingCP, Zadrozny2002TransformingCS, guo2017calibration, kumar2019verified, kisamori2020simulator}. 

In this regard, a key challenge is to rigorously assess and compare the performance of calibration methods. 
Without such assessments, we have limited ability to know when algorithms for improving calibration have a real effect, and when their improvements are merely artifacts due to random noise in finite-size datasets. 
As it turns out, existing works do not offer a satisfactory solution to this challenge.

In more detail, in this work, we consider the problem of detecting mis-calibration of predictive models using a finite validation dataset.
We focus on models whose probability predictions are continuously distributed---which is generally reasonable for many modern machine learning methods, including deep neural nets. We develop efficient and provably optimal algorithms to test their calibration.

Detecting mis-calibration has been studied from the perspective of statistical hypothesis testing.
The seminal work of
\cite{cox1958two} formulated a test of calibration for a collection of binary (yes-no) predictions, 
and proposed using a score test for a logistic regression model.
This has been widely used and further developed, leading to various tests for the so-called calibration slope and calibration intercept, which can validate various qualitative versions of model calibration, see e.g., \cite{hosmer1980goodness,miller1991validation,steyerberg2019clinical} and references therein.
In pioneering work,
\cite{miller1962statistical} suggested a chi-squared test for testing calibration of multiple series of binary predictions.
To deal with the challenging problem of setting critical values (i.e., how large of an empirical mis-calibration is statistically significant?) for testing calibration, bootstrap methods have become common, see e.g.,
\cite{harrell2015regression}.
We refer to Section \ref{relw} for more details and for a discussion of other related works.

\begin{figure}
    \centering
    \includegraphics{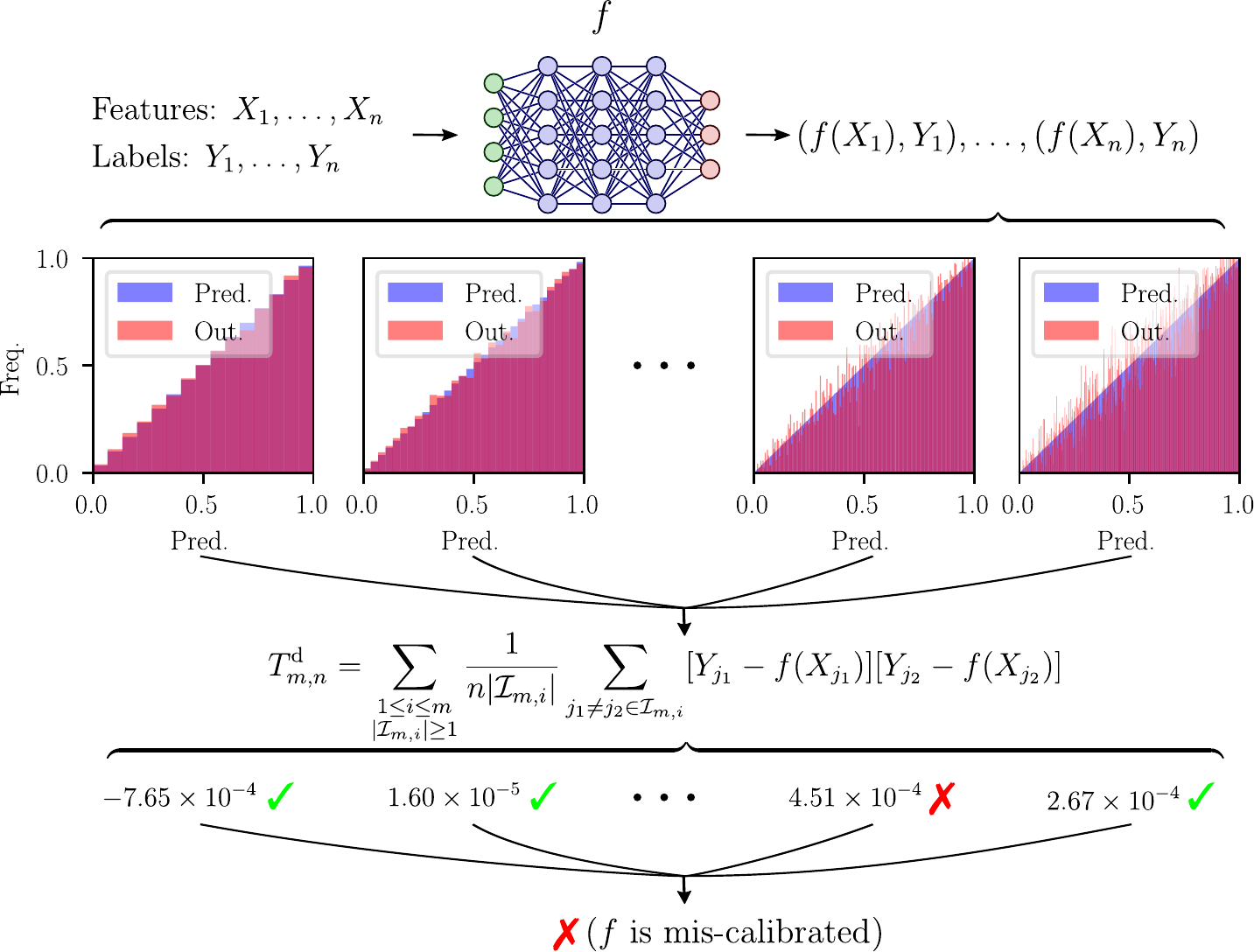}
    \caption{An overview of adaptive T-Cal. For a given probability predictor $f$, we compute $T_{m, n}^\text{d}$, the debiased plug-in estimator (DPE), binned over several scales (See \eqref{eq:onesampledef} for the definition).
    We then compare each value with the hypothetical distribution of DPE that we would get if the model were perfectly calibrated. The hypothesis of perfect calibration is rejected if at least one of the scales is detected to be mis-calibrated. This multi-scale approach ensures that T-Cal adaptively detects mis-calibration.}
    \label{fig:overview}
\end{figure}

In contrast to the above works that aim to test calibration slopes and intercepts, we aim to develop a \emph{nonparametric hypothesis test} for calibration, which does not assume a specific functional form (such as a logistic regression model), for the deviations to be detected from perfect calibration.
A nonparametric approach has the advantage that it can detect subtle forms of mis-calibration even after re-calibration by parametric methods.
However, existing approaches for nonparametric testing often rely on ad hoc techniques for binning the probability predictions, which is a limitation because the results can depend on the way that the binning has been performed \citep{harrell2015regression,steyerberg2019clinical}. In contrast, our \emph{adaptive tests} automatically select an optimal binning scheme.
Finally, 
as a new development in the area of testing calibration, 
T-Cal has theoretically guaranteed \emph{minimax optimality properties} for detecting certain reasonable types of smooth mis-calibration. 
These properties make T-Cal both practically and theoretically appealing.

We consider a given multi-class probabilistic classifier, and are interested in testing if it is calibrated.
We make the following contributions:

\begin{itemize}
  
    \item As a candidate test statistic, we consider the plug-in estimator of  $\ell_2$-expected calibration error (ECE), which is the expectation of the squared distance between the probability predictions and class probabilities given these predictions. 
    This is also known as the mean calibration error \citep[e.g.,][p. 105]{harrell2015regression}. 
    While the plug-in estimator is biased (i.e., its expectation is not zero even under perfect calibration), we show how to construct a \emph{debiased plug-in estimator} (DPE).
    
    We consider detecting mis-calibration when the deviation between predicted class probabilities and their true values---the ``mis-calibration curve"---satisfies a classical smoothness condition known as H\"{o}lder continuity. We later show that
    a smoothness condition is essentially unavoidable.
    Under this condition, we show that T-Cal can detect mis-calibration if the ECE is sufficiently large and the number of bins is chosen appropriately, depending on the smoothness (Theorem \ref{thm:onesamplethm}).

    \item To make T-Cal practical, we present a version that is adaptive to the unknown smoothness parameter (Theorem \ref{thm:adaptiveone}). This makes T-Cal fully tuning-free and practically useful.
    From a theoretical perspective, adaptivity only requires a minor additional increase in the level of mis-calibration that can be detected; by a $\log n$ factor.

    \item We support our theoretical results with a broad range of experiments. We provide simulations, which support our theoretical optimality results. We also provide experiments with several popular deep neural net architectures (ResNet-50, VGG-19, DenseNet-121, etc), on benchmark datasets (CIFAR 10 and 100, ImageNet) and several standard post-hoc calibration methods (Platt scaling, histogram binning, isotonic regression, etc).
    
    \item  To complement these results, we argue that T-Cal is optimal, by providing a number of fundamental lower bounds.  We prove that detecting mis-calibration from a finite dataset is only possible when the mis-calibration curve is sufficiently smooth, and it is not possible when the curve is just continuous (Proposition \ref{impossibility}).
    
    When the mis-calibration curves are H\"older smooth,
    we show that the calibration error required for reliable detection of mis-calibration has to be appropriately large (Theorem \ref{thm:multi-lower-bound}).
    This minimax result relies on Ingster’s (or the chi-squared) method.  Combined with our previous results, this shows that T-Cal is minimax optimal. 
    
    \item 
    To further put our problem in context, we show that testing calibration can be reduced to a well-known problem in statistical inference---the two-sample goodness-of-fit problem---by a novel randomization technique. Based on this insight, and building on the results of \cite{arias2018remember, kim2020minimax} on goodness-of-fit testing, we present another asymptotically minimax optimal test for mis-calibration that matches the lower bound (Theorem \ref{thm:reductiontest}). 
    While this method is theoretically optimal, it relies on sample splitting and is not as sample-efficient as our previous method in experiments.
  
    \item In the proofs, we have the following innovations:
    \begin{enumerate}
        \item We introduce an \emph{equal-volume binning scheme} for the probability simplex $\Delta_{K - 1}$ (Appendix \ref{multibinning}).
        We decompose the probability simplex into hypersimplices by taking intersections with smaller hypercubes composing the unit hypercube $[0, 1]^K$.
        Then we further decompose the hypersimplices into equal-volume simplices using results in polyhedral combinatorics.
        This construction enables us to extend proof techniques from the nonparametric hypothesis testing literature to our setting.
        
        \item
        To analyze our plug-in estimator, we need to deal with terms involving probability scores of inputs, which are continuous random variables. This is different from the structure of chi-squared statistics such as that of \cite{ingster1987minimax}.
        Thus, computing the mean and variance of the DPE requires a different analysis.
                
        \item While densities on the probability simplex can take arbitrary positive values, the conditional expectation of probability predictions has to lie in the probability simplex. This requires a careful construction of alternative distributions to use Ingster's method.
    \end{enumerate}
\end{itemize}

Our numerical results can be reproduced with code available at \url{https://github.com/dh7401/T-Cal}.

We now summarize some key takeaways:
\begin{itemize}
    \item {\bf The need for statistical significance to claim calibration.} It is crucial to perform rigorous statistical tests to assess the calibration of machine learning methods. While models with smaller empirical $\text{ECE}$ generally tend to be better calibrated, these values can be highly influenced by noise and randomness inherent in finite datasets.  
    Hence it is crucial to develop and use tools to assess statistical significance---such as the hypothesis tests of calibration that we develop---when claiming improved calibration.
    \item {\bf Potential suboptimality of popular approaches.}  
    The currently prevalent usage of popular metrics, such as the empirical $\text{ECE}$, may be suboptimal. The current standard is to evaluate mis-calibration metrics using a fixed number of bins (such as 15) of the probability scores, for all prediction models (ResNet, VGG, etc), and all datasets. Our results show theoretically that the optimal number of bins increases with the level of oscillations and non-smoothness expected in the probability predictor. Modern machine learning methods are becoming more and more over-parametrized and data-adaptive. This suggests that it is ever more important to use a careful model- and data-adaptive test (and number of bins) when testing calibration.
\end{itemize}

\subsection{Related Works}
\label{relw}
There is a great body of related work on evaluating the calibration of prediction methods, on improving calibration accuracy, and on nonparametric hypothesis testing techniques. We review the most closely related works. 

\paragraph{Broader Context.}
Broadly speaking, the study of calibration is an important part of the study of classification, prediction,  analytics, and forecasting \citep[e.g.,][etc]{hilden1978measurement,miller1991validation,miller1993validation,steyerberg2010assessing,hand1997construction,jolliffe2012forecast,gneiting2014probabilistic,van2015calibration,harrell2015regression,tetlock2016superforecasting,shah2018big,steyerberg2019clinical}.

\paragraph{Calibration.} As recounted in \cite{lichtenstein1977calibration}, research on calibration dates back at least to the early 1900s, when meteorologists suggested expressing predictions as probabilities and comparing them to observed empirical frequencies.
Calibration has since been studied in a variety of areas, including meteorology, statistics, medicine, computer science, and social science; and under a variety of names, such as realism or realism of confidence, 
appropriateness of confidence,
validity, external validity, secondary validity, and
reliability \citep{lichtenstein1977calibration}.
A general finding in this area is that human forecasters are often overconfident and thus mis-calibrated \cite[e.g.,][etc]{keren1991calibration}, as codified for instance in Tversky and Kahneman's celebrated work on prospect theory \citep{kahneman2013prospect}.

Beyond our hypothesis testing perspective, approaches to study calibration include Bayesian perspectives \cite[e.g.,][etc.]{dawid1982well,kadane1982subjectivist} and online settings \cite[e.g.,][etc.]{foster1998asymptotic,vovk2005good}.
See also Section 10.9 of
\cite{harrell2015regression}, 
Section 15.3 of
\cite{steyerberg2019clinical},
and
\cite{hastie1997classification,ivanov1999ready,garczarek2002classification,buja2005loss,toll2008validation,gebel2009multivariate,serrano2012calibration,van2019calibration, huang2020tutorial}, among others.

\paragraph{Calibration Measures.}
Proper scoring rules \citep{good1952rational,de1962does,savage1971elicitation,winkler1996scoring,degroot1983comparison,gneiting2007probabilistic} such as the  Brier score \citep{brier1950verification} and negative log-likelihood \citep[e.g.,][etc]{winkler1996scoring} are objective functions of two probability distributions (a true distribution and a predicted distribution). They are minimized when the predicted distribution equals the true distribution; see also \cite{bickel2007some}. 

As discussed in 
Sections 4.5 and 10.9 of 
\cite{harrell2015regression}, 
some of the standard techniques in the area include
plotting calibration curves, also known as reliability diagrams (estimated probabilities against predicted ones);
which can be bias-corrected using the bootstrap; 
and re-calibration by fitting statistical models to these curves \citep[e.g.,][etc]{austin2014graphical}.
More recently, the notion of ECE, also known as mean absolute calibration error \citep[e.g.,][p. 105]{harrell2015regression} is popularized in \cite{naeini2015obtaining} and later generalized to multi-class settings in \cite{vaicenavicius2019evaluating}.
\cite{gupta2021calibration} develop a binning-free calibration measure based on the Kolmogorov-Smirnov test.
\cite{arrieta2022metrics} introduce calibration metrics building on Komogorov-Smirnov and Kuiper statistics.

\paragraph{Calibration in Modern Machine Learning.}
\cite{guo2017calibration} draw attention to the mis-calibration of modern neural networks and compare different recalibration methods based on ECE.
Many other works \citep{milios2018dirichlet, kull2019beyond, zhang2020mix} also evaluate their methods using ECE or its variants.
In the following works of \cite{vaicenavicius2019evaluating, kumar2019verified}, it has been recognized that ECE evaluated on a fixed binning scheme can underestimate the calibration error.
The limitation of fixed binning has been known for the analogous problems of testing probability distributions and densities, see e.g., \cite{mann1942choice}, or page 19 of \cite{ingster2012nonparametric}.
\cite{kumar2019verified} proposes a debiased ECE, but only for probability predictors with a finite number of outputs.
\cite{nixon2019measuring} empirically study various versions of ECE obtained by adjusting hyperparameters involved in the estimator of ECE (such as norm, binning scheme, and class conditionality), and find that the choice of calibration measure is crucial when comparing different calibration methods.
\cite{roelofs2022mitigating} propose a heuristic for choosing an optimal number of bins when computing ECE.
\cite{zhang2020mix} use kernel density estimation to estimate ECE without relying on a binning scheme.
\cite{pmlr-v119-zhao20e} show that individual calibration is possible by randomized predictions and propose a training objective to enforce individual calibration.
See also \cite{niculescu2005predicting,kull2017beyond,bai2021don}, among others.

There has been interest in a variety of forms of calibration. We study the strongest form, multi-class calibration, which is stronger than other definitions such as marginal calibration and confidence (top) calibration \citep{vaicenavicius2019evaluating, widmann2019calibration}.

\paragraph{Nonparametric Hypothesis Testing.}
\cite{ingster1986asymptotic} derives the minimax testing rate for two-sample testing where $s$-H\"older continuous densities on $[0, 1]$ are separated in an $L^2$ sense, and shows that the chi-squared test achieves the minimax optimal rate $n^{-2s / (4s + 1)}$.
\cite{ingster1987minimax} extends this result to $L^p$ metrics and derives the minimax optimal rate $n^{-s / (2s + 1 - \max\{2, p\}^{-1})}$ for $1 \leq p < \infty$ and $(n / \log n)^{s / (2s + 1)}$ for $p = \infty$.
\cite{ingster2000adaptive} proposes an adaptive version of the test at the cost of $(\log \log n)^{s / (4s + 1)}$ factor in the minimax rate.
\cite{arias2018remember} extend these results to densities on $[0, 1]^d$ and show the minimax rate $n^{-2s / (4s + d)}$.
\cite{kim2020minimax} prove that a permutation test can also achieve the same optimal rate.
\cite{butucea2006nonparametric} study two-sample testing for one-dimensional densities in Besov spaces; they also prove adaptivity.
See also \cite{balakrishnan2018hypothesis,donoho2015higher,jin2016rare, chhor2021goodness, dubois2021goodness,berrett2021optimal} for reviews and further related works.

\paragraph{Nonparametric Functional Estimation.}
\cite{bickel1988estimating} study the problem of estimating the quadratic integral functional of the $k$-th derivative of $s$-H\"older probability densities on $\R$, and prove that the optimal convergence rate is $n^{-[4(s - k) / (4s + 1) \wedge 1/2]}$.
\cite{donoho1990minimax, brown1996asymptotic} show an analogous result for regression functions in the Gaussian white noise model.
\cite{birge1995estimation} generalize these results to smooth integral functionals of H\"older densities and their derivatives, and prove the same convergence rate.
\cite{kerkyacharian1996estimating} 
provide optimal Haar wavelet-based estimators of cubic functionals of densities over the broader class of  Besov spaces;
and also discuss estimating integrals of other powers of the density.
\cite{laurent1996efficient} 
studies estimation of functionals of the form $\int \phi(f(x),x) d\mu(x)$ of densities $f$, where $\phi$ is a sufficiently smooth function and $\mu$ is a measure.
This work 
constructs estimators attaining the optimal parametric rate using orthogonal projections, including showing semiparametric efficiency, 
when the smoothness $s>d/4$ in dimension $d$.
\cite{robins2008higher, gine2008simple, tchetgen2008minimax} introduce estimation method using higher-order U-statistics.
\cite{efromovich1996optimal, cai2006optimal, gine2008simple, mukherjee2015lepski} propose estimation methods adaptive to unknown smoothness $s$ based on Lepski's method \citep{lepski1991problem, lepski1997optimal}. 
See \cite{gine2021mathematical} for a more thorough review of related literature.

\paragraph{Hypothesis Testing for Calibration.}
\cite{cox1958two} formulates a test of calibration for a collection of Bernoulli random variables, as a test that their success probabilities are equal to some given values; and proposed using a score test for a logistic regression model.
These tests are referred to as testing the calibration slope and intercept, and they are part of a broader hierarchy of calibration \citep{van2016calibration}.
See also \cite{miller1991validation,steyerberg2019clinical} and references therein.
\cite{miller1962statistical}, Section 5, suggests a chi-squared test for testing calibration of a collection of sequences of Bernoulli random variables. \cite{spiegelhalter1986probabilistic} proposes a test of calibration based on the Brier score, for discrete-valued probability predictors.
The Hosmer-Lemeshow test \citep{hosmer1980goodness} is a goodness-of-fit test for logistic regression models.
The test is based on a chi-squared statistic that measures differences between expected and observed numbers of events in subgroups, and thus has, on the surface, a similarity to the types of test statistics we consider. There are also related tests for comparing predictors \citep{schervish1989general,diebold1995comparing}.

\cite{seillier1993testing} study testing the calibration of sequential probability forecasts.
\cite{brocker2007increasing} study the bootstrap-based procedure they call consistency resampling to produce standard error bars in reliability diagrams; without focusing on its optimality. 
For testing the calibration of forecasted densities, \cite{dawid1984statistical,diebold1998evaluating} propose the probability integral transform (PIT).
\cite{held2010score} propose a score-based approach for testing calibration.
\cite{vaicenavicius2019evaluating} use consistency resampling to test a hypothesis of perfect calibration; again without studying its optimality.
\cite{widmann2019calibration} propose kernel-based mis-calibration measures together with their estimators, and argue that the estimators can be viewed as calibration test statistics.
\cite{tamas2021exact} suggest distribution-free hypothesis tests for the null $H_0: \E[Y \mid X] = X$ based on conditional kernel mean embedding.

\paragraph{Note on Terminology.} The term calibration sometimes has a different meaning in a variety of areas of human activity, including measurement technology, engineering, economics, and even statistics, etc., see e.g., \cite{franklin1999calibration,dawkins2001calibration,kodovsky2009calibration,osborne1991statistical,vovk2020conformal, angelopoulos2021learn}.
These generally mean adjusting a measurement to agree with a desired standard, within a specified accuracy.
However, in our work, we focus on the notion of probabilistic calibration described so far.

\fussy

\subsection{Notations}
For an integer $d\geq 1$, and a vector $\mathbf{v}\in \R^d$, we  refer to the coordinates of $\mathbf{v}$ as both $[\mathbf{v}]_1,\ldots, [\mathbf{v}]_d$ and $v_1,\ldots, v_d$.
For any $p\geq 1$, and for an integer $K \geq 2$, we denote the $\ell_p$-norm of $\mathbf{x} = (x_1, \dots, x_K)^\top\in \R^K$ by $\left\Vert \mathbf{x} \right\Vert_p := ( \sum_{i = 1}^K |x_i|^p )^{1/p}$. When $p$ is unspecified, $\Vert \cdot \Vert$ stands for $\Vert \cdot \Vert_2$.
For an event $A$, we denote by $I(A)$ its indicator random variable, where $I(A) = 1$ if event $A$ happens, and $I(A)=0$ otherwise.
For two real numbers $a,b$, we denote $a\wedge b:=  \min(a,b)$.
For two sequences $(a_n)_{n\geq 1}$ and  $(b_n)_{n\geq 1}$ with $b_n \neq 0$,
we write $a_n \asymp b_n$ if $0 < \liminf_n a_n / b_n \leq \limsup_n a_n / b_n < \infty$.
When the index $n$ is self-evident, we may omit it above.
We use the Bachmann-Landau asymptotic notations $\Omega(\cdot), \Theta(\cdot)$ to hide constant factors in inequalities and use $\tilde{\Omega}(\cdot), \tilde{\Theta}(\cdot)$ to also hide logarithmic factors.
For a Lebesgue measurable set $A \subseteq \R^d$, we denote by $\mathds{1}_{A}:\R^d \to \{0,1\}$ its indicator function where $\mathds{1}_{A}(\mathbf{x}) = 1$ if $\mathbf{x} \in A$ and $\mathds{1}_{A}(\mathbf{x}) = 0$ otherwise.
For a real number $s \in \mathbb{R}$, we denote the largest integer less than or equal to $s$ by $\lfloor s \rfloor$.
Also, the smallest integer greater than or equal to $s$ is denoted by $\lceil s \rceil$.

\section{Definitions and Setup}
\label{sec:prelim}

For $K \geq 2$, consider a $K$-class classification problem where $X \in \mathcal{X}$ is the input feature vector (for instance, an image)
and $Y \in \mathcal{Y} :=\{ \mathbf{y} = (y_1, \dots, y_K)^\top \in \{0,1\}^{K}: \sum_{i=1}^K y_i = 1\}$ is the one-hot encoded output label (for instance, the indicator of the class of the image: building, vehicle, etc).

We consider a probabilistic classifier $f$ mapping the feature space to probability distributions over $K$ classes. Formally, the output space is the $(K-1)$-dimensional probability simplex $\Delta_{K-1}$, 
\beqs
\Delta_{K - 1} := \{\mathbf{z} = (z_1,\ldots, z_K)^\top \in [0, 1]^K : z_1 + \cdots + z_K = 1 \},
\eeqs
i.e.,  $f:\mathcal{X} \to \Delta_{K - 1}$.
For any $k\in\{1,\ldots,K\}$, the individual component $[f(X)]_k$ denotes the predicted probability of the $k$-th class. 
Thus, $f$ is also referred to as a probability predictor.
The probability predictor $f$ is assumed to be pre-trained on data that are independent of our calibration data at hand.

We assume that the feature-label pair $(X,Y)$ has an unknown joint probability distribution $P$ on $\mathcal{X} \times \mathcal{Y}$.
Calibration requires that the predicted probabilities of correctness are equal to the true probabilities. 
Thus, given that we predicted the probabilities $f(X)=\mathbf{z}$, and thus $[f(X)]_k = [\mathbf{z}]_k$, the true probability that $[Y]_k=1$ should be equal to $[f(X)]_k=[\mathbf{z}]_k$. Thus, for almost every $\mathbf{z}$, calibration requires that for all $k=1,\dots,K$,
$$P[[Y]_k=1\mid f(X)=\mathbf{z}] = [\mathbf{z}]_k.$$

We can reformulate this in a way that is more convenient to study.
The map $(\mathbf{x}, \mathbf{y}) \mapsto (f(\mathbf{x}), \mathbf{y})$ induces a probability distribution on $\Delta_{K - 1} \times \mathcal{Y}$; where we can think of $(\mathbf{x},\mathbf{y})$ as a realization of $(X,Y)$.
As will be discussed shortly, calibration only depends on the joint distribution of $(f(X), Y)$.
For this reason, we also denote the joint distribution of $(f(X), Y)$ by $P$ when there is no confusion.
We write $Z := f(X)$ for the predicted probabilities corresponding to $X$.

We define the \emph{regression function}  $\reg:\Delta_{K - 1} \to \Delta_{K - 1}$ as $$\reg(\mathbf{z}):=\E[Y \mid f(X)=\mathbf{z}] = \E[Y \mid Z = \mathbf{z}],$$
where the expectation is conditioned on the score $Z$ with $(Z,Y)\sim P$.
Note that each component, for $k=1,\ldots, K$, has the form 
$\E[[Y]_k\mid f(X)=\mathbf{z}] = P[[Y]_k=1\mid f(X)=\mathbf{z}].$
Especially for binary classification, this is also referred to as the \emph{calibration curve} of the probabilistic classifier $f$ \citep{harrell2015regression}.
Since we are particularly interested in continuous probability predictors, we assume that the marginal distribution $P_{Z}$ of $Z$ has a density with respect to the uniform measure on $\Delta_{K - 1}$. 
Then, this expectation is well-defined almost everywhere. 

In this language, the probabilistic classifier $f$ is \emph{perfectly calibrated} if $\reg(Z)=Z$
almost everywhere.\footnote{In the binary case ($K = 2$), we identify $\Delta_{K - 1}$ with $[0, 1]$ via the map $(z, 1 - z)^\top \mapsto z$ and use $\mathcal{Y} = \{0, 1\}$ instead of the one-hot encoded output space.
We say $f$ is perfectly calibrated if $\reg(z) := P(Y = 1 \mid f(X) = z) = z$ almost everywhere.}
Further, it turns out that it is important to study \emph{the deviations from calibration}. For this reason, we define the \emph{residual function} $\res:\Delta_{K - 1} \to \mathbb{R}^K$ as  $$\res(\mathbf{z}) := \text{reg}_f(\mathbf{z}) - \mathbf{z},$$ 
so that perfect calibration amounts to $\res(Z)=0$ almost everywhere.
When $(Z, Y)$ have a  joint distribution $P$, we sometimes write  $\res=\text{res}_{f, P}$ to display the dependence of the mis-calibration curve on $P$.
As we will see, the structure of the residual function crucially determines our ability to detect mis-calibration.
In analogy to the notion of calibration curves mentioned above, we may also call $\res$ the \emph{mis-calibration curve} of the probabilistic classifier $f$.

We observe calibration data $(Z_i, Y_i) \in \Delta_{K - 1} \times \mathcal{Y}$, $i\in \{1, \dots, n\}$, sampled i.i.d. from $P$, and denote their joint product distribution as $P^n$. Our goal is to rigorously test if $f$ is perfectly calibrated based on this finite calibration dataset.
The calibration properties of the probabilistic classifier $f$ can be expressed equivalently in terms of the distribution $P$ of $(f(X),Y) = (Z,Y)$. 
Therefore, we will sometimes refer to testing the calibration of the distribution $P$, and the probabilistic classifier will be implicit.

\paragraph{Expected Calibration Error.}
The $\ell_p$-ECE (Expected Calibration Error)
for the distribution $P$, also known as the mean calibration error \citep[e.g.,][p. 105]{harrell2015regression}, is
\beq\label{lpece}
\ell_p\text{-ECE}(f) 
= \ell_p\text{-ECE}_P(f)
= \E_{Z \sim P_Z} \left[ \|\reg(Z)-Z]\|_p^p \right]^\frac{1}{p}
=\E_{Z \sim P_Z} \left[ \sum_{k = 1}^K \left\vert [\res(Z)]_k \right\vert^p \right]^\frac{1}{p}.
\eeq

In words, this quantity measures the average over all classes $k=1,\ldots, K$ and over the data distribution $X\sim P_X$
of the per-class 
error $[\res(\mathbf{z})]_k=\E[[Y]_k \mid f(X)=\mathbf{z}]-[\mathbf{z}]_k$ 
between the predicted probability of class $k$ for input $X$---i.e., $[\mathbf{z}]_k =[f(X)]_k$---and the actual probability $\E[[Y]_k \mid f(X)=\mathbf{z}]=P[[Y]_k=1 \mid f(X)=\mathbf{z}]$ of that class.
For instance, when the number of classes is $K=2$, 
and the power is $p=1$,
we have $\ell_1\text{-ECE}(f)
=\E_{X \sim P_X} \sum_{k = 1}^2|P[[Y]_k=1 \mid f(X)] - [f(X)]_k|
= 2 \E_{X \sim P_X}\left|P[[Y]_1=1 \mid f(X)] - [f(X)]_1\right|$.

\paragraph{H\"older Continuity.}

\sloppy

We describe the notion of H\"older continuity for functions defined on $\Delta_{K - 1}$.
For simplicity, we only provide the definition for $K = 2$.
See Appendix \ref{sec:holderdef} for the complete definition for general $K\geq 2$.

Identifying $\Delta_1$ with $[0, 1]$ via the map $(z, 1 - z)^\top \mapsto z$, a function $g:\Delta_1 \to \mathbb{R}$ can be equivalently understood as a function $g: [0, 1] \to \mathbb{R}$.
For an integer $d \geq 0$ and a function $g: [0, 1] \to \mathbb{R}$, let $g^{(d)}$ be the $d$-th derivative of the function $g$.
For a real number $s$, we denote the smallest integer greater than or equal to $s$ by $\lceil s \rceil$.

For a H\"older smoothness parameter $s > 0$ and a H\"older constant $L > 0$, let $\cH_K(s,L)$ be the class of 
$(s,L)$-H\"older continuous
functions $g: [0, 1] \to \R$ satisfying, for all $x_1, x_2 \in [0, 1]$
\begin{equation}
\label{eq:binholderdef}
    \left\vert g^{(\lceil s\rceil - 1)}(x_1) - g^{(\lceil s\rceil - 1)}(x_2) \right\vert \leq L \left\vert x_1 - x_2 \right\vert^{s - \lceil s \rceil + 1}.
\end{equation}
In particular,
$\cH_K(1,L)$ denotes all $L$-Lipschitz functions.
We consider $L>0$ as an arbitrary fixed constant, and we do not display the dependence of our results on its value.
For instance, when the Lipschitz constant is 
$L=1$, 
and the H\"older smoothness parameter is $s = 1.5$,
this is the set of real-valued functions $g$ defined on $[0, 1]$ such that 
for all 
$x_1, x_2 \in [0,1]$,
$|g'(x_1)-g'(x_2)| \leq L |x_1 - x_2|^{0.5}$.

\paragraph{Goal.}
Our goal is to test the null hypothesis of perfect calibration, i.e., $\res=0$, against the alternative hypothesis that the model is mis-calibrated. To quantify mis-calibration, we use the notion of the $\ell_p\mbox{-ECE}(f)$ from \eqref{lpece}. 
We study the signal strength needed so that reliable mis-calibration detection is possible.
Further, we assume that the mis-calibration curves are H\"older continuous because we will show that by only assuming continuity, reliable detection of mis-calibration is impossible.
In Remark \ref{hol-assu}, We will also discuss what happens when the mis-calibration function is not H\"older smooth.

Let $\mathcal{P}$ be the family of all distributions $P$ over $(Z, Y) \in \Delta_{K - 1} \times \mathcal{Y}$ such that the marginal distribution $P_Z$ of $Z$ has a density with respect to the uniform measure on $\Delta_{K - 1}$.
Define
the collection $\mathcal{P}_0$ of joint distributions $P$ of $(Z, Y)$ under which the probability predictor $f$ is perfectly calibrated:
\begin{align*}
    \mathcal{P}_0 := \left\{ P \in \mathcal{P}:\, \text{res}_{f, P}(Z) = 0, \, P_Z\text{-a.s.} \right\}.
\end{align*}
For a H\"older smoothness parameter $s$ and a H\"older constant $L$, let $\mathcal{P}_{s, L, K}$ be the family of probability distributions $P \in \mathcal{P}$ over the predictions and labels
$(f(X), Y) = (Z, Y) \in \Delta_{K - 1} \times \mathcal{Y}$ under which the residual map $\mathbf{z} \mapsto [\text{res}_{f, P}(\mathbf{z})]_k$ (i.e., the map $\text{res}_{f}$ under the distribution $(Z,Y)\sim P$) 
belongs to the class of $(s,L)$-H\"older continuous functions $\mathcal{H}_K(s, L)$ for every $k \in \{1,\ldots,K\}$.
For a separation rate $\ep>0$, define
the collection $ \mathcal{P}_1(\ep, p, s)$ of joint distributions $P \in \mathcal{P}_{s, L, K}$ under which the $\ell_p$-ECE of $f$ is at least $\ep$:
\begin{align}
\label{p1}
    \mathcal{P}_1(\ep, p, s) := \left\{ P \in \mathcal{P}_{s, L, K}: \ell_p\text{-ECE}_P(f) \geq \ep \right\}.
\end{align}
We will also refer to these distributions as \emph{$\ep$-mis-calibrated}.
Our goal is to test the \emph{null hypothesis} of calibration against the \emph{alternative} of an $\ep$-calibration error:
\begin{equation}
\label{eq:setup}
H_0: P \in \mathcal{P}_0 \quad \mbox{versus } \quad H_{1}: P \in \Alt.
\end{equation}

Although we consider the null hypothesis of perfect calibration, we generally do not expect a model trained on finite data to be perfectly calibrated.
In this regard, the purpose of testing \eqref{eq:setup} is to check if there is statistically significant evidence of mis-calibration, and \emph{not} to check whether the predictor $f$ is perfectly calibrated.
As usual in hypothesis testing, not rejecting the null hypothesis does not mean that we accept that $f$ is perfectly calibrated but means that there is no statistically significant evidence of mis-calibration.
In this case,
 to gain more confidence that the model is calibrated,
one may consider testing other hypotheses  about calibration---such as top-$k$ calibration, \citep{guo2017calibration}---or collecting more data;
of course, this may require dealing with multiple testing problems.
Meanwhile, since the null of calibration is not rejected, one may use the classifier as if it was calibrated until evidence to the contrary is presented.

Moreover, in Remark \ref{rmk:compositenull}, we also provide results for the null hypothesis of a small enough calibration error.


\fussy

\paragraph{Hypothesis Testing.}
We recall some notions from hypothesis testing \cite[e.g.,][etc]{lehmann2005testing,ingster2012nonparametric} that we use to formulate our problem.
A \emph{test} $\xi$ is a function\footnote{To be rigorous, a Borel measurable function.} $\xi: (\Delta_{K - 1} \times \mathcal{Y})^n \rightarrow \{0, 1\}$
of the data, given a dataset $S= \{(X_i,Y_i)\}_{i = 1}^n \in (\Delta_{K - 1} \times \mathcal{Y})^n$,  the decision $\xi(S)$ of rejecting the null hypothesis.
In other words, for a given dataset $S$, $\xi(S)=1$ means that we detect mis-calibration, and $\xi(S)=0$ means that we do not detect mis-calibration.

Denote the set of all \emph{level} $\alpha \in (0, 1)$ tests, which have a \emph{false detection rate (or, false positive rate; type I error)} bounded by $\alpha$, as
$$\Phi_{n}(\alpha) := \left\{\xi: \sup_{P \in \mathcal{P}_0} P(\xi=1) \leq \alpha \right\}.$$
    The probability $P(\xi=1)$ is taken with respect to the distribution of the sample.
For
$\ep>0$ and $P\in \mathcal{P}_1(\ep, p, s)$ from \eqref{p1},
we want to minimize the \emph{false negative rate (type II error)}
$P(\xi=0)$, the probability of not detecting mis-calibration.
We consider the worst possible value (maximum or rather supremum) $\sup_{P \in \mathcal{P}_1(\ep, p, s)}P(\xi=0)$ of the type II error, over all distributions $P \in \mathcal{P}_1(\ep, p, s)$. 
We then want to minimize this over all tests $\xi\in\Phi_{n}(\alpha)$ 
that appropriately control the level, leading to the \emph{minimax risk} (minimax type II error)
$$R_{n}(\ep, p, s) := \inf_{\xi\in\Phi_{n}(\alpha)} \sup_{P \in \mathcal{P}_1(\ep, p, s)} P(\xi=0).$$
In words, among all tests that have a false detection rate of $\alpha<1$ using a sample of size $n$, we want to find the one with the best possible (smallest) mis-detection rate over all $\ep$-mis-calibrated distributions.

We consider $\alpha \in(0,1)$ as a fixed constant, and we do not display the dependence of our results on its value.
We want to understand how large the $\ell_p\text{-ECE}$ (as measured by $\ep$ in $\mathcal{P}_1(\ep,p,s)$) needs to be to ensure reliable detection of mis-calibration. This amounts to finding $\ep'$ such that the best possible worst-case risk $R_{n}(\ep', p, s)$ is small.
For a fixed $\beta \in (0, 1 - \alpha)$, the minimum separation (signal strength) for $s$-H\"older functions, in the $\ell_p$-norm, needed for a minimax type II error of at most $\beta$ is defined as
\begin{equation*}
\ep_n(\beta; p, s):=\ep_n(p, s) = \inf \{ \ep': R_{n}(\ep', p, s) \leq \beta \}.
\end{equation*}
Since $\beta \in (0,1-\alpha)$ is fixed, we usually omit the dependence of $\ep_n$ on this value.

\begin{remark}[Comparison with classical nonparametric hypothesis testing]\label{rem:diff}
As we summarize in Section 1.1, prior works such as  \cite{ingster1987minimax, ingster2000adaptive, ingster2012nonparametric, berman2014lp} have studied the problem of testing that the $L^p$ norm of a function is zero against the alternative that it is nonzero, where the function is either a probability density or a regression function in the Gaussian white noise model. 
Our task here is different from the classical problem since $\text{reg}_f$ is not a probability density, and we are not provided independent observations of the function $\text{reg}_f$ or $\text{res}_f$ in the Gaussian white noise model.
Rather, our observation model is closer to multinomial regression; which is heteroskedastic and differs from the above models.
While our proposed test shares ideas with the chi-squared test of \cite{ingster1987minimax, ingster2000adaptive}, it requires a different analysis for the above-mentioned reasons.
\end{remark}
 
\section{An Adaptive Debiased Calibration Test}
\label{sec:onesample}

Here we describe our main test for calibration.
This relies on a debiased plug-in estimator for  $\ell_2\text{-ECE}(f)^2$.
We prove that the test is minimax optimal and discuss why debiasing is necessary.
We also provide an adaptive plug-in test, which can adapt to an unknown H\"older smoothness parameter $s$.

 \subsection{Debiased Plug-in Estimator}
\label{sec:onesampleoptimal}

The calibration error of a continuous probability predictor $f$ is often estimated by a discretized plug-in estimator associated with a partition (or binning) of the probability simplex $\Delta_{K - 1}$ \cite[e.g.,][]{cox1958two,harrell2015regression}.
The early work of \cite{cox1958two} already recommended grouping together similar probability forecasts.
More recently, \cite{guo2017calibration} divide the interval $[0, 1]$ into bins of equal width and compute the (top-1) ECE by averaging the difference between confidence and accuracy in each bin.
\cite{vaicenavicius2019evaluating} generalize this idea to $K$-class classification and data-dependent partitions. 

In this work, we use an equal-volume partition $\mathcal{B}_m$ of the probability simplex $\Delta_{K - 1}$, which is parametrized by a binning scheme parameter $m \in \mathbb{N}_{+}$. 
The partition $\mathcal{B}_m$ consists of $m^{K - 1}$ simplices with equal volumes and diameters proportional to $m^{-1}$.
To construct a such partition, we first divide the simplex $\Delta_{K - 1}$ into $K - 1$ hypersimplices---generalizations of the standard probability simplex that can have more vertices and edges---by taking intersections with $m^{-1}$-scaled and translated $K$-dimensional hypercubes.
The hypersimplices are further divided into unit volume simplices using the result of \cite{stanley1977eulerian, sturmfels1996grobner}.
The construction of $\mathcal{B}_m$ is elaborated in Appendix B.3.
The purpose of using an equal-volume partition $\mathcal{B}_m$ is only for a simpler description of our results, and any partition with $\Theta(m^{-K + 1})$ volumes and $\Theta(m^{-1})$ diameters can be used.

Let us denote the sets comprising the partition as $\mathcal{B}_m = \{B_1, \dots, B_{m^{K - 1}}\}$.
For each $i \in \{1, \dots, m^{K - 1}\}$, define the indices of data points falling into the bin $B_i$ as
$\mathcal{I}_{m, i} := \{j: Z_j \in B_i,1\leq j\leq n\}$.
Then,  for each $i \in \{1, \dots, m^{K - 1}\}$, the averaged difference between probability predictions $Z_j = f(X_j)$ and true labels $Y_j$ for the probability predictions in $B_i$ is
$|\mathcal{I}_{m, i}|^{-1} \sum_{j \in \mathcal{I}_{m, i}} (Y_j - Z_j)$.
This estimates $\E[Y - Z \mid Z \in B_i] = \E[\res(Z) \mid Z \in B_i]$.
Now, the quantity $\ell_2\text{-ECE}(f)^2 = \E[\Vert \res(Z) \Vert^2]$ can be approximated by piecewise averaging as $\sum_{1\leq i\leq m^{K-1}} P_Z(B_i) \Vert \E[\res(Z) \mid Z \in B_i] \Vert^2$. 
Plugging in the estimate $\||\mathcal{I}_{m, i}|^{-1} \sum_{j \in \mathcal{I}_{m, i}} (Y_j - Z_j)\|^2$ of $\|\E[\res(Z) \mid Z \in B_i]\|^2$, 
we can define a plug-in estimator of $\ell_2\text{-ECE}(f)^2$ as follows:
\begin{equation}
\label{eq:onesamplebiased}
    T_{m, n} ^\text{b} := \sum_{\substack{1 \leq i \leq m^{K - 1}\\ |\mathcal{I}_{m, i}| \geq 1}} \frac{|\mathcal{I}_{m, i}|}{n} \left\Vert \frac{1}{|\mathcal{I}_{m, i}|} \sum_{j \in \mathcal{I}_{m, i}} (Y_j - Z_j) \right\Vert^2.
\end{equation}
Above, the sum is taken over bins $B_i$ containing at least one datapoint.
As will be discussed in Section \ref{sec:biasfailure}, the plug-in estimator is biased in the sense that its expectation is not zero under perfectly calibrated distributions.
Moreover, it does not lead to an optimal test statistic.
Informally, this happens because we are estimating both $\E[ Y \mid Z \in B_i]$ and $\E[Z \mid Z \in B_i]$ with the same sample $(Z_i, Y_i), i \in \{1, \dots, n\}$. 
We hence define the \emph{Debiased Plug-in Estimator} (DPE):

\begin{equation}
\label{eq:onesampledef}
    T_{m, n}^\textnormal{d} :=  \sum_{\substack{1 \leq i \leq m^{K - 1}\\ |\mathcal{I}_{m, i}| \geq 1}} \frac{|\mathcal{I}_{m, i}|}{n} \left[ \left\Vert \frac{1}{|\mathcal{I}_{m, i}|} \sum_{j \in \mathcal{I}_{m, i}} (Y_j - Z_j) \right\Vert^2 - \frac{1}{|\mathcal{I}_{m, i}|^2} \sum_{j \in \mathcal{I}_{m, i}}  \left\Vert Y_j - Z_j \right\Vert^2 \right].
\end{equation}

The debiasing term in \eqref{eq:onesampledef} ensures that $T_{m,n}^\text{d}$ has mean zero under a distribution $P \in \mathcal{P}_0$ under which $f$ is a calibrated probability predictor.
Due to the discretization, the mean of $T_{m, n}^\text{d}$ is not exactly $\ell_2\text{-ECE}(f)^2$ under $P \in \mathcal{P}_1(\ep, p, s)$, but the debiasing makes it comparable to $\ell_2\text{-ECE}(f)^2$.
This will be a crucial step when proving the optimality of $T_{m,n}^\text{d}$.

\begin{remark}[Connection to nonparametric functional estimation]\label{rem:ustatistic}
The definition of $T_{m, n}^d$ is closely related to the U-statistic for estimating the quadratic integral functional of a probability density \citep{kerkyacharian1996estimating, laurent1996efficient}.
To see this, let $\{ \phi_i (x) = P_Z(B_i)^{-1/2} 1_{B_i}(x): i = 1, \dots, m^{K - 1} \}$ be the Haar scaling functions
associated to the partition $\mathcal{B}_m$.
For each $1 \leq k \leq K$, the U-statistic
\begin{align*}
\frac{1}{n(n - 1)} \sum_{1 \leq i \leq m^{K - 1}} \sum_{1 \leq j_1 \neq j_2 \leq n} [Y_{j_1} - Z_{j_1}]_k [Y_{j_2} - Z_{j_2}]_k \phi_i(Z_{j_1}) \phi_i(Z_{j_2})
\end{align*}
is an unbiased estimate of $\int_{\Delta_{K - 1}} [\text{res}_f(z)]_k^2dP_Z(z)$.
Summing over $1 \leq k \leq K$ and plugging in $P_Z(B_i) \approx |\mathcal{I}_{m, i}| / n$, we recover (6) with the minor modification of changing $n \to n - 1$ in the scaling.

However, as noted in Remark \ref{rem:diff}, our problem differs from those studied in classical nonparametric statistic literature.
Specifically, our definition of $T_{m, n}^\textnormal{d}$ additionally requires an estimation of $P_Z(B_i)$ by $|\mathcal{I}_{m, i}| / n$.
Therefore, prior results on nonparametric functional estimation 
\citep{bickel1988estimating, donoho1990minimax, birge1995estimation} 
cannot be directly applied to $\ell_2\text{-ECE}(f)^2$.

\cite{wang2008effect, shen2020optimal} consider quadratic functional estimation for an unknown distribution of covariates and show that the minimax rate also depends on 
the
H\"older smoothness of the covariate density function. 
\end{remark}

\sloppy
In the following theorem, we prove that $T_{m, n}^\text{d}$ leads to a minimax optimal test when the number of bins is chosen in a specific way, namely $m \asymp n^{2 / (4s + K - 1)}$.
Crucially, the number of bins required decreases with the smoothness parameter $s$. In this sense, our result parallels the well-known results on the optimal choice of the number of bins for testing probability distributions and densities \citep{mann1942choice,ingster2012nonparametric}.

The guarantee on the power (or, Type II error control) requires the following mild condition, stated in
Assumption \ref{assumption:densitybound}. This ensures that the probability of each bin is proportional to the inverse of the number of bins up to some absolute constant. In particular, this holds if the density of the probabilities predicted is close to uniform. 
This assumption is necessary when extending the results of \cite{arias2018remember, kim2020minimax} to a general base probability measure $\mu$ of the probability predictions over the probability simplex.
See Appendix \ref{sec:twobackground} for more discussion.
\begin{assumption}[Bounded marginal density]
\label{assumption:densitybound}
Let $\nu$ be the uniform probability measure on the probability simplex $\Delta_{K - 1}$.
There exist constants $\nu_l, \nu_u > 0$ such that $\nu_l \leq dP_Z / d\nu \leq \nu_u$ almost everywhere.
\end{assumption}

\begin{algorithm}[t]
\caption{T-Cal: an optimal test for calibration (based on debiased plug-in estimation of the calibration error)}\label{alg:onesampletest}
\begin{algorithmic}
\STATE \textbf{Input:} Probability predictor $f: \mathcal{X} \to \Delta_{K - 1}$; i.i.d. sample $\{(X_i, Y_i) \in \mathcal{X} \times \mathcal{Y} : i \in \{1, \dots, n\} \}$; false detection rate $\alpha \in (0, 1)$; true detection rate $\beta \in (0, 1 - \alpha)$; H\"older smoothness $s$
\STATE \textbf{Initialize:}  $m_* \gets \lfloor n^{2 / (4s + K - 1)} \rfloor$; $T_{m_*, n}^\text{d}\gets 0$; $Z_i \gets f(X_i)$ for $1\leq i\leq n$; define $\{B_1, \dots, B_{m^{K - 1}}\}$ as in Appendix \ref{multibinning}
\FOR{$i = 1$ \TO $m_*^{K-1}$}
\STATE  $\mathcal{I}_{m_*, i} \gets \{j: Z_j \in B_i,\,1\leq j\leq n \}$
\STATE $T_{m_*, n}^\text{d} \gets T_{m_*, n}^\text{d}+  \frac{|\mathcal{I}_{m_*, i}|}{n} \left[ \left\Vert \frac{1}{|\mathcal{I}_{m_*, i}|} \sum_{j \in \mathcal{I}_{m_*, i}} (Y_j - Z_j) \right\Vert^2 - \frac{1}{|\mathcal{I}_{m_*, i}|^2} \sum_{j \in \mathcal{I}_{m_*, i}}  \left\Vert Y_j - Z_j \right\Vert^2 \right]$ 
\ENDFOR
\STATE $\xi_{m_*, n} \gets I\left( T_{m_*, n}^\text{d} \geq \frac{\sqrt{2}K}{\sqrt{\alpha}} \left(m_*^\frac{K- 1}{2} n^{-1} \wedge m_*^{-\frac{K - 1}{2}}\right) \right)$
\STATE \textbf{Output:} Reject $H_0$ if $\xi_{m_*, n} = 1$
\end{algorithmic}
\end{algorithm}

\begin{theorem}[Calibration test via debiased plug-in estimation]
\label{thm:onesamplethm}
    Suppose $p \leq 2$ and assume that the H\"older smoothness parameter $s$ is known.
    For a binning scheme parameter $m \in \mathbb{N}_+$, let
    $$\xi_{m, n}(\alpha) = \xi_{m, n} := I\left( T_{m, n}^\text{d} \geq \sqrt{\frac{2K^2}{\alpha}} \left(m^\frac{K- 1}{2} n^{-1} \wedge m^{-\frac{K - 1}{2}}\right) \right).$$
    Under Assumption \ref{assumption:densitybound} and for $m_* = \lfloor n^{2 / (4s + K - 1)} \rfloor$, we have
    \begin{enumerate}
        \item {\bf False detection rate control.}
        For every $P$ for which $f$ is perfectly calibrated, i.e., for $P\in \mathcal{P}_0$, the probability of falsely claiming mis-calibration is at most $\alpha$, i.e., 
        $P( \xi_{m_*, n} = 1) \leq \alpha$.
    
        \item {\bf True detection rate control.} There exists $c > 0$ depending only on $(s, L, K, \nu_l, \nu_u, \alpha, \beta)$ such that when $$\ep \geq c n^{-2s/(4s + K - 1)},$$
        then for every $P \in \mathcal{P}_1(\ep, p, s)$---i.e., when $f$ is mis-calibrated with an $\ell_p\textnormal{-ECE}$ of $\ep$---the power (true positive rate) is bounded as $P( \xi_{m_*, n} = 1 ) \geq 1 - \beta$.
        
    \end{enumerate}
\end{theorem}

\fussy

The proof can be found in Appendix \ref{sec:proofonesamplethm}.
The proof follows the classical structure of upper bound arguments in nonparametric hypothesis testing, see e.g., \cite{arias2018remember,kim2020minimax} for recent examples.
We compute the mean and variance of $T_{m_*, n}^\text{d}$ under null distributions $P_0 \in \mathcal{P}_0$ and alternative distributions $P_1 \in \mathcal{P}_{s, L, K}$ with a large ECE.
Using Lemma 13, we can find a lower bound on $\mathbb{E}_{P_1}[T_{m_*, n}^\text{d}] - \mathbb{E}_{P_0}[T_{m_*, n}^\text{d}]$.
The variances $\text{Var}_{P_0}(T_{m_*, n}^\text{d})$ and $\text{Var}_{P_1}(T_{m_*, n}^\text{d})$ can be also upper bounded.
We argue that the mean difference $\mathbb{E}_{P_1}[T_{m_*, n}^\text{d}] - \mathbb{E}_{P_0}[T_{m_*, n}^\text{d}]$ is significantly larger than the square root of the variances $\text{Var}_{P_0}(T_{m_*, n}^\text{d})$ and $\text{Var}_{P_1}(T_{m_*, n}^\text{d})$.
The conclusion follows from Chebyshev's inequality.

Combined with our lower bound in Theorem \ref{thm:multi-lower-bound}, this result shows the desired property that our test is \emph{minimax optimal}.
This holds for all $p \leq 2$, so that the test is minimax optimal even when the mis-calibration is measured in  the $\ell_p$ norm with $p < 2$.
This is consistent with experimental findings such as those of \cite{nixon2019measuring},  where  the empirical $\ell_2$-ECE performs better than the empirical $\ell_1$-ECE as a measure of calibration error.
Also see Section \ref{sec:powanalysis} for a comparison of the empirical $\ell_1$-ECE and $\ell_2$-ECE as a test statistic.

Although we present explicit critical values in Theorem \ref{thm:onesamplethm}, they can be conservative in practice, as in other works in nonparametric testing \citep{ingster1987minimax, arias2018remember, kim2020minimax}.
Therefore, we recommend choosing the critical values via a version of bootstrap: consistency resampling \citep{brocker2007increasing, vaicenavicius2019evaluating}.
See Appendix \ref{sec:critcompare} for further details on choosing critical values.

\begin{remark}\label{rmk:compositenull}
So far, we considered the null hypothesis of perfect calibration.
However, since the predictor $f$ is trained on a finite dataset, we cannot expect it to be perfectly calibrated.
We can extend Theorem \ref{thm:onesamplethm} to the null hypothesis of ``small enough'' mis-calibration, namely, for any given constant $c_0>0$, an $\ell_p\textnormal{-ECE}$ of at most $c_0 n^{-2s / (4s + K - 1)}$.
Then, the true and false positive rates of the test
$$\xi_{m, n}^\textnormal{comp} := I\left( T_{m, n}^\text{d} \geq \frac{K}{\sqrt{\alpha}}  \sqrt{2 \left(m^{K - 1} n^{-2} \wedge m^{-(K - 1)} \right) + 5c_0^2 n^{-\frac{4s}{4s + K - 1}} \left( n^{-1} \wedge m^{-(K - 1)} \right)} \right)$$
can be controlled as in Theorem \ref{thm:onesamplethm}.
See Appendix \ref{sec:proofcompositenull} for the proof.
\end{remark}

\subsection{An Adaptive Test}
The binning scheme used in our
plug-in test requires knowing the smoothness parameter $s$ to be minimax optimal.
However, in practice, this parameter is usually unknown.
Can we design an adaptive test that does not require knowing this parameter? 
Here we answer this question in the affirmative.
As in prior works in nonparametric hypothesis testing, e.g., \cite{ingster2000adaptive, arias2018remember, kim2020minimax}, we propose an \emph{adaptive test} that can adapt to an unknown H\"older smoothness parameter $s$.
The idea is to evaluate the plug-in test over a variety of partitions, and thus be able to detect mis-calibration at various different scales.

In more detail, we evaluate the test with a number of bins ranging over a dyadic grid $ 2, 2^2, \ldots, 2^B$.
In addition, to make sure that we control the false detection rate, we need to divide the level $\alpha$ by the number of tests performed.
Thus, for a number $B = \lceil \frac{2}{K - 1} \log_2(n / \sqrt{\log n}) \rceil$
of tests performed, 
we let the adaptive test
\beq\label{at}
\xi_n^\textnormal{ad} := \max_{1\leq b\leq B} \xi_{2^b, n} \left( \frac{\alpha}{B} \right)
\eeq
detect mis-calibration if any of the debiased plug-in tests $\xi_{2^b, n} \left( \alpha / B \right)$, with the number of bins $2^b$, $b\in\{1,\ldots,B\}$, detects mis-calibration at level $\alpha/B$. We summarize the procedure in Algorithm \ref{alg:adptest}.

\begin{theorem}[Adaptive plug-in test]
\label{thm:adaptiveone}
Suppose $p \leq 2$.
Under Assumption \ref{assumption:densitybound}, the adaptive test from \eqref{at} enjoys 
\begin{enumerate}
        \item {\bf False detection rate control.} 
        For every $P$ for which $f$ is perfectly calibrated, i.e., for $P\in \mathcal{P}_0$, the probability of falsely claiming mis-calibration is at most $\alpha$, i.e., 
        $P \left( \xi_{n}^\textnormal{ad} = 1 \right) \leq \alpha$.
        \item {\bf True detection rate control.} There exists $c_\textnormal{ad} > 0$ depending on $(s, L, K, \nu_l, \nu_u, \alpha, \beta)$ such that the power (true positive rate) is lower bounded as  $P( \xi_{ n}^\textnormal{ad} = 1 ) \geq 1 - \beta$ for every $P \in \mathcal{P}_1(\ep, p, s)$---i.e., when $f$ is mis-calibrated with an $\ell_p\textnormal{-ECE}$ of at least $\ep \geq c_\textnormal{ad} (n / \sqrt{\log n})^{-2s/(4s + K - 1)}$.

\end{enumerate}
\end{theorem}
See Appendix \ref{sec:proofadaptiveone} for the proof.
Compared to the non-adaptive test, this test requires a mild additional factor of $(\log n)^{s / (4s + K - 1)}$ in the separation rate $\ep$ to guarantee detection.
It is well understood in the area of nonparametric hypothesis testing that some adaptation cost is unavoidable, see for instance 
\cite{spokoiny1996adaptive, ingster2000adaptive}. For more discussion, see Remark \ref{rem:adapt}.

\begin{remark}
    We remark that the false detection rate control of Theorem \ref{thm:onesamplethm} and \ref{thm:adaptiveone} does not require a H\"older smoothness assumption.
\end{remark}

\begin{algorithm}[t]
\caption{Adaptive T-Cal: an adaptive test for calibration}\label{alg:adptest}
\begin{algorithmic}
\STATE \textbf{Input:} Probability predictor $f: \mathcal{X} \to \Delta_{K - 1}$; i.i.d. sample $\{(X_i, Y_i) \in \mathcal{X} \times \mathcal{Y} : i \in \{1, \dots, n\} \}$; false detection rate $\alpha \in (0, 1)$; true detection rate $\beta \in (0, 1 - \alpha)$

\STATE \textbf{Initialize:} $B \gets \lceil \frac{2}{K - 1} \log_2(n / \sqrt{\log n})\rceil$; $Z_i \gets f(X_i)$ for $1\leq i \leq n$; $\xi_n^\text{ad}\gets 0$
\FOR{$b=1$ \TO B}
\STATE Compute $\xi_{2^b, n}(\frac{\alpha}{B})$ as in Algorithm \ref{alg:onesampletest}
\IF{$\xi_{2^b, n}(\frac{\alpha}{B})=1$}
\STATE $\xi_n^\text{ad} \gets 1$
\STATE \textbf{break}
\ENDIF
\ENDFOR
\STATE \textbf{Output:} Reject $H_0$ if $\xi_{n}^\text{ad} = 1$
\end{algorithmic}
\end{algorithm}

\subsection{Necessity of Debiasing} 
\label{sec:biasfailure}

Recall from \eqref{eq:onesamplebiased} that $T_{m, n}^\text{b}$ is the plug-in estimator of $\ell_2\text{-ECE}(f)^2$ without the debiasing term in \eqref{eq:onesampledef}.
We argue that this biased estimator is not an optimal test statistic, even for $m = m_* = \lfloor n^{2 / (4s + K - 1)}\rfloor$ from Theorem \ref{thm:onesamplethm} (which is optimal for the debiased test), by presenting a failure case in the following example.

\sloppy
\begin{example}[Failure of naive plug-in]
\label{ex:miscal}
Consider binary classification problem with $K = 2$, $m_* = \lfloor n^{2 / (4s + 1)} \rfloor$ (assumed to be divisible by four), and the partition 
$$\mathcal{B}_{m_*} = \{B_1, \dots, B_{m_*}\} 
= \left\{\left[0, \frac{1}{m_*}\right), \dots, \left[\frac{m_* - 1}{m_*}, 1\right] \right\}.$$
Let $P_0$ be the distribution over $(Z, Y) \in [0, 1] \times \{0, 1\}$ given by $Z \stackrel{P_0}{\sim} \textnormal{Unif}([0, 1])$ and $Y \mid Z = z \stackrel{P_0}{\sim} \textnormal{Ber}(z)$ for all $z \in [0, 1]$.
Under $P_0$, the probability predictor $f$ is perfectly calibrated, i.e., $P_0 \in \mathcal{P}_0$.
Let $\zeta: \R \to \R$ be the function defined by
\begin{equation}
\label{eq:zetadef}
    \zeta(x) :=  e^{-\frac{1}{x(1 - x)}} \mathds{1}_{(0, 1)}(x).
\end{equation}
Let $g: [0, 1] \to [0, 1]$ be the function (corresponding to the calibration curve of the probability predictor $f$) \begin{equation}
\label{eq:biasreg}
    g(z) := z - \rho m_*^{-s} \sum_{j = 0}^{\frac{m_*}{4} - 1} \zeta\left( m_*z - \frac{m_*}{4} - j \right) + \rho m_*^{-s} \sum_{j = \frac{m_*}{4}}^{\frac{m_*}{2} - 1} \zeta\left( m_*z - \frac{m_*}{4} - j \right)
\end{equation}
for $s \in (\frac{1}{4}, \frac{1}{2})$, $\rho > 0$, and $\zeta$ defined in \eqref{eq:zetadef}.
Define the distribution $P_1$ over $(Z, Y)$ by 
$$Z \stackrel{P_1}{\sim} \textnormal{Unif}([0, 1]) \textnormal{ and } Y \mid Z = z \stackrel{P_1}{\sim} \textnormal{Ber}(g(z))$$ 
for all $z \in [0, 1]$.
As we will show in the proof of Theorem \ref{thm:multi-lower-bound}, (1) the mis-calibration curve $g(z) - z = \res(z)$ is $s$-H\"{o}lder and (2) $\ell_2\textnormal{-ECE}_{P_1}(f) = \Theta(n^{-2s / (4s + 1)})$.

As can be seen in Figure \ref{fig:miscal}, the probability predictor $f$ under $P_1$ is an example of a mis-calibrated predictor, as $f(X)$ is smaller than $\E[Y \mid f(X)]$ when $f(X)$ is above 0.5; and vice versa.
However, the mean of $T_{m_*, n}^\textnormal{b}$ under the mis-calibrated distribution $P_1$ is surprisingly smaller than the mean under the calibrated distribution $P_0$ when $n$ is large enough (Proposition \ref{prop:biasmean}).

That is, the statistic $T_{m_*, n}^\textnormal{b}$ does not capture the amount of mis-calibration,
and therefore the calibration test based on it will not perform well.
Figure \ref{fig:biashist} confirms this finding, and Figure \ref{fig:debiashist} displays that this effect can be removed by using the debiased statistic $T_{m_*, n}^\textnormal{d}$.
\end{example}
\fussy

\begin{proposition}[Failure of naive plug-in test]
\label{prop:biasmean}
    Let $P_0$ and $P_1$ be the distributions defined in Example \ref{ex:miscal}, and $m_* = \lfloor n^{2 / (4s + 1)} \rfloor$.
    Then $\E_{P_0}[T_{m_*, n}^\textnormal{b}] \geq \E_{P_1}[T_{m_*, n}^\textnormal{b}]$ for all large enough $n \in \mathbb{N}_+$.
\end{proposition}
See Appendix \ref{sec:proofbiasmean} for the proof.
We remark that it is possible to avoid the phenomenon in Proposition \ref{prop:biasmean}, by choosing a different $m$.
Proposition \ref{prop:biasmean} aims only to highlight 
that
the effect of the bias in $T_{m, n}^\text{b}$ can be extreme in certain cases, and we do not claim that $m = m_*$ is also the optimal choice for the biased statistic.

We finally comment on the related results of \cite{brocker2012estimating, ferro2012bias, kumar2019verified}.
In \cite{brocker2012estimating, ferro2012bias}, the plug-in estimator of the squared $\ell_2$-ECE is decomposed into terms related to reliability and resolution.
Based on this observation, \cite{kumar2019verified} propose a debiased estimator for the squared $\ell_2$-ECE and show an improved sample complexity for estimation.
However, their analysis is restricted to the binary classification case and probability predictors with only finitely many output values.
It is not clear how to adapt their method to predictors with continuous outputs, because this would require discretizing the outputs.
Our debiased plug-in estimator $T_{m,n}^\text{d}$ is more general, as it can be used for multi-class problems and continuous probability predictors $f$.
Also, our reason to introduce $T_{m, n}^\text{d}$ (testing) differs from theirs (estimation).

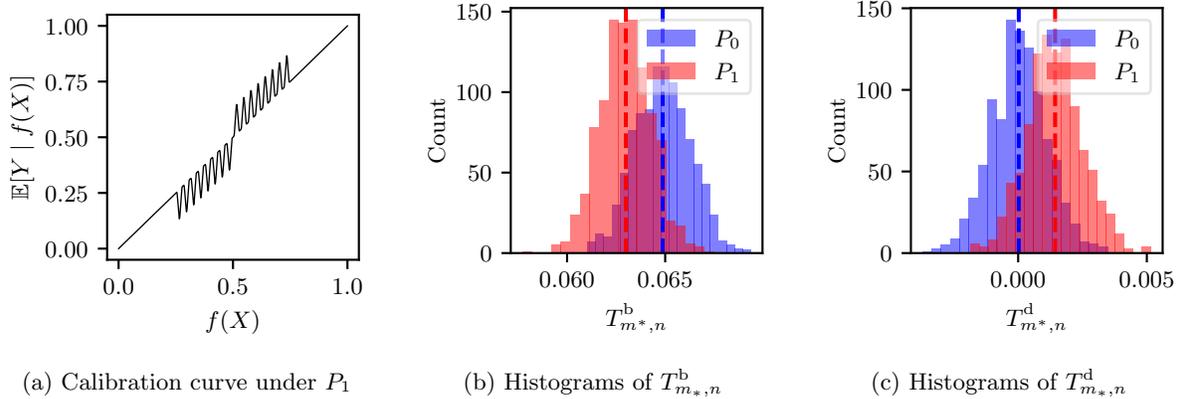
\begin{figure}
    \centering
    \captionsetup[subfigure]{justification=centering}
    \begin{subfigure}{0.32\textwidth}
        \centering
        \input{images/altdist.pgf}
        \caption{Calibration curve under $P_1$}
        \label{fig:miscal}
    \end{subfigure}
    \begin{subfigure}{0.32\textwidth}
        \centering
        \input{images/biased.pgf}
        \caption{Histograms of $T_{m_*, n}^{\text{b}}$}
        \label{fig:biashist}
    \end{subfigure}
    \begin{subfigure}{0.32\textwidth}
        \centering
        \input{images/debiased.pgf}
        \caption{Histograms of $T_{m_*, n}^{\text{d}}$}
        \label{fig:debiashist}
    \end{subfigure}
    
    \caption{\textbf{(a)} A graph of the calibration curve $z \mapsto g(z) = \E_{P_1}[Y \mid f(X) = z]$ defined in \eqref{eq:biasreg}. When the true label probability is above/below 0.5, the model outputs a smaller/larger score. Hence $f$ is a mis-calibrated probability predictor under $P_1$. \textbf{(b)} Histograms of $T_{m_*, n}^\text{b}$ and $T_{m_*, n}^\text{d}$ under $P_0$ and $P_1$ are obtained from 1,000 independent observations. We use the parameters $n = 10,000 $, $s = 0.3$, and $\rho = 100$. The dashed line indicates the empirical mean of each distribution. Note that the biased estimator $T_{m_*, n}^\text{b}$ has a smaller mean under $P_1$, which aligns with Proposition \ref{prop:biasmean}. \textbf{(c)} We see this effect disappears after debiasing and that the mean of $T_{m_*, n}^\textnormal{d}$ becomes zero.}
    \label{fig:debiasexample}
\end{figure}

\section{Experiments}
We perform experiments on both synthetic and empirical datasets to support our theoretical results. These experiments suggest that T-Cal is in general superior to state-of-the-art methods.

\begin{figure}
    \centering
    \captionsetup[subfigure]{justification=centering}
    \begin{subfigure}{0.27\textwidth}
        \centering
        \input{images/altdist2.pgf}
        \caption{Calibration curve under $P_{1, m}$}
        \label{fig:gmshape}
    \end{subfigure}
    \begin{subfigure}{0.72\textwidth}
        \centering
        \input{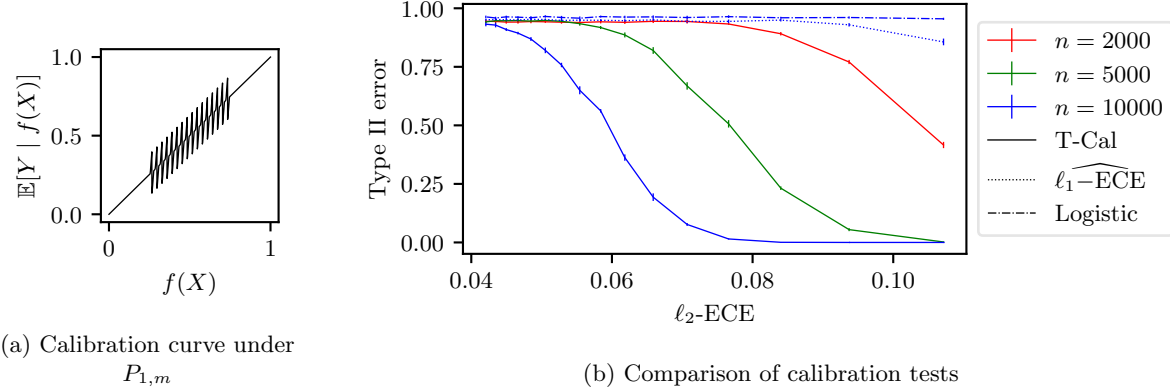}
        \caption{Comparison of calibration tests}
        \label{fig:compare}
    \end{subfigure}

    \caption{\textbf{(a)} A graph of the calibration curve $z \mapsto g_m(z) = \E_{P_{1, m}}[Y \mid f(X) = z]$ defined in \eqref{eq:gmdef}. The mis-calibration curve alternates between negative and positive values, making detection challenging. \textbf{(b)} We compare our test $\xi_{m_*, n}$ with other commonly used calibration tests.
    Since our test optimally adjusts the number of bins $m_* = \lfloor n^{2 / (4s + K - 1)} \rfloor$ according to the sample size $n$, it can detect mis-calibration over smaller and smaller intervals as $n$ grows.
    On the other hand, the plug-in test $\xi_{m,n}$, with a fixed-in-$n$ binning scheme parameter $m$, fails to detect mis-calibration over intervals smaller than the bin width. This issue remains when the sample size $n$ increases.
    The test based on the calibration slope and intercept also suffers from the same issue. Standard error bars are plotted over 10 repetitions.}
    \label{fig:exp1}
\end{figure}

\subsection{Synthetic Data: Power Analysis}
\label{sec:powanalysis}

Let $P_0 \in \mathcal{P}_0$ be the distribution defined in Example \ref{ex:miscal}---a distribution under which $f$ is perfectly calibrated.
For $m \in \mathbb{N}_+$, $s > 0$, $\rho > 0$, and $\zeta: \R \to \R$ from \eqref{eq:zetadef}, define $g_m : [0, 1] \to [0, 1]$ by
\begin{equation}
\label{eq:gmdef}
    g_m(z) := z + \rho m^{-s} \sum_{j = 0}^{m - 1} (-1)^j \zeta\left(2mz - \frac{m}{2} - j \right).
\end{equation}
This function oscillates strongly, as shown in Figure \ref{fig:gmshape}.
Let $P_{1, m}$ be the distribution over $(Z, Y) \in [0, 1] \times \{0, 1\}$ given by $Z \stackrel{P_{1, m}}{\sim} \text{Unif}([0, 1])$ and $Y \mid Z = z \stackrel{P_{1, m}}{\sim} \text{Ber}(g_m(z))$ for all $z \in [0, 1]$.
Under $P_{1, m}$, the probability predictor $f$ is mis-calibrated with an $\ell_p$-ECE of at least $\ep = \rho \Vert \zeta \Vert_{L^p} m^{-s}$. However, since the mis-calibration curve $g_m$ oscillates strongly, mis-calibration can be challenging to detect.

We study the type II error of tests against the alternative where the mis-calibration is specified as $H_1: (Z, Y) \sim P_{1, m}$.
This gives a lower bound on the worst-case type II error over the alternative hypothesis $\mathcal{P}_1(\ep, p, s)$.
We repeat the experiment for different values of $m$ to obtain a plot of $\ell_2$-ECE versus type II error.

\paragraph{Comparison of Tests.} We compare the test $\xi_{m_*, n}$ with classical calibration tests dating back to \cite{cox1958two}, and discussed in \cite{harrell2015regression, vaicenavicius2019evaluating}.
\cite{harrell2015regression} refits a logistic model
$$P(Y = 1 \mid Z) = \frac{1}{1 + \exp\left( -\left( \gamma_0 + \gamma_1 \log \frac{Z}{1 - Z} \right) \right)}$$
on the sample $\{(Z_i, Y_i): i \in \{1, \dots, n\}\}$ and tests the null hypothesis of $\gamma_0 = 0$ and $\gamma_1 = 1$.
Specifically, we perform the score test \citep{rao1948large, silvey1959lagrangian}, with the test statistic derived from the gradient of log-likelihood with respect to the tested parameters.
There are several approaches to set the critical values, including by using the asymptotic distribution theory of sampling statistics under the null hypothesis, or by data reuse methods such as the bootstrap.
We estimate the critical values via 1000 Monte Carlo simulations.

\begin{figure}
    \centering
    \input{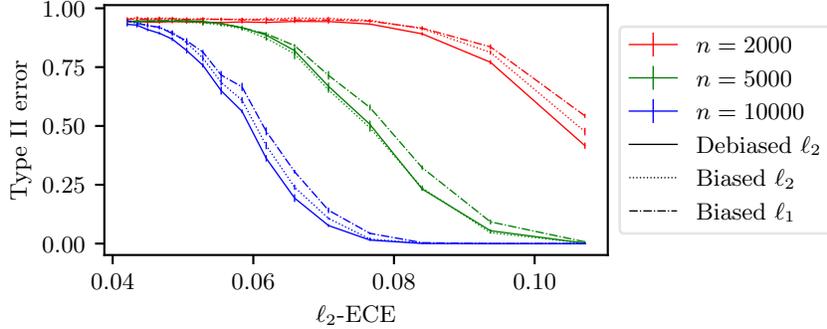}
    \caption{Type II error comparison for $T_{m_*, n}^\text{d}$ (T-Cal), $T_{m_*, n}^\text{b}$, and $T_{m_*, n}^{\ell_1}$. Using $\ell_2$ is better than $\ell_1$, and debiased $\ell_2$ (T-Cal) is better than biased $\ell_2$. Standard error bars are plotted over 10 repetitions.}
    \label{fig:l1vsl2}
\end{figure}

\cite{vaicenavicius2019evaluating} use $\widehat{\ell_1\text{-ECE}}$, the plug-in estimator for $\ell_1\text{-ECE}$, as their test statistic.
They approximate the distribution of $\widehat{\ell_1\text{-ECE}}$ by a bootstrapping procedure called consistency resampling (in which both the probability predictions and the labels are resampled) and compute a $p$-value based on this approximation.
This test also uses a plug-in estimator as the test statistic but differs from T-Cal as it is neither debiased nor adaptive.
Since the data-generating distribution is known in this synthetic experiment, we set the critical value via 1000 Monte Carlo simulations.

We control the false detection rate at a level $\alpha = 0.05$ and run experiments for $n = 2,000, 5,000$, and $10,000$.
We find that our proposed test achieves the lowest type II error, see Figure \ref{fig:compare}.
We also find that other tests do not leverage the growing sample size $n$.
For this reason, we only display $n = 10,000$ for the other two tests. As can be seen in Figure \ref{fig:compare}, T-Cal outperforms other testing methods in true detection rate by a large margin.

\paragraph{$\ell_1\textnormal{-ECE}$ versus $\ell_2\textnormal{-ECE}$.}
To confirm the effectiveness of using the $\ell_2$-ECE estimator $T_{m_*, n}^\text{d}$, we compare it with a plug-in $\ell_1$-ECE estimator defined as
$$T_{m, n}^{\ell_1} := \sum_{\substack{1 \leq i \leq m^{K - 1}\\|\mathcal{I}_{m, i}| \geq 1}} \frac{|\mathcal{I}_{m, i}|}{n} \left\Vert \frac{1}{|\mathcal{I}_{m, i}|} \sum_{j \in \mathcal{I}_{m, i}} (Y_j - Z_j) \right\Vert_1.$$
At the moment, it is unknown how to debias this estimator.
We use the optimal binning parameter $m = m_* = \lfloor n^{2 / (4s + K - 1)} \rfloor$ for both $\ell_1$ and $\ell_2$ estimators; because it is unknown what the $\ell_1$-optimal binning scheme is.
Also, to isolate the effect of debiasing, we compare the biased $\ell_2$ estimator $T_{m_*,n}^\text{b}$ as well, with the same number of bins.
In Figure \ref{fig:l1vsl2}, we see the $\ell_2$ estimators consistently outperform the $\ell_1$ estimator, regardless of debiasing. 
While it is a common practice to use a plug-in estimator of $\ell_1$-ECE, our result suggests T-Cal compares favorably to it.

\paragraph{Minimum Detection Rate.}
We perform an experiment to support the result on the minimum detection rate of T-Cal, presented in Theorem \ref{thm:onesamplethm}.
For each $n$, we find the largest integer, denoted $m(n)$, such that the type II error against $H_1: (Z, Y) \sim P_{1, m(n)}$ is less than $0.05$.
We compute $\ep_n := \ell_2\text{-ECE}_{P_{1, m(n)}}(f) = \rho \Vert \zeta \Vert_{L^2} m(n)^{-s}$ (a lower bound on the minimum detection rate) and plot $\log \ep_n$ versus $\log n$  in Figure \ref{fig:detection}.
We see that the logarithm decreases as $n$ grows, with the slope $-\frac{2s}{4s + 1}$ predicted by Theorem \ref{thm:onesamplethm}.

\begin{figure}
    \centering
    \captionsetup[subfigure]{justification=centering}
    \begin{subfigure}{0.4\textwidth}
        \centering
        \input{images/detection_rate0.pgf}
        \caption{$s = 0.3$, $\rho = 25$}
        \label{fig:detection0}
    \end{subfigure}
    \begin{subfigure}{0.4\textwidth}
        \centering
        \input{images/detection_rate1.pgf}
        \caption{$s = 0.5$, $\rho = 50$}
        \label{fig:detection1}
    \end{subfigure}

    \caption{The dots are $\log \ep_n$ computed for different sample sizes $n$. The red line has a slope $-\frac{2s}{4s + 1}$. Standard error bars are plotted over 10 repetitions. See the text for more details.}
    \label{fig:detection}
\end{figure}
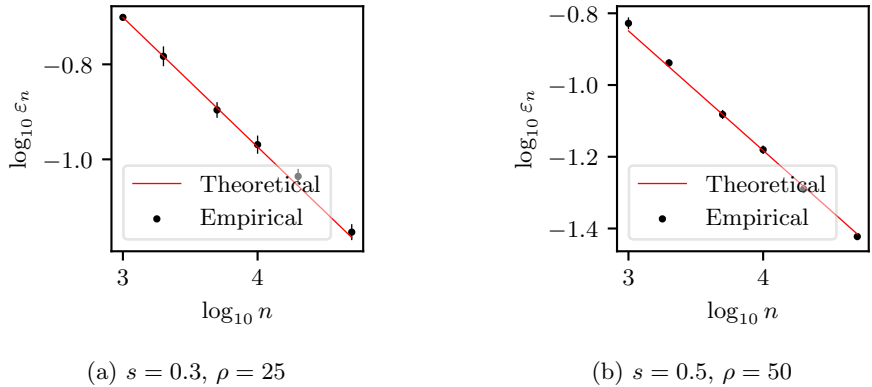

\subsection{Results on Empirical Datasets}
\label{sec:exp}
\begin{table}[ht]
\centering
    \begin{tabular}{lcccccc}
    \toprule
\multirow{2}{*}{} & \multicolumn{2}{c}{DenseNet 121} & \multicolumn{2}{c}{ResNet 50} & \multicolumn{2}{c}{VGG-19} \\
                  & $\widehat{\ell_1\text{-ECE}}$        & Calibrated?        & $\widehat{\ell_1\text{-ECE}}$      & Calibrated?       & $\widehat{\ell_1\text{-ECE}}$      & Calibrated?     \\\midrule
No Calibration     & 2.02\%            & reject                   & 2.23\%          & reject                  & 2.13\%         & reject                \\                 
Platt Scaling     & 2.32\%            & reject                   & 1.78\%          &    reject                & 1.71\%         & reject                \\
Poly. Scaling     & 1.71\%            & reject                     & 1.29\%          & reject                  & 0.90\%         & accept              \\
Isot. Regression                  & 1.16\%            & reject                  & 0.62\%          & reject                  & 1.13\%         & accept  \\
Hist. Binning                  & 0.97\%            & reject                   & 1.12\%          & reject                  & 1.28\%         & reject  \\
Scal. Binning                  & 1.94\%            & reject                  & 1.21\%          & reject                  & 1.67\%         & reject  \\
\bottomrule          
\end{tabular}
\caption{The values of the empirical $\ell_1\text{-ECE}$ \citep{guo2017calibration} and the testing results, via adaptive T-Cal and multiple binomial testing, of models trained on CIFAR-10.}
\label{tab:cifar-10}
\end{table}

To verify the performance of adaptive T-Cal empirically, we apply it to the probability predictions output by deep neural networks trained on several datasets. Since our goal is to test calibration, we calculate the probabilities predicted by pre-trained models on the test sets. 
As in \citep[][etc]{guo2017calibration,kumar2019verified,nixon2019measuring}, we binarize the test labels by taking the top-1 confidence as the new probability prediction, and the labels as the results of the top-$1$ classification, i.e., $\widetilde{Z}=\max_{1\leq k\leq K}[Z]_k$ and $\widetilde{Y}=I(\text{correctly classified by the top-1 prediction})$. 
This changes the problem of detecting the full-class mis-calibration to testing the mis-calibration of a binary classifier. Hence, we choose $K=2$ for adaptive T-Cal in the experiments below. We refer readers to \citep{Gupta2021ToplabelCA} for more details about binarization via the top-$1$ prediction.

\paragraph{CIFAR-10.} For the CIFAR-10 dataset, the models are DenseNet 121, ResNet 50, and VGG-19. 
We first apply the adaptive test directly to the $10,000$ uncalibrated probability predictions output by each model, with the false detection rate controlled at the level $\alpha=0.05$. 

For every choice of the number of bins $m$, we estimate the critical value by taking the upper $5\%$ quantile of the values of the test statistic over $3,000$ bootstrap re-samples of the probability predictions. The labels are also chosen randomly, following Bernoulli distributions with the probability prediction as the success probability. 
We also provide the values of the standard empirical $\ell_1\text{-ECE}$ calculated with \cite{guo2017calibration}'s approach for the reader's reference, and with $15$ equal-width bins.

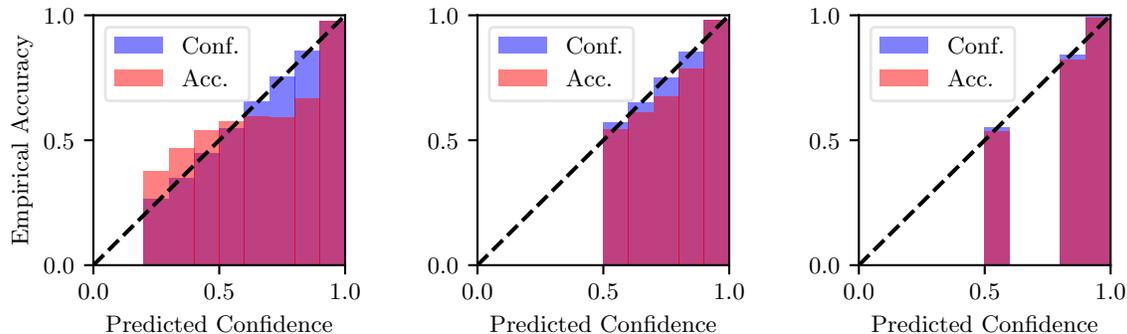
\begin{figure}
    \centering
    \captionsetup[subfigure]{justification=centering}
    \begin{subfigure}{0.33\textwidth}
        \centering
        \input{images/diagramplatt.pgf}
        \label{fig:platt}
    \end{subfigure}
    \begin{subfigure}{0.30\textwidth}
        \centering
        \input{images/diagrampoly.pgf}
        \label{fig:poly}
    \end{subfigure}
    \begin{subfigure}{0.30\textwidth}
        \centering
        \input{images/diagramhistogram.pgf}
        \label{fig:hist}
    \end{subfigure}
    \caption{The reliability diagrams for VGG-19, trained on CIFAR-10, calibrated by Platt scaling (left), polynomial scaling (middle), and histogram binning (right). The bins (bars) containing less than $10$ data points, where the sample noise dominates, are omitted for clarity. The dashed lines correspond to perfect calibration.}
    \label{fig:diagram}
\end{figure}

We then test the probability predictions of these three models calibrated by several post-calibration methods: Platt scaling \citep{Platt1999ProbabilisticOF}, polynomial scaling, isotonic regression \citep{Zadrozny2002TransformingCS}, histogram binning \citep{Zadrozny2001ObtainingCP}, and scaling-binning \citep{kumar2019verified}. To this end, we split the original dataset of $10,000$ images into $2$ sets---of sizes $2,000$ and $8,000$. 
The first set is used to calibrate the model, and the second is used to perform adaptive T-Cal and calculate the empirical $\ell_1\text{-ECE}$.
In polynomial scaling, we use polynomials of order $3$ to do regression on all the prediction-label pairs $(Z_i,Y_i)$, and truncate the calibrated prediction values into the interval $[0,1]$. We set the binning scheme in both histogram binning and scaling binning as $15$ equal-mass  bins.  Our implementation is adapted from \cite{kumar2019verified}. 

Since the recalibrated probability predictions output by the latter two methods belong to a finite set, we use a test based on the binomial distribution.
See Appendix \ref{sec:discretetest} for details.
For completeness, we also provide the debiased empirical $\ell_2\text{-ECE}$ values \citep{kumar2019verified} for models calibrated by the two discrete methods, see the details in Table \ref{tab:cifar-10-2}, Appendix \ref{sec:more_experiments}.

The results are listed in Table \ref{tab:cifar-10}, where we use ``accept" to denote that the test does not reject. The models with smaller empirical $\ell_1$\text{-ECE} are more likely to be accepted, by adaptive T-Cal and by multiple binomial testing, as perfectly calibrated. This can be further illustrated by the three empirical reliability diagrams given in Figure \ref{fig:diagram}, where the model's predictions calibrated by Platt scaling (left) are visually  more ``mis-calibrated''  than those calibrated by polynomial scaling (middle) and histogram binning (right).

\paragraph{CIFAR-100. } We perform the same experimental procedure for three models pre-trained on the CIFAR-100 dataset: MobileNet-v2, ResNet 56, and ShuffleNet-v2 \citep{chanyaofo}. The test set provided by CIFAR-100 is split into two parts, containing $2,000$ and $8,000$ images, respectively. 
Since the regression functions $\reg$ of models trained on the larger CIFAR-100 dataset can be more complicated than those of models trained on CIFAR-10, we set the polynomial degree as five in polynomial scaling. 

The results are listed in Table \ref{tab:cifar-100}. The values of the debiased empirical $\ell_2\text{-ECE}$ \citep{kumar2019verified} for the two discrete calibration methods are provided in Table \ref{tab:cifar-100-2}, Appendix \ref{sec:more_experiments}. The results roughly align with the magnitude of the empirical ECE value. 

However, as can be observed in the column corresponding to ResNet 56, 
this trend is certainly \emph{not} monotone. 
The calibrated ResNet 56 with the empirical ECE $1.84\%$ is accepted while the calibrated ResNet 56 with a smaller value $1.57\%$ is rejected. 
Furthermore, the test results reveal that models with relatively large (or small) empirical $\ell_1\text{-ECE}$ values may not necessarily be poorly (or well) calibrated since the $\ell_1\text{-ECE}$ values measured can be highly dominated by the sample noise.

\begin{table}
\centering
    \begin{tabular}{lcccccc}
    \toprule
\multirow{2}{*}{} & \multicolumn{2}{c}{MobileNet-v2} & \multicolumn{2}{c}{ResNet 56} & \multicolumn{2}{c}{ShuffleNet-v2} \\
                  & $\widehat{\ell_1\text{-ECE}}$        & Calibrated?        & $\widehat{\ell_1\text{-ECE}}$      & Calibrated?       & $\widehat{\ell_1\text{-ECE}}$      & Calibrated?     \\\midrule
No Calibration      & 11.87\%            & reject                   & 15.2\%          & reject                  & 9.08\%         & reject               \\                  
Platt Scaling       & 1.40\%            & accept                   & 1.84\%          & accept                & 1.34\%         & accept                \\
Poly. Scaling       & 1.69\%            & reject                   & 1.91\%         & reject                  & 1.81\%         &  accept               \\
Isot. Regression    & 1.76\%            & accept                   & 2.33\%           & reject                  & 1.38\%         & accept  \\
Hist. Binning       & 1.66\%            & reject                   & 2.44\%          & reject                  & 2.77\%         & reject  \\
Scal. Binning       & 1.85\%            & reject                   & 1.57\%          & reject                  & 1.65\%         & accept  \\
\bottomrule          
\end{tabular}
\caption{The values of the empirical $\ell_1\text{-ECE}$ \citep{guo2017calibration} and the testing results, via adaptive T-Cal and multiple binomial testing, of models trained on CIFAR-100.}
\label{tab:cifar-100}
\end{table}

\paragraph{ImageNet. } We repeat the above experiments on models pre-trained on the ImageNet dataset. We examine three pre-trained models provided in the torchvision package in PyTorch: DenseNet 161, ResNet 152, and EfficientNet-b7. We split the validation set of $50,000$ images into a calibration set and a test set---of sizes $10,000$ and $40,000$, respectively. We use polynomials of degree $5$ in polynomial scaling.

The results are listed in Table \ref{tab:imagenet}. The values of the debiased empirical $\ell_2\text{-ECE}$ \citep{kumar2019verified} are provided in Table \ref{tab:imagenet-2}, in Appendix \ref{sec:more_experiments}. As can be seen, the test results here generally align with the empirical ECE values.

\begin{remark}
    While it is hard to verify the H\"older smoothness of $\text{res}_f$ in these empirical datasets, we believe that some level of justification would be possible assuming (1) enough differentiability on $f$ (which we think to be true for common neural net architectures and activation functions), (2) $Y$ being deterministic given $X$ (which is reasonable for low-noise datasets such as ImageNet), and 
    (3) the set of inputs $\mathcal{C}_k = \{x \in \mathcal{X}: (Y | X = x) = e_k\}$ corresponding to each class being a Lipschitz domain, locally the graph of a Lipschitz continuous function.
Given these assumptions, each coordinate of the regression function $\text{reg}_f(z)$ can be written as $[\text{reg}_f(z)]_k = \mathbb{E}[[Y]_k | f(X) = z] = P_Z(\mathcal{C}_k \cap \{f(x) = z\}) / P_Z(\{f(x) = z\})$; assuming the denominator is strictly positive.
Then, H\"older continuity may follow from the inverse function theorem and the coarea formula \cite[Theorem 3.2.3]{federer2014geometric}, which expresses the measure of a level set as an integral of the Jacobian.

However, this argument still requires making the essentially unverifiable Assumption (3) from the above paragraph.
In some cases, this assumption can be viewed as reasonable: for instance, one may reasonably think that image manifolds for classes in ImageNet are locally Lipschitz; and thus Assumption (3) may hold.
Under such conditions, this rough argument may provide an idea of why H\"older smoothness could be reasonable for certain predictive models such as neural net architectures.

When the H\"older condition does not hold, we still have the type I error guarantee, but may not have type II error control. 
We expect that the H\"older condition might be drastically violated when there is obvious discontinuity of the predictors/regression functions (e.g., a decision tree/random forest trained on data having discrete features).
\end{remark}

\begin{table}
\centering
    \begin{tabular}{lcccccc}
    \toprule
\multirow{2}{*}{} & \multicolumn{2}{c}{DenseNet 161} & \multicolumn{2}{c}{ResNet 152} & \multicolumn{2}{c}{EfficientNet-b7} \\
                  & $\widehat{\ell_1\text{-ECE}}$        & Calibrated?        & $\widehat{\ell_1\text{-ECE}}$      & Calibrated?       & $\widehat{\ell_1\text{-ECE}}$      & Calibrated?     \\\midrule
No Calibration      & 5.67\%            & reject                   & 4.99\%          & reject                 & 2.82\%        & reject              \\                  
Platt Scaling       & 1.58\%           &    reject                & 1.41\%          & reject                & 1.90\%         & reject                \\
Poly. Scaling       & 0.62\%            &  accept                  & 0.64\%         &  accept                 & 0.71\%         & accept                \\
Isot. Regression    & 0.63\%            & reject                   & 0.80\%           & reject                  & 1.06\%         & reject  \\
Hist. Binning       & 0.46\%            & reject                   & 1.26\%          & reject                  & 0.88\%         & reject  \\
Scal. Binning       & 1.55\%            & reject                   & 1.40\%          & reject                  & 1.97\%         & reject  \\
\bottomrule          
\end{tabular}
\caption{The values of the empirical $\ell_1\text{-ECE}$ \citep{guo2017calibration} and the testing results, via adaptive T-Cal and multiple binomial testing, of models trained on ImageNet.}
\label{tab:imagenet}
\end{table}

\section{Lower Bounds for Detecting Mis-calibration}

To complement our results on the performance of the plug-in tests proposed earlier, we now show some fundamental lower bounds for detecting mis-calibration. We also provide a reduction that allows us to test calibration via two-sample tests and sample splitting. We show that this has a minimax optimal performance, but empirically does not perform as well as our previous test.

\subsection{Impossibility for General Continuous Mis-calibration Curves}
\label{ic}

In Proposition \ref{impossibility}, we show that detecting mis-calibration is impossible, even when the sample size $n$ is arbitrarily large unless the mis-calibration curve $\res$ has some level of smoothness. 
Intuitively, if the mis-calibration curve can be arbitrarily non-smooth, then it can oscillate between positive and negative values with arbitrarily high frequency, and these oscillations cannot be detected from a finite sample.

In this regard, one needs to be careful when concluding the quality of calibration from a finite sample. If we only assume that the mis-calibration curve is a continuous function of the probability predictions, then it is impossible to tell apart calibrated and mis-calibrated models. Further, for more complex models such as deep neural networks, one expects the predicted probabilities to be able to capture larger and larger classes of functions; thus this result is even more relevant for modern large-scale machine learning. 

\sloppy
Let $\mathcal{P}_1^\text{cont}(\ep, p)$ be the family of probability distributions $P$ over $(Z, Y)$
such that $\ell_p\text{-ECE}_P(f) \geq \ep$ and every entry of the mis-calibration curve $\text{res}_{f, P}$ is continuous.
This is a larger set of distributions than $ \mathcal{P}_{1}(\ep, p, s)$ in \eqref{p1}, because we only assume continuity, not H\"older smoothness. 
Denote the corresponding minimax type II error by $R_n^\text{cont}(\ep, p)$, namely
$$R_n^\text{cont}(\ep, p) := \inf_{\xi \in \Phi_n(\alpha)} \sup_{P \in \mathcal{P}_1^\textnormal{cont}(\ep, p)} P(\xi=0).$$
This has the same interpretation as before, namely, it is the best possible false negative rate for detecting mis-calibration for data distributions belonging to $\mathcal{P}_1^\text{cont}(\ep, p)$, in a worst-case sense.

\begin{proposition}[Impossibility of detecting mis-calibration]
\label{impossibility}
Let $\ep_0 = 0.1$.
For any level $\alpha \in (0,1)$, the minimax type II error $R_{n}^\textnormal{cont}(\ep_0, p)$ for testing the null hypothesis of calibration at level $\alpha$ against the hypothesis $P \in \mathcal{P}_1^\textnormal{cont}(\ep_0, p)$ of general continuous mis-calibration curves satisfies $R_{n}^\textnormal{cont}(\ep_0, p) \geq 1-\alpha$ for all $n$.
\end{proposition}
\fussy

In words, this result shows that for a certain fixed $\ell_p$ calibration error $\ep_0$, and for a fixed false positive rate $\alpha>0$, the false negative rate is at least $1-\alpha$. Thus, it is not possible to detect mis-calibration in this setting. The choice of the constant $\ep_0=0.1$ is arbitrary and can be replaced by any other constant; the result holds with minor modifications to the proof.

The proof can be found in Appendix \ref{sec:proofimpossibility}.
We make a few remarks on related results.
While Example 3.2 of \cite{kumar2019verified} demonstrates that---related to earlier results on probability distribution and density estimation \citep{mann1942choice}---using a binned estimator of ECE can arbitrarily underestimate the calibration error, we show a fundamental failure not due to binning, but instead due to the finite sample size.
Also, our result echoes Theorem 3 of \cite{gupta2020distribution} which states that asymptotically perfect calibration is only possible for probability predictors with a countable support; but this does not overlap with Proposition \ref{impossibility} as our conclusion is about the impossibility of detecting mis-calibration.

\subsection{H\"older Alternatives}
\label{sec:holderlowerbound}

As is customary in nonparametric statistics \citep{ingster1987minimax, low1997nonparametric, gyorfi2002distribution,ingster2012nonparametric}, we consider
testing against H\"older continuous alternatives; 
or, differently put, detecting mis-calibration when the mis-calibration curves are H\"older continuous.
This excludes the pathological examples where 
the mis-calibration curves oscillate widely that were discussed in Section \ref{ic};
but still allow a very rich class of possible mis-calibration curves, including non-smooth ones.

Theorem \ref{thm:multi-lower-bound} states that, for a $K$-class classification problem and for alternatives with a H\"older smoothness parameter $s$, the mis-calibration of a model can be detected only when the calibration error is
of order
$\Omega(n^{-2s/(4s + K - 1)})$.
In other words, the smallest possible calibration error that can be detected using a sample of size $n$ is of order $n^{-2s/(4s + K - 1)}$. 

Testing calibration of a probability predictor in our nonparametric model leads to rates that are slower than the parametric case $n^{-1/2}$.
This is because ${2s/(4s + K - 1)}<1/2$ for $s>0$ and $K\geq 2$.
The rate becomes even slower as the number of classes $K$ grows.
This indicates that evaluating model calibration on a small-sized dataset can be problematic. Further, it suggests that multi-class calibration may be even harder to achieve.

This rate is what one may expect based on results for similar problems in nonparametric hypothesis testing \citep{ingster2012nonparametric}, with $K-1$ interpreted as the dimension.
Specifically, the rate is equal to the minimum separation rate in two-sample goodness-of-fit testing for densities on $\Delta_{K - 1}$.
This connection to two-sample testing will be made clear in Section \ref{sec:upperbound}.

\sloppy
\begin{theorem}[Lower bound for detecting mis-calibration]
\label{thm:multi-lower-bound}
Given a level $\alpha \in (0, 1)$ and $\beta \in (0, 1 - \alpha)$, consider the hypothesis testing problem \eqref{eq:setup}, in which we test the calibration of the $K$-class probability predictor $f$ assuming $(s, L)$-H\"older continuity of mis-calibration curves as defined in \eqref{eq:holderdef}.
There exists  $c_\textnormal{lower}>0$ depending only on $(p, s,L,K,\alpha,\beta)$ such that, for any $p > 0$, the minimum $\ell_p\textnormal{-ECE}$ of $f$, i.e. $\ep_n(p, s)$, required to have a test with a false positive rate (type I error) at most $\alpha$ and with a true positive rate (power) at least $1 - \beta$ satisfies $\ep_n(p, s)\geq c_\textnormal{lower}n^{-2s / (4s +  K - 1)}$ for all $n$. 
\end{theorem}
\fussy

See Appendix \ref{sec:proofmultilowerbound} for the proof.
The proofs of both Proposition \ref{impossibility} and Theorem \ref{thm:multi-lower-bound} are based on Ingster's method, also known as the chi-squared or Ingster-Suslina method \citep{ingster1987minimax,ingster2012nonparametric}.
Informally, Ingster's method states that if we can select alternative distributions with an average likelihood ratio to a null distribution close to unity, then no test with a fixed level can control the minimax type II error below a certain threshold.


\sloppy
\begin{remark}[H\"older smoothness assumption]\label{hol-assu}
Since the residual function $\textnormal{res}_{f, P}$ depends on the unknown joint distribution $P$ of $(Z, Y)$,  the H\"older continuity of the map $\mathbf{z} \mapsto [\res(\mathbf{z})]_k$ for each $k \in \{1, \dots, K\}$ is in general an assumption that we need to make.
When the residual map $\res$ does not satisfy the H\"older assumption, we still have the false detection rate control in Theorem \ref{thm:onesamplethm} and \ref{thm:adaptiveone}, but we cannot guarantee the true detection rate control and the lower bound in Theorem \ref{thm:multi-lower-bound}.
Extending our approach beyond the H\"older assumption may be possible in future work, inspired by works in nonparametric hypothesis testing that study---for instance---Besov spaces of functions \citep{ingster2012nonparametric}. 
\end{remark}
\fussy

\section{Reduction to Two-sample Goodness-of-fit Testing}
\label{sec:upperbound}
\sloppy
To further put our work in context in the literature on nonparametric hypothesis testing, in this section we carefully examine the connections between the problem of testing calibration, and a well-known problem in that area.  
Specifically, we describe a novel randomization scheme that allows us to reduce the null hypothesis of perfect calibration to a hypothesis of equality of two distributions---making a strong connection to the problem of \emph{two-sample goodness of fit testing}.
In other words, we can use a calibrated probabilistic classifier and randomization to generate two samples from an identical distribution. For 
a mis-calibrated classifier the same scheme will generally result in two samples from two different distributions. 

If a classifier is perfectly calibrated, then its class probability predictions will match the true prediction-conditional class probabilities.
Therefore, randomly sampling labels according to the classifier's probability predictions will yield a sample from the empirical distribution.
We rely on sample splitting to obtain two samples: the empirical and the generated one.
Then we can use any classical test to check if the two samples are generated from the same distribution.  As we will show, the resulting test has a theoretically optimal detection rate. However, due to the sample splitting step, its empirical performance is inferior to the test based on the debiased plug-in estimator from Section \ref{sec:onesample}.

\fussy
We split our sample into two parts. 
For $i \in \{\lfloor n/2\rfloor +1, \ldots, n\}$, we generate random variables $\tilde{Y}_i$ following the
categorical distribution $\text{Cat}(Z_i)$ over classes $\mathcal{Y}= \{1, \ldots, K\}$, with a $K$-class probability distribution $Z_i=f(X_i)$  predicted by the classifier $f$.
These $\tilde Y_i$ are independent of each other and of $Y_1, \dots, Y_{\lfloor n/2 \rfloor}$, due to the sample splitting step.
For each $k \in \{1,\ldots,K\}$, define
$$\mathcal{V}_k := \left\{Z_i: [Y_i]_k = 1, 1 \leq i \leq \left\lfloor \frac{n}{2} \right\rfloor \right\}$$
and
$$\mathcal{W}_k := \left\{Z_i: [\tilde Y_i]_k = 1, \left\lfloor \frac{n}{2} \right\rfloor + 1 \leq i \leq n \right\}.$$
By construction, $\mathcal{V}_k$ is an i.i.d. sample from the distribution on
the probability simplex
$\Delta_{K - 1}$ with a density\footnote{We assume that the densities $\pi_k^\mathcal{V}$ and $\pi_k^\mathcal{W}$ are well defined, and in particular that $\E[Y]_k>0$ and $\E[Z]_k>0$ for every $k \in \{1,\ldots,K\}$.
This follows from Assumption \ref{assumption:classprob}, which will be introduced later in the section.}
$$\pi_k^\mathcal{V}(\mathbf{z}) := \frac{[\reg(\mathbf{z})]_k}{\int_{\Delta_{K - 1}} [\reg(\mathbf{z})]_k dP_Z(\mathbf{z})} = \frac{[\reg(\mathbf{z})]_k}{\E[Y]_k}$$
with respect to $P_Z$.
Similarly, $\mathcal{W}_k$ is an i.i.d. sample from the distribution on $\Delta_{K - 1}$ with a density
$$\pi_k^\mathcal{W}(\mathbf{z}) := \frac{[\mathbf{z}]_k}{\int_{\Delta_{K - 1}} [\mathbf{z}]_k dP_Z(\mathbf{z})} = \frac{[\mathbf{z}]_k}{\E[Z]_k}.$$
Now we consider testing the null hypothesis
$$H_{0}: \pi_k^\mathcal{V} = \pi_k^\mathcal{W} \text{ for all $k\in\{1,\ldots,K\}$} 
$$
against the complement of $H_{0}$.
We claim that if we use an appropriate test for this null hypothesis, with an additional procedure to rule out ``easily detectable'' alternatives, then we can obtain a test that attains the optimal rate specified in Theorem \ref{thm:multi-lower-bound}. 

We describe the main idea of this reduction.
A formal result can be found in Theorem \ref{thm:reductiontest}.
Let $P \in \mathcal{P}_1(\ep, p, s)$ and assume $\ep = \tilde{\Omega}(n^{-2s/(4s + K - 1)})$.\footnote{Here we use the notation $\tilde{\Omega}(\cdot)$ to include the adaptive case.
See Corollary \ref{cor:adaptivetwo} and Remark \ref{rem:adapt} for further details.}
The squared distance between $\pi_k^\mathcal{V}$ and $\pi_k^\mathcal{W}$ in $L^2(P_Z)$ is
\begin{align}
\label{eq:l2distance}
    &\int_{\Delta_{K - 1}} \left( \frac{[\reg(\mathbf{z})]_k}{\E[Y]_k} - \frac{[\mathbf{z}]_k}{\E[Z]_k} \right)^2 dP_Z(\mathbf{z}) \nonumber 
    = \int_{\Delta_{K - 1}} \left( \frac{[\res(\mathbf{z})]_k}{\E[Y]_k} + \frac{[\mathbf{z}]_k}{\E[Y]_k} - \frac{[\mathbf{z}]_k}{\E[Z]_k} \right)^2 dP_Z(\mathbf{z}) \nonumber\\ 
    &\geq \frac{1}{(\E[Y]_k)^2} \int_{\Delta_{K - 1}} [\res(\mathbf{z})]_k^2 dP_Z(\mathbf{z}) + \frac{2\E[Z - Y]_k}{(\E[Y]_k)^2 \E[Z]_k} \int_{\Delta_{K - 1}} [\mathbf{z}]_k [\res(\mathbf{z})]_k dP_Z(\mathbf{z}).
\end{align}
Further,
\begin{align}
\label{eq:totalexp}
    &\int_{\Delta_{K - 1}} [\mathbf{z}]_k [\res(\mathbf{z})]_k dP_Z(\mathbf{z}) = \E_P[[Z]_k\E[Y - Z| Z]_k] = \E_P[[Z]_k [Y - Z]_k].
\end{align}
Since $\E[Y - Z]_k = \E[[Z]_k[Y - Z]_k] = 0$ and $\V([Y - Z]_k), \V([Z]_k[Y - Z]_k) \leq 1$ under $H_0$,
we can detect mis-calibration for the alternatives $P \in \mathcal{P}_1(\ep, p, s)$ such that $\E_{P}[Y - Z]_k = \Omega(n^{-1/2})$ or $\E_{P}[[Z]_k[Y - Z]_k] = \Omega(n^{-1/2})$ by rejecting the null hypothesis $H_0$ of calibration if $\frac{1}{n}\sum_{i = 1}^n [Y_i - Z_i]_k \geq c n^{-1/2}$ or $\frac{1}{n}\sum_{i = 1}^n [Z_i]_k [Y_i - Z_i]_k \geq c n^{-1/2}$ for some $c > 0$. 
For the remaining alternatives, choose $k_0 \in \{1,\ldots,K\}$ such that
\begin{align*}
     \frac{1}{(\E[Y]_{k_0})^2} \int_{\Delta_{K - 1}} [\res(\mathbf{z})]_k^2 dP_Z(\mathbf{z}) \geq \frac{\ep^2}{K(\E[Y]_{k_0})^2} = \tilde{\Omega}(n^{-\frac{4s}{4s + K - 1}}).
\end{align*}
Then, $\Vert \pi_k^\mathcal{V} - \pi_k^\mathcal{W}\Vert_{L^2(P_Z)}^2$ is at least
\begin{align*}
    &\tilde{\Omega}(n^{-\frac{4s}{4s + K - 1}}) +  \frac{2\E[Z - Y]_{k_0} \E[[Z]_{k_0} [Y - Z]_{k_0}]}{(\E[Y]_{k_0})^2 \E[Z]_{k_0}} = \tilde{\Omega}(n^{-\frac{4s}{4s + K - 1}}).
\end{align*}
\sloppy
Since $|\mathcal{V}_{k_0}|, |\mathcal{W}_{k_0}| = \Theta(n)$ with high probability, the power of the test can be controlled using standard results on two sample testing.  The full procedure is described in Algorithm \ref{alg:reductiontest}.

In general,
for a positive integer $d>0$,
we allow using an arbitrary deterministic two-sample testing procedure $\textnormal{\texttt{TS}}_{\alpha, \beta}: ([0, 1]^d)^{n_1} \times ([0, 1]^d)^{n_2} \to \{0, 1\}$, which takes in two $d$-dimensional samples $\{V_1, \dots,V_{n_1}\}$, $\{W_1, \dots, W_{n_2}\}$ and outputs  ``1" if and only if the null hypothesis is rejected.
The two samples $\{V_1, \dots,V_{n_1}\}$ and $\{W_1, \dots, W_{n_2}\}$ are sampled i.i.d. from distributions with densities $f_1$ and $f_2$, respectively, with respect to an appropriate  probability measure $\mu$ on $[0, 1]^d$.
Further, it is assumed that $f_1 - f_2$ is $(s, L)$-H\"older continuous for a H\"older smoothness parameter $s > 0$ and a H\"older constant $L > 0$.
Given $\alpha \in (0, 1)$ and $\beta \in (0, 1 - \alpha)$, the
two-sample 
test is required to satisfy,
for some $c_\textnormal{ts} > 0$ depending on $(s, L, d, \alpha, \beta)$ and on $\nu_l, \nu_u$ from Assumption \ref{assumption:densitybound} to be introduced next,
\begin{align}
\label{eq:twosamplecontrol}
    &P(\textnormal{\texttt{TS}}_{\alpha, \beta}(V_1, \dots, V_{n_1}, W_1, \dots, W_{n_2}) = 1) \leq \alpha \quad \text{if } f_1 = f_2, \nonumber \\
    &P(\textnormal{\texttt{TS}}_{\alpha, \beta}(V_1, \dots, V_{n_1}, W_1, \dots, W_{n_2}) = 0) \leq \beta \quad \text{if } \left\Vert f_1 - f_2 \right\Vert_{L^2(\mu)} \geq c_\textnormal{ts}(n_1 \wedge n_2)^{-\frac{2s}{4s + d}}.
\end{align}
There are a number of such tests proposed in prior work, see e.g., \cite{ingster2012nonparametric,arias2018remember,kim2020minimax} and Appendix \ref{sec:twobackground}. Our general approach allows using any of these.
It is also known that there are \emph{adaptive} tests $\texttt{TS}^\text{ad}$
that do not require knowing the H\"older smoothness parameter $s$.
In the adaptive setting, the best-known minimum required separation in this general dimensional situation is
\begin{equation}
\label{eq:twoadaptivecontrol}
    \left\Vert f_1 - f_2 \right\Vert_{L^2(\mu)} \geq c_\text{ad} \left( \frac{n_1 \wedge n_2}{\log \log (n_1 \wedge n_2)} \right)^{-\frac{2s}{4s + d}}
\end{equation}
for some $c_\text{ad} > 0$.
See Appendix \ref{sec:twobackground} for examples of $\texttt{TS}$ and $\texttt{TS}^\text{ad}$.

\fussy
Next, we state an additional assumption required in our theorem.
See Appendix \ref{sec:twobackground} for more discussion.
Assumption \ref{assumption:classprob} guarantees that every class appears in the dataset. This is reasonable in many practical settings, as classes that do not appear can be omitted.
\begin{assumption}[Lower bounded class probability]
\label{assumption:classprob}
There exists a constant $d_c > 0$ such that $\E[Y]_k > d_c$ for all $k \in \{1,\ldots,K\}$.
\end{assumption}
Our result is as follows.

\begin{algorithm}[t]
\caption{Sample splitting calibration test $\xi_n^\text{split}$}\label{alg:reductiontest}
\begin{algorithmic}
\STATE \textbf{Input:} Probability predictor $f: \mathcal{X} \to \Delta_{K - 1}$; i.i.d. sample $\{(X_i, Y_i) \in \mathcal{X} \times \mathcal{Y} : i \in \{1, \dots, n\} \}$; false detection rate $\alpha \in (0, 1)$; true detection rate $\beta \in (0, 1 - \alpha)$; H\"older smoothness $s$; minimax optimal two-sample density test $\texttt{TS}$
\STATE \textbf{Procedure:} $Z_i \gets f(X_i)$ for $i \in \{1, \dots, n\}$; independently sample $\tilde{Y}_i \sim \text{Cat}(Z_i)$ for $i \in \{\lfloor \frac{n}{2} \rfloor + 1, \dots, n\}$
\FOR{$k = 1$ \TO $K$}
\STATE $T_{1, k} \gets \frac{1}{n} \sum_{i = 1}^n [Y_i - Z_i]_k$, $T_{2, k} \gets \frac{1}{n} \sum_{i = 1}^n [Z_i]_k [Y_i - Z_i]_k$
\STATE $\mathcal{V}_k \gets \left\{Z_i: [Y_i]_k = 1, 1 \leq i \leq \left\lfloor \frac{n}{2} \right\rfloor \right\}$, $\mathcal{W}_k \gets \left\{Z_i: [\tilde Y_i]_k = 1, \left\lfloor \frac{n}{2} \right\rfloor + 1 \leq i \leq n \right\}$
\STATE $b_k \gets I\Big(|T_{1, k}| \geq \sqrt{\frac{3K}{\alpha n}}\Big) \vee  I\Big(|T_{2, k}| \geq \sqrt{\frac{3K}{\alpha n}}\Big) \vee \texttt{TS}_{\frac{\alpha}{3K}, \frac{\beta}{2}} (\mathcal{V}_k, \mathcal{W}_k)$
\ENDFOR
\STATE \textbf{Output:} Reject $H_0$ if $\xi_n^\text{split} := \max\{b_k: k \in \{1,\ldots,K\}\} = 1$
\end{algorithmic}
\end{algorithm}

\sloppy
\begin{theorem}[Optimal calibration test via sample splitting]
\label{thm:reductiontest}
Suppose $p \leq 2$ and let $\xi_n^\textnormal{split}$ be the test described in Algorithm \ref{alg:reductiontest}.
Assume the H\"older smoothness parameter $s$ is known.
Under Assumption \ref{assumption:densitybound} and \ref{assumption:classprob}, we have 
\begin{enumerate}
    \item {\bf  False detection rate control.}
       For every $P$ for which $f$ is perfectly calibrated, i.e., for $P\in \mathcal{P}_0$, the probability of falsely claiming mis-calibration is at most $\alpha$, i.e., 
       $P(\xi_n^\textnormal{split} = 1) \leq \alpha$.

    \item {\bf True detection rate control.}
    There exists $c_\textnormal{split} > 0$ depending on $(s, L, K, \nu_l, \nu_u, d_c, \alpha, \beta)$ such that the power (true positive rate) is bounded as $P(\xi_n^\textnormal{split} = 1) \geq 1 - \beta$ for every $P \in \mathcal{P}_1(\ep, p, s)$---i.e., when $f$ is mis-calibrated with an $\ell_p\textnormal{-ECE}$ of at least  $\ep \geq c_\textnormal{split} n^{-2s/(4s + K - 1)}$.
\end{enumerate}
\end{theorem}
\fussy
The proof is in Appendix \ref{sec:proofreductiontest}.
Theorem \ref{thm:multi-lower-bound} and Theorem \ref{thm:reductiontest} together imply that the minimax optimal detection rate for calibration is  $\ep_n(p, s) \asymp n^{-2s / (4s + K - 1)}$.
By replacing $\texttt{TS}$ with an adaptive test $\texttt{TS}^\text{ad}$, we obtain an adaptive version of the test $\xi_n^\text{split}$.

\sloppy
\begin{corollary}[Adaptive test via sample splitting]
\label{cor:adaptivetwo}
Suppose $p \leq 2$ and let $\xi_n^\textnormal{ad-s}$ be the test described in Algorithm \ref{alg:reductiontest} with \textnormal{\texttt{TS}} replaced by an adaptive two-sample test $\textnormal{\texttt{TS}}^\textnormal{ad}$.
Under Assumption \ref{assumption:densitybound} and \ref{assumption:classprob}, we have 
\begin{enumerate}
    \item {\bf  False detection rate control.}  For every $P$ for which $f$ is perfectly calibrated, i.e., for $P\in \mathcal{P}_0$, the probability of falsely claiming mis-calibration is at most $\alpha$, i.e.,  $P(\xi_n^\textnormal{ad-s} = 1) \leq \alpha$.
    
    \item {\bf True detection rate control.} There exists $c_\textnormal{ad-s} > 0$ depending on $(s, L, K, \nu_l, \nu_u, d_c, \alpha, \beta)$ such that the power (true positive rate) is bounded as $P(\xi_n^\textnormal{ad-s} = 1) \geq 1 - \beta$ for every $P \in \mathcal{P}_1(\ep, p, s)$---i.e., when $f$ is mis-calibrated with an $\ell_p\textnormal{-ECE}$ of at least $\ep \geq c_\textnormal{ad-s} (n / \log\log n)^{-2s/(4s + K - 1)}$.
\end{enumerate}
\end{corollary}

\begin{remark}[Adaptation cost and optimality]
\label{rem:adapt}
\cite{spokoiny1996adaptive, ingster2000adaptive} develop an adaptive chi-squared test for one-dimensional goodness-of-fit testing which can adapt to an unknown H\"older smoothness parameter $s$ while only losing a $(\log \log n)^{s / (4s + 1)}$ factor in the separation rate.
The test was proven to be minimax optimal in the adaptive setting.
\cite{arias2018remember} extend the adaptive test to a general dimension $d$ and attain an adaptive test at the cost of $\sqrt{\log n}$ factor.
\cite{kim2020minimax} provide a stronger analysis for their permutation test and reduce the adaptation cost to $(\log \log n)^{2s / (4s + d)}$.
To our knowledge, the minimax optimality of these adaptive tests in general dimensions is so far not established.

While the adaptive test in Corollary \ref{cor:adaptivetwo} requires an additional factor of $(\log\log n)^{2s / (4s + K - 1)}$ in the separation rate, Theorem \ref{thm:adaptiveone} requires a factor of $(\log n)^{s / (4s + K - 1)}$.
This gap comes from the requirement \eqref{eq:twoadaptivecontrol} which we borrow from \cite{kim2020minimax}.
Theorem 6.1 and Lemma C.1 of \cite{kim2020minimax}
develops combinatorial concentration inequalities to improve a polynomial dependency on $\alpha$ in the separation rate to a logarithmic dependency.
This results in the $(\log\log n)^{2s / (4s + K - 1)}$ factor in their adaptive test.
Since our proof of Theorem \ref{thm:onesamplethm} uses a quadratic tail bound from Chebyshev's inequality, the adaptation cost in Theorem \ref{thm:adaptiveone} is $(\log n)^{s / (4s + K - 1)}$. 
Currently, it appears challenging to improve the polynomial dependence for $T_{m, n}^\textnormal{d}$, due to its complicated conditional structure.

\end{remark}
\fussy

\subsection{Comparison with the Debiased Plug-in Test}

\begin{figure}
    \centering
    \captionsetup[subfigure]{justification=centering}
    \begin{subfigure}{0.32\textwidth}
        \centering
        \input{images/plugin_vs_splitting100.pgf}
        \caption{$s = 0.6$, $\rho = 100$}
        \label{fig:plugvssplit100}
    \end{subfigure}
    \begin{subfigure}{0.32\textwidth}
        \centering
        \input{images/plugin_vs_splitting200.pgf}
        \caption{$s = 0.8$, $\rho = 200$}
        \label{fig:plugvssplit200}
    \end{subfigure}

    \caption{Type II error comparison for $\xi_{m_*, n}$ and $\xi_n^\text{split}$. The horizontal dashed line indicates a type II error of $1 - \alpha = 0.95$. Since $\xi_n^\text{split}$ relies on sampling splitting, its effective sample size is much smaller than that of the plug-in test. This results in higher type II errors as can be seen in the figure. Standard error bars are plotted over 10  repetitions.}
    \label{fig:type2error}
\end{figure}
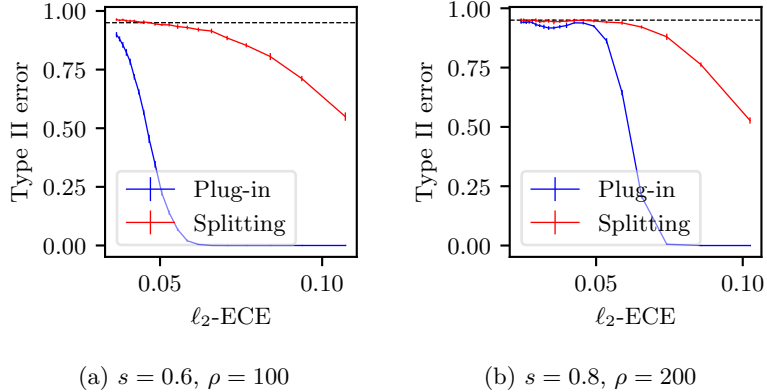

We compare the empirical performances of the debiased plug-in test $\xi_{m_*, n}$ and the sample splitting test $\xi_n^\text{split}$.
As described in Section \ref{sec:powanalysis}, we study the type II error against the fixed alternative where the mis-calibration is specified as $H_1: (Z, Y) \sim P_{1, m}$, for various values of $m$.
We use a sample size of $n = 20,000$ and pairs of H\"older smoothness and scaling parameter indicated in Figure \ref{fig:type2error}.
The critical value for $\alpha = 0.05$ and the corresponding type II error are estimated via 1,000 Monte Carlo simulations.
For the sample splitting test $\xi_n^\text{split}$, we use the chi-squared two-sample test of \cite{arias2018remember}.

Since $\xi_n^\text{split}$ relies on sample splitting and discards some of the observations, its effective sample size is smaller than that of the debiased plug-in test.
For this reason, we find that T-Cal outperforms the sample splitting test by a large margin.
While the sample splitting test reveals a theoretically interesting connection to two-sample density testing, it appears empirically suboptimal.

\section{Conclusion}
\sloppy
This paper studied the problem of testing model calibration from a finite sample.
We analyzed the plug-in estimator of $\ell_2\text{-ECE}(f)^2$ as a test statistic for calibration testing.
We discovered that the estimator needs debiasing and becomes minimax optimal when the number of bins is chosen appropriately.
We also provided an adaptive version of the test, which can be used without knowing the H\"older smoothness parameter $s$.
We tested T-Cal with a broad range of experiments, including several neural net architectures and post-hoc calibration methods.

On the theoretical side, we provided an impossibility result for testing calibration against general continuous alternatives.
Assuming that the calibration curve is $s$-H\"older-smooth, we derived a lower bound of $\Omega(n^{-2s / (4s+K-1)})$ on the calibration error required for a model to be distinguished from a perfectly calibrated one.
We also discussed a reduction to two-sample testing and showed that the resulting test also matches the lower bound.

Interesting future directions include (1) developing a testing framework for comparing calibration of predictive models, (2) extending the theoretical result to $\ell_p\text{-ECE}$ with $p > 2$ and other calibration concepts such as top(-$k$), within-$k$, and marginal calibration, (3) developing a minimax estimation theory for calibration error, and (4) establishing local rates for testing calibration.

\fussy


\section*{Acknowledgements}

This work was supported in part by the NSF TRIPODS 1934960, NSF DMS 2046874 (CAREER),  ARO W911NF-20-1-0080, DCIST, and NSF CAREER award CIF-1943064, and Air Force Office of Scientific Research Young
Investigator Program (AFOSR-YIP) \#FA9550-20-1-0111 award.
We thank the editors and associate editor for their work in handling our manuscript; and the reviewers for their very thorough reading and many helpful suggestions that have significantly improved our work.
We are grateful to a number of individuals for their comments, discussion, and feedback;
in particular, we would like to thank
Sivaraman Balakrishnan,
Chao Gao,
Chirag Gupta, 
Lucas Janson,
Adel Javanmard,
Edward Kennedy,
Ananya Kumar, 
Yuval Kluger,
Mohammed Mehrabi,
Alexander Podkopaev,
Yury Polyanskiy,
Mark Tygert, and
Larry Wasserman.

\section{Appendix}

\paragraph{Notations.}
For completeness and to help the reader, we present (and in some cases recall) some notations. 
We will use the symbols $:=$ or $=:$ to define quantities in equations.
We will occasionally use bold font for vectors.
For an integer $d \geq 1$, we denote $[d] := \{1,\ldots,d\}$ and $\mathbf{1}_d := (1,1,\ldots, 1)^\top \in \R^d$.
For a vector $\mathbf{v} \in \R^d$, we will sometimes write $[\mathbf{v}]_i$ for the $i$-th coordinate of $\mathbf{v}$, for any $i\in [d]$.
The minimum of two scalars $a,b \in \R$ is denoted by $\min(a,b)$ or $a \wedge b$; their maximum is denoted by $\max(a,b)$ or $a \vee b$.
We denote the $d$-dimensional Lebesgue measure by $\text{Leb}_d$.
For a function $h: \R^d \to \R$, $1 \leq p < \infty$, and a measure $\mu$ on $\R^d$, we let $\left\Vert h \right\Vert_{L^p(\mu)} := (\int |h|^p d \mu )^{1/p}$.
When $\mu = \text{Leb}_d$, we omit $\mu$ and write $\left\Vert h \right\Vert_{L^p}$.
If $p = \infty$, then $\left\Vert h \right\Vert_{L^p} := \text{ess sup}_{x \in \R^d} |h(x)|$.
We denote the $\ell_p$-norm of $\mathbf{x} = (x_1, \dots, x_d)^\top\in \R^d$ by $\left\Vert \mathbf{x} \right\Vert_p := ( \sum_{i = 1}^d |x_i|^p )^{1/p}$.
When $p$ is unspecified, $\Vert \cdot \Vert$ stands for $\Vert \cdot \Vert_2$.

For two sequences $(a_n)_{n\geq 1}$ and  $(b_n)_{n\geq 1}$ with $b_n \neq 0$,
we write $a_n \asymp b_n$ if $0 < \liminf_n a_n / b_n \leq \limsup_n a_n / b_n < \infty$.
When the index $n$ is self-evident, we may omit it above.
We use the Bachmann-Landau asymptotic notations $\Omega(\cdot), \Theta(\cdot)$ to hide constant factors in inequalities and use $\tilde{\Omega}(\cdot), \tilde{\Theta}(\cdot)$ to also hide logarithmic factors.
For a Lebesgue measurable set $A \subseteq \R^d$, we denote by $\mathds{1}_{A}:\R^d \to \{0,1\}$ its indicator function where $\mathds{1}_{A}(\mathbf{x}) = 1$ if $\mathbf{x} \in A$ and $\mathds{1}_{A}(\mathbf{x}) = 0$ otherwise.
For a real number $s \in \mathbb{R}$, we denote the largest integer less than or equal to $s$ by $\lfloor s \rfloor$.
Also, the smallest integer greater than or equal to $s$ is denoted by $\lceil s \rceil$.

For an integer $d \geq 1$, a vector $\boldsymbol \gamma = (\gamma_1, \dots, \gamma_{d})^\top \in \mathbb{N}^{d}$ is called a multi-index.
We write $|\boldsymbol \gamma| := \gamma_1 + \cdots + \gamma_d$.
For a vector $\mathbf{x} = (x_1, \dots, x_d) \in \R^d$ and a multi-index $\boldsymbol \gamma = (\gamma_1, \dots, \gamma_d)^\top \in \mathbb{N}^d$, we write $\mathbf{x}^{\boldsymbol \gamma} := x_1^{\gamma_1} \cdots x_d^{\gamma_d}$.
For a sufficiently smooth function $f:\R^d \to \R$, we denote its partial derivative of order $\boldsymbol \gamma = (\gamma_1, \dots, \gamma_{d})^\top$ by  $f^{(\boldsymbol \gamma)} := \partial^{\gamma_1}_{1} \cdots \partial^{\gamma_d}_{d} f$.
A pure partial derivative with respect to an individual coordinate $i \in [d]$ is also denoted as $\partial_i f$.
For two sets $S,T$, a map $f:S\to T$, and a subset $S' \subseteq S$, we denote by $f(S')$ the image of $S'$ under $f$.
The support of a function $f:S\to \mathbb{R}$ is the set of points $\supp(f) := \{a\in S: f(a) \neq 0\}$.

The uniform distribution on a compact set $S \subseteq \R^d$ is denoted by $\text{Unif}(S)$.
The binomial distribution with $n \in \mathbb{N}$ trials and success probability $p \in [0, 1]$ is denoted by $\text{Bin}(n, p)$, and write $\text{Ber}(p) := \text{Bin}(1, p)$.
For an integer $d \geq 2$, we let $\Delta_{d - 1} := \{\mathbf{z} = (z_1,\ldots, z_d)^\top \in [0, 1]^d : z_1 + \cdots + z_d = 1 \}$ be the $(d - 1)$-dimensional probability simplex.
We denote the multinomial distribution with $n \in \mathbb{N}$ trials and class probability vector $\mathbf{p} \in \Delta_{d - 1}$ by $\text{Multi}(n, \mathbf{p})$, and write $\text{Cat}(\mathbf{p}) := \text{Multi}(1, \mathbf{p})$.
For a joint distribution $(X,Y) \sim P$, we will write $P_X,P_Y$ for the marginal distributions of $X,Y$, respectively.
For a distribution $Q$ and a random variable $Z\sim Q$, we will denote expectations of functions of $Z$ with respect to $Q$ as $\E f(Z)$, $\E_Z f(Z)$, $\E_{Q} f(Z)$, or $\E_{Z \sim Q} f(Z)$.
We abbreviate almost surely by  ``a.s.", and almost everywhere by ``a.e."

\subsection{Proofs}

\subsubsection{Proof of Theorem \ref{thm:onesamplethm}}\label{sec:proofonesamplethm}

We first state a Lemma used in the proof.
This lemma generalizes Lemma 3 in \cite{arias2018remember} from the uniform measure on the cube to a general probability measure on the probability simplex.
See Section 3.2.2 of \cite{ingster2012nonparametric} for a discussion of how such results connect to geometric notions like Kolmogorov diameters.

\begin{lemma}
\label{discretel2multi}
    For $m \in \mathbb{N}_+$, let $\mathcal{B}_m = \{B_1, \dots, B_{m^{K - 1}}\}$ be the partition of $\Delta_{K - 1}$ defined in Appendix \ref{multibinning}, and $\mu$ be a probability measure on $\Delta_{K - 1}$ such that $\mu(B_i) > 0$ for all $i \in [m^{K - 1}]$.
    For any continuous function $h : \Delta_{K - 1} \rightarrow \mathbb{R}$, define
    $$W_m[h] := \sum_{i = 1}^{m^{K - 1}} \frac{\int_{B_i} h(\mathbf{z}) d\mu(\mathbf{z})}{\mu(B_i)} \mathds{1}_{B_i}.$$
    There are $b_1, b_2 > 0$ depending on $(K, s, L)$ such that for every $h \in \mathcal{H}_K(s, L)$,
    $$\left\Vert W_m[h] \right\Vert_{L^2(\mu)} \geq b_1 \left\Vert h \right\Vert_{L^2(\mu)} - b_2 m^{-s}.$$
    In other words,
    $$ \left(\sum_{i = 1}^{m^{K - 1}} \frac{\mathbb{E}_\mu[h(Z) I(Z \in B_i)]^2}{\mu(B_i)} \right)^\frac{1}{2} \geq b_1 \E_\mu[h(Z)^2]^\frac{1}{2} - b_2 m^{-s}.$$
\end{lemma}
The proof can be found in Appendix \ref{app:prooflemma}.

\paragraph{Overview of the proof.}
The proof follows the classical structure of upper bound arguments in nonparametric hypothesis testing, see e.g., \cite{arias2018remember,kim2020minimax} for recent examples.
We compute or bound the mean and variance of $T_{m_*, n}^\text{d}$ under null distributions $P_0 \in \mathcal{P}_0$ and alternative distributions $P_1 \in \mathcal{P}_{s, L, K}$ with a large ECE.
Using Lemma \ref{discretel2multi}, we can find a lower bound on $\mathbb{E}_{P_1}[T_{m_*, n}^\text{d}] - \mathbb{E}_{P_0}[T_{m_*, n}^\text{d}]$.
The variances $\V_{P_0}(T_{m_*, n}^\text{d})$ and $\V_{P_1}(T_{m_*, n}^\text{d})$ can be also upper bounded.
We argue that the mean difference $\mathbb{E}_{P_1}[T_{m_*, n}^\text{d}] - \mathbb{E}_{P_0}[T_{m_*, n}^\text{d}]$ is significantly larger than the square root of the variances $\V_{P_0}(T_{m_*, n}^\text{d})$ and $\V_{P_1}(T_{m_*, n}^\text{d})$.
The conclusion follows from Chebyshev's inequality.

\begin{proof}
Let $N_i := |\mathcal{I}_{m_*, i}|$ for each $i \in [m_*^{K - 1}]$ and $\mathcal{I} := \{i \in [m_*^{K - 1}]: N_i \geq 1\}$.
Also write $\mathbf{N} := (N_1, \dots, N_{m_*^{K - 1}})^\top$, $\mathbf{Z} := (Z_1, \dots, Z_n)^\top$, $\barY := Y - Z$, and $\barY_j := Y_j - Z_j$ for all $j\in [n]$.
By Assumption \ref{assumption:densitybound},
$$\int_{\Delta_{K - 1}} \Vert \res(\mathbf{z}) \Vert^2 dP_Z(\mathbf{z}) \leq  \nu_u \int_{\Delta_{K - 1}} \Vert \res(\mathbf{z}) \Vert^2 d\mathbf{z},$$
where the latter integral is with respect to the uniform measure $\text{Unif}(\Delta_{K - 1})$ on $\Delta_{K - 1}$.
Therefore, we may assume $P_Z = \text{Unif}(\Delta_{K - 1})$ by merging $\nu_u$ with $c$.
We prove the theorem for $p = 2$.
Then, the general case follows since $\mathcal{P}_1(\ep, p, s) \subseteq \mathcal{P}_1(\ep, 2, s)$ for all $p \leq 2$.

Let $P_0 \in \mathcal{P}_0$ and $P_1 \in \mathcal{P}_{s, L, K}$ be a null and an alternative distribution over $(Z, Y)$, respectively.
Write $\ep := \ell_p\text{-ECE}_{P_1}(f)$.
Under $P_0$ and conditioned on $\mathbf{Z}$, recalling $T_{m, n}^\textnormal{d}$ from \eqref{eq:onesampledef},
\begin{align*}
    &\E_{P_0}[T_{m_*, n}^\textnormal{d} \mid \mathbf{Z}] = \frac{1}{n} \sum_{i \in \mathcal{I}} \frac{1}{N_i} \sum_{j_1 \neq j_2 \in \mathcal{I}_{m_*, i}} \E_{P_0}\left[ \barY_{j_1}^\top \barY_{j_2} \mid \mathbf{Z} \right] = 0
\end{align*}
because $\E_{P_0}[ \barY_{j_1}^\top \barY_{j_2} \mid \mathbf{Z}] = \E_{P_0}[\barY_{j_1} \mid \mathbf{Z}]^\top \E_{P_0}[\barY_{j_2} \mid \mathbf{Z}] = 0$ for all $j_1 \neq j_2 \in \mathcal{I}_{m_*, i}$.
Therefore,
\begin{equation}
\label{eq:nullmeanzero}
\E_{P_0}[T_{m_*, n}^\textnormal{d}] = \E_{P_0}[\E_{P_0}[T_{m_*, n}^\textnormal{d} \mid \mathbf{Z}]] = 0.
\end{equation}
Also,
\begin{align*}
    \text{Var}_{P_0}(T_{m_*, n}^\textnormal{d} \mid \mathbf{Z}) &= \frac{1}{n^2} \sum_{i \in \mathcal{I}} \frac{1}{N_i^2} \text{Var}_{P_0}\left( 2\sum_{j_1 < j_2 \in \mathcal{I}_{m_*, i}} \barY_{j_1}^\top \barY_{j_2} \mid \mathbf{Z} \right)\\
    &= \frac{1}{n^2} \sum_{i \in \mathcal{I}}  \frac{4}{N_i^2} \sum_{j_1 < j_2 \in \mathcal{I}_{m_*, i}} \text{Var}_{P_0}\left( \barY_{j_1}^\top \barY_{j_2} \mid \mathbf{Z} \right).
\end{align*}
Here we used that the cross terms in the expansion of $\text{Var}_{P_0}( \sum_{j_1 < j_2 \in \mathcal{I}_{m_*, i}} \barY_{j_1} \barY_{j_2} \mid \mathbf{Z})$ vanish since for $j_1,j_2,j_3 \in \mathcal{I}_{m_*, i}$ that are pairwise different,
\begin{align*}
    \text{Cov}_{P_0}\left(\barY_{j_1}^\top \barY_{j_2}, \barY_{j_1}^\top \barY_{j_3} \mid \mathbf{Z}\right)
    &= \E_{P_0}\left[\barY_{j_2}^\top \barY_{j_1}\barY_{j_1}^\top \barY_{j_3} \mid \mathbf{Z} \right] - \E_{P_0}\left[ \barY_{j_1}^\top \barY_{j_2}  \mid \mathbf{Z}\right] \E_{P_0}\left[ \barY_{j_1}^\top \barY_{j_3}  \mid \mathbf{Z}\right]\\
    &= \E_{P_0}\left[\barY_{j_2} \mid \mathbf{Z}\right]^\top \E_{P_0}\left[\barY_{j_1} \barY_{j_1}^\top \mid \mathbf{Z}\right] \E_{P_0}\left[\barY_{j_3} \mid \mathbf{Z}\right] = 0,
\end{align*}
and for $j_1, j_2, j_3, j_4 \in \mathcal{I}_{m_*, i}$ that are pairwise different,
$\text{Cov}_{P_0}\left( \barY_{j_1}^\top \barY_{j_2}, \barY_{j_3}^\top \barY_{j_4} \mid \mathbf{Z} \right) = 0$
by the independence of $\barY_{j_1}, \barY_{j_2}, \barY_{j_3}, \barY_{j_4}$ given $\mathbf{Z}$.
Further, since $\text{Var}_{P_0}(\barY_{j_1}^\top \barY_{j_2} \mid \mathbf{Z}) \leq K^2$ for all $j_1 < j_2 \in \mathcal{I}_{m_*, i}$,
\begin{equation*}
    \text{Var}_{P_0}(T_{m_*,n}^\text{d} \mid \mathbf{Z}) \leq \frac{1}{n^2} \sum_{i \in \mathcal{I}} \frac{2K^2 N_i(N_i - 1)}{N_i^2} \leq 2K^2 n^{-2} \sum_{i = 1}^{m_*^{K - 1}} I(N_i \geq 2).
\end{equation*}
Thus, by the law of total variance,
\begin{align*}
\text{Var}_{P_0}(T_{m_*, n}^\textnormal{d}) &= \E_{P_0} [\text{Var}_{P_0}(T_{m_*, n}^\textnormal{d} \mid \mathbf{Z})] + \text{Var}_{P_0}(\E_{P_0}[T_{m_*, n}^\textnormal{d} \mid \mathbf{Z}]) \leq 2 K^2 n^{-2} \sum_{i = 1}^{m_*^{K - 1}} P_0(N_i \geq 2)\\
&= 2K^2 m_*^{K - 1} n^{-2} P_0(N_1 \geq 2).
\end{align*}
Since $N_1$ follows a Binomial distribution with $n$ trials and success probability $m_*^{-K + 1}$, and as
$(1 - x)^{n - 1} \geq 1 - (n - 1) x$ for any $x \in [0, 1]$, we see that
\begin{align}
\label{eq:n1upper}
   P_0(N_1 \geq 2) =  1- \left( 1 - \frac{1}{m_*^{K - 1}} \right)^{n-1}\left( 1 + \frac{n-1}{m_*^{K - 1}} \right) \leq 1 \wedge \frac{n^2}{m_*^{2(K - 1)}}.
\end{align}
Therefore, defining $\tau^2$ below,
\begin{align}
\label{eq:nullvarbound}
\text{Var}_{P_0}(T_{m_*, n}^\text{d}) \leq 2K^2\left( m_*^{K - 1} n^{-2} \wedge m_*^{-(K - 1)} \right) =: \tau^2.
\end{align}

Under $P_1$, we have
\begin{align*} 
    \mathbb{E}_{P_1}[T_{m_*, n}^\textnormal{d} \mid \mathbf{Z}] &= \frac{1}{n} \sum_{i \in \mathcal{I}} \frac{1}{N_i} \sum_{j_1 \neq j_2 \in \mathcal{I}_{m_*, i}} \E_{P_1}\left[ \barY_{j_1}^\top \barY_{j_2} \mid \mathbf{Z} \right]\\
    &= \frac{1}{n} \sum_{i \in \mathcal{I}} \frac{1}{N_i} \sum_{j_1 \neq j_2 \in \mathcal{I}_{m_*, i}} \res(Z_{j_1})^\top \res(Z_{j_2}),
\end{align*}
since for each $i \in \mathcal{I}$, $\E_{P_1}[\barY_{j_1}^\top \barY_{j_2} \mid \mathbf{Z}] = \E_{P_1}[\barY_{j_1} \mid \mathbf{Z}]^\top \E_{P_1}[\barY_{j_2} \mid \mathbf{Z}] = \res(Z_{j_1})^\top \res(Z_{j_2})$ for all $j_1 \neq j_2 \in \mathcal{I}_{m_*, i}$.
Moreover,
\begin{align}
    \E_{P_1}[T_{m_*, n}^\textnormal{d} \mid \mathbf{N}] &= \E_{P_1}[\E_{P_1}[T_{m_*, n}^\text{d} \mid \mathbf{Z}] \mid \mathbf{N}]\nonumber\\
    &= \frac{1}{n} \sum_{i \in \mathcal{I}} \frac{N_i(N_i - 1)}{N_i} \mathbb{E}_{P_1}[\res(Z) \mid Z \in B_i]^\top \mathbb{E}_{P_1}[\res(Z) \mid Z \in B_i] \nonumber \\
    &=\frac{1}{n} \sum_{i \in \mathcal{I}} (N_i - 1) \left\Vert \E_{P_1} \left[ \res(Z) \mid Z \in B_i \right] \right\Vert^2\label{etd}
\end{align}
and
\begin{align*}
    \E_{P_1}[T_{m_*,n}^\text{d}] &= \E_{P_1}[\E_{P_1}[T_{m_*, n}^\text{d} \mid \mathbf{N}]]\\
    &= \frac{1}{n} \sum_{i = 1}^{m_*^{K - 1}}  \E_{P_1}[I(N_i \geq 1)(N_i - 1)] \left\Vert \E_{P_1}[\res(Z) \mid Z \in B_i] \right\Vert^2.
\end{align*}
For any $x \in [0, 1]$, we have $1 - nx + \binom{n}{2} x^2 - \binom{n}{3} x^3 \leq (1 - x)^n \leq 1 - nx + \binom{n}{2} x^2.$
Applying the inequality for $x = m_*^{-(K - 1)}$, we derive
\begin{align*}
   &\frac{1}{4} \left( \frac{n}{m_*^{K - 1}} \wedge \frac{n^2}{m_*^{2(K - 1)}} \right) \leq \binom{n}{2} \frac{1}{m_*^{2(K - 1)}} - \binom{n}{3} \frac{1}{m_*^{3(K - 1)}}\\
   &\leq  \E_{P_1}[I(N_i \geq 1)(N_i - 1)] = \frac{n}{m_*^{K - 1}} - 1 + \left(1 - \frac{1}{m_*^{K - 1}} \right)^{n} \leq \binom{n}{2} \frac{1}{m_*^{2(K - 1)}} \leq \frac{n^2}{m_*^{2(K - 1)}}.
\end{align*}
Also, since $(1 - x)^n \leq 1$ for $x \in [0, 1]$, we have
\begin{align*}
    \E_{P_1} [I(N_i \geq 1)(N_i - 1)] \leq \frac{n}{m_*^{K - 1}}.
\end{align*}
Therefore,
\begin{align}
\label{eq:lowandupper}
\frac{1}{4} \left(\frac{n}{m_*^{K - 1}} \wedge \frac{n^2}{m_*^{2(K - 1)}} \right) &\leq \E_{P_1}[I(N_i \geq 1)(N_i - 1)] \leq 
\left(\frac{n}{m_*^{K - 1}} \wedge \frac{n^2}{m_*^{2(K - 1)}} \right). 
\end{align}

By Lemma \ref{discretel2multi}, and as $\ep = (\sum_{k = 1}^K \E_{P_1}[[\res(Z)]_k^2])^{1/2} \geq K^{-1/2} \sum_{k = 1}^K (\E_{P_1}[[\res(Z)]_k^2])^{1/2}$ by the Cauchy-Schwarz inequality,
\begin{align}
\label{eq:disclower}
    & \sum_{i = 1}^{m_*^{K - 1}}  m_*^{-(K - 1)}\left\Vert \E_{P_1}[\res(Z) I(Z \in B_i)] \right\Vert^2 
    = \sum_{i = 1}^{m_*^{K - 1}} \sum_{k = 1}^K m_*^{-(K - 1)} \E_{P_1}\left[[\res(Z)]_k^2 I(Z \in B_i)\right]\nonumber\\
    &= \sum_{k = 1}^K W_{m_*}^2[[\res(Z)]_k]
    \geq \sum_{k = 1}^K  \left(b_1 \E_{P_1}\left[[\res(Z)]_k^2\right]^\frac{1}{2} - b_2 m_*^{-s}\right)^2  \nonumber \\
    &\geq \sum_{k = 1}^K  \left( b_1^2 \E_{P_1}\left[ [\res(Z)]_k^2 \right] - 2b_1b_2 \E_{P_1}\left[ [\res(Z)]_k^2 \right]^\frac{1}{2} m_*^{-s} \right) 
    \geq b_1^2 \ep^2 - 2\sqrt{K }b_1b_2 m_*^{-s} \ep.
\end{align}
By \eqref{eq:lowandupper} and \eqref{eq:disclower}, defining $\delta$ below,
\begin{equation}
\label{mu}
    \E_{P_1}[T_{m_*, n}^\textnormal{d}] \geq \frac{1}{4}(b_1^2\ep^2 - 2\sqrt{K}b_1b_2 m_*^{-s} \ep) \left(1 \wedge \frac{n}{m_*^{K - 1}} \right)=: \delta. 
\end{equation}

Moreover, we find
\begin{align*}
    &\text{Var}_{P_1}(T_{m_*, n}^\textnormal{d} \mid \mathbf{N}) = \frac{1}{n^2} \sum_{i \in \mathcal{I}} \frac{1}{N_i^2} \text{Var}_{P_1}\left( 2\sum_{j_1 < j_2 \in \mathcal{I}_{m_*, i}} \barY_{j_1}^\top \barY_{j_2} \mid \mathbf{N} \right)\\
    &= \frac{1}{n^2} \sum_{i \in \mathcal{I}} \frac{1}{N_i^2} \left[2N_i(N_i - 1) \text{Var}_{P_1}\left( \barY_{1}^\top \barY_{2} \mid Z_1, Z_2 \in B_i \right) \right.\\
    &\hspace{60pt} \left. + 4N_i(N_i - 1)(N_i - 2) \text{Cov}_{P_1}\left( \barY_{1}^\top \barY_{2}, \barY_{1}^\top \barY_{3}\mid Z_1, Z_2, Z_3 \in B_i \right) \right]\\
    &\leq \frac{1}{n^2} \sum_{i \in \mathcal{I}} \left( \frac{2K^2N_i(N_i - 1)}{N_i^2} + \frac{4K^2N_i(N_i - 1)(N_i - 2)}{N_i^2}  \Vert \E_{P_1}[\res(Z) \mid Z \in B_i] \Vert^2 \right).
\end{align*}
Further, by equations \eqref{eq:n1upper}, \eqref{eq:lowandupper} and since $\Vert \mathbb{E}_{P_1}[\res(Z) \mid Z \in B_i]\Vert^2 \leq \mathbb{E}_{P_1}[\Vert \res(Z) \Vert^2 \mid Z \in B_i]$,
\begin{align}
\label{eq:evbound}
    &\E_{P_1}[\text{Var}_{P_1}(T_{m_*, n}^\text{d} \mid \mathbf{N})]\nonumber\\
    &\leq \frac{1}{n^2} \sum_{i = 1}^{m_*^{K - 1}} \left( 2K^2 P_1(N_i \geq 2) + 4K^2 \E_{P_1}[(N_i - 1) I(N_i \geq 1)] \Vert \E_{P_1}[\res(Z) \mid Z \in B_i] \Vert^2 \right) \nonumber\\
    &\leq \tau^2 + 4K^2 m_*^{K - 1}n^{-2}  \E_{P_1}[(N_1 - 1) I(N_i \geq 1)] \sum_{i = 1}^{m_*^{K - 1}} \E_{P_1}[\Vert \res(Z) \Vert^2 I(Z \in B_i)] \nonumber\\
    &\leq \tau^2 + 4K^2 \ep^2 (n^{-1} \wedge m_*^{-(K - 1)}).
\end{align}
Also, from \eqref{etd},
\begin{align*}
    \text{Var}_{P_1}(\E_{P_1}[T_{m_*,n}^\text{d} \mid \mathbf{N}]) &\leq \frac{1}{n^2} \left(\sum_{i = 1}^{m_*^{K - 1}}  \sqrt{\text{Var}_{P_1} [I(N_i \geq 1)(N_i - 1)]} \left\Vert \E_{P_1} \left[ \res(Z) \mid Z \in B_i \right] \right\Vert^2 \right)^2\\
    &= \frac{1}{n^2} \text{Var}_{P_1} [I(N_1 \geq 1)(N_1 - 1)] m_*^{2(K - 1)} \ep^4.
\end{align*}
Writing $x = m_*^{-(K - 1)}$, we have
\begin{align*}
    \text{Var}_{P_1} [I(N_1 \geq 1)(N_1 - 1)] =  nx (1 - x) + (1 - x)^n - (1 - x)^{2n} - 2n x (1 - x)^n.
\end{align*}
Using that $1 - nx \leq (1 - x)^n$, we find
\begin{align*}
    \text{Var}_{P_1} [I(N_i \geq 1)(N_i - 1)] & = nx(1 - x) + (1 - x)^n [1 - (1 - x)^n - 2nx]\\
    &\leq nx(1 - x) \leq nx = \frac{n}{m_*^{K - 1}}.
\end{align*}
Similarly,
\begin{align*}
    \text{Var}_{P_1} [I(N_i \geq 1)(N_i - 1)] &\leq nx + (1 - x)^{n - 1}[1 - (1 - x)^n - 2nx]\\
    &\leq nx + (1 - (n - 1)x)(- nx) \leq n^2 x^2 = \frac{n^2}{m_*^{2(K - 1)}}.
\end{align*}

Therefore,
\begin{align}
\label{eq:vebound}
    \text{Var}_{P_1}(\E_{P_1}[T_{m_*,n}^\text{d} \mid \mathbf{N}]) \leq \left( \frac{m_*^{K - 1}}{n} \wedge 1\right) \ep^4.
\end{align}
By equations \eqref{eq:evbound}, \eqref{eq:vebound}, and the law of  total variance, defining $\sigma^2$ below,
\begin{align}
\label{sigma}
 \text{Var}_{P_1}(T_{m_*, n}^\text{d}) &= \text{Var}_{P_1}(\E_{P_1}[T_{m_*, n}^\text{d} \mid \mathbf{N}]) + \E_{P_1}[\text{Var}_{P_1}(T_{m_*, n}^\text{d} \mid \mathbf{N})] \nonumber\\
 &\leq \tau^2 + 4K^2 \ep^2 (n^{-1} \wedge m_*^{-(K - 1)}) + \left( \frac{m_*^{K - 1}}{n} \wedge 1\right) \ep^4 =: \sigma^2.
\end{align}
Recalling that $m_* = \lfloor n^{2 / (4s + K - 1)} \rfloor$,
we choose $c > 0$ such that $\ep \geq c n^{-2s/(4s + K - 1)}$ implies
\begin{align*}
    &\left(\sqrt{\frac{2}{\alpha}} + \sqrt{\frac{2}{\beta}} \right) K (m_*^\frac{K - 1}{2} n^{-1} \wedge m_*^{-\frac{K - 1}{2}}) + \sqrt{\frac{4}{\beta}} K \ep \left(n^{-\frac{1}{2}} \wedge m_*^{-\frac{K - 1}{2}} \right) + \sqrt{\frac{1}{\beta}} \left( \frac{m_*^{K - 1}}{n} \wedge 1\right)^\frac{1}{2} \ep^2 \\&\leq \frac{1}{4}( b_1^2\ep^2 - 2\sqrt{K}b_1b_2 m_*^{-s} \ep) \left(1 \wedge \frac{n}{m_*^{K - 1}} \right)
\end{align*}
for all large enough $n$.
For $\tau$, $\delta$, and $\sigma$ from \eqref{eq:nullvarbound}, \eqref{mu}, and \eqref{sigma}, this gives
\begin{equation}
\label{multiconstraint}
\frac{\tau}{\sqrt{\alpha}} + \frac{\sigma}{\sqrt{\beta}} \leq \delta.
\end{equation}
By equations \eqref{eq:nullmeanzero}, \eqref{eq:nullvarbound}, 
the definition of $\xi_{m_*, n}$ from Algorithm \eqref{alg:onesampletest},
and Chebyshev's inequality,
\begin{align}
\label{eq:cheby1}
    P_0(\xi_{m_*, n} = 1) &\leq P_0\left(T_{m_*, n}^\textnormal{d} \geq \frac{\tau}{\sqrt{\alpha}} \right) \leq \frac{\text{Var}_{P_0}(T_{m_*, n}^\textnormal{d})}{\tau^2 / \alpha} \leq \alpha.
\end{align}
By equations \eqref{mu}, \eqref{sigma}, \eqref{multiconstraint}, and Chebyshev's inequality,
\begin{align}
\label{eq:cheby2}
    &P_1\left(T_{m_*, n}^\textnormal{d} < \frac{\tau}{\sqrt{\alpha}} \right) \leq P_1 \left( T_{m_*, n}^\textnormal{d} - \E_{P_1}[T_{m_*, n}^\textnormal{d}] \leq \frac{\tau}{\sqrt{\alpha}} - \delta  \right) \nonumber \\
    & \leq P_1\left(\left|T_{m_*, n}^\textnormal{d} - \E_{P_1}[T_{m_*, n}^\textnormal{d}] \right| \geq \frac{\sigma}{\sqrt{\beta}} \right)
    \leq \frac{\text{Var}_{P_1}(T_{m_*, n}^\textnormal{d})}{\sigma^2 / \beta}
    \leq \beta.
\end{align}

By the above arguments, Theorem \ref{thm:onesamplethm} holds for all $n \geq N$, where $N \in \mathbb{N}_+$ depends on $(s, L, K, \nu_l, \nu_u, \alpha, \beta).$
If we require $c$ to further satisfy $c \geq N^{2s / (4s + K - 1)}$,
then the family $\mathcal{P}_1(\ep, p, s)$ is empty for $n < N$ given $\ep \geq c n^{-2s/(4s + K - 1)} > 1$.
Therefore, Theorem \ref{thm:onesamplethm} becomes vacuously true for $n < N$, and thereby true for all $n \in \mathbb{N}$.
This finishes the proof.

\end{proof}

\subsubsection{Proof of Lemma \ref{discretel2multi}}
\label{app:prooflemma}

We state and prove Lemma \ref{lem:taylor} and \ref{lem:poly} which we use in the proof of Lemma \ref{discretel2multi}.

\begin{lemma}
\label{lem:taylor}
    Fix $h \in \mathcal{H}_K(s, L)$ and $\mathbf{z}_0 \in \Delta_{K - 1}$.
    Let $u$ be the $(\lceil s \rceil - 1)$-th order Taylor series of $h$ at $\mathbf{z}_0$.
    There is $L'$ depending on $(K, s, L)$ such that
    \beq\label{taylor}
    |h(\mathbf{z}) - u(\mathbf{z})| \leq L'\left\Vert \mathbf{z} - \mathbf{z}_0 \right\Vert^s
    \eeq
    for all $\mathbf{z} \in \Delta_{K - 1}.$
\end{lemma}
        
\begin{proof}
Let $\psi = \pi_{-K}$ as in \eqref{eq:holderdef}.
Recall that for a multi-index $\boldsymbol \gamma = (\gamma_1, \dots, \gamma_{K-1})^\top \in \mathbb{N}^{K-1}$, we write $(\psi(\mathbf{z}) - \psi(\mathbf{z}_0))^{\boldsymbol \gamma}= \prod_{j\in[K-1]} (\psi_j(\mathbf{z}) - \psi_j(\mathbf{z}_0))^{\gamma_j}$.
By a Taylor series expansion, there exists $t \in [0, 1]$ such that
\begin{align*}
    h(\mathbf{z}) = (h \circ \psi^{-1}) ( \psi(\mathbf{z})) &= \sum_{\substack{\boldsymbol \gamma \in \mathbb{N}^{K - 1}\\0 \leq |\boldsymbol \gamma| \leq \lceil s \rceil - 2}} \frac{(h \circ \psi^{-1})^{(\boldsymbol \gamma)}(\psi(\mathbf{z}_0))}{|\boldsymbol \gamma|!} (\psi(\mathbf{z}) - \psi(\mathbf{z}_0))^{\boldsymbol \gamma}\\
    &+ \sum_{\substack{\boldsymbol \gamma \in \mathbb{N}^{K - 1}\\ |\boldsymbol \gamma| = \lceil s \rceil - 1}} \frac{(h \circ \psi^{-1})^{(\boldsymbol \gamma)}(t\psi(\mathbf{z}) + (1 - t) \psi(\mathbf{z}_0))}{|\boldsymbol \gamma|!} (\psi(\mathbf{z}) - \psi(\mathbf{z}_0))^{\boldsymbol \gamma}.
\end{align*}
Then, $h(\mathbf{z}) - u(\mathbf{z})$ equals
\begin{align*}
    &      \sum_{\substack{\boldsymbol \gamma \in \mathbb{N}^{K - 1}\\ |\boldsymbol \gamma| = \lceil s \rceil - 1}} \frac{(h \circ \psi^{-1})^{(\boldsymbol \gamma)}(t\psi(\mathbf{z}) + (1 - t) \psi(\mathbf{z}_0)) - (h \circ \psi^{-1})^{(\boldsymbol \gamma)}(\psi(\mathbf{z}_0))}{|\boldsymbol \gamma|!} (\psi(\mathbf{z}) - \psi(\mathbf{z}_0))^{\boldsymbol \gamma}.
\end{align*}
By the triangle inequality and the $s$-H\"older continuity of $h$,
\begin{align*}
    &|h(\mathbf{z}) - u(\mathbf{z})| \\
    &\leq \sum_{\substack{\boldsymbol \gamma \in \mathbb{N}^{K - 1}\\ |\boldsymbol \gamma| = \lceil s \rceil - 1}} \frac{|(h \circ \psi^{-1})^{(\boldsymbol \gamma)}(t\psi(\mathbf{z}) + (1 - t) \psi(\mathbf{z}_0)) - (h \circ \psi^{-1})^{(\boldsymbol \gamma)}(\psi(\mathbf{z}_0))|}{|\boldsymbol \gamma|!} \left\vert(\psi(\mathbf{z}) - \psi(\mathbf{z}_0))^{\boldsymbol \gamma} \right\vert\\
    &\leq \sum_{\substack{\boldsymbol \gamma \in \mathbb{N}^{K - 1}\\ |\boldsymbol \gamma| = \lceil s \rceil - 1}} \frac{L (t\Vert \psi(\mathbf{z}) - \psi(\mathbf{z}_0) \Vert)^{s - \lceil s \rceil + 1}}{|\boldsymbol \gamma|!} \Vert \psi(\mathbf{z}) - \psi(\mathbf{z}_0) \Vert^{\lceil s \rceil - 1}\\
    &\leq L'\Vert \psi(\mathbf{z}) - \psi(\mathbf{z}_0) \Vert^s \leq L' \Vert \mathbf{z} - \mathbf{z}_0 \Vert^s.
\end{align*}
\end{proof}
        
\begin{lemma}
\label{lem:poly}
Let $\mathcal{P}_q^K$ be the class of polynomials on $\mathbb{R}^K$ of degree at most $q$.
There are $a_1, a_2 > 0$ depending on $(K, q)$ such that
\beq\label{w-lb}
\left\Vert W_m[v] \right\Vert_{L^2(\mu)} \geq a_1 \left\Vert v \right\Vert_{L^2(\mu)}
\eeq
for every $v \in \mathcal{P}_q^K$ and $m \geq a_2$.
\end{lemma}

\sloppy
\begin{proof}
    If \eqref{w-lb} does not hold, then we can find a sequence $\{ m_l \}_{l = 1}^\infty$ increasing to infinity and a sequence of polynomials $\{v_l\}_{l = 1}^\infty \subseteq \mathcal{P}_q^K$ such that $\Vert W_{m_l}[v_l] \Vert_{L^2(\mu)} < \frac{1}{l} \left\Vert v_l \right\Vert_{L^2(\mu)}.$
    Dividing $v_l$ by $\left\Vert v_l \right\Vert_{L^2(\mu)}$, we may assume $\left\Vert v_l \right\Vert_{L^2(\mu)} = 1.$
    Now $\{v \in \mathcal{P}_q^K: \left\Vert v \right\Vert_{L^2(\mu)} = 1 \}$ is compact in the topology induced by the norm $\left\Vert \cdot \right\Vert_{L^2(\mu)} $, due to the Heine-Borel theorem because it is closed and bounded. Thus, we can find a convergent subsequence $\{v_{l_k}\}_{k = 1}^\infty$.
    Denote the limit by $v_\infty$.
    On one hand,
    \begin{align*}
        \left\Vert W_{m_{l_k}} [v_\infty] \right\Vert_{L^2(\mu)} &\leq \left\Vert W_{m_{l_k}} [v_\infty - v_{l_k}] \right\Vert_{L^2(\mu)} + \left\Vert W_{m_{l_k}} [v_{l_k}] \right\Vert_{L^2(\mu)} \leq \left\Vert v_\infty - v_l \right\Vert_{L^2(\mu)} + \frac{1}{l_k} \rightarrow 0.
    \end{align*}
    Here, we used that
    \begin{align}
    \label{contraction}
    \left\Vert W_m[v] \right\Vert_{L^2(\mu)}^2 &= \sum_{i = 1}^{m^{K - 1}} \frac{\left( \int_{B_i} v(\mathbf{z}) d\mu(\mathbf{z}) \right)^2}{\mu(B_i)}
    \leq \sum_{i = 1}^{m^{K - 1}} \int_{B_i} v(\mathbf{z})^2 d\mu(\mathbf{z})
    = \left\Vert v \right\Vert_{L^2(\mu)}^2.
    \end{align}
    On the other hand, $\Vert W_{m_{l_k}}[v_\infty] \Vert_{L^2(\mu)} \rightarrow \Vert v_\infty \Vert_{L^2(\mu)} = 1$
    since $W_{m_{l_k}}[v_\infty] \to v_\infty$ a.e. $\mu$ and $\Vert W_{m_{l_k}}[v_\infty] \Vert_{L^\infty} \leq \Vert v_\infty \Vert_{L^\infty}.$
    The conclusion follows due to the contradiction.
    
\end{proof}
\fussy

     Now we proceed with the proof of Lemma \ref{discretel2multi}.
    Since $a_1$ does not depend on $m$, for sufficiently large $m$,
    we can choose $r$ such that $m / r$ is an integer and $m \geq a_2$.
    Partition $\Delta_{K - 1}$ into $\mathcal{M} := (m / r)^{K - 1}$ simplices as described in Appendix \ref{multibinning} and call them $\tilde{B}_1, \dots, \tilde{B}_{\mathcal{M}}$.
    By this construction, we can ensure that each $\tilde{B}_j$, $j\in [\mathcal{M}]$, consists of $r^{K - 1}$ different simplices $B_i$ (also defined in  Appendix \ref{multibinning}).
    For $j\in [\mathcal{M}]$, let $u_j$ be the $(\lceil s \rceil - 1)$-th order Taylor expansion of $h$ at an arbitrary vertex of $\tilde B_j$.
    Define $u := \sum_{j = 1}^{\mathcal{M}} u_j \mathds{1}_{\tilde{B}_j}$.
    By equation \eqref{taylor},
    \beq\label{hubound}
    |h(\mathbf{z}) - u(\mathbf{z})| \leq L' \operatorname{diam}(\tilde B_1)^s = L' \operatorname{diam}(\Delta_{K - 1})^s \left ( \frac{r}{m} \right)^s  =: b m^{-s}\eeq
    for all $\mathbf{z} \in \Delta_{K - 1}$.
    Therefore, by 
    \eqref{contraction} and
    Lemma \ref{lem:taylor},
    \begin{align*}
        &\left\Vert W_m[h] \right\Vert_{L^2(\mu)} \geq \left\Vert W_m[u] \right\Vert_{L^2(\mu)} - \left\Vert W_m[u - h] \right\Vert_{L^2(\mu)}\\
        &\geq \left\Vert W_m[u] \right\Vert_{L^2(\mu)} - \left\Vert u - h \right\Vert_{L^2(\mu)}
        \geq \left\Vert W_m[u] \right\Vert_{L^2(\mu)} - b m^{-s}.
    \end{align*}
    Note that
    $$\left\Vert W_m[u] \right\Vert_{L^2(\mu)}^2 = \sum_{j = 1}^{\mathcal{M}} \left\Vert W_m[u_j \mathds{1}_{\tilde{B}_j}] \right\Vert_{L^2(\mu)}^2,$$
    and that Lemma \ref{lem:poly} with $q=\lceil s \rceil - 1$
    can be applied to $\tilde B_j$ and its $r^{K - 1}$ sub-simplices to get
    \begin{align*}
        \left\Vert W_m[u_j \mathds{1}_{\tilde{B}_j}] \right\Vert_{L^2(\mu)}^2 \geq a_1^2 \left\Vert u_j \mathds{1}_{\tilde{B}_j} \right\Vert_{L^2(\mu)}^2.
    \end{align*}
    Thus,
    \begin{align*}
        \left\Vert W_m[u] \right\Vert_{L^2(\mu)}^2
        &\geq \sum_{j = 1}^{\mathcal{M}} a_1^2 \left\Vert u_j \mathds{1}_{\tilde{B}_j} \right\Vert_{L^2(\mu)}^2
        = a_1^2 \left\Vert u \right\Vert_{L^2(\mu)}^2.
    \end{align*}
    In conclusion, combining the above inequalities, and by \eqref{hubound}
    \begin{align*}
        &\left\Vert W_m[h] \right\Vert_{L^2(\mu)} \geq a_1\left\Vert u \right\Vert_{L^2(\mu)} - bm^{-s} \geq a_1(\left\Vert h \right\Vert_{L^2(\mu)} - bm^{-s}) - bm^{-s}
        =: b_1 \left\Vert h \right\Vert_{L^2(\mu)} - b_2m^{-s}.
    \end{align*}
This finishes the proof of Lemma \ref{discretel2multi}.

\subsubsection{Proof of Remark \ref{rmk:compositenull}}
\label{sec:proofcompositenull}
\sloppy
We follow the same strategy in Appendix \ref{sec:proofonesamplethm}.
For null distributions $P_0$ and alternative distributions $P_1$ such that $\ell_p\text{-ECE}_{P_0}(f) \leq c_0 n^{-2s / (4s + K - 1)}$ and $\ep:= \ell_p\text{-ECE}_{P_1}(f) \geq c_1 n^{-2s / (4s + K - 1)}$, we show the mean difference $\E_{P_1}[T_{m_*, n}^\text{d}] - \E_{P_0}[T_{m_*, n}^\text{d}]$ is larger than $\V_{P_0}(T_{m_*, n}^\text{d})^{1/ 2}$ and $\V_{P_1}(T_{m_*, n}^\text{d})^{1/2}$.

\fussy

While $\E_{P_0}[T_{m_*, n}^\text{d}] = 0$ under the null hypothesis of perfect calibration, we now have
\begin{align}
\label{eq:compnullmean}
\E_{P_0}[T_{m_*, n}^\text{d}] &= \frac{1}{n} \sum_{i = 1}^{m_*^{K - 1}} \E_{P_0}[I(N_i \geq 1)(N_i - 1)] \Vert \E_{P_0}[\res(Z) \mid Z \in B_i] \Vert^2 \nonumber\\
&\leq (m_*^{-(K - 1)} \wedge m_*^{-2(K - 1)} n ) \sum_{i = 1}^{m_*^{K - 1}} \E_{P_0}[\Vert \res(Z) \Vert^2 \mid Z \in B_i] \nonumber\\
&\leq c_0^2 n^{-\frac{4s}{4s + K - 1}} (1 \wedge m_*^{-(K - 1)} n ).
\end{align}
Therefore,
$$\E_{P_1}[T_{m_*, n}^\text{d}] - \E_{P_0}[T_{m_*, n}^\text{d}] \geq \delta -  c_0^2 n^{-\frac{4s}{4s + K - 1}} (1 \wedge m_*^{-(K - 1)} n ) =: \delta'.$$
By the equation \eqref{sigma},
\begin{align*}
    \V_{P_0}(T_{m_*, n}^\text{d}) \leq \tau^2 + 5K^2 c_0^2 n^{-\frac{4s}{4s + K - 1}} (n^{-1} \wedge m_*^{-(K - 1)}) =: (\tau')^2.
\end{align*}
Similar to \eqref{multiconstraint}, we can choose large enough $c_1 > 0$ such that
$$\frac{\tau'}{\sqrt{\alpha}} + \frac{\sigma}{\sqrt{\beta}} \leq \delta'.$$
The conclusion follows from Chebyshev's inequality as in \eqref{eq:cheby1} and \eqref{eq:cheby2}.

\subsubsection{Proof of Theorem \ref{thm:adaptiveone}}
\label{sec:proofadaptiveone}
By the union bound, for $P \in \mathcal{P}_0$,
$$P(\xi_n^\text{ad} = 1) \leq \sum_{b = 1}^B P\left(\xi_{2^b, n}\left( \frac{\alpha}{B}\right) = 1\right) \leq \sum_{b = 1}^B \frac{\alpha}{B} = \alpha.$$
There exists $b_0 \in \{1,\ldots,B\}$ such that $2^{b_0 - 1} < (n / \sqrt{\log n})^{2 / (4s + K - 1)} \leq 2^{b_0}$.
Let $m_0 = 2^{b_0}$ and repeat the argument in the proof of Theorem \ref{thm:onesamplethm}.
The condition \eqref{multiconstraint} for type II error control is now changed to
$$ \sqrt{\frac{2}{\alpha}} K m_0^{\frac{K - 1}{2}} n^{-1} \sqrt{\log n} + \sqrt{\frac{7}{\beta}} K m_0^{\frac{K - 1}{2}} n^{-1} \leq \frac{1}{4}(b_1^2\ep^2 - 2\sqrt{K}b_1b_2 m_0^{-s} \ep),$$
which is satisfied when $\ep \geq c_\text{ad} ( n / \sqrt{\log n})^{-2s / (4s + K - 1)}$ for a sufficiently large $c_\text{ad} > 0$.
Assuming $\ep \geq c_\text{ad} ( n / \sqrt{\log n})^{-2s / (4s + K - 1)}$ and $P \in \mathcal{P}_1(\ep, p, s)$, we have
$$P(\xi_n^\text{ad} = 1) \geq P(\xi_{m_0, n} = 1) \geq 1 - \beta.$$
This finishes the proof.

\subsubsection{Proof of Proposition \ref{prop:biasmean}}\label{sec:proofbiasmean}

\paragraph{Overview of the proof.}
We repeat the computation in Appendix \ref{sec:proofonesamplethm}.
However, due to the bias term, now the mean difference $\mathbb{E}_{P_1}[T_{m_*, n}^\text{b}] - \mathbb{E}_{P_0}[T_{m_*, n}^\text{b}]$ cannot be lower bounded by a positive number.
Instead, we prove that $\mathbb{E}_{P_0}[T_{m_*, n}^\text{b}] \geq \mathbb{E}_{P_1}[T_{m_*, n}^\text{b}]$ holds for all large enough $n$.

\begin{proof}
We use the same notations as in Appendix \ref{sec:proofonesamplethm}.
Since
\begin{align*}
    \E_{P_0}[T_{m_*, n}^\text{b} \mid \mathbf{Z}] &= \frac{1}{n} \sum_{i \in \mathcal{I}} \frac{1}{N_i} \left[ \sum_{j \in \mathcal{I}_{m_*,i}} \E_{P_0}\left[ \barY_j^2 \mid \mathbf{Z} \right] + \sum_{j_1 \neq j_2 \in \mathcal{I}_{m_*, i}} \E_{P_0}\left[\barY_{j_1} \barY_{j_2} \mid \mathbf{Z} \right] \right]\\
    &= \frac{1}{n} \sum_{i \in \mathcal{I}} \frac{1}{N_i} \sum_{j \in \mathcal{I}_{m_*, i}} \left( Z_j - Z_j^2 \right),
\end{align*}
we have
\begin{equation}
\label{eq:biasnullmean}
\E_{P_0}[T_{m_*, n}^\text{b} \mid \mathbf{N}] = \E_{P_0}[\E_{P_0}[T_{m_*, n}^\text{b} \mid \mathbf{Z}] \mid \mathbf{N}] = \frac{1}{n} \sum_{i \in \mathcal{I}} \E_{P_0}\left[ Z - Z^2 \mid Z \in B_i \right].
\end{equation}
Similarly,
\begin{align*}
    &\E_{P_1}[T_{m_*, n}^\text{b} \mid \mathbf{Z}] = \frac{1}{n} \sum_{i \in \mathcal{I}} \frac{1}{N_i} \left[ \sum_{j \in \mathcal{I}_{m_*, i}} \E_{P_1}\left[ \barY_j^2 \mid \mathbf{Z} \right] + \sum_{j_1 \neq j_2 \in \mathcal{I}_{m_*, i}} \E_{P_1}\left[\barY_{j_1} \barY_{j_2} \mid \mathbf{Z} \right] \right]\\
    &= \frac{1}{n} \sum_{i \in \mathcal{I}} \frac{1}{N_i} \left[ \sum_{j \in \mathcal{I}_{m_*, i}} (\reg(Z_j) - \reg(Z_j)^2 + \res(Z_j)^2) + \sum_{j_1 \neq j_2 \in \mathcal{I}_{m_*, i}} \res(Z_{j_1}) \res(Z_{j_2}) \right],
\end{align*}
and thus
\begin{align}
\label{eq:biasaltmean}
    &\E_{P_1}[T_{m_*, n}^\text{b} \mid \mathbf{N}] = \E_{P_1}[\E_{P_1}[T_{m_*, n}^\text{b} \mid \mathbf{Z}] \mid \mathbf{N}] \\
    =& \frac{1}{n} \sum_{i \in \mathcal{I}} \left(\E_{P_1} \left[ \reg(Z) - \reg(Z)^2 + \res(Z)^2 \mid Z \in B_i \right] + (N_i - 1) \E_{P_1}\left[ \res(Z) \mid Z \in B_i \right]^2 \right).\nonumber 
\end{align}
Since $Z \sim \text{Unif}([0, 1])$ under both $P_0$ and $P_1$, the equations \eqref{eq:biasnullmean} and \eqref{eq:biasaltmean} imply
\begin{align*}
    &\E_{P_0}[T_{m_*, n}^\text{b} \mid \mathbf{N}] - \E_{P_1}[T_{m_*, n}^\text{b} \mid \mathbf{N}]\\
    &= \frac{1}{n} \sum_{i \in \mathcal{I}} \left(  \E[\res(Z) (2Z - 1) \mid Z \in B_i] - (N_i - 1) \E[\res(Z) \mid Z \in B_i]^2 \right)\\
    &\geq  \frac{1}{n} \sum_{i \in \mathcal{I}} \E[\res(Z) (2Z - 1) \mid Z \in B_i] - \rho^2 \left\Vert \zeta \right\Vert_{L^1}^2 m_*^{-2s}.
\end{align*}
Here we used that $\E[\res(Z) \mid Z \in B_i]^2 \leq \rho^2 \Vert \zeta \Vert_{L^1}^2 m_*^{-2s}$ for all $i \in [m_*]$ and $\sum_{i \in \mathcal{I}} (N_i - 1) \leq n$.
Taking total expectation,
\begin{align}
\label{eq:biasmeandiff1}
    &\E_{P_0}[T_{m_*, n}^\text{b}] - \E_{P_1}[T_{m_*, n}^\text{b}] \geq \frac{1}{n} \E \left[ \sum_{i \in \mathcal{I}} \E[\res(Z) (2Z - 1) \mid Z \in B_i]\right] - \rho^2 \left\Vert \zeta \right\Vert_{L^1}^2 m_*^{-2s} \nonumber \\
    &= \frac{1}{n} \sum_{i = 1}^{m_*} P(N_i \geq 1) \E[\res(Z) (2Z - 1) \mid Z \in B_i] - \rho^2 \left\Vert \zeta \right\Vert_{L^1}^2 m_*^{-2s} \nonumber \\
    &= \frac{1}{n} P\left(N_1 \geq 1\right) \sum_{i = 1}^{m_*} \E[\res(Z) (2Z - 1) \mid Z \in B_i] - \rho^2 \left\Vert \zeta \right\Vert_{L^1}^2 m_*^{-2s}.
\end{align}
From \eqref{eq:biasreg}, we see that $\E[\res(Z)(2Z - 1) \mid Z \in B_i] \geq 0$ for all $i \in [m_*]$.
Thus,
\begin{align}
\label{eq:biasmeandiff2}
    &\sum_{i = 1}^{m_*} \E[\res(Z) (2Z - 1) \mid Z \in B_i] \geq \sum_{i = \frac{m_*}{4} + 1}^{\frac{m_*}{8}} \E[\res(Z) (2Z - 1) \mid Z \in B_i] \nonumber\\
    &\geq \frac{1}{4} \sum_{i = \frac{m_*}{4} + 1}^{\frac{m_*}{8}} \E[-\res(Z) \mid Z \in B_i] = \frac{\rho}{32} \left\Vert \zeta \right\Vert_{L^1} m_*^{1 - s}.
\end{align}
Combining \eqref{eq:biasmeandiff1} and \eqref{eq:biasmeandiff2}, we find
\begin{align}
\label{eq:biasmeandiff3}
    \E_{P_0}[T_{m_*, n}^\text{b}] - \E_{P_1}[T_{m_*, n}^\text{b}] \geq  \frac{\rho}{32} P(N_1 \geq 1)  \left\Vert \zeta \right\Vert_{L^1} m_*^{1 - s}n^{-1} - \rho^2 \left\Vert \zeta \right\Vert_{L^1}^2 m_*^{-2s}.
\end{align}
Since $m_* = \lfloor n^{2 / (4s + 1)} \rfloor$ and $\frac{2}{4s + 1} < 1$, we find
$$\lim_{n \rightarrow \infty} P\left(N_1 \geq 1\right) = \lim_{n \rightarrow \infty}  1 - \left(1 - \frac{1}{m_*}\right)^n= 1.$$
Also, we have $m_*^{1 - s}n^{-1} \asymp n^{(1 - 6s) /(4s + 1)}$ and $m_*^{-2s} \asymp n^{-4s / (4s + 1)}$ with $\frac{1 - 6s}{4s + 1} > \frac{-4s}{4s + 1}$.
In conclusion, the RHS of \eqref{eq:biasmeandiff3} is positive for all large enough $n$.

\end{proof}

\subsubsection{Ingter's method}\label{sec:proofingster}
    \begin{lemma}[Ingster's method for the lower bound]
\label{ingster}
    Let $P_0 \in \mathcal{P}_0$ and $P_1, \dots, P_M \in \mathcal{P}_1(\ep, p, s)$ be probability distributions on $\Delta_{K - 1} \times \mathcal{Y}$, and suppose that $P_1, \dots, P_M$ are absolutely continuous with respect to $P_0$.
    For an i.i.d. sample $\{(Z_i, Y_i): i \in \{1, \dots, n\}\}$ from $P_0$,
    define the average likelihood ratio between $P_1, \dots, P_M$ and $P_0$ as
    $$L_n := \frac{1}{M} \sum_{i = 1}^M \prod_{j = 1}^n \frac{dP_i}{dP_0} (Z_j, Y_j).$$
    If $\mathbb{E}_{P_0}[L_n^2] \leq 1 + (1 - \alpha - \beta)^2$, then
    the minimax type II error (false negative rate) for testing $H_0: P \in \mathcal{P}_0$ against $H_1: P \in \mathcal{P}_1(\ep, p, s)$ at level $\alpha$ satisfies
    $R_n(\ep, p, s) \geq \beta$ and
    the minimum separation rate to ensure type II error at most $\beta$ obeys
    $\ep_n(\beta; p, s)\geq \ep$.
\end{lemma}

    The proof follows from the results of \cite{ingster1987minimax,ingster2012nonparametric}; see also Lemma G.1 in \cite{kim2020minimax} for a very clear statement.
    By definition, it holds that
    \begin{align*}
        R_n(\ep, p, s) &= \inf_{\xi \in \Phi_n(\alpha)} \sup_{P \in \mathcal{P}_1(\ep, p, s)} \E_P[1 - \xi]
        \geq \inf_{\xi \in \Phi_n(\alpha)} \frac{1}{M} \sum_{i = 1}^M \E_{P_i}[1 - \xi]
        \\
        &= \inf_{\xi \in \Phi_n(\alpha)} \left( \E_{P_0}[1 - \xi] + \frac{1}{M} \sum_{i = 1}^M \E_{P_i}[1 - \xi] - \E_{P_0}[1 - \xi] \right)\\
        &\geq 1 - \alpha + \inf_{\xi \in \Phi_n(\alpha)} \left( \frac{1}{M} \sum_{i = 1}^M \E_{P_i} [1 - \xi] - \E_{P_0} [1 - \xi] \right).
    \end{align*}
    where the last inequality holds because $ \E_{P_0} [\xi] \leq \alpha$.
    Further,
    \begin{align*}
        &\left\vert \frac{1}{M} \sum_{i = 1}^M \E_{P_i}[1 - \xi] - \E_{P_0}[1 - \xi] \right\vert = \left\vert \E_{P_0}[\xi] - \frac{1}{M} \sum_{i = 1}^M \E_{P_i}[\xi] \right\vert= \left\vert \E_{P_0}[\xi] - \E_{P_0}[\xi L_n] \right\vert\\
        &\leq \E_{P_0}[|L_n - 1|]
        \leq \sqrt{\E_{P_0}[L_n^2] - 1} \leq 1 - \alpha - \beta
    \end{align*}
    by a change of variables and the Cauchy-Schwarz inequality.
    Therefore, we have
    $R_n(\ep, p, s) \geq 1 - \alpha - (1 - \alpha - \beta) = \beta$.
    Finally, since $\ep\mapsto R_n(\ep, p, s)$ is non-increasing, we find $\ep_n(p, s) \geq \ep$.

\subsubsection{Proof of Proposition \ref{impossibility}}\label{sec:proofimpossibility}

\paragraph{Overview of the proof.}
We construct distributions $P_1, \dots, P_M$ over $(Z, Y)$ under which the predictor $f$ has an $\ell_p$-ECE of at least $\ep_0 = 0.1$.
We can choose the mis-calibration curves of $P_1, \dots, P_M$ to be orthogonal in $L^2$, so that the cross terms in the expansion of $\E_{P_0}[L_n^2]$ cancel out.
By choosing $M$ sufficiently large, we can ensure that $\E_{P_0}[L_n^2]$ is at most $1 + (1 - \alpha - \beta)^2$.
The conclusion follows from Lemma \ref{ingster}.

\begin{proof}
We prove Proposition \ref{impossibility} for the binary case.
The generalization to the multi-class case follows the same argument and is omitted.
    The construction in this proof is inspired by \cite{ingster1987minimax, ingster2000adaptive, burnashev1979minimax}.
    Let $P_0$ be a null distribution 
    over $(Z, Y) \in [0,1] \times \{0, 1\}$
    defined as follows: the distribution of the predicted probabilities follows
    $Z \stackrel{P_0}{\sim} \text{Unif}([0,1])$ and
    $P_0(Y = 1 \mid Z=z) = z$ for all $z \in [0,1]$.
    Under $P_0$, the probability predictor $f$ is perfectly calibrated.
    For each $i \in [M]$, let
    $$g_i(u) := \begin{cases} u + \sqrt{\frac{u(1 - u)}{3}} \sin\left(2 i \pi(u - \frac{1}{4})\right) & u \in [\frac{1}{4}, \frac{3}{4}], \\ 
    u & u \notin [\frac{1}{4}, \frac{3}{4}],\end{cases}$$
    and define $P_i$ as follows:
    $Z \stackrel{P_i}{\sim} \text{Unif}([0, 1])$
    and $P_i(Y = 1 \mid Z = z)= g_i(z)$ for all $z \in [0,1]$.
  
    It can be verified that $0 \leq g_i(u) \leq 1$ for all $u \in [0, 1]$.
    Since $p \geq 1$, for all $i\in [M]$, the $\ell_p$-ECE of the probability predictor $f$ under $P_i$ is lower bounded as
    \begin{align*}
        \ell_p\text{-ECE}_{P_i}(f) &\geq \ell_1\text{-ECE}_{P_i}(f) = 2\int_0^1 |g_i(u) - u| du\\
        &= 2\int_{\frac{1}{4}}^{\frac{3}{4}} \sqrt{\frac{u(1 - u)}{3}} \left\vert \sin \left(2i\pi\left(u - \frac{1}{4}\right)\right) \right\vert du
        \geq 0.1.
    \end{align*}
    Thus we know that $P_i \in \mathcal{P}_1^\text{cont}(\ep_0, p)$ for all $i \in [M]$.
    Now, observe that
    $$L_n = \frac{1}{M} \sum_{i = 1}^M \prod_{j = 1}^n \frac{dP_i}{dP_0}(Z_j, Y_j) = \frac{1}{M} \sum_{i = 1}^M \prod_{j = 1}^n \frac{1 - Y_j + (2Y_j - 1)g_i(Z_j)}{1 - Y_j + (2Y_j - 1)Z_j}$$
    and thus, for a random variable $(Z, Y) \sim P_0$, and defining $A_{a,b}$ below
    \begin{align*}
        \mathbb{E}_{P_0}[L_n^2] &= \frac{1}{M^2} \sum_{a, b \in [M]} \mathbb{E}_{P_0} \left[ \prod_{j = 1}^n \frac{1 - Y_j + (2Y_j - 1)g_{a}(Z_j)}{1 - Y_j + (2Y_j - 1)Z_j} \cdot \frac{1 - Y_j + (2Y_j - 1)g_{b}(Z_j)}{1 - Y_j + (2Y_j - 1)Z_j} \right]\\
        &= \frac{1}{M^2} \sum_{a, b \in [M]} \mathbb{E}_{P_0}\left[\frac{1 - Y + (2Y - 1)g_{a}(Z)}{1 - Y + (2Y - 1)Z} \cdot \frac{1 - Y + (2Y - 1)g_{b}(Z)}{1 - Y + (2Y - 1)Z} \right]^n\\
        &=: \frac{1}{M^2} \sum_{a, b \in [M]} \mathbb{E}_{P_0}\left[A_{a,b}\right]^n.
    \end{align*}
    In the second line, we have used the independence of the observations.
    If $a = b$, then
    \begin{align*}
        \mathbb{E}_{P_0}\left[A_{a,b}\right]= &\int_0^1 u \frac{g_{a}(u)^2}{u^2} + (1 - u) \frac{(1 - g_{a}(u))^2}{(1 - u)^2} du
        =1 + \int_0^1 \frac{(g_{a}(u) - u)^2}{u(1 - u)} du\\
        = &1 + \int_{\frac{1}{4}}^{\frac{3}{4}} \frac{1}{3} \sin^2\left(2a \pi \left(u - \frac{1}{4}\right)\right) du
        = \frac{13}{12}.
    \end{align*}
    If $a \neq b$, then
    \begin{align*}
        \mathbb{E}_{P_0}\left[A_{a,b}\right]= &\int_0^1 u \frac{g_{a}(u)g_{b}(u)}{u^2} + (1 - u) \frac{(1 - g_{a}(u))(1 - g_{b}(u))}{(1 - u)^2} du\\
        = &1 + \int_0^1 \frac{(g_{a}(u) - u)(g_{b}(u) - u)}{u(1 - u)} du\\
        &= 1 + \int_{\frac{1}{4}}^{\frac{3}{4}} \frac{1}{3} \sin\left(2a \pi \left(u - \frac{1}{4}\right)\right) \sin\left(2b \pi \left(u - \frac{1}{4}\right)\right) du
        = 1.
    \end{align*}
    Therefore,
    $$\mathbb{E}_{P_0}[L_n^2] = \frac{1}{M^2}\left( M \left( \frac{13}{12} \right)^n + (M^2 - M) \right).$$
    Choose a large enough $M \in \mathbb{N}_+$ such that $M \geq [(13/12)^n-1]/(1-\alpha-\beta)^2$.
    Then,
    \begin{align*}
        \frac{1}{M^2}\left( M \left( \frac{13}{12} \right)^n + (M^2 - M) \right) \leq 1 + (1 - \alpha - \beta)^2,
    \end{align*}
    and the result follows by Lemma \ref{ingster}.
    
\end{proof}

\subsubsection{Proof of Theorem \ref{thm:multi-lower-bound}}\label{sec:proofmultilowerbound}
\paragraph{Overview of the proof.}
We construct $m^{K - 1}$ distributions under which the predictor $f$ has an $\ell_p$-ECE of $\Omega(n^{-2s / (4s + K - 1)})$.
The mis-calibration curves are constructed by linearly combining bump functions with disjoint supports.
By properly scaling them, we can guarantee H\"older continuity.
Also, the mis-calibration curves are chosen to be ``almost'' orthogonal in $L^2$, so that the cross terms in the expansion of $\E_{P_0}[L_n^2]$ are small.
We use Lemma \ref{ingster} to conclude.

\begin{proof}
    The  proof is inspired by the lower bound arguments in  \cite{arias2018remember}.
    For $m := \lceil n^{2 / (4s + K - 1)} \rceil$ and $\boldsymbol \eta \in \{ \pm 1 \}^{[m]^{K - 1}}$,
    we define alternative distributions $P_{\boldsymbol \eta} \in \mathcal{P}_1(\ep , p, s)$ with $\ep := c_\textnormal{lower}n^{-2s / (4s + K-1)}$ and use Lemma \ref{ingster} to prove $\ep_n(p, s) \geq c_\text{lower} n^{-2s / (4s + K - 1)}$.
    Let $\zeta: \R \to \R$ be the function from \eqref{eq:zetadef}.
    It can be verified that $\zeta$ is infinitely differentiable and its derivatives of every order are bounded.
    For $\psi = \pi_{-K}:(z_1, \dots, z_K)^\top \mapsto (z_1, \dots, z_{K - 1})^\top$, we see that $[\frac{1}{2K}, \frac{1}{K}]^{K - 1} \subseteq \psi(\Delta_{K - 1} \cap [\frac{1}{2K}, 1]^K)$.
    For each $\mathbf{j} = (j_1,\ldots, j_{K-1})^\top \in [m]^{K-1}$,
    define $\Psi_\mathbf{j}: \R^{K-1} \to \R$ by
    \begin{align}\label{eqn:Psi-def}
        \Psi_{\mathbf{j}}(x_1, \dots, x_{K-1}) := m^{-s}   \prod_{k = 1}^{K-1} \zeta \Big( m(2K x_k - 1)  - j_k + 1 \Big).
    \end{align}
    Then, each $\Psi_\mathbf{j}$ is supported on the cube
    $$\supp (\Psi_\mathbf{j}) = \prod_{k = 1}^{K-1} \left(\frac{j_k - 1 + m}{2Km}, \frac{j_k + m}{2Km}\right).$$
    The sets $\supp(\Psi_\mathbf{j})$ are disjoint for different indices $\mathbf{j} \in [m]^{K-1}$, and we have
    $$\bigcup_{\mathbf{j} \in [m]^{K - 1}} \supp(\Psi_\mathbf{j}) \subseteq \left[ \frac{1}{2K}, \frac{1}{K} \right]^{K - 1} \subseteq \psi\left(\Delta_{K - 1} \cap \left[\frac{1}{2K}, 1\right]^K\right).$$
    
    Let $c_{\alpha, \beta} := \left(\log\left(1 + (1 - \alpha - \beta)^2 \right)  \right)^{1 / 4}$ and
    \beq
    \label{cprime-1}
    \rho := \left(\max_{t\in \{0, \ldots, \lceil s \rceil\} } \left\Vert \zeta^{(t)} \right\Vert_{L^\infty}^{K-1}\right)^{-1} \left( \frac{1}{2K} \wedge \frac{L (2K)^{-\lceil s \rceil}}{2\sqrt{K - 1}} \wedge \frac{L (2K)^{-\lceil s \rceil + 1}}{4} \wedge \frac{c_{\alpha, \beta} (2K)^{\frac{K - 1}{2}}}{2\sqrt{K!}}\right).
    \eeq
    By the definition of $\rho$ in \eqref{cprime-1}, we see that
    \begin{equation}
    \label{multilower-1}
        \rho \left\Vert \zeta \right\Vert_{L^\infty}^{K - 1} \leq \frac{1}{2K},
    \end{equation}
    \begin{equation}
    \label{multilower-2}
        \rho \sqrt{K - 1} (2K)^{\lceil s \rceil} \left(\max_{t\in \{0, \ldots, \lceil s \rceil\}} \left\Vert \zeta^{(t)} \right\Vert_{L^\infty}^{K-1}\right) \leq \frac{L}{2},
    \end{equation}
    \begin{equation}
    \label{multilower-3}
        2\rho  (2K)^{\lceil s \rceil - 1}\left(\max_{t\in \{0, \ldots, \lceil s \rceil- 1\} } \left\Vert \zeta^{(t)} \right\Vert_{L^\infty}^{K-1}\right) \leq \frac{L}{2},
    \end{equation}
    and
    \begin{equation}
    \label{multilower-4}
        4\rho^2 K! (2K)^{-K + 1} \left\Vert \zeta \right\Vert_{L^2}^{2(K-1)} \leq c_{\alpha, \beta}^2 \leq 1 .
    \end{equation}
    
    For each $\boldsymbol \eta \in \{\pm 1\}^{[m]^{K-1}}$, define $g_{\boldsymbol \eta}: \Delta_{K - 1} \to \R^{K}$ by
    $$g_{\boldsymbol \eta}(\mathbf{z}) := \mathbf{z} + \rho  \left( \sum_{\mathbf{j} \in [m]^{K-1}} {\boldsymbol \eta}_{\mathbf{j}} \left(\Psi_\mathbf{j} \circ \psi\right)(\mathbf{z}) \right) (1, -1, 0, \dots, 0)^\top.$$
    Then we have $g_{\boldsymbol \eta}(\Delta_{K - 1}) \subseteq \Delta_{K - 1}$ for all $\boldsymbol \eta \in \{\pm 1\}^{[m]^{K - 1}}$.
    This is because
    \begin{align*}
    \sum_{k = 1}^K [g_{\boldsymbol \eta}(\mathbf{z})]_k = \sum_{k = 1}^K [\mathbf{z}]_k + \rho \left( \sum_{\mathbf{j} \in [m]^{K-1}} {\boldsymbol \eta}_{\mathbf{j}} \left(\Psi_{\mathbf{j}} \circ \psi\right)(\mathbf{z}) \right) - \rho \left( \sum_{\mathbf{j} \in [m]^{K-1}} {\boldsymbol \eta}_{\mathbf{j}} \left(\Psi_{\mathbf{j}} \circ \psi\right)(\mathbf{z}) \right) = 1
    \end{align*}
    for all $\mathbf{z} \in \Delta_{K - 1}$ and
    \begin{align*}
            [g_{\boldsymbol \eta}(\mathbf{z})]_k &= \begin{cases}[\mathbf{z}]_k \pm \rho \sum_{\mathbf{j} \in [m]^{K - 1}} \boldsymbol \eta_{\mathbf{j}} (\Psi_\mathbf{j} \circ \psi)(\mathbf{z}) \geq \frac{1}{2K} -\rho m^{-s} \left\Vert \zeta \right\Vert_{L^\infty}^{K-1},  \\ &\hspace{-50pt} \text{if } \mathbf{z} \in [\frac{1}{2K}, 1]^K, k \in \{1, 2\},\\  [\mathbf{z}]_k, &\hspace{-50pt} \text{otherwise.} \end{cases}\\
            &\geq 0
    \end{align*}
    for all $\mathbf{z} \in \Delta_{K - 1}$ and $k \in \{1,\ldots,K\}$ by \eqref{multilower-1}.
        
     Next, we claim that each coordinate function of the mapping $\mathbf{z} \mapsto g_{\boldsymbol \eta}(\mathbf{z}) - \mathbf{z}$ belongs to $\mathcal{H}_K(s, L)$, i.e., for all multi-indices $\boldsymbol \gamma \in \mathbb{N}^{K-1}$ with $|\boldsymbol \gamma| = \lceil s \rceil - 1$ and $\mathbf{x}_1, \mathbf{x}_2 \in \psi(\Delta_{K - 1})$,
    \begin{equation}
    \label{resholder}
    \rho\left\vert \sum_{\mathbf{j} \in [m]^{K - 1}} {\boldsymbol \eta}_\mathbf{j} \left( \Psi_\mathbf{j}^{(\boldsymbol \gamma)}(\mathbf{x}_1) - \Psi_\mathbf{j}^{(\boldsymbol \gamma)}(\mathbf{x}_2) \right) \right\vert \leq L \left\Vert \mathbf{x}_1 - \mathbf{x}_2 \right\Vert^{s - \lceil s \rceil + 1}.
    \end{equation}
    If $m \left\Vert \mathbf{x}_1 - \mathbf{x}_2\right\Vert \leq 1$, then from the mean value theorem
    \begin{align}\label{eqn:gvuoqwndas}
        \rho \left\vert \Psi_\mathbf{j}^{(\boldsymbol \gamma)}(\mathbf{x}_1) - \Psi_\mathbf{j}^{(\boldsymbol \gamma)}(\mathbf{x}_2) \right\vert &\leq \rho \sqrt{K - 1} \max_{\substack{\boldsymbol \gamma \in \mathbb{N}^{K - 1}\\|\boldsymbol \gamma'| = \lceil s \rceil}} \left\Vert \Psi_\mathbf{j}^{(\boldsymbol \gamma')} \right\Vert_{L^\infty} \left\Vert \mathbf{x}_1 - \mathbf{x}_2 \right\Vert.
    \end{align}
    By the definition of $\Psi_\mathbf{j}$ in  \eqref{eqn:Psi-def}, for any $\boldsymbol \gamma' = (\gamma_1', \dots, \gamma_{K - 1}')^\top \in \mathbb{N}^{K - 1}$ with $|\boldsymbol \gamma'| = \lceil s \rceil$,
    \begin{align}\label{eqn:viowejfa}
        \left\Vert \Psi_\mathbf{j}^{(\boldsymbol \gamma^\prime)} \right\Vert_{L^\infty} = m^{|\boldsymbol{\gamma}^\prime|-s}(2K)^{|\boldsymbol{\gamma}^\prime|}\prod_{k=1}^{K-1}\left\Vert \zeta^{(\gamma'_k)} \right\Vert_{L^\infty} \leq  m^{\lceil s\rceil -s}(2K)^{\lceil s\rceil}\max_{t\in \{0, \ldots, \lceil s \rceil\}} \left\Vert \zeta^{(t)} \right\Vert_{L^\infty}^{K-1}.
    \end{align}
    Plugging \eqref{eqn:viowejfa} into \eqref{eqn:gvuoqwndas}, and using \eqref{multilower-2}, we reach 
    \begin{align*}
        \rho \left\vert \Psi_\mathbf{j}^{(\boldsymbol \gamma)}(\mathbf{x}_1) - \Psi_\mathbf{j}^{(\boldsymbol \gamma)}(\mathbf{x}_2) \right\vert
        &\leq\rho \sqrt{K - 1} m^{\lceil s \rceil - s} (2K)^{\lceil s \rceil} \left(\max_{t\in \{0, \ldots, \lceil s \rceil\}} \left\Vert \zeta^{(t)} \right\Vert_{L^\infty}^{K-1}\right)  \left\Vert \mathbf{x}_1 - \mathbf{x}_2 \right\Vert\\
        &\leq \frac{L}{2} m^{\lceil s \rceil - s} \left\Vert \mathbf{x}_1 - \mathbf{x}_2 \right\Vert 
        \leq \frac{L}{2} \left\Vert \mathbf{x}_1 - \mathbf{x}_2 \right\Vert^{s - \lceil s \rceil + 1}.
    \end{align*}
    If $m \left\Vert \mathbf{x}_1 - \mathbf{x}_2 \right\Vert > 1$, using \eqref{multilower-3}, we similarly get
        \begin{align*}
        &\rho \left\vert \Psi_\mathbf{j}^{(\boldsymbol \gamma)}(\mathbf{x}_1) - \Psi_{\mathbf{j}}^{(\boldsymbol \gamma)}(\mathbf{x}_2) \right\vert \leq 2\rho \max_{|\boldsymbol \gamma'| = \lceil s \rceil - 1} \left\Vert \Psi_\mathbf{j}^{(\boldsymbol \gamma')} \right\Vert_{L^\infty}\\
        &\leq 2\rho m^{\lceil s \rceil - s - 1} (2K)^{\lceil s \rceil - 1} \left(\max_{t\in \{0, \ldots, \lceil s \rceil- 1\} } \left\Vert \zeta^{(t)} \right\Vert_{L^\infty}^{K-1}\right) \leq \frac{L}{2} m^{\lceil s \rceil - s - 1}
        \leq \frac{L}{2} \left\Vert \mathbf{x}_1 - \mathbf{x}_2 \right\Vert^{s - \lceil s \rceil + 1}. 
    \end{align*}
     Thus, for any $\mathbf{j} \in [m]^{K - 1}$ and $\mathbf{x}_1,\mathbf{x}_2\in \psi(\Delta_{K-1})$, it holds that
    \begin{equation}
    \label{singleholder}
        \rho\left\vert \Psi_\mathbf{j}^{(\boldsymbol \gamma)}(\mathbf{x}_1) - \Psi_\mathbf{j}^{(\boldsymbol \gamma)}(\mathbf{x}_2) \right\vert \leq \frac{L}{2} \left\Vert \mathbf{x}_1 - \mathbf{x}_2 \right\Vert^{s - \lceil s \rceil + 1}.
    \end{equation}
    
    Given $\mathbf{x}_1, \mathbf{x}_2 \in \psi(\Delta_{K - 1})$, there can be two cases: (1) there exists $\mathbf{j}_1 \in [m]^{K - 1}$ such that
    $$\rho\left( \sum_{\mathbf{j} \in [m]^{K-1}} \boldsymbol \eta_{\mathbf{j}} \left( \Psi_\mathbf{j}^{(\boldsymbol \gamma)} (\mathbf{x}_1) - \Psi_\mathbf{j}^{(\boldsymbol \gamma)} (\mathbf{x}_2) \right) \right) = \rho \boldsymbol \eta_{\mathbf{j}_1} \left( \Psi_{\mathbf{j}_1}^{(\boldsymbol \gamma)}(\mathbf{x}_1) - \Psi_{\mathbf{j}_1}^{(\boldsymbol \gamma)}(\mathbf{x}_2) \right);$$
    or (2) there exist distinct $\mathbf{j}_1, \mathbf{j}_2 \in [m]^{K - 1}$ such that $\mathbf{x}_1 \in \supp(\Psi_{\mathbf{j}_1})$ and $\mathbf{x}_2 \in \supp(\Psi_{\mathbf{j}_2})$.
    In the first case, \eqref{resholder} directly follows from \eqref{singleholder}.
    In the second case, choose a point $\mathbf{x}_3$ on the line segment connecting $\mathbf{x}_1$ and $\mathbf{x}_2$ such that $\mathbf{x}_3 \notin \supp(\Psi_{\mathbf{j}_1}) \cup \supp(\Psi_{\mathbf{j}_2})$.
    Such a point exists since $\supp(\Psi_{\mathbf{j}_1})$, $\supp(\Psi_{\mathbf{j}_2})$ are open and $\supp(\Psi_{\mathbf{j}_1}) \cap \supp(\Psi_{\mathbf{j}_2}) = \varnothing$.
    For any $\boldsymbol \gamma \in \mathbb{N}^{K - 1}$ with $|\boldsymbol \gamma| = \lceil s \rceil - 1$, we have
        \begin{align*}
        &\rho \left\vert \sum_{\mathbf{j} \in [m]^{K-1}} \boldsymbol \eta_\mathbf{j} \left( \Psi_\mathbf{j}^{(\boldsymbol \gamma)}(\mathbf{x}_1) - \Psi_\mathbf{j}^{(\boldsymbol \gamma)}(\mathbf{x}_2) \right) \right\vert =\rho\left\vert \boldsymbol \eta_{\mathbf{j}_1} \Psi_{\mathbf{j}_1}^{(\boldsymbol \gamma)}(\mathbf{x}_1) - \boldsymbol \eta_{\mathbf{j}_2} \Psi_{\mathbf{j}_2}^{(\boldsymbol \gamma)}(\mathbf{x}_2) \right\vert\\
        &=\rho\left\vert \boldsymbol \eta_{\mathbf{j}_1} \Psi_{\mathbf{j}_1}^{(\boldsymbol \gamma)}(\mathbf{x}_1) -\boldsymbol \eta_{\mathbf{j}_1} \Psi_{\mathbf{j}_1}^{(\boldsymbol \gamma)}(\mathbf{x}_3) + \boldsymbol \eta_{\mathbf{j}_2}\Psi_{\mathbf{j}_2}^{(\boldsymbol \gamma)}(\mathbf{x}_3) -\boldsymbol \eta_{\mathbf{j}_2} \Psi_{\mathbf{j}_2}^{(\boldsymbol \gamma)}(\mathbf{z}_2) \right\vert\\
        &\leq  \rho \left\vert \Psi_{\mathbf{j}_1}^{(\boldsymbol \gamma)}(\mathbf{x}_1) - \Psi_{\mathbf{j}_1}^{(\boldsymbol \gamma)}(\mathbf{x}_3) \right\vert + \rho \left\vert \Psi_{\mathbf{j}_2}^{(\boldsymbol \gamma)}(\mathbf{x}_3) - \Psi_{\mathbf{j}_2}^{(\boldsymbol \gamma)}(\mathbf{x}_2) \right\vert\\
        &\leq \frac{L}{2} \left\Vert \mathbf{x}_1 - \mathbf{x}_3 \right\Vert^{s - \lceil s \rceil + 1} + \frac{L}{2} \left\Vert \mathbf{x}_3 - \mathbf{x}_2 \right\Vert^{s - \lceil s \rceil + 1}
        \leq L\left\Vert \mathbf{x}_1 - \mathbf{x}_2 \right\Vert^{s - \lceil s \rceil + 1}. 
    \end{align*}
    The second inequality holds because of \eqref{singleholder}, and the last inequality holds because $\Vert \mathbf{x}_1 - \mathbf{x}_3 \Vert, \Vert \mathbf{x}_3 - \mathbf{x}_2 \Vert \leq \Vert \mathbf{x}_1 - \mathbf{x}_2 \Vert$ and $s - \lceil s \rceil + 1 > 0$.
    This finishes the proof of \eqref{resholder}.
    
    Now, let $P_0$ and $P_{\boldsymbol \eta}$, $\boldsymbol \eta \in \{\pm 1\}^{[m]^{K-1}}$, be the distributions of $(Z, Y) \in \Delta_{K-1} \times \mathcal{Y}$ characterized by
    $$Z \stackrel{P_0}{\sim} \text{Unif}(\Delta_{K - 1}), \quad Y \mid Z = \mathbf{z} \stackrel{P_0}{\sim} \text{Cat}(\mathbf{z}) \text{ for all } \mathbf{z} \in \Delta_{K - 1},$$
    and
    $$Z \stackrel{P_{\boldsymbol \eta}}{\sim} \text{Unif}(\Delta_{K - 1}), \quad Y \mid Z = \mathbf{z} \stackrel{P_{\boldsymbol \eta}}{\sim} \text{Cat}(g_{\boldsymbol \eta}(\mathbf{z})) \text{ for all } \mathbf{z} \in \Delta_{K - 1}.$$
    We have $P_0 \in \mathcal{P}_0$ by definition.
    For $Z \sim \text{Unif}(\Delta_{K - 1})$ and $\mathbf{j}_0 := \mathbf{1}_{K - 1} \in [m]^{K - 1}$,
    \begin{align*}
        \ell_p\text{-ECE}_{P_{\boldsymbol \eta}}(f)^p &= \E\left[ \sum_{k = 1}^K \left\vert [g_{\boldsymbol \eta}(Z) - Z]_k \right\vert^p \right]
        = 2\rho^p \E \left[ \left\vert \sum_{\mathbf{j} \in [m]^{K-1}} \boldsymbol \eta_\mathbf{j} (\Psi_\mathbf{j} \circ \psi)(Z) \right\vert^p \right]\\
        &= 2 (K-1)! \rho^p m^{K-1} \int_{\R^{K - 1}} |\Psi_{\mathbf{j}_0}(\mathbf{x})|^p d\mathbf{x}.
    \end{align*}
    Further,
    \begin{align}
    \label{bumplp}
        m^{K-1} \int_{\R^{K - 1}} |\Psi_{\mathbf{j}_0}(\mathbf{x})|^p d\mathbf{x} &= m^{-ps} \prod_{k = 1}^{K-1} \left( m \int_\R \left\vert \zeta\left( m(2 K x_k - 1) \right) \right\vert^p dx_k \right) \nonumber \\
        &= m^{-ps} (2K)^{-K + 1} \left\Vert \zeta \right\Vert_{L^p}^{(K-1)p}.
    \end{align}
    Thus, we have
    \begin{equation}
    \label{alternative-ece}
        \ell_p\text{-ECE}_{P_{\boldsymbol \eta}}(f) = (2(2K)^{-K + 1} (K-1)!)^\frac{1}{p} \rho m^{-s} \left\Vert \zeta \right\Vert_{L^p}^{K-1} \geq c_\textnormal{lower} n^{-\frac{2s}{4s + K-1}}
    \end{equation}
    for some $c_\textnormal{lower} > 0$ because
    $$\lim_{n \rightarrow \infty} \frac{m^{-s}}{n^{-\frac{2s}{4s + K-1}}} = \lim_{n \to \infty} \left( \lceil n^{\frac{2}{4s + K-1}} \rceil \right)^{-s} n^{-\frac{2s}{4s + K-1}} = 1.$$
    From \eqref{resholder} and \eqref{alternative-ece}, we see that $P_{\boldsymbol \eta} \in \mathcal{P}_1(\ep, p, s)$ with $\ep = c_\textnormal{lower}n^{-2s / (4s + K-1)}$ for all $\boldsymbol \eta \in \{\pm 1\}^{[m]^{K-1}}$.
    
    The final step is to apply Lemma \ref{ingster}.
    Given $n$ i.i.d. observations $\{(Z_i, Y_i): i \in [n]\}$, the average likelihood ratio between $P_{\boldsymbol \eta}$, ${\boldsymbol \eta} \in \{\pm 1\}^{[m]^{K-1}}$, and $P_0$ is
    $$L_n = \frac{1}{2^{m^{K-1}}} \sum_{\boldsymbol \eta \in \{\pm 1\}^{[m]^{K-1}}} \prod_{i = 1}^n \frac{[g_{\boldsymbol \eta}(Z_i)]_{\argmax_k [Y_i]_k}}{[Z_i]_{\argmax_k [Y_i]_k}}.$$
    Let $\boldsymbol \eta^1$, $\boldsymbol \eta^2$ be independent random variables uniformly drawn from $\{\pm 1\}^{[m]^{K-1}}$, and $(Z, Y) \sim P_0$.
    Then,
    \begin{align*}
        \E_{P_0}[L_n^2] &= \E_{\boldsymbol \eta^1, \boldsymbol \eta^2} \E_{P_0} \prod_{i = 1}^n \frac{[g_{\boldsymbol \eta^1}(Z_i)]_{\argmax_k [Y_i]_k}}{[Z_i]_{\argmax_k [Y_i]_k}} \cdot \frac{[g_{\boldsymbol \eta^2}(Z_i)]_{\argmax_k [Y_i]_k}}{[Z_i]_{\argmax_k [Y_i]_k}}\\
        &= \E_{\boldsymbol \eta^1, \boldsymbol \eta^2} \left( \E_{P_0} \frac{[g_{\boldsymbol \eta^1}(Z)]_{\argmax_k [Y]_k} [g_{\boldsymbol \eta^2}(Z)]_{\argmax_k [Y]_k}}{[Z]_{\argmax_k [Y]_k}^2} \right)^n.
    \end{align*}
    Moreover,
    \begin{align*}
        &\E_{P_0} \frac{[g_{\boldsymbol \eta^1}(Z)]_{\argmax_k [Y]_k} [g_{\boldsymbol \eta^2}(Z)]_{\argmax_k [Y]_k}}{[Z]_{\argmax_k [Y]_k}^2}
        = \int_{\Delta_{K - 1}} \left( \sum_{k = 1}^K [\mathbf{z}]_k \frac{[g_{\boldsymbol \eta^1}(\mathbf{z})]_k [g_{\boldsymbol \eta^2}(\mathbf{z})]_k}{[\mathbf{z}]_k^2}\right) d\mathbf{z}\\
        &= 1 + \int_{\Delta_{K - 1}} \left( \sum_{k = 1}^K \frac{[g_{\boldsymbol \eta^1}(\mathbf{z}) - \mathbf{z}]_k [g_{\boldsymbol \eta^2}(\mathbf{z}) - \mathbf{z}]_k}{[\mathbf{z}]_k} \right) d\mathbf{z}\\
        &= 1 + \rho^2 \int_{\Delta_{K - 1}} \left( \sum_{\mathbf{j} \in [m]^{K-1}} \boldsymbol \eta^1_\mathbf{j} (\Psi_\mathbf{j} \circ \psi)(\mathbf{z}) \right)\left( \sum_{\mathbf{j} \in [m]^{K-1}} \boldsymbol \eta^2_\mathbf{j} (\Psi_\mathbf{j} \circ \psi)(\mathbf{z}) \right) \left(\frac{1}{[\mathbf{z}]_0} + \frac{1}{[\mathbf{z}]_1} \right)d\mathbf{z}
    \end{align*}
    where the integral $\int_{\Delta_{K - 1}}$ is over $\text{Unif}(\Delta_{K - 1})$.
    Since $\{\supp(\Psi_\mathbf{j}): \mathbf{j} \in [m]^{K - 1} \}$ is a collection of pairwise disjoint sets,
    \begin{align*}
        &\int_{\Delta_{K - 1}} \left( \sum_{\mathbf{j} \in [m]^{K-1}} \boldsymbol \eta^1_\mathbf{j} (\Psi_\mathbf{j} \circ \psi)(\mathbf{z}) \right)\left( \sum_{\mathbf{j} \in [m]^{K-1}} \boldsymbol \eta^2_\mathbf{j} (\Psi_\mathbf{j} \circ \psi)(\mathbf{z}) \right) \left(\frac{1}{[\mathbf{z}]_0} + \frac{1}{[\mathbf{z}]_1} \right)d\mathbf{z}\\
        &= \int_{\Delta_{K - 1}} \left( \sum_{\mathbf{j} \in [m]^{K-1}} \boldsymbol \eta^1_\mathbf{j} \boldsymbol \eta^2_\mathbf{j} (\Psi_\mathbf{j} \circ \psi)(\mathbf{z})^2 \right) \left( \frac{1}{[\mathbf{z}]_0} + \frac{1}{[\mathbf{z}]_1} \right) d\mathbf{z}.
    \end{align*}
    The random variable $(\boldsymbol \eta^1_\mathbf{j} \boldsymbol \eta^2_\mathbf{j})_{\mathbf{j} \in [m]^{K - 1}}$ is also uniformly distributed on $\{\pm 1\}^{[m]^{K - 1}}$.
    Hence,
    \begin{align}
    \label{eq:lnsquare}
        \E_{P_0}[L_n^2] &= \E_{\boldsymbol \eta^1} \left(1 + \rho^2  \sum_{\mathbf{j} \in [m]^{K - 1}} \boldsymbol \eta^1_\mathbf{j}  \int_{\Delta_{K - 1}}  (\Psi_\mathbf{j} \circ \psi)(\mathbf{z})^2  \left(\frac{1}{[\mathbf{z}]_0} + \frac{1}{[\mathbf{z}]_1} \right) d\mathbf{z} \right)^n.
    \end{align}
    Since $\bigcup_{\mathbf{j} \in [m]^{K - 1}} \supp(\Psi_\mathbf{j}) \subseteq \psi(\Delta_{K - 1} \cap [\frac{1}{2K}, 1]^{K})$,
    we have $\frac{1}{[\mathbf{z}]_0} + \frac{1}{[\mathbf{z}]_1}\leq 4K$ for every $\mathbf{z} \in \Delta_{K - 1}$ such that $(\Psi_\mathbf{j} \circ \psi)(\mathbf{z}) \neq 0$ for at least one $\mathbf{j} \in [m]^{K - 1}$.
    Thus,
    \begin{align}
    \label{eq:sumbound}
        &\left\vert \rho^2  \sum_{\mathbf{j} \in [m]^{K-1}} \boldsymbol \eta_\mathbf{j}  \int_{\Delta_{K - 1}}  (\Psi_\mathbf{j} \circ \psi)(\mathbf{z})^2  \left(\frac{1}{[\mathbf{z}]_0} + \frac{1}{[\mathbf{z}]_1} \right) d\mathbf{z} \right\vert \nonumber \\
        &\leq \rho^2 \sum_{\mathbf{j} \in [m]^{K-1}} \left\vert \int_{\Delta_{K - 1}} (\Psi_\mathbf{j} \circ \psi)(\mathbf{z})^2  \left(\frac{1}{[\mathbf{z}]_0} + \frac{1}{[\mathbf{z}]_1} \right) d\mathbf{z} \right\vert \nonumber\\
        &\leq \rho^2 \sum_{\mathbf{j} \in [m]^{K-1}} 4K \left\vert \int_{\Delta_{K - 1}} (\Psi_\mathbf{j} \circ \psi)(\mathbf{z})^2 d\mathbf{z} \right\vert
        = 4 \rho^2 K! m^{K-1} \int_{\R^{K - 1}} \Psi_{\mathbf{j}_0}(\mathbf{x})^2 d\mathbf{x}.
    \end{align}
    The last equality is because $\text{Unif}(\Delta_{K - 1})$ has density $(K - 1)!$ with respect to $\text{Leb}_{K - 1}$, when projected to $\R^{K - 1}$.
    Also, by \eqref{multilower-4} and \eqref{bumplp},
    \begin{align*}
    \label{eq:sumbound2}
         4 \rho^2 K! m^{K-1} \int_{\R^{K - 1}} \Psi_{\mathbf{j}_0}(\mathbf{x})^2 d\mathbf{x} &= 4\rho^2 K! m^{-2s} (2K)^{-K + 1} \left\Vert \zeta \right\Vert_{L^2}^{2(K-1)} \leq c_{\alpha, \beta}^2 m^{-2s} \leq 1.
    \end{align*}
    By \eqref{eq:lnsquare} and that $(1 + x)^n \leq \exp(nx)$ for all $x \in (-1, 1]$,
    \begin{align}
        \E_{P_0}[L_n^2] &\leq \E_{\boldsymbol \eta^1} \exp \left( n \rho^2 \sum_{\mathbf{j} \in [m]^{K-1}} \boldsymbol \eta^1_\mathbf{j}  \int_{\Delta_{K - 1}}  (\Psi_\mathbf{j} \circ \psi)(\mathbf{z})^2  \left(\frac{1}{[\mathbf{z}]_0} + \frac{1}{[\mathbf{z}]_1} \right) d\mathbf{z}  \right).
    \end{align}

    Since $\{\boldsymbol \eta^1_\mathbf{j} : \mathbf{j} \in [m]^{K - 1}\}$ is a set of i.i.d. random variables drawn from $\text{Unif}(\{ \pm 1\})$, we have, with $\cosh (x) := [\exp(x)+ \exp(-x)]/2$,
    \begin{align*}
        &\E_{\boldsymbol \eta^1} \exp \left( n \rho^2 \sum_{\mathbf{j} \in [m]^{K-1}} \boldsymbol \eta^1_\mathbf{j}  \int_{\Delta_{K - 1}}  (\Psi_\mathbf{j} \circ \psi)(\mathbf{z})^2  \left(\frac{1}{[\mathbf{z}]_0} + \frac{1}{[\mathbf{z}]_1} \right) d\mathbf{z}  \right)\\
        &= \prod_{\mathbf{j} \in [m]^{K-1}} \E_{\boldsymbol \eta^1_\mathbf{j}} \exp \left(  n \rho^2 \boldsymbol \eta^1_\mathbf{j}  \int_{\Delta_{K - 1}}  (\Psi_\mathbf{j} \circ \psi)(\mathbf{z})^2  \left(\frac{1}{[\mathbf{z}]_0} + \frac{1}{[\mathbf{z}]_1} \right) d\mathbf{z} \right)\\
        &= \prod_{\mathbf{j} \in [m]^{K-1}} \cosh \left(  n \rho^2 \int_{\Delta_{K - 1}}  (\Psi_\mathbf{j} \circ \psi)(\mathbf{z})^2  \left(\frac{1}{[\mathbf{z}]_0} + \frac{1}{[\mathbf{z}]_1} \right) d\mathbf{z} \right).
    \end{align*}
    Similarly to \eqref{eq:sumbound} and \eqref{eq:sumbound2}, we have
    \begin{align}
    \label{eq:summandbound}
        \left\vert n \rho^2 \int_{\Delta_{K - 1}}  (\Psi_\mathbf{j} \circ \psi)(\mathbf{z})^2  \left(\frac{1}{[\mathbf{z}]_0} + \frac{1}{[\mathbf{z}]_1} \right) d\mathbf{z} \right\vert \leq c_{\alpha, \beta}^2 m^{-2s - K + 1} n \leq 1
    \end{align}
    for each $\mathbf{j} \in [m]^{K-1}$.
    Using that $\cosh(x) \leq 1 + x^2 \leq e^{x^2}$ for $x \in [-1, 1]$,
    \begin{align*}
        &\prod_{\mathbf{j} \in [m]^{K-1}} \cosh \left(  n \rho^2 \int_{\Delta_{K - 1}}  (\Psi_\mathbf{j} \circ \psi)(\mathbf{z})^2  \left(\frac{1}{[\mathbf{z}]_0} + \frac{1}{[\mathbf{z}]_1} \right) d\mathbf{z} \right)\\
        &\leq \exp \left( \sum_{\mathbf{j} \in [m]^{K-1}} \left( n \rho^2 \int_{\Delta_{K - 1}}  (\Psi_\mathbf{j} \circ \psi)(\mathbf{z})^2  \left(\frac{1}{[\mathbf{z}]_0} + \frac{1}{[\mathbf{z}]_1} \right) d\mathbf{z}  \right)^2 \right).
    \end{align*}
    Again from \eqref{eq:summandbound}, it follows that
    \begin{align*}
        &\exp \left( \sum_{\mathbf{j} \in [m]^{K-1}} \left( n \rho^2 \int_{\Delta_{K - 1}}  (\Psi_\mathbf{j} \circ \psi)(\mathbf{z})^2  \left(\frac{1}{[\mathbf{z}]_0} + \frac{1}{[\mathbf{z}]_1} \right) d\mathbf{z}  \right)^2 \right)\\
        &\leq \exp( c_{\alpha, \beta}^4 m^{-4s - K + 1} n^2)
        \leq 1 + (1 - \alpha - \beta)^2.
    \end{align*}
    In conclusion, $\ep_n(p, s) \geq c_\textnormal{lower} n^{-2s / (4s + K - 1)}$ by Lemma \ref{ingster}.
    
\end{proof}

\subsubsection{Proof of Theorem \ref{thm:reductiontest}}\label{sec:proofreductiontest}

\paragraph{Overview of the proof.}
Under $P \in \mathcal{P}_0$, we prove that $T_{1, k}$ and $T_{2, k}$ have zero mean, and their variances are bounded by unity.
By rejecting $H_0$ when $|T_{1, k}| \geq \sqrt{3K / \alpha n}$ or $|T_{2, k}| \geq \sqrt{3K / \alpha n}$, we can filter out distributions $P \in \mathcal{P}_1(\ep, p, s)$ such that $\E_P [T_{1, k}] = \Omega(n^{-1 / 2})$ or $\E_P [T_{2, k}] = \Omega(n^{-1 / 2})$.
For the remaining cases, we compute $\Vert \pi_k^\mathcal{V} - \pi_k^\mathcal{W} \Vert_{L^2(P_Z)}$ and show it is lower bounded by $\Omega(n^{-2s / (4s + K - 1)})$.
We conclude using the minimax optimality of the two-sample test $\texttt{TS}$.

\begin{proof}
We prove the theorem for $p = 2$.
Then, the general case follows since $\mathcal{P}_1(\ep, p, s) \subseteq \mathcal{P}_1(\ep, 2, s)$ for all $p \leq 2$.
    Assume $P \in \mathcal{P}_0$.
    By the union bound,
    \begin{align*}
        &P(\xi_n^\textnormal{split} = 1)
        \leq \sum_{k = 1}^K \left[ P\left( |T_{1, k}| \geq \sqrt{\frac{3K}{\alpha n}} \right) + P\left( |T_{2, k}| \geq \sqrt{\frac{3K}{\alpha n}} \right) + P\left(\texttt{TS}_{\frac{\alpha}{3K}, \frac{\beta}{2}}(\mathcal{V}_k, \mathcal{W}_k) = 1 \right)\right].
    \end{align*}
    Moreover, for all $k\in \{1,\ldots,K\}$,
    $$\E_P[Y - Z]_k = \E_P[\E_P[[Y - Z]_k | Z]] = \E_P[[\E_P[Y | Z] - Z]_k ] = 0$$
    and
    $\text{Var}_P([Y - Z]_k) = \E_P[[Y - Z]_k^2] \leq 1.$
    Thus, by Chebyshev's inequality
    \begin{align*}
        P\left( |T_{1, k}| \geq \sqrt{\frac{3K}{\alpha n}} \right) &\leq \frac{\alpha n}{3K} \text{Var}_P(T_{1, k}) = \frac{\alpha }{3K} \text{Var}_P([Y - Z]_k) \leq \frac{\alpha}{3K}.
    \end{align*}
    Similarly, we have
    \begin{align*}
        P\left( |T_{2, k}| \geq \sqrt{\frac{3K}{\alpha n}} \right) &\leq \frac{\alpha n}{3K} \text{Var}_P(T_{2, k}) = \frac{\alpha}{3K} \text{Var}_P([Z]_k[Y - Z]_k) \leq \frac{\alpha}{3K}.
    \end{align*}
    From \eqref{eq:twosamplecontrol}, we know that
    $P(\texttt{TS}_{\frac{\alpha}{3K}, \frac{\beta}{2}}(\mathcal{V}_k, \mathcal{W}_k) = 1) \leq \frac{\alpha}{3K}.$
    Therefore,
    \begin{align*}
        P(\xi_n^\text{split} = 1) \leq \sum_{k = 1}^K \left( \frac{\alpha}{3K} + \frac{\alpha}{3K} + \frac{\alpha}{3K} \right) = \alpha.
    \end{align*}
    
    Let $P \in \mathcal{P}_1(\ep, p, s)$ and suppose that, for some $k \in \{1,\ldots,K\}$,
    \begin{equation}
    \label{eq:firstviolate}
    |\E_P[T_{1, k}] | = |\E_P[Y - Z]_k| \geq \frac{1}{\sqrt{n}} \left( \sqrt{\frac{3K}{\alpha}} + \frac{1}{\sqrt{\beta}} \right).
    \end{equation}
    By Chebyshev's inequality,
    \begin{align*}
        P\left(|T_{1, k} - \E_P[T_{1, k}]| \leq \frac{1}{\sqrt{\beta n}}  \right) \geq 1 - \beta n \text{Var}_P(T_{1, k}) \geq 1 - \beta.
    \end{align*}
    Note that \eqref{eq:firstviolate} and $|T_{1, k} - \E_P[T_{1, k}]| \leq 1 / \sqrt{\beta n}$ imply
    \begin{align*}
        |T_{1, k}| \geq |\E_P[T_{1, k}]| - |T_{1, k} - \E_P[T_{1, k}]| \geq \sqrt{\frac{3K}{\alpha n}}.
    \end{align*}
    Therefore,
    $$P(\xi_n^\text{split} = 1) \geq P\left(|T_{1, k}| \geq \sqrt{\frac{3K}{\alpha n}} \right) \geq  P\left(|T_{1, k} - \E_P[T_{1, k}]| \leq \frac{1}{\sqrt{\beta n}}  \right) \geq 1 - \beta.$$
    The same conclusion can be drawn when
    $$|\E_P[T_{2, k}]| = |\E_P[[Z]_k[Y - Z]_k]| \geq \frac{1}{\sqrt{n}} \left( \sqrt{\frac{3K}{\alpha}} + \frac{1}{\sqrt{\beta}} \right)$$
    for some $k \in \{1,\ldots,K\}$.
    
    Now it remains to prove the claim for $P \in H(\ep, p, s)$ such that
    \begin{align}
    \label{eq:firstorder}
        |\E_P[Y - Z]_k| \vee |\E_P[[Z]_k[Y - Z]_k]| <  \frac{1}{\sqrt{n}} \left( \sqrt{\frac{3K}{\alpha}} + \frac{1}{\sqrt{\beta}} \right)
    \end{align}
    for every $k \in \{1,\ldots,K\}$.
    Since
    $$\ell_2\text{-ECE}_P(f)^2 = \sum_{k = 1}^K \int_{\Delta_{K - 1}} [\res(\mathbf{z})]_k^2 dP_Z(\mathbf{z}) \geq \ep^2,$$
    we can choose $k_0 \in \{1,\ldots,K\}$ such that
    $\int_{\Delta_{K - 1}} [\res(\mathbf{z})]_{k_0}^2 dP_Z(\mathbf{z}) \geq \frac{\ep^2}{K}.$
    Choose $c_\text{split}'  > 0$ such that, for $d_c$ from Assumption \ref{assumption:classprob} and $c_\text{ts}$ from \eqref{eq:twosamplecontrol},
    \begin{equation}
    \label{eq:cupper}
        \frac{(c_\text{split}')^2}{K} \geq \frac{4}{d_c^3} \left(\sqrt{\frac{3K}{\alpha}} + \frac{1}{\sqrt{\beta}} \right)^2 + c_\text{ts}^2 \left( \frac{d_c}{8} \right)^{-\frac{4s}{4s + K - 1}}.
    \end{equation}
    There exists $N \in \mathbb{N}_+$ such that     for all $n \geq N$,
    \begin{equation}
    \label{eq:largencondition}
        \frac{1}{\sqrt{n}} \left( \sqrt{\frac{3K}{\alpha}} + \frac{1}{\sqrt{\beta}} \right) \leq \frac{d_c}{2}, \quad 2\left( \frac{2}{e}\right)^{\frac{d_c n}{8}} \leq \frac{\beta}{2}.
    \end{equation}
    \sloppy
    
    Let $c_\text{split} = c_\text{split}' \vee N^{2s / (4s + K - 1)}.$
    If $n < N$, then $\Alt$ is empty since $\ep \geq c_\text{split} n^{-2s/(4s + K - 1)} > 1$, so the claim is vacuously true.
    Assume $n \geq N$.
    By \eqref{eq:firstorder}, \eqref{eq:largencondition}, and Assumption \ref{assumption:classprob},
    \begin{equation}
    \label{eq:fxpositive}
        \E_P[Z]_{k_0} \geq \E_P[Y]_{k_0} - |\E_P[Y - Z]_{k_0}| \geq \frac{d_c}{2}.
    \end{equation}
    By \eqref{eq:l2distance}, \eqref{eq:totalexp}, \eqref{eq:firstorder}, and \eqref{eq:fxpositive}, $\Vert \pi_{k_0}^\mathcal{V} - \pi_{k_0}^\mathcal{W} \Vert_{L^2(P_Z)}^2$ is lower bounded by
    \fussy
    \begin{align*}
        &\frac{1}{(\E_P[Y]_{k_0})^2} \int_{\Delta_{K - 1}} [\res(\mathbf{z})]_{k_0}^2 dP_Z(\mathbf{z}) + \frac{2\E_P[Z - Y]_{k_0} \E_P[[Z]_{k_0} [Y - Z]_{k_0}]}{(\E_P[Y]_{k_0})^2 \E_P[Z]_{k_0}}\\
        &\geq \frac{\ep^2}{K} - \frac{4}{d_c^3} \left( \sqrt{\frac{3K}{\alpha}} + \frac{1}{\sqrt{\beta}} \right)^2  n^{-1}.
    \end{align*}
    Further by $\ep \geq c_\text{split}' n^{-2s / (4s + K - 1)}$ and \eqref{eq:cupper},
    \begin{align}
    \label{eq:corcompare}
        \frac{\ep^2}{K} - \frac{4}{d_c^3} \left( \sqrt{\frac{3K}{\alpha}} + \frac{1}{\sqrt{\beta}} \right)^2  n^{-1} &\geq \left(\frac{(c'_\text{split})^2}{K} - \frac{4}{d_c^3} \left( \sqrt{\frac{3K}{\alpha}} + \frac{1}{\sqrt{\beta}} \right)^2 \right) n^{-\frac{4s}{4s + K - 1}} \nonumber \\
        &\geq c_\text{ts}^2 \left( \frac{d_c n}{8} \right)^{-\frac{4s}{4s + K - 1}}.
    \end{align}
    In conclusion,
    $$\left\Vert \pi_{k_0}^\mathcal{V} - \pi_{k_0}^\mathcal{W} \right\Vert_{L^2(P_Z)} \geq c_\text{ts}  \left( \frac{d_c n}{8} \right)^{-\frac{2s}{4s + K - 1}} .$$
    Note that $\pi_{k_0}^\mathcal{V} - \pi_{k_0}^\mathcal{W}$ is $s$-H\"older since it is a linear combination of two $s$-H\"older functions $\mathbf{z} \mapsto [\res(\mathbf{z})]_{k_0}$ and $\mathbf{z} \mapsto [\mathbf{z}]_{k_0}$, possibly with different H\"older constants.
    Thus by \eqref{eq:twosamplecontrol}, we have
    \begin{equation}
    \label{eq:conditionalcontrol}
    P\left(\texttt{TS}_{\frac{\alpha}{3K}, \frac{\beta}{2}}(\mathcal{V}_{k_0}, \mathcal{W}_{k_0}) = 1 \mid |\mathcal{V}_{k_0}| = v, |\mathcal{W}_{k_0}| = w \right) \geq 1 - \frac{\beta}{2}
    \end{equation}
    given that $v, w \geq \frac{d_c n}{8}$.
    For convenience, assume that $n$ is even (if required, drop an observation).
    Since $|\mathcal{V}_{k_0}| \sim \text{Bin}(\frac{n}{2}, \E_P[Y]_{k_0})$ and $|\mathcal{W}_{k_0}| \sim \text{Bin}(\frac{n}{2}, \E_P[Z]_{k_0})$, we find
    $$P\left(|\mathcal{V}_{k_0}| < \frac{n \E_P[Y]_{k_0}}{4}\right) \leq \left(\frac{2}{e}\right)^{\frac{n \E_P[Y]_{k_0}}{4}}, \quad P\left(|\mathcal{W}_{k_0}| < \frac{n \E_P[Z]_{k_0}}{4}\right) \leq \left(\frac{2}{e}\right)^{ \frac{n \E_P[Z]_{k_0}}{4}} $$
    by Chernoff's inequality (Exercise 2.3.2 of \cite{vershynin2018high}).
    Therefore,
    \begin{align*}
        &P\left(|\mathcal{V}_{k_0}| < \frac{d_c n}{8} \text{ or } |\mathcal{W}_{k_0}| < \frac{d_c n}{8} \right) \leq P\left( |\mathcal{V}_{k_0}| < \frac{d_c n}{8} \right) + P\left( |\mathcal{W}_{k_0}| < \frac{d_c n}{8} \right)\\
        &\leq P\left( |\mathcal{V}_{k_0}| < \frac{n\E_P[Y]_{k_0}}{4}\right) + P\left( |\mathcal{W}_{k_0}| < \frac{n\E_P[Z]_{k_0}}{4}\right)\\
        &\leq \left(\frac{2}{e}\right)^{\frac{n \E_P[Y]_{k_0}}{4}} + \left(\frac{2}{e}\right)^{ \frac{n \E_P[Z]_{k_0}}{4}} \leq 2\left( \frac{2}{e}\right)^{\frac{d_c n}{8}} \leq \frac{\beta}{2},
    \end{align*}
    and thus
    \begin{equation}
    \label{eq:vwconcentration}
        P\left( |\mathcal{V}_{k_0}|, |\mathcal{W}_{k_0}| \geq \frac{d_c n}{8}\right) \geq 1 - \frac{\beta}{2}.
    \end{equation}
    The last inequality holds by \eqref{eq:cupper}.
    Finally by \eqref{eq:conditionalcontrol} and \eqref{eq:vwconcentration},
    \begin{align*}
        &P\left( \texttt{TS}_{\frac{\alpha}{3K}, \frac{\beta}{2}}(\mathcal{V}_{k_0}, \mathcal{W}_{k_0}) = 1 \right)\\
        &\geq \int_{\{(v, w) \in \mathbf{z}^2: v, w \geq \frac{d_c n}{8}\}} P\left(\texttt{TS}_{\frac{\alpha}{3K}, \frac{\beta}{2}}(\mathcal{V}_{k_0}, \mathcal{W}_{k_0}) = 1 \mid |\mathcal{V}_{k_0}| = v, |\mathcal{W}_{k_0}| = w \right) dP_{|\mathcal{V}_{k_0}|, |\mathcal{W}_{k_0}|}(v, w)\\
        &\geq \left(1 - \frac{\beta}{2} \right) P\left( |\mathcal{V}_{k_0}|, |\mathcal{W}_{k_0}| \geq \frac{d_c n}{8}\right) \geq \left(1 - \frac{\beta}{2} \right)^2 \geq 1 - \beta
    \end{align*}
    where the integral is with respect to the joint distribution of $|\mathcal{V}_{k_0}|$ and $|\mathcal{W}_{k_0}|$.
    This proves that
    $$P(\xi_n^\text{split} = 1) \geq P\left( \texttt{TS}_{\frac{\alpha}{3K}, \frac{\beta}{2}}(\mathcal{V}_{k_0}, \mathcal{W}_{k_0}) = 1 \right) \geq 1 - \beta.$$
    
\end{proof}

\subsubsection{Proof of Corollary \ref{cor:adaptivetwo}}\label{sec:proofadaptivetwo}
The proof of Theorem \ref{thm:reductiontest} can be repeated by replacing \eqref{eq:twosamplecontrol} with \eqref{eq:twoadaptivecontrol}.
The only difference is \eqref{eq:corcompare}, where we used $n^{-4s / (4s + K - 1)} \geq n^{-1}$.
In the adaptive setting, we instead need $(n / \log \log n)^{-4s / (4s + K - 1)} \geq n^{-1}$.
This inequality holds for all large $n \in \mathbb{N}_+$, say $n \geq N'$.
Then, we can define $c_\text{ad-s}$ to be larger than $(N' / \log \log N')^{2s / (4s + K - 1)}$, so that $\mathcal{P}_1(\ep, p, s)$ becomes empty when $n < N'$.

\subsection{Background}
\subsubsection{H\"older Continuity on the Probability Simplex}
\label{sec:holderdef}
To define derivatives---and thus the class of functions we study---on $\Delta_{K - 1}$, a coordinate chart $\psi: \Delta_{K - 1} \rightarrow \R^{K - 1}$ has to be specified.
For example, we can consider the canonical projection $\pi_{-k}: (z_1, \dots, z_K)^\top \mapsto (z_1, \dots, z_{k - 1}, z_{k + 1}, \dots, z_K)^\top$.
The definition of H\"older smoothness below depends on the choice of $\psi$. We assume $\psi = \pi_{-K}$, but all conclusions and proofs remain the same for any choice of the coordinate chart $\psi$.

For an integer $d \geq 1$, a vector $\boldsymbol \gamma = (\gamma_1, \dots, \gamma_{d})^\top \in \mathbb{N}^{d}$ is called a multi-index.
We write $|\boldsymbol \gamma| := \gamma_1 + \cdots + \gamma_d$.
For a sufficiently smooth function $f:\R^d \to \R$, we denote its partial derivative of order $\boldsymbol \gamma = (\gamma_1, \dots, \gamma_{d})^\top$ by  $f^{(\boldsymbol \gamma)} := \partial^{\gamma_1}_{1} \cdots \partial^{\gamma_d}_{d} f$.
For a H\"older smoothness parameter $s > 0$ and a H\"older constant $L > 0$, let $\cH_K(s,L)$ be the class of 
$(s,L)$-H\"older continuous
functions $g: \Delta_{K - 1} \to \R$ satisfying, for all $\mathbf{x}_1, \mathbf{x}_2 \in \psi(\Delta_{K - 1})$ and multi-indices $\boldsymbol \gamma \in \mathbb{N}^{K - 1}$ with $|\boldsymbol \gamma| = \lceil s\rceil - 1$,
\begin{equation}
\label{eq:holderdef}
    \left\vert (g \circ \psi^{-1})^{(\boldsymbol \gamma)}(\mathbf{x}_1) - (g \circ \psi^{-1})^{(\boldsymbol \gamma)}(\mathbf{x}_2) \right\vert \leq L \left\Vert \mathbf{x}_1 - \mathbf{x}_2 \right\Vert^{s - \lceil s \rceil + 1}.
\end{equation}
In particular,
$\cH_K(1,L)$ denotes all $L$-Lipschitz functions.
]
\subsubsection{Two-sample Goodness-of-fit Tests}
\label{sec:twobackground}
Here we state and slightly extend the results of \cite{arias2018remember, kim2020minimax}.
For $d \in \mathbb{N}_+$, let $\mu$ be a measure on $[0, 1]^d$ which is absolutely continuous with respect to $\textnormal{Leb}_d$ and satisfies
\begin{equation}
\label{eq:density}
    \nu_l \leq \frac{d\mu}{d \textnormal{Leb}_d} \leq \nu_u
\end{equation}
almost everywhere for some constants $\nu_l, \nu_u > 0$.
For $n_1, n_2 \in \mathbb{N}_+$, suppose we have two samples $\{V_1, \dots, V_{n_1}\}$ and $\{W_1, \dots, W_{n_2}\}$ i.i.d. sampled from the distributions on $[0, 1]^d$ with densities $f_1$ and $f_2$, respectively, with respect to $\mu$.
We also assume $f_1 - f_2$ is $(s, L)$-H\"older continuous for a H\"older smoothness parameter $s > 0$ and a H\"older constant $L > 0$.

For $m \in \mathbb{N}_+$ and $\mathbf{i} = (i_1, \dots, i_d)^\top \in [m]^d$, let $R_{m, \mathbf{i}} := \prod_{k = 1}^d \left[\frac{i_k - 1}{m}, \frac{i_k}{m} \right)$.
$$v_{\mathbf{i}, m} := \left|\left\{j \in [n_1]: V_j \in R_{m, \mathbf{i}} \right\}\right|,$$
$$w_{\mathbf{i}, m} := \left|\left\{j \in [n_2]: W_j \in R_{m, \mathbf{i}} \right\}\right|.$$
The unnormalized chi-squared statistic is defined by
$$\Gamma_{m, n_1, n_2} := \sum_{\mathbf{i} \in [m]^d} (n_2 v_{\mathbf{i}, m} - n_1 w_{\mathbf{i}, m})^2.$$
\begin{theorem}[Chi-squared test, \cite{arias2018remember}]
\label{thm:twosample}
Consider the two-sample goodness-of-fit testing problem described above.
Assume the H\"older smoothness parameter $s$ is known and let $m_* = \lfloor (n_1 \wedge n_2)^{2 / (4s + d)} \rfloor$.
For any $\alpha \in (0, 1)$ and $\beta \in (0, 1 - \alpha)$, there exist $c > 0$ depending on $(d, L)$ and $c_\textnormal{ts} > 0$ depending on $(s, d, L, \nu_l, \nu_u, \alpha, \beta)$ such that for $\tau := n_1 n_2 (n_1 + n_2) + cn_1 n_2 m_*^{-d / 2}$,
\begin{alignat*}{2}
    &P(\Gamma_{m_*, n_1, n_2} \geq \tau ) \leq \alpha \quad &&\text{if } f_1 = f_2,\\
    &P(\Gamma_{m_*, n_1, n_2} \geq \tau ) \geq 1 - \beta \quad &&\text{if } \left\Vert f_1 - f_2 \right\Vert_{L^2(\mu)} \geq c_\textnormal{ts}(n_1 \wedge n_2)^{-\frac{2s}{4s + d}}.
\end{alignat*}
\end{theorem}

For $m \in \mathbb{N}_+$, let $k_m: ([0, 1]^d)^4 \to \mathbb{R}$ be the kernel
\begin{align*}
    k_m(v_1, v_2, w_1, w_2) &= \sum_{\mathbf{i} \in [m]^d} \left( \mathds{1}_{R_{m, \mathbf{i}}}(v_1) \mathds{1}_{R_{m, \mathbf{i}}}(v_2) + \mathds{1}_{R_{m, \mathbf{i}}}(w_1) \mathds{1}_{R_{m, \mathbf{i}}}(w_2) \right.\\
    &\qquad\qquad - \left.\mathds{1}_{R_{m, \mathbf{i}}}(v_1) \mathds{1}_{R_{m, \mathbf{i}}}(w_2) - \mathds{1}_{R_{m, \mathbf{i}}}(w_1) \mathds{1}_{R_{m, \mathbf{i}}}(v_2) \right).
\end{align*}
For the two samples $\{V_1, \dots, V_{n_1}\}$, $\{W_1, \dots, W_{n_2}\}$ described above, define
$$U_{m, n_1, n_2} := \frac{1}{n_1(n_1 - 1)n_2(n_2 - 1)} \sum_{i_1 \neq i_2 \in [n_1]} \sum_{j_1 \neq j_2 \in [n_2]} k_m(V_{i_1}, V_{i_2}, W_{j_1}, W_{j_2}).$$
For any $\alpha \in (0, 1)$, the $1 - \alpha$ quantile $c_{1 - \alpha, m, n_1, n_2}$ of the U-statistic $U_{m, n_1, n_2}$ can be found by the permutation procedure described in Section 2.1 of \cite{kim2020minimax}.

\sloppy
\begin{theorem}[Permutation test, \cite{kim2020minimax}]
\label{thm:permutation}
Consider the two-sample goodness-of-fit testing described above.
Assume the H\"older smoothness parameter $s$ is known and let $m_* = \lfloor (n_1 \wedge n_2)^{2 / (4s + d)} \rfloor$.
For any $\alpha \in (0, 1)$ and $\beta \in (0, 1 - \alpha)$, there exists $c_\textnormal{ts} > 0$ depending on $(s, d, L, \nu_l, \nu_u, \alpha, \beta)$ such that
\begin{alignat*}{2}
    &P(U_{m_*, n_1, n_2} \geq c_{1 - \alpha, m_*, n_1, n_2} ) \leq \alpha \quad &&\text{if } f_1 = f_2,\\
    &P(U_{m_*, n_1, n_2} \geq c_{1 - \alpha, m_*, n_1, n_2} ) \geq 1 - \beta \quad &&\text{if } \left\Vert f_1 - f_2 \right\Vert_{L^2(\mu)} \geq c_\textnormal{ts}(n_1 \wedge n_2)^{-\frac{2s}{4s + d}}.
\end{alignat*}
\end{theorem}

\fussy

\begin{corollary}[Multi-scale permutation test, \cite{kim2020minimax}]
\label{cor:permutation}
Consider the two-sample goodness-of-fit testing problem described above.
Let $B = \lceil \frac{2}{d} \log_2(\frac{n_1 \wedge n_2}{\log\log (n_1 \wedge n_2)}) \rceil$ and define
$$\xi_{n_1, n_2}^\textnormal{perm} := \max_{b \in \{1,\ldots,B\}} I(U_{2^b, n_1, n_2} \geq c_{1 - \alpha / B, 2^b, n_1, n_2}).$$
For any $\alpha \in (0, 1)$ and $\beta \in (0, 1 - \alpha)$, there exists $c_\textnormal{ad} > 0$ depending on $(d, L, \nu_l, \nu_u, \alpha, \beta)$ such that
\begin{alignat*}{2}
    &P(\xi_{n_1, n_2}^\textnormal{perm} = 1) \leq \alpha \quad &&\text{if } f_1 = f_2,\\
    &P(\xi_{n_1, n_2}^\textnormal{perm} = 1 ) \geq 1 - \beta \quad &&\text{if } \left\Vert f_1 - f_2 \right\Vert_{L^2(\mu)} \geq c_\textnormal{ad}\left(\frac{n_1 \wedge n_2}{\log\log (n_1 \wedge n_2)} \right)^{-\frac{2s}{4s + d}}.
\end{alignat*}

\end{corollary}

\begin{remark}[Comment on the proofs]
Theorem \ref{thm:twosample}, Theorem \ref{thm:permutation}, and Corollary \ref{cor:permutation} are generalization of Theorem 4 of \cite{arias2018remember}, Proposition 4.6 of \cite{kim2020minimax}, and Proposition 7.1 of \cite{kim2020minimax}, respectively.
The original statements are for $\mu = \textnormal{Leb}_d$. 
The proofs in \cite{arias2018remember, kim2020minimax} can be adapted with only minor differences.
For example, equation (93) of \cite{arias2018remember} is still true assuming \eqref{eq:density} above.
Also, we proved Lemma \ref{discretel2multi} in this paper, which is a generalization of Lemma 3 of \cite{arias2018remember} to a general measure $\mu$.
Other parts of the proofs in \cite{arias2018remember, kim2020minimax} can be repeated without any modification.
\end{remark}

\subsubsection{Binning Scheme for the Probability Simplex}
\label{multibinning}
Here we describe a binning scheme for the probability simplex $\Delta_{K - 1} \subseteq \R^{K}$ into equal volume  simplices---i.e., affine transforms of the standard probability simplex.

\paragraph{Hypersimplex.}
For $u \geq 2$ and $v \in [u - 1]$, define the $(u - 1)$-dimensional polytope hypersimplex $\Delta_{u-1, v}$---a generalization of the standard probability simplex that can have more vertices and edges---as
$$\Delta_{u-1, v} := \left\{ (x_1, \dots, x_{u})^\top \in [0, 1)^u: x_1 + \cdots + x_{u} = v \right\}.$$
Let $A_{u - 1, v - 1}$ be the Eulerian number:
$$A_{u - 1, v - 1} := \sum_{i = 0}^{v} (-1)^i \binom{u}{i} (v - i)^{u - 1}.$$
It is known that the hypersimplex $\Delta_{u-1, v}$ can be partitioned into $A_{u - 1, v - 1}$ simplices \citep{stanley1977eulerian, sturmfels1996grobner} whose volumes are identical to that of the unit probability simplex $\Delta_{u - 1}$.

\paragraph{Construction.}
Let $m \in \mathbb{N}_+$ and $R_{m, \mathbf{i}} := \prod_{k = 1}^K \left[\frac{i_k - 1}{m}, \frac{i_k}{m} \right)$ for $\mathbf{i} = (i_1, \dots, i_K)^\top \in [m]^{K}$.
The hypercube $R_{m, \mathbf{i}}$ has a positive-volume intersection with $\Delta_{K - 1}$ when 
\beq
\label{hypersimplexcondition}
m + 1 \leq \sum_{k = 1}^K i_k \leq m + K - 1.
\eeq
Suppose that \eqref{hypersimplexcondition} holds and we write $|{\mathbf{i}}| = \sum_{k = 1}^K i_k$.
Then, the intersection is 
\begin{align*}
    \Delta_{K - 1} \cap R_{m, \mathbf{i}} &= \left\{(x_1, \dots, x_K)^\top \in R_{m, \mathbf{i}}: x_1 + \cdots + x_K = 1  \right\}\\
    &= \frac{\mathbf{i} - \mathbf{1}_K}{m} + \frac{1}{m} \left\{ (x_1, \dots, x_K)^\top \in [0, 1)^{K}: x_1 + \cdots + x_K = m + K - |\mathbf{i}|\right\},
\end{align*}
which is a $\frac{1}{m}$-scaled and translated version of the hypersimplex $\Delta_{K-1, m + K - |\mathbf{i}|}$.
Recall that this hypersimplex can be further partitioned into $A_{K - 1, m + K - |\mathbf{i}| - 1}$ simplices with the volume $\text{vol}(\Delta_{K- 1}) / m^{K- 1}$. Let $\mathcal{S}_\mathbf{i}$ be the set of such $A_{K - 1, m + K - |\mathbf{i}| - 1}$ simplices.
Now, we define 
\begin{equation*}\label{eqn:partition1}
    \mathcal{B}_m := \bigcup_{\substack{\mathbf{i} \in [m]^K \\ m + 1 \leq |\mathbf{i}| \leq m + K - 1}} \left\{ \frac{\mathbf{i} - \mathbf{1}_K}{m} + \frac{1}{m} \Delta : \Delta \in \mathcal{S}_\mathbf{i} \right\},
\end{equation*}
\sloppy
which is the collection of all simplices obtained from decomposing $\Delta_{K - 1} \cap R_{m, \mathbf{i}}$ for each $\mathbf{i} \in [m]^{K}$ satisfying \eqref{hypersimplexcondition}.
Since each simplex in $\mathcal{B}_m$ has a volume $\text{vol}(\Delta_{K - 1}) / m^{K - 1}$, it follows that $|\mathcal{B}_m| = m^{K - 1}$.
Noting that there are $\binom{j - 1}{K - 1}$ multi-indices $\mathbf{i} \in [m]^K$ with $|\mathbf{i}| = j$, it also directly follows from Worpitzky's identity (Equation 6.37 in \cite{10.5555/562056}) that
\begin{align*}
    |\mathcal{B}_m| &= \sum_{j = m + 1}^{m + K - 1} \binom{j - 1}{K - 1} A_{K - 1, m + K - j - 1} = \sum_{j = 0}^{K - 2} \binom{m + j}{K - 1} A_{K - 1, K - j - 2}\\
    &= \sum_{j = 0}^{K - 2} \binom{m + j}{K - 1} A_{K - 1, j} = m^{K - 1}.
\end{align*}
We finally index the partition as $\mathcal{B}_m = \{B_1, \dots, B_{m^{K - 1}}\}.$
\fussy

\subsubsection{Calibration Test for Discrete Predictions}
\label{sec:discretetest}
Testing calibration of discrete probability predictions has been studied in \cite{cox1958two, miller1962statistical, harrell2015regression}.
Here we describe a test based on testing multiple binomial parameters.
This test is related to the chi-squared test from \cite{miller1962statistical}, but does not use the asymptotic distribution of the test statistic when choosing critical values.
Let $\{v_1, \dots, v_t\} \subseteq [0, 1]$ be the range of a discrete-valued probability predictor $f$.
For each $i \in [t]$, we let $N_i = |\{j \in [n]: f(X_j) = v_i\}|$, $M_i = |\{j \in [n]: f(X_j) =v_i, Y_j = 1\}|$, and $p_i = P(Y = 1 \mid f(X) = v_i)$.
In this setting, the random variable $M_i$, under the null hypothesis of perfect calibration, follows the binomial distribution $\text{Binom}(N_i, p_i)$ with $N_i$ trials and success probability $p_i$, given $N_i$.
We use an exact binomial test to test the null hypothesis $H_{0, i}: p_i = v_i$ for each $i \in [t]$ and apply the Bonferroni correction to control the false detection rate under the null hypothesis $H_0 = \cap_{i = 1}^t H_{0, i}$.

\subsection{More Experiments}
\subsubsection{Debiased ECE from Kumar et al. (2019)}\label{sec:more_experiments}
In this subsection, we provide additional experiments, to evaluate calibration methods whose outputs range over a finite set.
When the predicted probabilities belong to a finite set, \cite{kumar2019verified} propose a debiased version of the empirical $\ell_2\text{-ECE}$, whose sample complexity required is smaller than that of the plug-in estimator; see Section \ref{sec:biasfailure} for more discussion. 
We calculate \cite{kumar2019verified}'s debiased estimator along with calibration testing results.
We use models trained on CIFAR-10 (Table \ref{tab:cifar-10-2}), CIFAR-100 (Table \ref{tab:cifar-100-2}), and ImageNet (Table \ref{tab:imagenet-2}). The values of the debiased $\ell_2\text{-ECE}$ estimator are typically very small, which is consistent with the  experimental results in \cite{kumar2019verified}. As can be observed, there is no clear relation between the debiased empirical $\ell_2\text{-ECE}$ values and test results.

\begin{table}[ht]
\centering
    \begin{tabular}{lcccccc}
    \toprule
\multirow{2}{*}{} & \multicolumn{2}{c}{DenseNet 121} & \multicolumn{2}{c}{ResNet 50} & \multicolumn{2}{c}{VGG-19} \\
                  & $\widehat{\ell_2\text{-ECE}^{\mathrm{db}}}$        & Calibrated?        & $\widehat{\ell_2\text{-ECE}^{\mathrm{db}}}$      & Calibrated?       & $\widehat{\ell_2\text{-ECE}^{\mathrm{db}}}$      & Calibrated?     \\\midrule
Hist. Bin.                 & 0.02\%            & reject                   & 0.02\%         & reject                  & 0.05\%         & reject  \\
Scal. Bin.                & 0.11\%            & reject                   & 0.10\%          & reject                  & 0.20\%         &    reject\\
\bottomrule          
\end{tabular}
\caption{The values of the debiased empirical $\ell_2\text{-ECE}$ \citep{kumar2019verified} and the testing results, via multiple binomial testing, of models trained on CIFAR-10 with two discrete calibration methods.}
\label{tab:cifar-10-2}
\end{table}

\begin{table}[ht]
\centering
    \begin{tabular}{lcccccc}
    \toprule
\multirow{2}{*}{} & \multicolumn{2}{c}{MobileNet-v2} & \multicolumn{2}{c}{ResNet 56} & \multicolumn{2}{c}{ShuffleNet-v2} \\
                  & $\widehat{\ell_2\text{-ECE}^{\mathrm{db}}}$        & Calibrated?        & $\widehat{\ell_2\text{-ECE}^{\mathrm{db}}}$      & Calibrated?       & $\widehat{\ell_2\text{-ECE}^{\mathrm{db}}}$      & Calibrated?     \\\midrule
Hist. Bin.                 & 0.04\%            & reject                   & 0.09\%         & reject                  & 0.15\%         & reject  \\
Scal. Bin.               & 0.04\%            & reject                   & 0.03\%          & reject                  & 0.02\%         & accept   \\
\bottomrule          
\end{tabular}
\caption{The values of the debiased empirical $\ell_2\text{-ECE}$ \citep{kumar2019verified} and the testing results, via multiple binomial testing, of models trained on CIFAR-100 with two discrete calibration methods.}
\label{tab:cifar-100-2}
\end{table}

\begin{table}[H]
\centering
    \begin{tabular}{lcccccc}
    \toprule
\multirow{2}{*}{} & \multicolumn{2}{c}{DenseNet 161} & \multicolumn{2}{c}{ResNet 152} & \multicolumn{2}{c}{EfficientNet-b7} \\
                  & $\widehat{\ell_2\text{-ECE}^{\mathrm{db}}}$        & Calibrated?        & $\widehat{\ell_2\text{-ECE}^{\mathrm{db}}}$      & Calibrated?       & $\widehat{\ell_2\text{-ECE}^{\mathrm{db}}}$      & Calibrated?     \\\midrule
Hist. Bin.                 & 0.01\%           & reject                   & 0.03\%         & reject                  & 0.02\%         & reject   \\
Scal. Bin.                & 0.05\%            & reject                   & 0.03\%          & reject                  & 0.06\%         & reject   \\
\bottomrule          
\end{tabular}
\caption{The values of the debiased empirical $\ell_2\text{-ECE}$ \citep{kumar2019verified} and the testing results, via multiple binomial testing, of models trained on ImageNet with two discrete calibration methods.}
\label{tab:imagenet-2}
\end{table}

\subsubsection{Estimation of ECE via DPE}\label{sec:estimation}
\begin{figure}[!ht]
    \centering
    \includegraphics{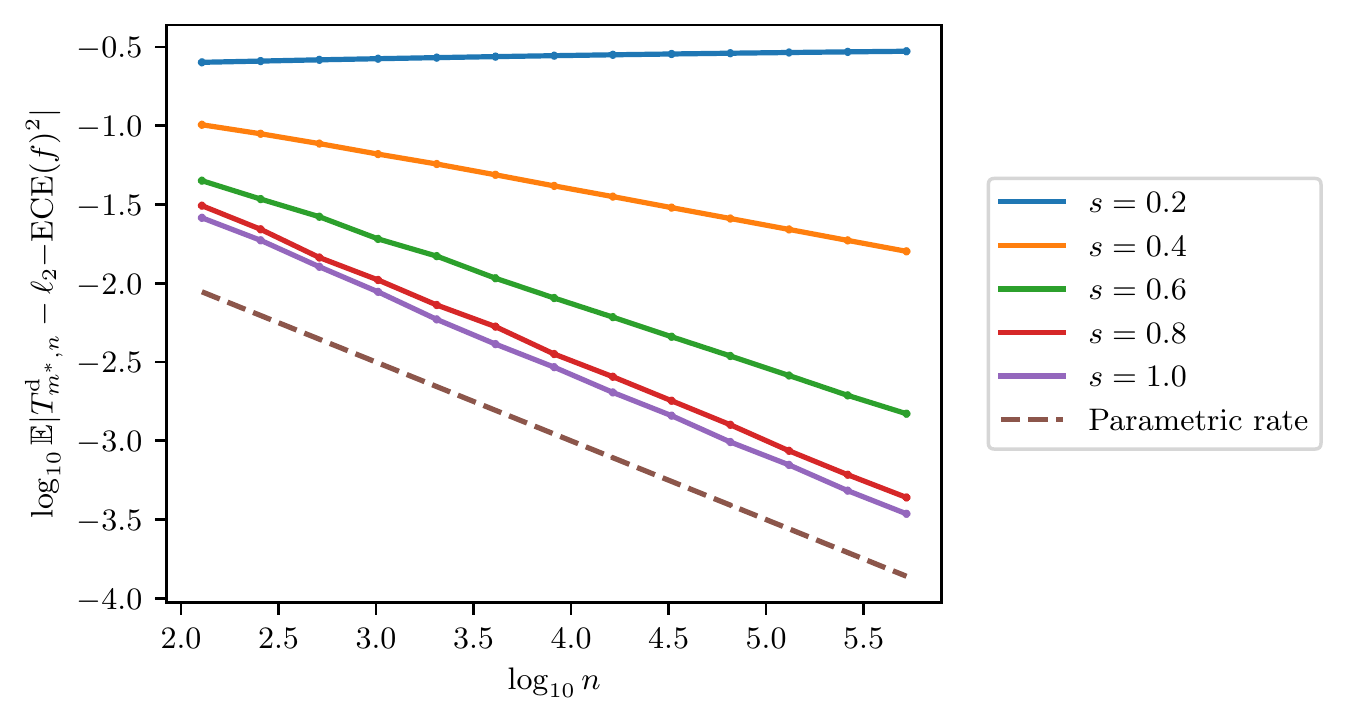}
    \caption{Estimation error of $T_{m_*, n}^\text{d}$, varying $n$, $s$ over a log scale. 
    The dashed line has slope $-\frac{1}{2}$, which corresponds to the parametric convergence rate.
    The lines for $s = 0.4$ and $0.6$ fail to achieve the parametric convergence rate, even though $s$ is larger than the threshold value 0.25 where the parametric rate arises in standard nonparametric estimation.}
    \label{fig:estimation}
\end{figure}

The DPE $T_{m, n}^{\text{d}}$ is used as a test statistic in the main text, but it can also be interpreted as an estimate of $\ell_2\text{-ECE}(f)^2$.
While it builds on ideas from nonparametric functional estimation (Remark \ref{rem:ustatistic}), the convergence rate $n^{-[4s/(4s + K - 1) \wedge 1/ 2]}$ does not directly follow here, since our construction involves an additional estimation step $P_Z(B_i) \approx |\mathcal{I}_{m, i}| / n$.
If we use the same binning parameter $m = m_* \asymp n^{2 / (4s + K - 1)}$, then the ratio between the standard deviation and the mean of the estimate $|\mathcal{I}_{m, i}| / n$ is 
$$\sqrt{(1 - P_Z(B_i)) / n P_Z(B_i)} \asymp n^{(K - 1 -4s) / [2(4s + K - 1)]},$$
which is large for small $s$.
For this reason, the DPE fails to achieve the parametric convergence rate when $s \geq \frac{K - 1}{4}$.
We experimentally support this claim in Figure \ref{fig:estimation}, where we use $K = 2$ and $Z \sim \text{Unif}([0, 1])$, with a deterministic choice of $Y = 1$.
Optimal estimation of $\ell_2\text{-ECE}(f)^2$ remains an open problem.

\subsubsection{Comparison of Critical Values}\label{sec:critcompare}
We compare three choices of critical values, where we sample $\{(\tilde Z_i, \tilde Y_i)\}_{i = 1}^n$ according to the following rules.
\begin{enumerate}
    \item Oracle Monte Carlo: $\tilde Z_i \sim P_Z$, $\tilde Y_i \sim \text{Cat}(\tilde{Z}_i)$.

    \item Full bootstrapping (consistency resampling): $\tilde Z_i \sim \text{Unif}(\{Z_i\}_{i = 1}^n)$, $\tilde Y_i \sim \text{Cat}(\tilde {Z}_i)$.

    \item $Y$-only bootstrapping: $\tilde{Z}_i = Z_i$, $\tilde Y_i \sim \text{Cat}(\tilde {Z}_i)$.
\end{enumerate}
Here, $\{(Z_i, Y_i)\}_{i = 1}^n$ is the original calibration sample, and all sampling is done independently.
We repeat the above sampling procedures $N$ times to create $\{(\tilde Z_i^{j}, \tilde Y_i^{j})\}_{i = 1}^n$, $j \in [N]$ and compute the DPE $T_{m_*, n}^{\text{d}, j}$ test statistics for each $j \in [N]$.
Denote their order statistics by $T_{m_*, n}^{\text{d}, (1)} \leq \cdots \leq T_{m_*, n}^{\text{d}, (N)}$.
For oracle Monte Carlo and $Y$-only bootstrapping, the DPEs $T_{m_*, n}^\text{d}, T_{m_*, n}^{\text{d}, 1}, \dots, T_{m_*, n}^{\text{d}, N}$ are exchangeable under the null, so (assuming ties happen with zero probability)
$$P(T_{m_*, n}^\text{d} \geq T_{m_*, n}^{\text{d}, (j)}) = \frac{N + 1 - j}{N + 1},$$
for any $j \in [N]$.
In other words, the test 
that rejects when
$I(T_{m_*, n}^\text{d} \geq T_{m_*, n}^{\text{d}, (j)})$ has type I error $\frac{j}{N + 1}$.
We do not have such a guarantee for consistency resampling.

\begin{figure}
    \centering
    \includegraphics{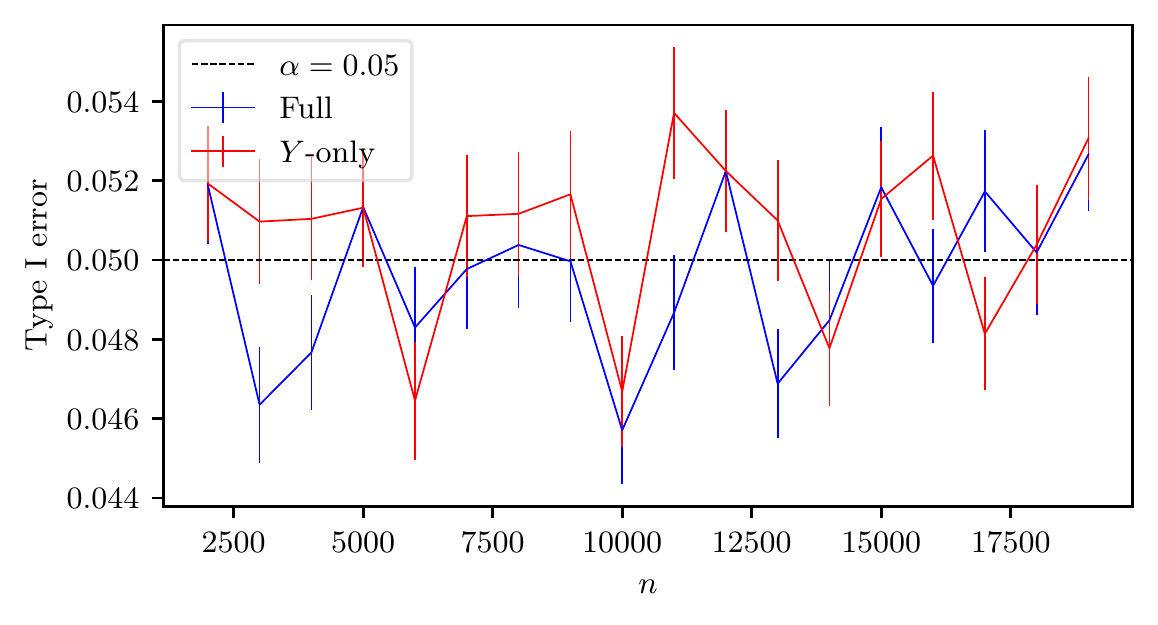}
    \caption{Type I error of the full/$Y$-only bootstrapping. We use $N = 19, \alpha = 0.05$ and compute the type I error from 10,000 oracle Monte Carlo trials. Standard error bars are plotted over 1,000 repetitions.}
    \label{fig:critval}
\end{figure}

Figure \ref{fig:critval} compares the type I errors 
when using 
critical values based on full/$Y$-only bootstrapping.
We see that there is no significant difference between the two, and they stay relatively close to the nominal level $\alpha = 0.05$.
Since consistency resampling and $Y$-only bootstrapping give randomized critical values, the true type I error is estimated by the sample mean of type I errors obtained from 10,000 independent trials.

\subsubsection{Cross-fitting}
The idea of cross-fitting \citep[see e.g.,][for related ideas]{Hajek1962, schick1986, newey2018cross, kennedy2020optimal} can be applied to the sample splitting test in Section \ref{sec:upperbound}.
Specifically, we compute the two-sample test statistic again by swapping the role of $\{(Z_i, Y_i)\}_{i = 1}^{\lfloor n / 2 \rfloor}$ and $\{(Z_i, Y_i)\}_{i = \lfloor n / 2 \rfloor + 1}^{n}$, and use the average of two test statistics.
The critical value for $\alpha = 0.05$ and the corresponding type II error are estimated via 1,000 oracle Monte Carlo simulations.
In Figure \ref{fig:crossfit}, we compare the type II error of the cross-fitted test statistics with other tests discussed in the main text.
The experimental setup is identical to Figure \ref{fig:type2error}.
We observe that the cross-fitting procedure does not significantly improve the power.

\begin{figure}
    \centering
    \captionsetup[subfigure]{justification=centering}
    \begin{subfigure}{0.32\textwidth}
        \centering
        \includegraphics{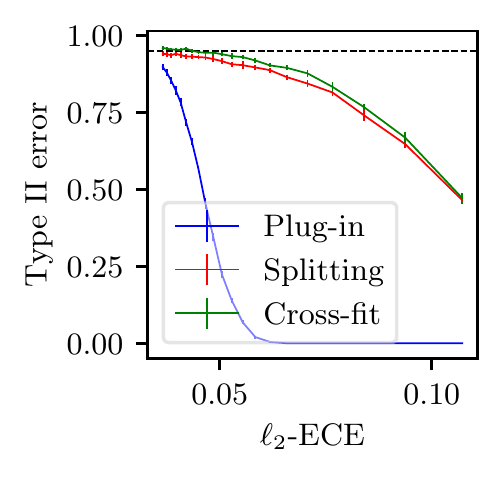}
        \caption{$s = 0.6$, $\rho = 100$}
        \label{fig:crossfit100}
    \end{subfigure}
    \begin{subfigure}{0.32\textwidth}
        \centering
        \includegraphics{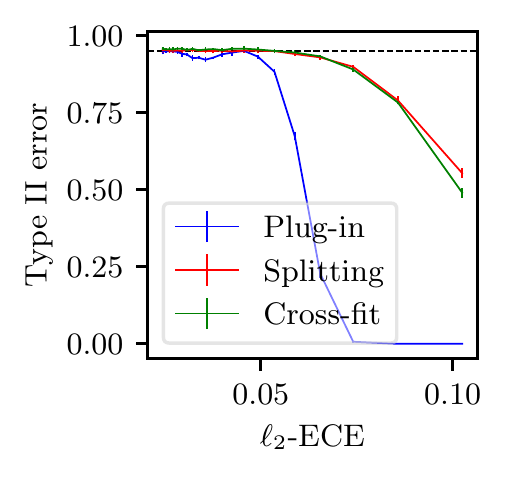}
        \caption{$s = 0.8$, $\rho = 200$}
        \label{fig:crossfit200}
    \end{subfigure}
    \caption{Type II error comparison for plug-in, sample-splitting, and cross-fitting test. The horizontal dashed line indicates a type II error of $1 - \alpha = 0.95$. Standard error bars are plotted over 10  repetitions.}
    \label{fig:crossfit}
\end{figure}

\newpage
{\small
\setlength{\bibsep}{0.2pt plus 0.3ex}
\bibliographystyle{plainnat-abbrev}
\bibliography{reference,reference_calib}
}

\end{document}

%% file: images/altdist.pgf
\begingroup%
\makeatletter%
\begin{pgfpicture}%
\pgfpathrectangle{\pgfpointorigin}{\pgfqpoint{2.045864in}{1.872389in}}%
\pgfusepath{use as bounding box, clip}%
\begin{pgfscope}%
\pgfsetbuttcap%
\pgfsetmiterjoin%
\definecolor{currentfill}{rgb}{1.000000,1.000000,1.000000}%
\pgfsetfillcolor{currentfill}%
\pgfsetlinewidth{0.000000pt}%
\definecolor{currentstroke}{rgb}{1.000000,1.000000,1.000000}%
\pgfsetstrokecolor{currentstroke}%
\pgfsetdash{}{0pt}%
\pgfpathmoveto{\pgfqpoint{0.000000in}{0.000000in}}%
\pgfpathlineto{\pgfqpoint{2.045864in}{0.000000in}}%
\pgfpathlineto{\pgfqpoint{2.045864in}{1.872389in}}%
\pgfpathlineto{\pgfqpoint{0.000000in}{1.872389in}}%
\pgfpathclose%
\pgfusepath{fill}%
\end{pgfscope}%
\begin{pgfscope}%
\pgfsetbuttcap%
\pgfsetmiterjoin%
\definecolor{currentfill}{rgb}{1.000000,1.000000,1.000000}%
\pgfsetfillcolor{currentfill}%
\pgfsetlinewidth{0.000000pt}%
\definecolor{currentstroke}{rgb}{0.000000,0.000000,0.000000}%
\pgfsetstrokecolor{currentstroke}%
\pgfsetstrokeopacity{0.000000}%
\pgfsetdash{}{0pt}%
\pgfpathmoveto{\pgfqpoint{0.606171in}{0.488889in}}%
\pgfpathlineto{\pgfqpoint{1.923671in}{0.488889in}}%
\pgfpathlineto{\pgfqpoint{1.923671in}{1.772389in}}%
\pgfpathlineto{\pgfqpoint{0.606171in}{1.772389in}}%
\pgfpathclose%
\pgfusepath{fill}%
\end{pgfscope}%
\begin{pgfscope}%
\pgfsetbuttcap%
\pgfsetroundjoin%
\definecolor{currentfill}{rgb}{0.000000,0.000000,0.000000}%
\pgfsetfillcolor{currentfill}%
\pgfsetlinewidth{0.803000pt}%
\definecolor{currentstroke}{rgb}{0.000000,0.000000,0.000000}%
\pgfsetstrokecolor{currentstroke}%
\pgfsetdash{}{0pt}%
\pgfsys@defobject{currentmarker}{\pgfqpoint{0.000000in}{-0.048611in}}{\pgfqpoint{0.000000in}{0.000000in}}{%
\pgfpathmoveto{\pgfqpoint{0.000000in}{0.000000in}}%
\pgfpathlineto{\pgfqpoint{0.000000in}{-0.048611in}}%
\pgfusepath{stroke,fill}%
}%
\begin{pgfscope}%
\pgfsys@transformshift{0.666058in}{0.488889in}%
\pgfsys@useobject{currentmarker}{}%
\end{pgfscope}%
\end{pgfscope}%
\begin{pgfscope}%
\definecolor{textcolor}{rgb}{0.000000,0.000000,0.000000}%
\pgfsetstrokecolor{textcolor}%
\pgfsetfillcolor{textcolor}%
\pgftext[x=0.666058in,y=0.391667in,,top]{\color{textcolor}\rmfamily\fontsize{9.000000}{10.800000}\selectfont \(\displaystyle {0.0}\)}%
\end{pgfscope}%
\begin{pgfscope}%
\pgfsetbuttcap%
\pgfsetroundjoin%
\definecolor{currentfill}{rgb}{0.000000,0.000000,0.000000}%
\pgfsetfillcolor{currentfill}%
\pgfsetlinewidth{0.803000pt}%
\definecolor{currentstroke}{rgb}{0.000000,0.000000,0.000000}%
\pgfsetstrokecolor{currentstroke}%
\pgfsetdash{}{0pt}%
\pgfsys@defobject{currentmarker}{\pgfqpoint{0.000000in}{-0.048611in}}{\pgfqpoint{0.000000in}{0.000000in}}{%
\pgfpathmoveto{\pgfqpoint{0.000000in}{0.000000in}}%
\pgfpathlineto{\pgfqpoint{0.000000in}{-0.048611in}}%
\pgfusepath{stroke,fill}%
}%
\begin{pgfscope}%
\pgfsys@transformshift{1.264921in}{0.488889in}%
\pgfsys@useobject{currentmarker}{}%
\end{pgfscope}%
\end{pgfscope}%
\begin{pgfscope}%
\definecolor{textcolor}{rgb}{0.000000,0.000000,0.000000}%
\pgfsetstrokecolor{textcolor}%
\pgfsetfillcolor{textcolor}%
\pgftext[x=1.264921in,y=0.391667in,,top]{\color{textcolor}\rmfamily\fontsize{9.000000}{10.800000}\selectfont \(\displaystyle {0.5}\)}%
\end{pgfscope}%
\begin{pgfscope}%
\pgfsetbuttcap%
\pgfsetroundjoin%
\definecolor{currentfill}{rgb}{0.000000,0.000000,0.000000}%
\pgfsetfillcolor{currentfill}%
\pgfsetlinewidth{0.803000pt}%
\definecolor{currentstroke}{rgb}{0.000000,0.000000,0.000000}%
\pgfsetstrokecolor{currentstroke}%
\pgfsetdash{}{0pt}%
\pgfsys@defobject{currentmarker}{\pgfqpoint{0.000000in}{-0.048611in}}{\pgfqpoint{0.000000in}{0.000000in}}{%
\pgfpathmoveto{\pgfqpoint{0.000000in}{0.000000in}}%
\pgfpathlineto{\pgfqpoint{0.000000in}{-0.048611in}}%
\pgfusepath{stroke,fill}%
}%
\begin{pgfscope}%
\pgfsys@transformshift{1.863785in}{0.488889in}%
\pgfsys@useobject{currentmarker}{}%
\end{pgfscope}%
\end{pgfscope}%
\begin{pgfscope}%
\definecolor{textcolor}{rgb}{0.000000,0.000000,0.000000}%
\pgfsetstrokecolor{textcolor}%
\pgfsetfillcolor{textcolor}%
\pgftext[x=1.863785in,y=0.391667in,,top]{\color{textcolor}\rmfamily\fontsize{9.000000}{10.800000}\selectfont \(\displaystyle {1.0}\)}%
\end{pgfscope}%
\begin{pgfscope}%
\definecolor{textcolor}{rgb}{0.000000,0.000000,0.000000}%
\pgfsetstrokecolor{textcolor}%
\pgfsetfillcolor{textcolor}%
\pgftext[x=1.264921in,y=0.225000in,,top]{\color{textcolor}\rmfamily\fontsize{9.000000}{10.800000}\selectfont \(\displaystyle f(X)\)}%
\end{pgfscope}%
\begin{pgfscope}%
\pgfsetbuttcap%
\pgfsetroundjoin%
\definecolor{currentfill}{rgb}{0.000000,0.000000,0.000000}%
\pgfsetfillcolor{currentfill}%
\pgfsetlinewidth{0.803000pt}%
\definecolor{currentstroke}{rgb}{0.000000,0.000000,0.000000}%
\pgfsetstrokecolor{currentstroke}%
\pgfsetdash{}{0pt}%
\pgfsys@defobject{currentmarker}{\pgfqpoint{-0.048611in}{0.000000in}}{\pgfqpoint{0.000000in}{0.000000in}}{%
\pgfpathmoveto{\pgfqpoint{0.000000in}{0.000000in}}%
\pgfpathlineto{\pgfqpoint{-0.048611in}{0.000000in}}%
\pgfusepath{stroke,fill}%
}%
\begin{pgfscope}%
\pgfsys@transformshift{0.606171in}{0.547230in}%
\pgfsys@useobject{currentmarker}{}%
\end{pgfscope}%
\end{pgfscope}%
\begin{pgfscope}%
\definecolor{textcolor}{rgb}{0.000000,0.000000,0.000000}%
\pgfsetstrokecolor{textcolor}%
\pgfsetfillcolor{textcolor}%
\pgftext[x=0.280556in, y=0.503827in, left, base]{\color{textcolor}\rmfamily\fontsize{9.000000}{10.800000}\selectfont \(\displaystyle {0.00}\)}%
\end{pgfscope}%
\begin{pgfscope}%
\pgfsetbuttcap%
\pgfsetroundjoin%
\definecolor{currentfill}{rgb}{0.000000,0.000000,0.000000}%
\pgfsetfillcolor{currentfill}%
\pgfsetlinewidth{0.803000pt}%
\definecolor{currentstroke}{rgb}{0.000000,0.000000,0.000000}%
\pgfsetstrokecolor{currentstroke}%
\pgfsetdash{}{0pt}%
\pgfsys@defobject{currentmarker}{\pgfqpoint{-0.048611in}{0.000000in}}{\pgfqpoint{0.000000in}{0.000000in}}{%
\pgfpathmoveto{\pgfqpoint{0.000000in}{0.000000in}}%
\pgfpathlineto{\pgfqpoint{-0.048611in}{0.000000in}}%
\pgfusepath{stroke,fill}%
}%
\begin{pgfscope}%
\pgfsys@transformshift{0.606171in}{0.838934in}%
\pgfsys@useobject{currentmarker}{}%
\end{pgfscope}%
\end{pgfscope}%
\begin{pgfscope}%
\definecolor{textcolor}{rgb}{0.000000,0.000000,0.000000}%
\pgfsetstrokecolor{textcolor}%
\pgfsetfillcolor{textcolor}%
\pgftext[x=0.280556in, y=0.795532in, left, base]{\color{textcolor}\rmfamily\fontsize{9.000000}{10.800000}\selectfont \(\displaystyle {0.25}\)}%
\end{pgfscope}%
\begin{pgfscope}%
\pgfsetbuttcap%
\pgfsetroundjoin%
\definecolor{currentfill}{rgb}{0.000000,0.000000,0.000000}%
\pgfsetfillcolor{currentfill}%
\pgfsetlinewidth{0.803000pt}%
\definecolor{currentstroke}{rgb}{0.000000,0.000000,0.000000}%
\pgfsetstrokecolor{currentstroke}%
\pgfsetdash{}{0pt}%
\pgfsys@defobject{currentmarker}{\pgfqpoint{-0.048611in}{0.000000in}}{\pgfqpoint{0.000000in}{0.000000in}}{%
\pgfpathmoveto{\pgfqpoint{0.000000in}{0.000000in}}%
\pgfpathlineto{\pgfqpoint{-0.048611in}{0.000000in}}%
\pgfusepath{stroke,fill}%
}%
\begin{pgfscope}%
\pgfsys@transformshift{0.606171in}{1.130639in}%
\pgfsys@useobject{currentmarker}{}%
\end{pgfscope}%
\end{pgfscope}%
\begin{pgfscope}%
\definecolor{textcolor}{rgb}{0.000000,0.000000,0.000000}%
\pgfsetstrokecolor{textcolor}%
\pgfsetfillcolor{textcolor}%
\pgftext[x=0.280556in, y=1.087236in, left, base]{\color{textcolor}\rmfamily\fontsize{9.000000}{10.800000}\selectfont \(\displaystyle {0.50}\)}%
\end{pgfscope}%
\begin{pgfscope}%
\pgfsetbuttcap%
\pgfsetroundjoin%
\definecolor{currentfill}{rgb}{0.000000,0.000000,0.000000}%
\pgfsetfillcolor{currentfill}%
\pgfsetlinewidth{0.803000pt}%
\definecolor{currentstroke}{rgb}{0.000000,0.000000,0.000000}%
\pgfsetstrokecolor{currentstroke}%
\pgfsetdash{}{0pt}%
\pgfsys@defobject{currentmarker}{\pgfqpoint{-0.048611in}{0.000000in}}{\pgfqpoint{0.000000in}{0.000000in}}{%
\pgfpathmoveto{\pgfqpoint{0.000000in}{0.000000in}}%
\pgfpathlineto{\pgfqpoint{-0.048611in}{0.000000in}}%
\pgfusepath{stroke,fill}%
}%
\begin{pgfscope}%
\pgfsys@transformshift{0.606171in}{1.422343in}%
\pgfsys@useobject{currentmarker}{}%
\end{pgfscope}%
\end{pgfscope}%
\begin{pgfscope}%
\definecolor{textcolor}{rgb}{0.000000,0.000000,0.000000}%
\pgfsetstrokecolor{textcolor}%
\pgfsetfillcolor{textcolor}%
\pgftext[x=0.280556in, y=1.378941in, left, base]{\color{textcolor}\rmfamily\fontsize{9.000000}{10.800000}\selectfont \(\displaystyle {0.75}\)}%
\end{pgfscope}%
\begin{pgfscope}%
\pgfsetbuttcap%
\pgfsetroundjoin%
\definecolor{currentfill}{rgb}{0.000000,0.000000,0.000000}%
\pgfsetfillcolor{currentfill}%
\pgfsetlinewidth{0.803000pt}%
\definecolor{currentstroke}{rgb}{0.000000,0.000000,0.000000}%
\pgfsetstrokecolor{currentstroke}%
\pgfsetdash{}{0pt}%
\pgfsys@defobject{currentmarker}{\pgfqpoint{-0.048611in}{0.000000in}}{\pgfqpoint{0.000000in}{0.000000in}}{%
\pgfpathmoveto{\pgfqpoint{0.000000in}{0.000000in}}%
\pgfpathlineto{\pgfqpoint{-0.048611in}{0.000000in}}%
\pgfusepath{stroke,fill}%
}%
\begin{pgfscope}%
\pgfsys@transformshift{0.606171in}{1.714048in}%
\pgfsys@useobject{currentmarker}{}%
\end{pgfscope}%
\end{pgfscope}%
\begin{pgfscope}%
\definecolor{textcolor}{rgb}{0.000000,0.000000,0.000000}%
\pgfsetstrokecolor{textcolor}%
\pgfsetfillcolor{textcolor}%
\pgftext[x=0.280556in, y=1.670645in, left, base]{\color{textcolor}\rmfamily\fontsize{9.000000}{10.800000}\selectfont \(\displaystyle {1.00}\)}%
\end{pgfscope}%
\begin{pgfscope}%
\definecolor{textcolor}{rgb}{0.000000,0.000000,0.000000}%
\pgfsetstrokecolor{textcolor}%
\pgfsetfillcolor{textcolor}%
\pgftext[x=0.225000in,y=1.130639in,,bottom,rotate=90.000000]{\color{textcolor}\rmfamily\fontsize{9.000000}{10.800000}\selectfont \(\displaystyle \mathbb{E}[Y \mid  f(X)]\)}%
\end{pgfscope}%
\begin{pgfscope}%
\pgfpathrectangle{\pgfqpoint{0.606171in}{0.488889in}}{\pgfqpoint{1.317500in}{1.283500in}}%
\pgfusepath{clip}%
\pgfsetrectcap%
\pgfsetroundjoin%
\pgfsetlinewidth{0.501875pt}%
\definecolor{currentstroke}{rgb}{0.000000,0.000000,0.000000}%
\pgfsetstrokecolor{currentstroke}%
\pgfsetdash{}{0pt}%
\pgfpathmoveto{\pgfqpoint{0.666058in}{0.547230in}}%
\pgfpathlineto{\pgfqpoint{0.970431in}{0.842390in}}%
\pgfpathlineto{\pgfqpoint{0.970670in}{0.842070in}}%
\pgfpathlineto{\pgfqpoint{0.972108in}{0.836685in}}%
\pgfpathlineto{\pgfqpoint{0.974503in}{0.812425in}}%
\pgfpathlineto{\pgfqpoint{0.983847in}{0.703611in}}%
\pgfpathlineto{\pgfqpoint{0.984805in}{0.704953in}}%
\pgfpathlineto{\pgfqpoint{0.986602in}{0.715978in}}%
\pgfpathlineto{\pgfqpoint{0.989836in}{0.759630in}}%
\pgfpathlineto{\pgfqpoint{0.997742in}{0.868452in}}%
\pgfpathlineto{\pgfqpoint{1.000736in}{0.873271in}}%
\pgfpathlineto{\pgfqpoint{1.007804in}{0.878909in}}%
\pgfpathlineto{\pgfqpoint{1.008163in}{0.878425in}}%
\pgfpathlineto{\pgfqpoint{1.009600in}{0.872752in}}%
\pgfpathlineto{\pgfqpoint{1.012116in}{0.846346in}}%
\pgfpathlineto{\pgfqpoint{1.021100in}{0.740177in}}%
\pgfpathlineto{\pgfqpoint{1.021818in}{0.740439in}}%
\pgfpathlineto{\pgfqpoint{1.022178in}{0.741249in}}%
\pgfpathlineto{\pgfqpoint{1.023855in}{0.750889in}}%
\pgfpathlineto{\pgfqpoint{1.026969in}{0.791135in}}%
\pgfpathlineto{\pgfqpoint{1.035354in}{0.905528in}}%
\pgfpathlineto{\pgfqpoint{1.038468in}{0.910030in}}%
\pgfpathlineto{\pgfqpoint{1.045296in}{0.915308in}}%
\pgfpathlineto{\pgfqpoint{1.045536in}{0.914984in}}%
\pgfpathlineto{\pgfqpoint{1.046973in}{0.909566in}}%
\pgfpathlineto{\pgfqpoint{1.049369in}{0.885249in}}%
\pgfpathlineto{\pgfqpoint{1.058712in}{0.776535in}}%
\pgfpathlineto{\pgfqpoint{1.059670in}{0.777902in}}%
\pgfpathlineto{\pgfqpoint{1.061467in}{0.788972in}}%
\pgfpathlineto{\pgfqpoint{1.064701in}{0.832684in}}%
\pgfpathlineto{\pgfqpoint{1.072487in}{0.940962in}}%
\pgfpathlineto{\pgfqpoint{1.075362in}{0.945971in}}%
\pgfpathlineto{\pgfqpoint{1.082669in}{0.951828in}}%
\pgfpathlineto{\pgfqpoint{1.083028in}{0.951337in}}%
\pgfpathlineto{\pgfqpoint{1.084466in}{0.945631in}}%
\pgfpathlineto{\pgfqpoint{1.086981in}{0.919167in}}%
\pgfpathlineto{\pgfqpoint{1.095965in}{0.813097in}}%
\pgfpathlineto{\pgfqpoint{1.096684in}{0.813378in}}%
\pgfpathlineto{\pgfqpoint{1.097043in}{0.814196in}}%
\pgfpathlineto{\pgfqpoint{1.098720in}{0.823879in}}%
\pgfpathlineto{\pgfqpoint{1.101835in}{0.864185in}}%
\pgfpathlineto{\pgfqpoint{1.110219in}{0.978477in}}%
\pgfpathlineto{\pgfqpoint{1.113334in}{0.982963in}}%
\pgfpathlineto{\pgfqpoint{1.120162in}{0.988226in}}%
\pgfpathlineto{\pgfqpoint{1.120401in}{0.987898in}}%
\pgfpathlineto{\pgfqpoint{1.121839in}{0.982446in}}%
\pgfpathlineto{\pgfqpoint{1.124234in}{0.958072in}}%
\pgfpathlineto{\pgfqpoint{1.133577in}{0.849460in}}%
\pgfpathlineto{\pgfqpoint{1.134536in}{0.850851in}}%
\pgfpathlineto{\pgfqpoint{1.136333in}{0.861967in}}%
\pgfpathlineto{\pgfqpoint{1.139567in}{0.905737in}}%
\pgfpathlineto{\pgfqpoint{1.147353in}{1.013917in}}%
\pgfpathlineto{\pgfqpoint{1.150228in}{1.018904in}}%
\pgfpathlineto{\pgfqpoint{1.157534in}{1.024747in}}%
\pgfpathlineto{\pgfqpoint{1.157894in}{1.024250in}}%
\pgfpathlineto{\pgfqpoint{1.159331in}{1.018509in}}%
\pgfpathlineto{\pgfqpoint{1.161847in}{0.991988in}}%
\pgfpathlineto{\pgfqpoint{1.170831in}{0.886016in}}%
\pgfpathlineto{\pgfqpoint{1.171549in}{0.886316in}}%
\pgfpathlineto{\pgfqpoint{1.171909in}{0.887144in}}%
\pgfpathlineto{\pgfqpoint{1.173586in}{0.896870in}}%
\pgfpathlineto{\pgfqpoint{1.176700in}{0.937235in}}%
\pgfpathlineto{\pgfqpoint{1.185085in}{1.051427in}}%
\pgfpathlineto{\pgfqpoint{1.188199in}{1.055897in}}%
\pgfpathlineto{\pgfqpoint{1.195027in}{1.061144in}}%
\pgfpathlineto{\pgfqpoint{1.195267in}{1.060811in}}%
\pgfpathlineto{\pgfqpoint{1.196704in}{1.055325in}}%
\pgfpathlineto{\pgfqpoint{1.199219in}{1.029239in}}%
\pgfpathlineto{\pgfqpoint{1.208323in}{0.922433in}}%
\pgfpathlineto{\pgfqpoint{1.209162in}{0.923146in}}%
\pgfpathlineto{\pgfqpoint{1.209281in}{0.923448in}}%
\pgfpathlineto{\pgfqpoint{1.210958in}{0.932856in}}%
\pgfpathlineto{\pgfqpoint{1.214073in}{0.972771in}}%
\pgfpathlineto{\pgfqpoint{1.222577in}{1.088082in}}%
\pgfpathlineto{\pgfqpoint{1.225692in}{1.092422in}}%
\pgfpathlineto{\pgfqpoint{1.232400in}{1.097666in}}%
\pgfpathlineto{\pgfqpoint{1.232759in}{1.097162in}}%
\pgfpathlineto{\pgfqpoint{1.234197in}{1.091387in}}%
\pgfpathlineto{\pgfqpoint{1.236712in}{1.064808in}}%
\pgfpathlineto{\pgfqpoint{1.245696in}{0.958936in}}%
\pgfpathlineto{\pgfqpoint{1.246415in}{0.959255in}}%
\pgfpathlineto{\pgfqpoint{1.246654in}{0.959763in}}%
\pgfpathlineto{\pgfqpoint{1.248331in}{0.968852in}}%
\pgfpathlineto{\pgfqpoint{1.251326in}{1.006367in}}%
\pgfpathlineto{\pgfqpoint{1.260070in}{1.124729in}}%
\pgfpathlineto{\pgfqpoint{1.263304in}{1.129064in}}%
\pgfpathlineto{\pgfqpoint{1.270132in}{1.137705in}}%
\pgfpathlineto{\pgfqpoint{1.272049in}{1.150344in}}%
\pgfpathlineto{\pgfqpoint{1.275403in}{1.199596in}}%
\pgfpathlineto{\pgfqpoint{1.282350in}{1.298165in}}%
\pgfpathlineto{\pgfqpoint{1.284027in}{1.302414in}}%
\pgfpathlineto{\pgfqpoint{1.284506in}{1.301824in}}%
\pgfpathlineto{\pgfqpoint{1.286063in}{1.294514in}}%
\pgfpathlineto{\pgfqpoint{1.288938in}{1.262251in}}%
\pgfpathlineto{\pgfqpoint{1.297203in}{1.163913in}}%
\pgfpathlineto{\pgfqpoint{1.298761in}{1.163687in}}%
\pgfpathlineto{\pgfqpoint{1.299000in}{1.163877in}}%
\pgfpathlineto{\pgfqpoint{1.307145in}{1.172857in}}%
\pgfpathlineto{\pgfqpoint{1.308942in}{1.182036in}}%
\pgfpathlineto{\pgfqpoint{1.311697in}{1.216493in}}%
\pgfpathlineto{\pgfqpoint{1.320561in}{1.337830in}}%
\pgfpathlineto{\pgfqpoint{1.321400in}{1.338894in}}%
\pgfpathlineto{\pgfqpoint{1.322119in}{1.337851in}}%
\pgfpathlineto{\pgfqpoint{1.323796in}{1.328637in}}%
\pgfpathlineto{\pgfqpoint{1.327030in}{1.288750in}}%
\pgfpathlineto{\pgfqpoint{1.334097in}{1.201581in}}%
\pgfpathlineto{\pgfqpoint{1.335654in}{1.199859in}}%
\pgfpathlineto{\pgfqpoint{1.336133in}{1.200109in}}%
\pgfpathlineto{\pgfqpoint{1.343680in}{1.207551in}}%
\pgfpathlineto{\pgfqpoint{1.345477in}{1.212727in}}%
\pgfpathlineto{\pgfqpoint{1.347633in}{1.231336in}}%
\pgfpathlineto{\pgfqpoint{1.352185in}{1.307514in}}%
\pgfpathlineto{\pgfqpoint{1.357335in}{1.371755in}}%
\pgfpathlineto{\pgfqpoint{1.358773in}{1.375362in}}%
\pgfpathlineto{\pgfqpoint{1.359372in}{1.374735in}}%
\pgfpathlineto{\pgfqpoint{1.360929in}{1.367386in}}%
\pgfpathlineto{\pgfqpoint{1.363804in}{1.335068in}}%
\pgfpathlineto{\pgfqpoint{1.371949in}{1.237028in}}%
\pgfpathlineto{\pgfqpoint{1.373386in}{1.236458in}}%
\pgfpathlineto{\pgfqpoint{1.373866in}{1.236810in}}%
\pgfpathlineto{\pgfqpoint{1.382011in}{1.245804in}}%
\pgfpathlineto{\pgfqpoint{1.383808in}{1.255021in}}%
\pgfpathlineto{\pgfqpoint{1.386563in}{1.289542in}}%
\pgfpathlineto{\pgfqpoint{1.395427in}{1.410776in}}%
\pgfpathlineto{\pgfqpoint{1.396265in}{1.411818in}}%
\pgfpathlineto{\pgfqpoint{1.396984in}{1.410757in}}%
\pgfpathlineto{\pgfqpoint{1.398661in}{1.401502in}}%
\pgfpathlineto{\pgfqpoint{1.401895in}{1.361560in}}%
\pgfpathlineto{\pgfqpoint{1.408962in}{1.274484in}}%
\pgfpathlineto{\pgfqpoint{1.410520in}{1.272787in}}%
\pgfpathlineto{\pgfqpoint{1.410999in}{1.273041in}}%
\pgfpathlineto{\pgfqpoint{1.418545in}{1.280488in}}%
\pgfpathlineto{\pgfqpoint{1.420342in}{1.285690in}}%
\pgfpathlineto{\pgfqpoint{1.422498in}{1.304355in}}%
\pgfpathlineto{\pgfqpoint{1.427050in}{1.380574in}}%
\pgfpathlineto{\pgfqpoint{1.432201in}{1.444719in}}%
\pgfpathlineto{\pgfqpoint{1.433638in}{1.448288in}}%
\pgfpathlineto{\pgfqpoint{1.434237in}{1.447645in}}%
\pgfpathlineto{\pgfqpoint{1.435794in}{1.440258in}}%
\pgfpathlineto{\pgfqpoint{1.438669in}{1.407886in}}%
\pgfpathlineto{\pgfqpoint{1.446814in}{1.309941in}}%
\pgfpathlineto{\pgfqpoint{1.448252in}{1.309389in}}%
\pgfpathlineto{\pgfqpoint{1.448731in}{1.309742in}}%
\pgfpathlineto{\pgfqpoint{1.456876in}{1.318750in}}%
\pgfpathlineto{\pgfqpoint{1.458673in}{1.328008in}}%
\pgfpathlineto{\pgfqpoint{1.461428in}{1.362590in}}%
\pgfpathlineto{\pgfqpoint{1.470292in}{1.483722in}}%
\pgfpathlineto{\pgfqpoint{1.471131in}{1.484742in}}%
\pgfpathlineto{\pgfqpoint{1.471849in}{1.483662in}}%
\pgfpathlineto{\pgfqpoint{1.473526in}{1.474367in}}%
\pgfpathlineto{\pgfqpoint{1.476761in}{1.434370in}}%
\pgfpathlineto{\pgfqpoint{1.483828in}{1.347388in}}%
\pgfpathlineto{\pgfqpoint{1.485385in}{1.345715in}}%
\pgfpathlineto{\pgfqpoint{1.485864in}{1.345972in}}%
\pgfpathlineto{\pgfqpoint{1.493530in}{1.353604in}}%
\pgfpathlineto{\pgfqpoint{1.495327in}{1.359281in}}%
\pgfpathlineto{\pgfqpoint{1.497603in}{1.380429in}}%
\pgfpathlineto{\pgfqpoint{1.503113in}{1.474051in}}%
\pgfpathlineto{\pgfqpoint{1.507426in}{1.519239in}}%
\pgfpathlineto{\pgfqpoint{1.508623in}{1.521182in}}%
\pgfpathlineto{\pgfqpoint{1.509222in}{1.520274in}}%
\pgfpathlineto{\pgfqpoint{1.510899in}{1.511287in}}%
\pgfpathlineto{\pgfqpoint{1.514014in}{1.473544in}}%
\pgfpathlineto{\pgfqpoint{1.521321in}{1.383669in}}%
\pgfpathlineto{\pgfqpoint{1.522878in}{1.382201in}}%
\pgfpathlineto{\pgfqpoint{1.523357in}{1.382482in}}%
\pgfpathlineto{\pgfqpoint{1.531263in}{1.390591in}}%
\pgfpathlineto{\pgfqpoint{1.533059in}{1.397535in}}%
\pgfpathlineto{\pgfqpoint{1.535575in}{1.424451in}}%
\pgfpathlineto{\pgfqpoint{1.545876in}{1.557675in}}%
\pgfpathlineto{\pgfqpoint{1.545996in}{1.557667in}}%
\pgfpathlineto{\pgfqpoint{1.546954in}{1.555805in}}%
\pgfpathlineto{\pgfqpoint{1.548871in}{1.542934in}}%
\pgfpathlineto{\pgfqpoint{1.552584in}{1.491863in}}%
\pgfpathlineto{\pgfqpoint{1.558454in}{1.421085in}}%
\pgfpathlineto{\pgfqpoint{1.560251in}{1.418644in}}%
\pgfpathlineto{\pgfqpoint{1.560610in}{1.418822in}}%
\pgfpathlineto{\pgfqpoint{1.567438in}{1.425348in}}%
\pgfpathlineto{\pgfqpoint{1.863785in}{1.714048in}}%
\pgfpathlineto{\pgfqpoint{1.863785in}{1.714048in}}%
\pgfusepath{stroke}%
\end{pgfscope}%
\begin{pgfscope}%
\pgfsetrectcap%
\pgfsetmiterjoin%
\pgfsetlinewidth{0.803000pt}%
\definecolor{currentstroke}{rgb}{0.000000,0.000000,0.000000}%
\pgfsetstrokecolor{currentstroke}%
\pgfsetdash{}{0pt}%
\pgfpathmoveto{\pgfqpoint{0.606171in}{0.488889in}}%
\pgfpathlineto{\pgfqpoint{0.606171in}{1.772389in}}%
\pgfusepath{stroke}%
\end{pgfscope}%
\begin{pgfscope}%
\pgfsetrectcap%
\pgfsetmiterjoin%
\pgfsetlinewidth{0.803000pt}%
\definecolor{currentstroke}{rgb}{0.000000,0.000000,0.000000}%
\pgfsetstrokecolor{currentstroke}%
\pgfsetdash{}{0pt}%
\pgfpathmoveto{\pgfqpoint{1.923671in}{0.488889in}}%
\pgfpathlineto{\pgfqpoint{1.923671in}{1.772389in}}%
\pgfusepath{stroke}%
\end{pgfscope}%
\begin{pgfscope}%
\pgfsetrectcap%
\pgfsetmiterjoin%
\pgfsetlinewidth{0.803000pt}%
\definecolor{currentstroke}{rgb}{0.000000,0.000000,0.000000}%
\pgfsetstrokecolor{currentstroke}%
\pgfsetdash{}{0pt}%
\pgfpathmoveto{\pgfqpoint{0.606171in}{0.488889in}}%
\pgfpathlineto{\pgfqpoint{1.923671in}{0.488889in}}%
\pgfusepath{stroke}%
\end{pgfscope}%
\begin{pgfscope}%
\pgfsetrectcap%
\pgfsetmiterjoin%
\pgfsetlinewidth{0.803000pt}%
\definecolor{currentstroke}{rgb}{0.000000,0.000000,0.000000}%
\pgfsetstrokecolor{currentstroke}%
\pgfsetdash{}{0pt}%
\pgfpathmoveto{\pgfqpoint{0.606171in}{1.772389in}}%
\pgfpathlineto{\pgfqpoint{1.923671in}{1.772389in}}%
\pgfusepath{stroke}%
\end{pgfscope}%
\end{pgfpicture}%
\makeatother%
\endgroup%

%% file: images/biased.pgf
\begingroup%
\makeatletter%
\begin{pgfpicture}%
\pgfpathrectangle{\pgfpointorigin}{\pgfqpoint{1.974096in}{1.927935in}}%
\pgfusepath{use as bounding box, clip}%
\begin{pgfscope}%
\pgfsetbuttcap%
\pgfsetmiterjoin%
\definecolor{currentfill}{rgb}{1.000000,1.000000,1.000000}%
\pgfsetfillcolor{currentfill}%
\pgfsetlinewidth{0.000000pt}%
\definecolor{currentstroke}{rgb}{1.000000,1.000000,1.000000}%
\pgfsetstrokecolor{currentstroke}%
\pgfsetdash{}{0pt}%
\pgfpathmoveto{\pgfqpoint{0.000000in}{0.000000in}}%
\pgfpathlineto{\pgfqpoint{1.974096in}{0.000000in}}%
\pgfpathlineto{\pgfqpoint{1.974096in}{1.927935in}}%
\pgfpathlineto{\pgfqpoint{0.000000in}{1.927935in}}%
\pgfpathclose%
\pgfusepath{fill}%
\end{pgfscope}%
\begin{pgfscope}%
\pgfsetbuttcap%
\pgfsetmiterjoin%
\definecolor{currentfill}{rgb}{1.000000,1.000000,1.000000}%
\pgfsetfillcolor{currentfill}%
\pgfsetlinewidth{0.000000pt}%
\definecolor{currentstroke}{rgb}{0.000000,0.000000,0.000000}%
\pgfsetstrokecolor{currentstroke}%
\pgfsetstrokeopacity{0.000000}%
\pgfsetdash{}{0pt}%
\pgfpathmoveto{\pgfqpoint{0.556596in}{0.520000in}}%
\pgfpathlineto{\pgfqpoint{1.874096in}{0.520000in}}%
\pgfpathlineto{\pgfqpoint{1.874096in}{1.803500in}}%
\pgfpathlineto{\pgfqpoint{0.556596in}{1.803500in}}%
\pgfpathclose%
\pgfusepath{fill}%
\end{pgfscope}%
\begin{pgfscope}%
\pgfpathrectangle{\pgfqpoint{0.556596in}{0.520000in}}{\pgfqpoint{1.317500in}{1.283500in}}%
\pgfusepath{clip}%
\pgfsetbuttcap%
\pgfsetmiterjoin%
\definecolor{currentfill}{rgb}{0.000000,0.000000,1.000000}%
\pgfsetfillcolor{currentfill}%
\pgfsetfillopacity{0.500000}%
\pgfsetlinewidth{0.000000pt}%
\definecolor{currentstroke}{rgb}{0.000000,0.000000,0.000000}%
\pgfsetstrokecolor{currentstroke}%
\pgfsetstrokeopacity{0.500000}%
\pgfsetdash{}{0pt}%
\pgfpathmoveto{\pgfqpoint{0.957548in}{0.520000in}}%
\pgfpathlineto{\pgfqpoint{1.000381in}{0.520000in}}%
\pgfpathlineto{\pgfqpoint{1.000381in}{0.579012in}}%
\pgfpathlineto{\pgfqpoint{0.957548in}{0.579012in}}%
\pgfpathclose%
\pgfusepath{fill}%
\end{pgfscope}%
\begin{pgfscope}%
\pgfpathrectangle{\pgfqpoint{0.556596in}{0.520000in}}{\pgfqpoint{1.317500in}{1.283500in}}%
\pgfusepath{clip}%
\pgfsetbuttcap%
\pgfsetmiterjoin%
\definecolor{currentfill}{rgb}{0.000000,0.000000,1.000000}%
\pgfsetfillcolor{currentfill}%
\pgfsetfillopacity{0.500000}%
\pgfsetlinewidth{0.000000pt}%
\definecolor{currentstroke}{rgb}{0.000000,0.000000,0.000000}%
\pgfsetstrokecolor{currentstroke}%
\pgfsetstrokeopacity{0.500000}%
\pgfsetdash{}{0pt}%
\pgfpathmoveto{\pgfqpoint{1.000381in}{0.520000in}}%
\pgfpathlineto{\pgfqpoint{1.043214in}{0.520000in}}%
\pgfpathlineto{\pgfqpoint{1.043214in}{0.621163in}}%
\pgfpathlineto{\pgfqpoint{1.000381in}{0.621163in}}%
\pgfpathclose%
\pgfusepath{fill}%
\end{pgfscope}%
\begin{pgfscope}%
\pgfpathrectangle{\pgfqpoint{0.556596in}{0.520000in}}{\pgfqpoint{1.317500in}{1.283500in}}%
\pgfusepath{clip}%
\pgfsetbuttcap%
\pgfsetmiterjoin%
\definecolor{currentfill}{rgb}{0.000000,0.000000,1.000000}%
\pgfsetfillcolor{currentfill}%
\pgfsetfillopacity{0.500000}%
\pgfsetlinewidth{0.000000pt}%
\definecolor{currentstroke}{rgb}{0.000000,0.000000,0.000000}%
\pgfsetstrokecolor{currentstroke}%
\pgfsetstrokeopacity{0.500000}%
\pgfsetdash{}{0pt}%
\pgfpathmoveto{\pgfqpoint{1.043214in}{0.520000in}}%
\pgfpathlineto{\pgfqpoint{1.086047in}{0.520000in}}%
\pgfpathlineto{\pgfqpoint{1.086047in}{0.604302in}}%
\pgfpathlineto{\pgfqpoint{1.043214in}{0.604302in}}%
\pgfpathclose%
\pgfusepath{fill}%
\end{pgfscope}%
\begin{pgfscope}%
\pgfpathrectangle{\pgfqpoint{0.556596in}{0.520000in}}{\pgfqpoint{1.317500in}{1.283500in}}%
\pgfusepath{clip}%
\pgfsetbuttcap%
\pgfsetmiterjoin%
\definecolor{currentfill}{rgb}{0.000000,0.000000,1.000000}%
\pgfsetfillcolor{currentfill}%
\pgfsetfillopacity{0.500000}%
\pgfsetlinewidth{0.000000pt}%
\definecolor{currentstroke}{rgb}{0.000000,0.000000,0.000000}%
\pgfsetstrokecolor{currentstroke}%
\pgfsetstrokeopacity{0.500000}%
\pgfsetdash{}{0pt}%
\pgfpathmoveto{\pgfqpoint{1.086047in}{0.520000in}}%
\pgfpathlineto{\pgfqpoint{1.128880in}{0.520000in}}%
\pgfpathlineto{\pgfqpoint{1.128880in}{0.772907in}}%
\pgfpathlineto{\pgfqpoint{1.086047in}{0.772907in}}%
\pgfpathclose%
\pgfusepath{fill}%
\end{pgfscope}%
\begin{pgfscope}%
\pgfpathrectangle{\pgfqpoint{0.556596in}{0.520000in}}{\pgfqpoint{1.317500in}{1.283500in}}%
\pgfusepath{clip}%
\pgfsetbuttcap%
\pgfsetmiterjoin%
\definecolor{currentfill}{rgb}{0.000000,0.000000,1.000000}%
\pgfsetfillcolor{currentfill}%
\pgfsetfillopacity{0.500000}%
\pgfsetlinewidth{0.000000pt}%
\definecolor{currentstroke}{rgb}{0.000000,0.000000,0.000000}%
\pgfsetstrokecolor{currentstroke}%
\pgfsetstrokeopacity{0.500000}%
\pgfsetdash{}{0pt}%
\pgfpathmoveto{\pgfqpoint{1.128880in}{0.520000in}}%
\pgfpathlineto{\pgfqpoint{1.171713in}{0.520000in}}%
\pgfpathlineto{\pgfqpoint{1.171713in}{0.941511in}}%
\pgfpathlineto{\pgfqpoint{1.128880in}{0.941511in}}%
\pgfpathclose%
\pgfusepath{fill}%
\end{pgfscope}%
\begin{pgfscope}%
\pgfpathrectangle{\pgfqpoint{0.556596in}{0.520000in}}{\pgfqpoint{1.317500in}{1.283500in}}%
\pgfusepath{clip}%
\pgfsetbuttcap%
\pgfsetmiterjoin%
\definecolor{currentfill}{rgb}{0.000000,0.000000,1.000000}%
\pgfsetfillcolor{currentfill}%
\pgfsetfillopacity{0.500000}%
\pgfsetlinewidth{0.000000pt}%
\definecolor{currentstroke}{rgb}{0.000000,0.000000,0.000000}%
\pgfsetstrokecolor{currentstroke}%
\pgfsetstrokeopacity{0.500000}%
\pgfsetdash{}{0pt}%
\pgfpathmoveto{\pgfqpoint{1.171713in}{0.520000in}}%
\pgfpathlineto{\pgfqpoint{1.214546in}{0.520000in}}%
\pgfpathlineto{\pgfqpoint{1.214546in}{1.160696in}}%
\pgfpathlineto{\pgfqpoint{1.171713in}{1.160696in}}%
\pgfpathclose%
\pgfusepath{fill}%
\end{pgfscope}%
\begin{pgfscope}%
\pgfpathrectangle{\pgfqpoint{0.556596in}{0.520000in}}{\pgfqpoint{1.317500in}{1.283500in}}%
\pgfusepath{clip}%
\pgfsetbuttcap%
\pgfsetmiterjoin%
\definecolor{currentfill}{rgb}{0.000000,0.000000,1.000000}%
\pgfsetfillcolor{currentfill}%
\pgfsetfillopacity{0.500000}%
\pgfsetlinewidth{0.000000pt}%
\definecolor{currentstroke}{rgb}{0.000000,0.000000,0.000000}%
\pgfsetstrokecolor{currentstroke}%
\pgfsetstrokeopacity{0.500000}%
\pgfsetdash{}{0pt}%
\pgfpathmoveto{\pgfqpoint{1.214546in}{0.520000in}}%
\pgfpathlineto{\pgfqpoint{1.257379in}{0.520000in}}%
\pgfpathlineto{\pgfqpoint{1.257379in}{1.253429in}}%
\pgfpathlineto{\pgfqpoint{1.214546in}{1.253429in}}%
\pgfpathclose%
\pgfusepath{fill}%
\end{pgfscope}%
\begin{pgfscope}%
\pgfpathrectangle{\pgfqpoint{0.556596in}{0.520000in}}{\pgfqpoint{1.317500in}{1.283500in}}%
\pgfusepath{clip}%
\pgfsetbuttcap%
\pgfsetmiterjoin%
\definecolor{currentfill}{rgb}{0.000000,0.000000,1.000000}%
\pgfsetfillcolor{currentfill}%
\pgfsetfillopacity{0.500000}%
\pgfsetlinewidth{0.000000pt}%
\definecolor{currentstroke}{rgb}{0.000000,0.000000,0.000000}%
\pgfsetstrokecolor{currentstroke}%
\pgfsetstrokeopacity{0.500000}%
\pgfsetdash{}{0pt}%
\pgfpathmoveto{\pgfqpoint{1.257379in}{0.520000in}}%
\pgfpathlineto{\pgfqpoint{1.300212in}{0.520000in}}%
\pgfpathlineto{\pgfqpoint{1.300212in}{1.270289in}}%
\pgfpathlineto{\pgfqpoint{1.257379in}{1.270289in}}%
\pgfpathclose%
\pgfusepath{fill}%
\end{pgfscope}%
\begin{pgfscope}%
\pgfpathrectangle{\pgfqpoint{0.556596in}{0.520000in}}{\pgfqpoint{1.317500in}{1.283500in}}%
\pgfusepath{clip}%
\pgfsetbuttcap%
\pgfsetmiterjoin%
\definecolor{currentfill}{rgb}{0.000000,0.000000,1.000000}%
\pgfsetfillcolor{currentfill}%
\pgfsetfillopacity{0.500000}%
\pgfsetlinewidth{0.000000pt}%
\definecolor{currentstroke}{rgb}{0.000000,0.000000,0.000000}%
\pgfsetstrokecolor{currentstroke}%
\pgfsetstrokeopacity{0.500000}%
\pgfsetdash{}{0pt}%
\pgfpathmoveto{\pgfqpoint{1.300212in}{0.520000in}}%
\pgfpathlineto{\pgfqpoint{1.343045in}{0.520000in}}%
\pgfpathlineto{\pgfqpoint{1.343045in}{1.497905in}}%
\pgfpathlineto{\pgfqpoint{1.300212in}{1.497905in}}%
\pgfpathclose%
\pgfusepath{fill}%
\end{pgfscope}%
\begin{pgfscope}%
\pgfpathrectangle{\pgfqpoint{0.556596in}{0.520000in}}{\pgfqpoint{1.317500in}{1.283500in}}%
\pgfusepath{clip}%
\pgfsetbuttcap%
\pgfsetmiterjoin%
\definecolor{currentfill}{rgb}{0.000000,0.000000,1.000000}%
\pgfsetfillcolor{currentfill}%
\pgfsetfillopacity{0.500000}%
\pgfsetlinewidth{0.000000pt}%
\definecolor{currentstroke}{rgb}{0.000000,0.000000,0.000000}%
\pgfsetstrokecolor{currentstroke}%
\pgfsetstrokeopacity{0.500000}%
\pgfsetdash{}{0pt}%
\pgfpathmoveto{\pgfqpoint{1.343045in}{0.520000in}}%
\pgfpathlineto{\pgfqpoint{1.385879in}{0.520000in}}%
\pgfpathlineto{\pgfqpoint{1.385879in}{1.497905in}}%
\pgfpathlineto{\pgfqpoint{1.343045in}{1.497905in}}%
\pgfpathclose%
\pgfusepath{fill}%
\end{pgfscope}%
\begin{pgfscope}%
\pgfpathrectangle{\pgfqpoint{0.556596in}{0.520000in}}{\pgfqpoint{1.317500in}{1.283500in}}%
\pgfusepath{clip}%
\pgfsetbuttcap%
\pgfsetmiterjoin%
\definecolor{currentfill}{rgb}{0.000000,0.000000,1.000000}%
\pgfsetfillcolor{currentfill}%
\pgfsetfillopacity{0.500000}%
\pgfsetlinewidth{0.000000pt}%
\definecolor{currentstroke}{rgb}{0.000000,0.000000,0.000000}%
\pgfsetstrokecolor{currentstroke}%
\pgfsetstrokeopacity{0.500000}%
\pgfsetdash{}{0pt}%
\pgfpathmoveto{\pgfqpoint{1.385879in}{0.520000in}}%
\pgfpathlineto{\pgfqpoint{1.428712in}{0.520000in}}%
\pgfpathlineto{\pgfqpoint{1.428712in}{1.413603in}}%
\pgfpathlineto{\pgfqpoint{1.385879in}{1.413603in}}%
\pgfpathclose%
\pgfusepath{fill}%
\end{pgfscope}%
\begin{pgfscope}%
\pgfpathrectangle{\pgfqpoint{0.556596in}{0.520000in}}{\pgfqpoint{1.317500in}{1.283500in}}%
\pgfusepath{clip}%
\pgfsetbuttcap%
\pgfsetmiterjoin%
\definecolor{currentfill}{rgb}{0.000000,0.000000,1.000000}%
\pgfsetfillcolor{currentfill}%
\pgfsetfillopacity{0.500000}%
\pgfsetlinewidth{0.000000pt}%
\definecolor{currentstroke}{rgb}{0.000000,0.000000,0.000000}%
\pgfsetstrokecolor{currentstroke}%
\pgfsetstrokeopacity{0.500000}%
\pgfsetdash{}{0pt}%
\pgfpathmoveto{\pgfqpoint{1.428712in}{0.520000in}}%
\pgfpathlineto{\pgfqpoint{1.471545in}{0.520000in}}%
\pgfpathlineto{\pgfqpoint{1.471545in}{1.278719in}}%
\pgfpathlineto{\pgfqpoint{1.428712in}{1.278719in}}%
\pgfpathclose%
\pgfusepath{fill}%
\end{pgfscope}%
\begin{pgfscope}%
\pgfpathrectangle{\pgfqpoint{0.556596in}{0.520000in}}{\pgfqpoint{1.317500in}{1.283500in}}%
\pgfusepath{clip}%
\pgfsetbuttcap%
\pgfsetmiterjoin%
\definecolor{currentfill}{rgb}{0.000000,0.000000,1.000000}%
\pgfsetfillcolor{currentfill}%
\pgfsetfillopacity{0.500000}%
\pgfsetlinewidth{0.000000pt}%
\definecolor{currentstroke}{rgb}{0.000000,0.000000,0.000000}%
\pgfsetstrokecolor{currentstroke}%
\pgfsetstrokeopacity{0.500000}%
\pgfsetdash{}{0pt}%
\pgfpathmoveto{\pgfqpoint{1.471545in}{0.520000in}}%
\pgfpathlineto{\pgfqpoint{1.514378in}{0.520000in}}%
\pgfpathlineto{\pgfqpoint{1.514378in}{1.118545in}}%
\pgfpathlineto{\pgfqpoint{1.471545in}{1.118545in}}%
\pgfpathclose%
\pgfusepath{fill}%
\end{pgfscope}%
\begin{pgfscope}%
\pgfpathrectangle{\pgfqpoint{0.556596in}{0.520000in}}{\pgfqpoint{1.317500in}{1.283500in}}%
\pgfusepath{clip}%
\pgfsetbuttcap%
\pgfsetmiterjoin%
\definecolor{currentfill}{rgb}{0.000000,0.000000,1.000000}%
\pgfsetfillcolor{currentfill}%
\pgfsetfillopacity{0.500000}%
\pgfsetlinewidth{0.000000pt}%
\definecolor{currentstroke}{rgb}{0.000000,0.000000,0.000000}%
\pgfsetstrokecolor{currentstroke}%
\pgfsetstrokeopacity{0.500000}%
\pgfsetdash{}{0pt}%
\pgfpathmoveto{\pgfqpoint{1.514378in}{0.520000in}}%
\pgfpathlineto{\pgfqpoint{1.557211in}{0.520000in}}%
\pgfpathlineto{\pgfqpoint{1.557211in}{0.983662in}}%
\pgfpathlineto{\pgfqpoint{1.514378in}{0.983662in}}%
\pgfpathclose%
\pgfusepath{fill}%
\end{pgfscope}%
\begin{pgfscope}%
\pgfpathrectangle{\pgfqpoint{0.556596in}{0.520000in}}{\pgfqpoint{1.317500in}{1.283500in}}%
\pgfusepath{clip}%
\pgfsetbuttcap%
\pgfsetmiterjoin%
\definecolor{currentfill}{rgb}{0.000000,0.000000,1.000000}%
\pgfsetfillcolor{currentfill}%
\pgfsetfillopacity{0.500000}%
\pgfsetlinewidth{0.000000pt}%
\definecolor{currentstroke}{rgb}{0.000000,0.000000,0.000000}%
\pgfsetstrokecolor{currentstroke}%
\pgfsetstrokeopacity{0.500000}%
\pgfsetdash{}{0pt}%
\pgfpathmoveto{\pgfqpoint{1.557211in}{0.520000in}}%
\pgfpathlineto{\pgfqpoint{1.600044in}{0.520000in}}%
\pgfpathlineto{\pgfqpoint{1.600044in}{0.882499in}}%
\pgfpathlineto{\pgfqpoint{1.557211in}{0.882499in}}%
\pgfpathclose%
\pgfusepath{fill}%
\end{pgfscope}%
\begin{pgfscope}%
\pgfpathrectangle{\pgfqpoint{0.556596in}{0.520000in}}{\pgfqpoint{1.317500in}{1.283500in}}%
\pgfusepath{clip}%
\pgfsetbuttcap%
\pgfsetmiterjoin%
\definecolor{currentfill}{rgb}{0.000000,0.000000,1.000000}%
\pgfsetfillcolor{currentfill}%
\pgfsetfillopacity{0.500000}%
\pgfsetlinewidth{0.000000pt}%
\definecolor{currentstroke}{rgb}{0.000000,0.000000,0.000000}%
\pgfsetstrokecolor{currentstroke}%
\pgfsetstrokeopacity{0.500000}%
\pgfsetdash{}{0pt}%
\pgfpathmoveto{\pgfqpoint{1.600044in}{0.520000in}}%
\pgfpathlineto{\pgfqpoint{1.642877in}{0.520000in}}%
\pgfpathlineto{\pgfqpoint{1.642877in}{0.722325in}}%
\pgfpathlineto{\pgfqpoint{1.600044in}{0.722325in}}%
\pgfpathclose%
\pgfusepath{fill}%
\end{pgfscope}%
\begin{pgfscope}%
\pgfpathrectangle{\pgfqpoint{0.556596in}{0.520000in}}{\pgfqpoint{1.317500in}{1.283500in}}%
\pgfusepath{clip}%
\pgfsetbuttcap%
\pgfsetmiterjoin%
\definecolor{currentfill}{rgb}{0.000000,0.000000,1.000000}%
\pgfsetfillcolor{currentfill}%
\pgfsetfillopacity{0.500000}%
\pgfsetlinewidth{0.000000pt}%
\definecolor{currentstroke}{rgb}{0.000000,0.000000,0.000000}%
\pgfsetstrokecolor{currentstroke}%
\pgfsetstrokeopacity{0.500000}%
\pgfsetdash{}{0pt}%
\pgfpathmoveto{\pgfqpoint{1.642877in}{0.520000in}}%
\pgfpathlineto{\pgfqpoint{1.685710in}{0.520000in}}%
\pgfpathlineto{\pgfqpoint{1.685710in}{0.587442in}}%
\pgfpathlineto{\pgfqpoint{1.642877in}{0.587442in}}%
\pgfpathclose%
\pgfusepath{fill}%
\end{pgfscope}%
\begin{pgfscope}%
\pgfpathrectangle{\pgfqpoint{0.556596in}{0.520000in}}{\pgfqpoint{1.317500in}{1.283500in}}%
\pgfusepath{clip}%
\pgfsetbuttcap%
\pgfsetmiterjoin%
\definecolor{currentfill}{rgb}{0.000000,0.000000,1.000000}%
\pgfsetfillcolor{currentfill}%
\pgfsetfillopacity{0.500000}%
\pgfsetlinewidth{0.000000pt}%
\definecolor{currentstroke}{rgb}{0.000000,0.000000,0.000000}%
\pgfsetstrokecolor{currentstroke}%
\pgfsetstrokeopacity{0.500000}%
\pgfsetdash{}{0pt}%
\pgfpathmoveto{\pgfqpoint{1.685710in}{0.520000in}}%
\pgfpathlineto{\pgfqpoint{1.728543in}{0.520000in}}%
\pgfpathlineto{\pgfqpoint{1.728543in}{0.562151in}}%
\pgfpathlineto{\pgfqpoint{1.685710in}{0.562151in}}%
\pgfpathclose%
\pgfusepath{fill}%
\end{pgfscope}%
\begin{pgfscope}%
\pgfpathrectangle{\pgfqpoint{0.556596in}{0.520000in}}{\pgfqpoint{1.317500in}{1.283500in}}%
\pgfusepath{clip}%
\pgfsetbuttcap%
\pgfsetmiterjoin%
\definecolor{currentfill}{rgb}{0.000000,0.000000,1.000000}%
\pgfsetfillcolor{currentfill}%
\pgfsetfillopacity{0.500000}%
\pgfsetlinewidth{0.000000pt}%
\definecolor{currentstroke}{rgb}{0.000000,0.000000,0.000000}%
\pgfsetstrokecolor{currentstroke}%
\pgfsetstrokeopacity{0.500000}%
\pgfsetdash{}{0pt}%
\pgfpathmoveto{\pgfqpoint{1.728543in}{0.520000in}}%
\pgfpathlineto{\pgfqpoint{1.771377in}{0.520000in}}%
\pgfpathlineto{\pgfqpoint{1.771377in}{0.545291in}}%
\pgfpathlineto{\pgfqpoint{1.728543in}{0.545291in}}%
\pgfpathclose%
\pgfusepath{fill}%
\end{pgfscope}%
\begin{pgfscope}%
\pgfpathrectangle{\pgfqpoint{0.556596in}{0.520000in}}{\pgfqpoint{1.317500in}{1.283500in}}%
\pgfusepath{clip}%
\pgfsetbuttcap%
\pgfsetmiterjoin%
\definecolor{currentfill}{rgb}{0.000000,0.000000,1.000000}%
\pgfsetfillcolor{currentfill}%
\pgfsetfillopacity{0.500000}%
\pgfsetlinewidth{0.000000pt}%
\definecolor{currentstroke}{rgb}{0.000000,0.000000,0.000000}%
\pgfsetstrokecolor{currentstroke}%
\pgfsetstrokeopacity{0.500000}%
\pgfsetdash{}{0pt}%
\pgfpathmoveto{\pgfqpoint{1.771377in}{0.520000in}}%
\pgfpathlineto{\pgfqpoint{1.814210in}{0.520000in}}%
\pgfpathlineto{\pgfqpoint{1.814210in}{0.536861in}}%
\pgfpathlineto{\pgfqpoint{1.771377in}{0.536861in}}%
\pgfpathclose%
\pgfusepath{fill}%
\end{pgfscope}%
\begin{pgfscope}%
\pgfpathrectangle{\pgfqpoint{0.556596in}{0.520000in}}{\pgfqpoint{1.317500in}{1.283500in}}%
\pgfusepath{clip}%
\pgfsetbuttcap%
\pgfsetmiterjoin%
\definecolor{currentfill}{rgb}{1.000000,0.000000,0.000000}%
\pgfsetfillcolor{currentfill}%
\pgfsetfillopacity{0.500000}%
\pgfsetlinewidth{0.000000pt}%
\definecolor{currentstroke}{rgb}{0.000000,0.000000,0.000000}%
\pgfsetstrokecolor{currentstroke}%
\pgfsetstrokeopacity{0.500000}%
\pgfsetdash{}{0pt}%
\pgfpathmoveto{\pgfqpoint{0.616482in}{0.520000in}}%
\pgfpathlineto{\pgfqpoint{0.666742in}{0.520000in}}%
\pgfpathlineto{\pgfqpoint{0.666742in}{0.528430in}}%
\pgfpathlineto{\pgfqpoint{0.616482in}{0.528430in}}%
\pgfpathclose%
\pgfusepath{fill}%
\end{pgfscope}%
\begin{pgfscope}%
\pgfpathrectangle{\pgfqpoint{0.556596in}{0.520000in}}{\pgfqpoint{1.317500in}{1.283500in}}%
\pgfusepath{clip}%
\pgfsetbuttcap%
\pgfsetmiterjoin%
\definecolor{currentfill}{rgb}{1.000000,0.000000,0.000000}%
\pgfsetfillcolor{currentfill}%
\pgfsetfillopacity{0.500000}%
\pgfsetlinewidth{0.000000pt}%
\definecolor{currentstroke}{rgb}{0.000000,0.000000,0.000000}%
\pgfsetstrokecolor{currentstroke}%
\pgfsetstrokeopacity{0.500000}%
\pgfsetdash{}{0pt}%
\pgfpathmoveto{\pgfqpoint{0.666742in}{0.520000in}}%
\pgfpathlineto{\pgfqpoint{0.717001in}{0.520000in}}%
\pgfpathlineto{\pgfqpoint{0.717001in}{0.520000in}}%
\pgfpathlineto{\pgfqpoint{0.666742in}{0.520000in}}%
\pgfpathclose%
\pgfusepath{fill}%
\end{pgfscope}%
\begin{pgfscope}%
\pgfpathrectangle{\pgfqpoint{0.556596in}{0.520000in}}{\pgfqpoint{1.317500in}{1.283500in}}%
\pgfusepath{clip}%
\pgfsetbuttcap%
\pgfsetmiterjoin%
\definecolor{currentfill}{rgb}{1.000000,0.000000,0.000000}%
\pgfsetfillcolor{currentfill}%
\pgfsetfillopacity{0.500000}%
\pgfsetlinewidth{0.000000pt}%
\definecolor{currentstroke}{rgb}{0.000000,0.000000,0.000000}%
\pgfsetstrokecolor{currentstroke}%
\pgfsetstrokeopacity{0.500000}%
\pgfsetdash{}{0pt}%
\pgfpathmoveto{\pgfqpoint{0.717001in}{0.520000in}}%
\pgfpathlineto{\pgfqpoint{0.767261in}{0.520000in}}%
\pgfpathlineto{\pgfqpoint{0.767261in}{0.520000in}}%
\pgfpathlineto{\pgfqpoint{0.717001in}{0.520000in}}%
\pgfpathclose%
\pgfusepath{fill}%
\end{pgfscope}%
\begin{pgfscope}%
\pgfpathrectangle{\pgfqpoint{0.556596in}{0.520000in}}{\pgfqpoint{1.317500in}{1.283500in}}%
\pgfusepath{clip}%
\pgfsetbuttcap%
\pgfsetmiterjoin%
\definecolor{currentfill}{rgb}{1.000000,0.000000,0.000000}%
\pgfsetfillcolor{currentfill}%
\pgfsetfillopacity{0.500000}%
\pgfsetlinewidth{0.000000pt}%
\definecolor{currentstroke}{rgb}{0.000000,0.000000,0.000000}%
\pgfsetstrokecolor{currentstroke}%
\pgfsetstrokeopacity{0.500000}%
\pgfsetdash{}{0pt}%
\pgfpathmoveto{\pgfqpoint{0.767261in}{0.520000in}}%
\pgfpathlineto{\pgfqpoint{0.817520in}{0.520000in}}%
\pgfpathlineto{\pgfqpoint{0.817520in}{0.562151in}}%
\pgfpathlineto{\pgfqpoint{0.767261in}{0.562151in}}%
\pgfpathclose%
\pgfusepath{fill}%
\end{pgfscope}%
\begin{pgfscope}%
\pgfpathrectangle{\pgfqpoint{0.556596in}{0.520000in}}{\pgfqpoint{1.317500in}{1.283500in}}%
\pgfusepath{clip}%
\pgfsetbuttcap%
\pgfsetmiterjoin%
\definecolor{currentfill}{rgb}{1.000000,0.000000,0.000000}%
\pgfsetfillcolor{currentfill}%
\pgfsetfillopacity{0.500000}%
\pgfsetlinewidth{0.000000pt}%
\definecolor{currentstroke}{rgb}{0.000000,0.000000,0.000000}%
\pgfsetstrokecolor{currentstroke}%
\pgfsetstrokeopacity{0.500000}%
\pgfsetdash{}{0pt}%
\pgfpathmoveto{\pgfqpoint{0.817520in}{0.520000in}}%
\pgfpathlineto{\pgfqpoint{0.867780in}{0.520000in}}%
\pgfpathlineto{\pgfqpoint{0.867780in}{0.579012in}}%
\pgfpathlineto{\pgfqpoint{0.817520in}{0.579012in}}%
\pgfpathclose%
\pgfusepath{fill}%
\end{pgfscope}%
\begin{pgfscope}%
\pgfpathrectangle{\pgfqpoint{0.556596in}{0.520000in}}{\pgfqpoint{1.317500in}{1.283500in}}%
\pgfusepath{clip}%
\pgfsetbuttcap%
\pgfsetmiterjoin%
\definecolor{currentfill}{rgb}{1.000000,0.000000,0.000000}%
\pgfsetfillcolor{currentfill}%
\pgfsetfillopacity{0.500000}%
\pgfsetlinewidth{0.000000pt}%
\definecolor{currentstroke}{rgb}{0.000000,0.000000,0.000000}%
\pgfsetstrokecolor{currentstroke}%
\pgfsetstrokeopacity{0.500000}%
\pgfsetdash{}{0pt}%
\pgfpathmoveto{\pgfqpoint{0.867780in}{0.520000in}}%
\pgfpathlineto{\pgfqpoint{0.918039in}{0.520000in}}%
\pgfpathlineto{\pgfqpoint{0.918039in}{0.722325in}}%
\pgfpathlineto{\pgfqpoint{0.867780in}{0.722325in}}%
\pgfpathclose%
\pgfusepath{fill}%
\end{pgfscope}%
\begin{pgfscope}%
\pgfpathrectangle{\pgfqpoint{0.556596in}{0.520000in}}{\pgfqpoint{1.317500in}{1.283500in}}%
\pgfusepath{clip}%
\pgfsetbuttcap%
\pgfsetmiterjoin%
\definecolor{currentfill}{rgb}{1.000000,0.000000,0.000000}%
\pgfsetfillcolor{currentfill}%
\pgfsetfillopacity{0.500000}%
\pgfsetlinewidth{0.000000pt}%
\definecolor{currentstroke}{rgb}{0.000000,0.000000,0.000000}%
\pgfsetstrokecolor{currentstroke}%
\pgfsetstrokeopacity{0.500000}%
\pgfsetdash{}{0pt}%
\pgfpathmoveto{\pgfqpoint{0.918039in}{0.520000in}}%
\pgfpathlineto{\pgfqpoint{0.968299in}{0.520000in}}%
\pgfpathlineto{\pgfqpoint{0.968299in}{0.831918in}}%
\pgfpathlineto{\pgfqpoint{0.918039in}{0.831918in}}%
\pgfpathclose%
\pgfusepath{fill}%
\end{pgfscope}%
\begin{pgfscope}%
\pgfpathrectangle{\pgfqpoint{0.556596in}{0.520000in}}{\pgfqpoint{1.317500in}{1.283500in}}%
\pgfusepath{clip}%
\pgfsetbuttcap%
\pgfsetmiterjoin%
\definecolor{currentfill}{rgb}{1.000000,0.000000,0.000000}%
\pgfsetfillcolor{currentfill}%
\pgfsetfillopacity{0.500000}%
\pgfsetlinewidth{0.000000pt}%
\definecolor{currentstroke}{rgb}{0.000000,0.000000,0.000000}%
\pgfsetstrokecolor{currentstroke}%
\pgfsetstrokeopacity{0.500000}%
\pgfsetdash{}{0pt}%
\pgfpathmoveto{\pgfqpoint{0.968299in}{0.520000in}}%
\pgfpathlineto{\pgfqpoint{1.018558in}{0.520000in}}%
\pgfpathlineto{\pgfqpoint{1.018558in}{1.177557in}}%
\pgfpathlineto{\pgfqpoint{0.968299in}{1.177557in}}%
\pgfpathclose%
\pgfusepath{fill}%
\end{pgfscope}%
\begin{pgfscope}%
\pgfpathrectangle{\pgfqpoint{0.556596in}{0.520000in}}{\pgfqpoint{1.317500in}{1.283500in}}%
\pgfusepath{clip}%
\pgfsetbuttcap%
\pgfsetmiterjoin%
\definecolor{currentfill}{rgb}{1.000000,0.000000,0.000000}%
\pgfsetfillcolor{currentfill}%
\pgfsetfillopacity{0.500000}%
\pgfsetlinewidth{0.000000pt}%
\definecolor{currentstroke}{rgb}{0.000000,0.000000,0.000000}%
\pgfsetstrokecolor{currentstroke}%
\pgfsetstrokeopacity{0.500000}%
\pgfsetdash{}{0pt}%
\pgfpathmoveto{\pgfqpoint{1.018558in}{0.520000in}}%
\pgfpathlineto{\pgfqpoint{1.068818in}{0.520000in}}%
\pgfpathlineto{\pgfqpoint{1.068818in}{1.320870in}}%
\pgfpathlineto{\pgfqpoint{1.018558in}{1.320870in}}%
\pgfpathclose%
\pgfusepath{fill}%
\end{pgfscope}%
\begin{pgfscope}%
\pgfpathrectangle{\pgfqpoint{0.556596in}{0.520000in}}{\pgfqpoint{1.317500in}{1.283500in}}%
\pgfusepath{clip}%
\pgfsetbuttcap%
\pgfsetmiterjoin%
\definecolor{currentfill}{rgb}{1.000000,0.000000,0.000000}%
\pgfsetfillcolor{currentfill}%
\pgfsetfillopacity{0.500000}%
\pgfsetlinewidth{0.000000pt}%
\definecolor{currentstroke}{rgb}{0.000000,0.000000,0.000000}%
\pgfsetstrokecolor{currentstroke}%
\pgfsetstrokeopacity{0.500000}%
\pgfsetdash{}{0pt}%
\pgfpathmoveto{\pgfqpoint{1.068818in}{0.520000in}}%
\pgfpathlineto{\pgfqpoint{1.119077in}{0.520000in}}%
\pgfpathlineto{\pgfqpoint{1.119077in}{1.742381in}}%
\pgfpathlineto{\pgfqpoint{1.068818in}{1.742381in}}%
\pgfpathclose%
\pgfusepath{fill}%
\end{pgfscope}%
\begin{pgfscope}%
\pgfpathrectangle{\pgfqpoint{0.556596in}{0.520000in}}{\pgfqpoint{1.317500in}{1.283500in}}%
\pgfusepath{clip}%
\pgfsetbuttcap%
\pgfsetmiterjoin%
\definecolor{currentfill}{rgb}{1.000000,0.000000,0.000000}%
\pgfsetfillcolor{currentfill}%
\pgfsetfillopacity{0.500000}%
\pgfsetlinewidth{0.000000pt}%
\definecolor{currentstroke}{rgb}{0.000000,0.000000,0.000000}%
\pgfsetstrokecolor{currentstroke}%
\pgfsetstrokeopacity{0.500000}%
\pgfsetdash{}{0pt}%
\pgfpathmoveto{\pgfqpoint{1.119077in}{0.520000in}}%
\pgfpathlineto{\pgfqpoint{1.169337in}{0.520000in}}%
\pgfpathlineto{\pgfqpoint{1.169337in}{1.725521in}}%
\pgfpathlineto{\pgfqpoint{1.119077in}{1.725521in}}%
\pgfpathclose%
\pgfusepath{fill}%
\end{pgfscope}%
\begin{pgfscope}%
\pgfpathrectangle{\pgfqpoint{0.556596in}{0.520000in}}{\pgfqpoint{1.317500in}{1.283500in}}%
\pgfusepath{clip}%
\pgfsetbuttcap%
\pgfsetmiterjoin%
\definecolor{currentfill}{rgb}{1.000000,0.000000,0.000000}%
\pgfsetfillcolor{currentfill}%
\pgfsetfillopacity{0.500000}%
\pgfsetlinewidth{0.000000pt}%
\definecolor{currentstroke}{rgb}{0.000000,0.000000,0.000000}%
\pgfsetstrokecolor{currentstroke}%
\pgfsetstrokeopacity{0.500000}%
\pgfsetdash{}{0pt}%
\pgfpathmoveto{\pgfqpoint{1.169337in}{0.520000in}}%
\pgfpathlineto{\pgfqpoint{1.219596in}{0.520000in}}%
\pgfpathlineto{\pgfqpoint{1.219596in}{1.742381in}}%
\pgfpathlineto{\pgfqpoint{1.169337in}{1.742381in}}%
\pgfpathclose%
\pgfusepath{fill}%
\end{pgfscope}%
\begin{pgfscope}%
\pgfpathrectangle{\pgfqpoint{0.556596in}{0.520000in}}{\pgfqpoint{1.317500in}{1.283500in}}%
\pgfusepath{clip}%
\pgfsetbuttcap%
\pgfsetmiterjoin%
\definecolor{currentfill}{rgb}{1.000000,0.000000,0.000000}%
\pgfsetfillcolor{currentfill}%
\pgfsetfillopacity{0.500000}%
\pgfsetlinewidth{0.000000pt}%
\definecolor{currentstroke}{rgb}{0.000000,0.000000,0.000000}%
\pgfsetstrokecolor{currentstroke}%
\pgfsetstrokeopacity{0.500000}%
\pgfsetdash{}{0pt}%
\pgfpathmoveto{\pgfqpoint{1.219596in}{0.520000in}}%
\pgfpathlineto{\pgfqpoint{1.269856in}{0.520000in}}%
\pgfpathlineto{\pgfqpoint{1.269856in}{1.489475in}}%
\pgfpathlineto{\pgfqpoint{1.219596in}{1.489475in}}%
\pgfpathclose%
\pgfusepath{fill}%
\end{pgfscope}%
\begin{pgfscope}%
\pgfpathrectangle{\pgfqpoint{0.556596in}{0.520000in}}{\pgfqpoint{1.317500in}{1.283500in}}%
\pgfusepath{clip}%
\pgfsetbuttcap%
\pgfsetmiterjoin%
\definecolor{currentfill}{rgb}{1.000000,0.000000,0.000000}%
\pgfsetfillcolor{currentfill}%
\pgfsetfillopacity{0.500000}%
\pgfsetlinewidth{0.000000pt}%
\definecolor{currentstroke}{rgb}{0.000000,0.000000,0.000000}%
\pgfsetstrokecolor{currentstroke}%
\pgfsetstrokeopacity{0.500000}%
\pgfsetdash{}{0pt}%
\pgfpathmoveto{\pgfqpoint{1.269856in}{0.520000in}}%
\pgfpathlineto{\pgfqpoint{1.320115in}{0.520000in}}%
\pgfpathlineto{\pgfqpoint{1.320115in}{1.253429in}}%
\pgfpathlineto{\pgfqpoint{1.269856in}{1.253429in}}%
\pgfpathclose%
\pgfusepath{fill}%
\end{pgfscope}%
\begin{pgfscope}%
\pgfpathrectangle{\pgfqpoint{0.556596in}{0.520000in}}{\pgfqpoint{1.317500in}{1.283500in}}%
\pgfusepath{clip}%
\pgfsetbuttcap%
\pgfsetmiterjoin%
\definecolor{currentfill}{rgb}{1.000000,0.000000,0.000000}%
\pgfsetfillcolor{currentfill}%
\pgfsetfillopacity{0.500000}%
\pgfsetlinewidth{0.000000pt}%
\definecolor{currentstroke}{rgb}{0.000000,0.000000,0.000000}%
\pgfsetstrokecolor{currentstroke}%
\pgfsetstrokeopacity{0.500000}%
\pgfsetdash{}{0pt}%
\pgfpathmoveto{\pgfqpoint{1.320115in}{0.520000in}}%
\pgfpathlineto{\pgfqpoint{1.370375in}{0.520000in}}%
\pgfpathlineto{\pgfqpoint{1.370375in}{1.110115in}}%
\pgfpathlineto{\pgfqpoint{1.320115in}{1.110115in}}%
\pgfpathclose%
\pgfusepath{fill}%
\end{pgfscope}%
\begin{pgfscope}%
\pgfpathrectangle{\pgfqpoint{0.556596in}{0.520000in}}{\pgfqpoint{1.317500in}{1.283500in}}%
\pgfusepath{clip}%
\pgfsetbuttcap%
\pgfsetmiterjoin%
\definecolor{currentfill}{rgb}{1.000000,0.000000,0.000000}%
\pgfsetfillcolor{currentfill}%
\pgfsetfillopacity{0.500000}%
\pgfsetlinewidth{0.000000pt}%
\definecolor{currentstroke}{rgb}{0.000000,0.000000,0.000000}%
\pgfsetstrokecolor{currentstroke}%
\pgfsetstrokeopacity{0.500000}%
\pgfsetdash{}{0pt}%
\pgfpathmoveto{\pgfqpoint{1.370375in}{0.520000in}}%
\pgfpathlineto{\pgfqpoint{1.420634in}{0.520000in}}%
\pgfpathlineto{\pgfqpoint{1.420634in}{0.688604in}}%
\pgfpathlineto{\pgfqpoint{1.370375in}{0.688604in}}%
\pgfpathclose%
\pgfusepath{fill}%
\end{pgfscope}%
\begin{pgfscope}%
\pgfpathrectangle{\pgfqpoint{0.556596in}{0.520000in}}{\pgfqpoint{1.317500in}{1.283500in}}%
\pgfusepath{clip}%
\pgfsetbuttcap%
\pgfsetmiterjoin%
\definecolor{currentfill}{rgb}{1.000000,0.000000,0.000000}%
\pgfsetfillcolor{currentfill}%
\pgfsetfillopacity{0.500000}%
\pgfsetlinewidth{0.000000pt}%
\definecolor{currentstroke}{rgb}{0.000000,0.000000,0.000000}%
\pgfsetstrokecolor{currentstroke}%
\pgfsetstrokeopacity{0.500000}%
\pgfsetdash{}{0pt}%
\pgfpathmoveto{\pgfqpoint{1.420634in}{0.520000in}}%
\pgfpathlineto{\pgfqpoint{1.470894in}{0.520000in}}%
\pgfpathlineto{\pgfqpoint{1.470894in}{0.663314in}}%
\pgfpathlineto{\pgfqpoint{1.420634in}{0.663314in}}%
\pgfpathclose%
\pgfusepath{fill}%
\end{pgfscope}%
\begin{pgfscope}%
\pgfpathrectangle{\pgfqpoint{0.556596in}{0.520000in}}{\pgfqpoint{1.317500in}{1.283500in}}%
\pgfusepath{clip}%
\pgfsetbuttcap%
\pgfsetmiterjoin%
\definecolor{currentfill}{rgb}{1.000000,0.000000,0.000000}%
\pgfsetfillcolor{currentfill}%
\pgfsetfillopacity{0.500000}%
\pgfsetlinewidth{0.000000pt}%
\definecolor{currentstroke}{rgb}{0.000000,0.000000,0.000000}%
\pgfsetstrokecolor{currentstroke}%
\pgfsetstrokeopacity{0.500000}%
\pgfsetdash{}{0pt}%
\pgfpathmoveto{\pgfqpoint{1.470894in}{0.520000in}}%
\pgfpathlineto{\pgfqpoint{1.521153in}{0.520000in}}%
\pgfpathlineto{\pgfqpoint{1.521153in}{0.570581in}}%
\pgfpathlineto{\pgfqpoint{1.470894in}{0.570581in}}%
\pgfpathclose%
\pgfusepath{fill}%
\end{pgfscope}%
\begin{pgfscope}%
\pgfpathrectangle{\pgfqpoint{0.556596in}{0.520000in}}{\pgfqpoint{1.317500in}{1.283500in}}%
\pgfusepath{clip}%
\pgfsetbuttcap%
\pgfsetmiterjoin%
\definecolor{currentfill}{rgb}{1.000000,0.000000,0.000000}%
\pgfsetfillcolor{currentfill}%
\pgfsetfillopacity{0.500000}%
\pgfsetlinewidth{0.000000pt}%
\definecolor{currentstroke}{rgb}{0.000000,0.000000,0.000000}%
\pgfsetstrokecolor{currentstroke}%
\pgfsetstrokeopacity{0.500000}%
\pgfsetdash{}{0pt}%
\pgfpathmoveto{\pgfqpoint{1.521153in}{0.520000in}}%
\pgfpathlineto{\pgfqpoint{1.571413in}{0.520000in}}%
\pgfpathlineto{\pgfqpoint{1.571413in}{0.553721in}}%
\pgfpathlineto{\pgfqpoint{1.521153in}{0.553721in}}%
\pgfpathclose%
\pgfusepath{fill}%
\end{pgfscope}%
\begin{pgfscope}%
\pgfpathrectangle{\pgfqpoint{0.556596in}{0.520000in}}{\pgfqpoint{1.317500in}{1.283500in}}%
\pgfusepath{clip}%
\pgfsetbuttcap%
\pgfsetmiterjoin%
\definecolor{currentfill}{rgb}{1.000000,0.000000,0.000000}%
\pgfsetfillcolor{currentfill}%
\pgfsetfillopacity{0.500000}%
\pgfsetlinewidth{0.000000pt}%
\definecolor{currentstroke}{rgb}{0.000000,0.000000,0.000000}%
\pgfsetstrokecolor{currentstroke}%
\pgfsetstrokeopacity{0.500000}%
\pgfsetdash{}{0pt}%
\pgfpathmoveto{\pgfqpoint{1.571413in}{0.520000in}}%
\pgfpathlineto{\pgfqpoint{1.621672in}{0.520000in}}%
\pgfpathlineto{\pgfqpoint{1.621672in}{0.528430in}}%
\pgfpathlineto{\pgfqpoint{1.571413in}{0.528430in}}%
\pgfpathclose%
\pgfusepath{fill}%
\end{pgfscope}%
\begin{pgfscope}%
\pgfsetbuttcap%
\pgfsetroundjoin%
\definecolor{currentfill}{rgb}{0.000000,0.000000,0.000000}%
\pgfsetfillcolor{currentfill}%
\pgfsetlinewidth{0.803000pt}%
\definecolor{currentstroke}{rgb}{0.000000,0.000000,0.000000}%
\pgfsetstrokecolor{currentstroke}%
\pgfsetdash{}{0pt}%
\pgfsys@defobject{currentmarker}{\pgfqpoint{0.000000in}{-0.048611in}}{\pgfqpoint{0.000000in}{0.000000in}}{%
\pgfpathmoveto{\pgfqpoint{0.000000in}{0.000000in}}%
\pgfpathlineto{\pgfqpoint{0.000000in}{-0.048611in}}%
\pgfusepath{stroke,fill}%
}%
\begin{pgfscope}%
\pgfsys@transformshift{0.851433in}{0.520000in}%
\pgfsys@useobject{currentmarker}{}%
\end{pgfscope}%
\end{pgfscope}%
\begin{pgfscope}%
\definecolor{textcolor}{rgb}{0.000000,0.000000,0.000000}%
\pgfsetstrokecolor{textcolor}%
\pgfsetfillcolor{textcolor}%
\pgftext[x=0.851433in,y=0.422778in,,top]{\color{textcolor}\rmfamily\fontsize{9.000000}{10.800000}\selectfont \(\displaystyle {0.060}\)}%
\end{pgfscope}%
\begin{pgfscope}%
\pgfsetbuttcap%
\pgfsetroundjoin%
\definecolor{currentfill}{rgb}{0.000000,0.000000,0.000000}%
\pgfsetfillcolor{currentfill}%
\pgfsetlinewidth{0.803000pt}%
\definecolor{currentstroke}{rgb}{0.000000,0.000000,0.000000}%
\pgfsetstrokecolor{currentstroke}%
\pgfsetdash{}{0pt}%
\pgfsys@defobject{currentmarker}{\pgfqpoint{0.000000in}{-0.048611in}}{\pgfqpoint{0.000000in}{0.000000in}}{%
\pgfpathmoveto{\pgfqpoint{0.000000in}{0.000000in}}%
\pgfpathlineto{\pgfqpoint{0.000000in}{-0.048611in}}%
\pgfusepath{stroke,fill}%
}%
\begin{pgfscope}%
\pgfsys@transformshift{1.364899in}{0.520000in}%
\pgfsys@useobject{currentmarker}{}%
\end{pgfscope}%
\end{pgfscope}%
\begin{pgfscope}%
\definecolor{textcolor}{rgb}{0.000000,0.000000,0.000000}%
\pgfsetstrokecolor{textcolor}%
\pgfsetfillcolor{textcolor}%
\pgftext[x=1.364899in,y=0.422778in,,top]{\color{textcolor}\rmfamily\fontsize{9.000000}{10.800000}\selectfont \(\displaystyle {0.065}\)}%
\end{pgfscope}%
\begin{pgfscope}%
\definecolor{textcolor}{rgb}{0.000000,0.000000,0.000000}%
\pgfsetstrokecolor{textcolor}%
\pgfsetfillcolor{textcolor}%
\pgftext[x=1.215346in,y=0.256111in,,top]{\color{textcolor}\rmfamily\fontsize{9.000000}{10.800000}\selectfont \(\displaystyle T_{m^*, n}^{\mathrm{b}}\)}%
\end{pgfscope}%
\begin{pgfscope}%
\pgfsetbuttcap%
\pgfsetroundjoin%
\definecolor{currentfill}{rgb}{0.000000,0.000000,0.000000}%
\pgfsetfillcolor{currentfill}%
\pgfsetlinewidth{0.803000pt}%
\definecolor{currentstroke}{rgb}{0.000000,0.000000,0.000000}%
\pgfsetstrokecolor{currentstroke}%
\pgfsetdash{}{0pt}%
\pgfsys@defobject{currentmarker}{\pgfqpoint{-0.048611in}{0.000000in}}{\pgfqpoint{0.000000in}{0.000000in}}{%
\pgfpathmoveto{\pgfqpoint{0.000000in}{0.000000in}}%
\pgfpathlineto{\pgfqpoint{-0.048611in}{0.000000in}}%
\pgfusepath{stroke,fill}%
}%
\begin{pgfscope}%
\pgfsys@transformshift{0.556596in}{0.520000in}%
\pgfsys@useobject{currentmarker}{}%
\end{pgfscope}%
\end{pgfscope}%
\begin{pgfscope}%
\definecolor{textcolor}{rgb}{0.000000,0.000000,0.000000}%
\pgfsetstrokecolor{textcolor}%
\pgfsetfillcolor{textcolor}%
\pgftext[x=0.395138in, y=0.476597in, left, base]{\color{textcolor}\rmfamily\fontsize{9.000000}{10.800000}\selectfont \(\displaystyle {0}\)}%
\end{pgfscope}%
\begin{pgfscope}%
\pgfsetbuttcap%
\pgfsetroundjoin%
\definecolor{currentfill}{rgb}{0.000000,0.000000,0.000000}%
\pgfsetfillcolor{currentfill}%
\pgfsetlinewidth{0.803000pt}%
\definecolor{currentstroke}{rgb}{0.000000,0.000000,0.000000}%
\pgfsetstrokecolor{currentstroke}%
\pgfsetdash{}{0pt}%
\pgfsys@defobject{currentmarker}{\pgfqpoint{-0.048611in}{0.000000in}}{\pgfqpoint{0.000000in}{0.000000in}}{%
\pgfpathmoveto{\pgfqpoint{0.000000in}{0.000000in}}%
\pgfpathlineto{\pgfqpoint{-0.048611in}{0.000000in}}%
\pgfusepath{stroke,fill}%
}%
\begin{pgfscope}%
\pgfsys@transformshift{0.556596in}{0.941511in}%
\pgfsys@useobject{currentmarker}{}%
\end{pgfscope}%
\end{pgfscope}%
\begin{pgfscope}%
\definecolor{textcolor}{rgb}{0.000000,0.000000,0.000000}%
\pgfsetstrokecolor{textcolor}%
\pgfsetfillcolor{textcolor}%
\pgftext[x=0.330902in, y=0.898108in, left, base]{\color{textcolor}\rmfamily\fontsize{9.000000}{10.800000}\selectfont \(\displaystyle {50}\)}%
\end{pgfscope}%
\begin{pgfscope}%
\pgfsetbuttcap%
\pgfsetroundjoin%
\definecolor{currentfill}{rgb}{0.000000,0.000000,0.000000}%
\pgfsetfillcolor{currentfill}%
\pgfsetlinewidth{0.803000pt}%
\definecolor{currentstroke}{rgb}{0.000000,0.000000,0.000000}%
\pgfsetstrokecolor{currentstroke}%
\pgfsetdash{}{0pt}%
\pgfsys@defobject{currentmarker}{\pgfqpoint{-0.048611in}{0.000000in}}{\pgfqpoint{0.000000in}{0.000000in}}{%
\pgfpathmoveto{\pgfqpoint{0.000000in}{0.000000in}}%
\pgfpathlineto{\pgfqpoint{-0.048611in}{0.000000in}}%
\pgfusepath{stroke,fill}%
}%
\begin{pgfscope}%
\pgfsys@transformshift{0.556596in}{1.363021in}%
\pgfsys@useobject{currentmarker}{}%
\end{pgfscope}%
\end{pgfscope}%
\begin{pgfscope}%
\definecolor{textcolor}{rgb}{0.000000,0.000000,0.000000}%
\pgfsetstrokecolor{textcolor}%
\pgfsetfillcolor{textcolor}%
\pgftext[x=0.266667in, y=1.319619in, left, base]{\color{textcolor}\rmfamily\fontsize{9.000000}{10.800000}\selectfont \(\displaystyle {100}\)}%
\end{pgfscope}%
\begin{pgfscope}%
\pgfsetbuttcap%
\pgfsetroundjoin%
\definecolor{currentfill}{rgb}{0.000000,0.000000,0.000000}%
\pgfsetfillcolor{currentfill}%
\pgfsetlinewidth{0.803000pt}%
\definecolor{currentstroke}{rgb}{0.000000,0.000000,0.000000}%
\pgfsetstrokecolor{currentstroke}%
\pgfsetdash{}{0pt}%
\pgfsys@defobject{currentmarker}{\pgfqpoint{-0.048611in}{0.000000in}}{\pgfqpoint{0.000000in}{0.000000in}}{%
\pgfpathmoveto{\pgfqpoint{0.000000in}{0.000000in}}%
\pgfpathlineto{\pgfqpoint{-0.048611in}{0.000000in}}%
\pgfusepath{stroke,fill}%
}%
\begin{pgfscope}%
\pgfsys@transformshift{0.556596in}{1.784532in}%
\pgfsys@useobject{currentmarker}{}%
\end{pgfscope}%
\end{pgfscope}%
\begin{pgfscope}%
\definecolor{textcolor}{rgb}{0.000000,0.000000,0.000000}%
\pgfsetstrokecolor{textcolor}%
\pgfsetfillcolor{textcolor}%
\pgftext[x=0.266667in, y=1.741129in, left, base]{\color{textcolor}\rmfamily\fontsize{9.000000}{10.800000}\selectfont \(\displaystyle {150}\)}%
\end{pgfscope}%
\begin{pgfscope}%
\definecolor{textcolor}{rgb}{0.000000,0.000000,0.000000}%
\pgfsetstrokecolor{textcolor}%
\pgfsetfillcolor{textcolor}%
\pgftext[x=0.211111in,y=1.161750in,,bottom,rotate=90.000000]{\color{textcolor}\rmfamily\fontsize{9.000000}{10.800000}\selectfont Count}%
\end{pgfscope}%
\begin{pgfscope}%
\pgfpathrectangle{\pgfqpoint{0.556596in}{0.520000in}}{\pgfqpoint{1.317500in}{1.283500in}}%
\pgfusepath{clip}%
\pgfsetbuttcap%
\pgfsetroundjoin%
\pgfsetlinewidth{1.505625pt}%
\definecolor{currentstroke}{rgb}{0.000000,0.000000,1.000000}%
\pgfsetstrokecolor{currentstroke}%
\pgfsetdash{{5.550000pt}{2.400000pt}}{0.000000pt}%
\pgfpathmoveto{\pgfqpoint{1.351609in}{0.520000in}}%
\pgfpathlineto{\pgfqpoint{1.351609in}{1.803500in}}%
\pgfusepath{stroke}%
\end{pgfscope}%
\begin{pgfscope}%
\pgfpathrectangle{\pgfqpoint{0.556596in}{0.520000in}}{\pgfqpoint{1.317500in}{1.283500in}}%
\pgfusepath{clip}%
\pgfsetbuttcap%
\pgfsetroundjoin%
\pgfsetlinewidth{1.505625pt}%
\definecolor{currentstroke}{rgb}{1.000000,0.000000,0.000000}%
\pgfsetstrokecolor{currentstroke}%
\pgfsetdash{{5.550000pt}{2.400000pt}}{0.000000pt}%
\pgfpathmoveto{\pgfqpoint{1.158901in}{0.520000in}}%
\pgfpathlineto{\pgfqpoint{1.158901in}{1.803500in}}%
\pgfusepath{stroke}%
\end{pgfscope}%
\begin{pgfscope}%
\pgfsetrectcap%
\pgfsetmiterjoin%
\pgfsetlinewidth{0.803000pt}%
\definecolor{currentstroke}{rgb}{0.000000,0.000000,0.000000}%
\pgfsetstrokecolor{currentstroke}%
\pgfsetdash{}{0pt}%
\pgfpathmoveto{\pgfqpoint{0.556596in}{0.520000in}}%
\pgfpathlineto{\pgfqpoint{0.556596in}{1.803500in}}%
\pgfusepath{stroke}%
\end{pgfscope}%
\begin{pgfscope}%
\pgfsetrectcap%
\pgfsetmiterjoin%
\pgfsetlinewidth{0.803000pt}%
\definecolor{currentstroke}{rgb}{0.000000,0.000000,0.000000}%
\pgfsetstrokecolor{currentstroke}%
\pgfsetdash{}{0pt}%
\pgfpathmoveto{\pgfqpoint{1.874096in}{0.520000in}}%
\pgfpathlineto{\pgfqpoint{1.874096in}{1.803500in}}%
\pgfusepath{stroke}%
\end{pgfscope}%
\begin{pgfscope}%
\pgfsetrectcap%
\pgfsetmiterjoin%
\pgfsetlinewidth{0.803000pt}%
\definecolor{currentstroke}{rgb}{0.000000,0.000000,0.000000}%
\pgfsetstrokecolor{currentstroke}%
\pgfsetdash{}{0pt}%
\pgfpathmoveto{\pgfqpoint{0.556596in}{0.520000in}}%
\pgfpathlineto{\pgfqpoint{1.874096in}{0.520000in}}%
\pgfusepath{stroke}%
\end{pgfscope}%
\begin{pgfscope}%
\pgfsetrectcap%
\pgfsetmiterjoin%
\pgfsetlinewidth{0.803000pt}%
\definecolor{currentstroke}{rgb}{0.000000,0.000000,0.000000}%
\pgfsetstrokecolor{currentstroke}%
\pgfsetdash{}{0pt}%
\pgfpathmoveto{\pgfqpoint{0.556596in}{1.803500in}}%
\pgfpathlineto{\pgfqpoint{1.874096in}{1.803500in}}%
\pgfusepath{stroke}%
\end{pgfscope}%
\begin{pgfscope}%
\pgfsetbuttcap%
\pgfsetmiterjoin%
\definecolor{currentfill}{rgb}{1.000000,1.000000,1.000000}%
\pgfsetfillcolor{currentfill}%
\pgfsetfillopacity{0.500000}%
\pgfsetlinewidth{1.003750pt}%
\definecolor{currentstroke}{rgb}{0.800000,0.800000,0.800000}%
\pgfsetstrokecolor{currentstroke}%
\pgfsetstrokeopacity{0.500000}%
\pgfsetdash{}{0pt}%
\pgfpathmoveto{\pgfqpoint{1.246700in}{1.354889in}}%
\pgfpathlineto{\pgfqpoint{1.786596in}{1.354889in}}%
\pgfpathquadraticcurveto{\pgfqpoint{1.811596in}{1.354889in}}{\pgfqpoint{1.811596in}{1.379889in}}%
\pgfpathlineto{\pgfqpoint{1.811596in}{1.716000in}}%
\pgfpathquadraticcurveto{\pgfqpoint{1.811596in}{1.741000in}}{\pgfqpoint{1.786596in}{1.741000in}}%
\pgfpathlineto{\pgfqpoint{1.246700in}{1.741000in}}%
\pgfpathquadraticcurveto{\pgfqpoint{1.221700in}{1.741000in}}{\pgfqpoint{1.221700in}{1.716000in}}%
\pgfpathlineto{\pgfqpoint{1.221700in}{1.379889in}}%
\pgfpathquadraticcurveto{\pgfqpoint{1.221700in}{1.354889in}}{\pgfqpoint{1.246700in}{1.354889in}}%
\pgfpathclose%
\pgfusepath{stroke,fill}%
\end{pgfscope}%
\begin{pgfscope}%
\pgfsetbuttcap%
\pgfsetmiterjoin%
\definecolor{currentfill}{rgb}{0.000000,0.000000,1.000000}%
\pgfsetfillcolor{currentfill}%
\pgfsetfillopacity{0.500000}%
\pgfsetlinewidth{0.000000pt}%
\definecolor{currentstroke}{rgb}{0.000000,0.000000,0.000000}%
\pgfsetstrokecolor{currentstroke}%
\pgfsetstrokeopacity{0.500000}%
\pgfsetdash{}{0pt}%
\pgfpathmoveto{\pgfqpoint{1.271700in}{1.603500in}}%
\pgfpathlineto{\pgfqpoint{1.521700in}{1.603500in}}%
\pgfpathlineto{\pgfqpoint{1.521700in}{1.691000in}}%
\pgfpathlineto{\pgfqpoint{1.271700in}{1.691000in}}%
\pgfpathclose%
\pgfusepath{fill}%
\end{pgfscope}%
\begin{pgfscope}%
\definecolor{textcolor}{rgb}{0.000000,0.000000,0.000000}%
\pgfsetstrokecolor{textcolor}%
\pgfsetfillcolor{textcolor}%
\pgftext[x=1.621700in,y=1.603500in,left,base]{\color{textcolor}\rmfamily\fontsize{9.000000}{10.800000}\selectfont \(\displaystyle P_0\)}%
\end{pgfscope}%
\begin{pgfscope}%
\pgfsetbuttcap%
\pgfsetmiterjoin%
\definecolor{currentfill}{rgb}{1.000000,0.000000,0.000000}%
\pgfsetfillcolor{currentfill}%
\pgfsetfillopacity{0.500000}%
\pgfsetlinewidth{0.000000pt}%
\definecolor{currentstroke}{rgb}{0.000000,0.000000,0.000000}%
\pgfsetstrokecolor{currentstroke}%
\pgfsetstrokeopacity{0.500000}%
\pgfsetdash{}{0pt}%
\pgfpathmoveto{\pgfqpoint{1.271700in}{1.429195in}}%
\pgfpathlineto{\pgfqpoint{1.521700in}{1.429195in}}%
\pgfpathlineto{\pgfqpoint{1.521700in}{1.516695in}}%
\pgfpathlineto{\pgfqpoint{1.271700in}{1.516695in}}%
\pgfpathclose%
\pgfusepath{fill}%
\end{pgfscope}%
\begin{pgfscope}%
\definecolor{textcolor}{rgb}{0.000000,0.000000,0.000000}%
\pgfsetstrokecolor{textcolor}%
\pgfsetfillcolor{textcolor}%
\pgftext[x=1.621700in,y=1.429195in,left,base]{\color{textcolor}\rmfamily\fontsize{9.000000}{10.800000}\selectfont \(\displaystyle P_1\)}%
\end{pgfscope}%
\end{pgfpicture}%
\makeatother%
\endgroup%

%% file: images/debiased.pgf
\begingroup%
\makeatletter%
\begin{pgfpicture}%
\pgfpathrectangle{\pgfpointorigin}{\pgfqpoint{2.037734in}{1.945621in}}%
\pgfusepath{use as bounding box, clip}%
\begin{pgfscope}%
\pgfsetbuttcap%
\pgfsetmiterjoin%
\definecolor{currentfill}{rgb}{1.000000,1.000000,1.000000}%
\pgfsetfillcolor{currentfill}%
\pgfsetlinewidth{0.000000pt}%
\definecolor{currentstroke}{rgb}{1.000000,1.000000,1.000000}%
\pgfsetstrokecolor{currentstroke}%
\pgfsetdash{}{0pt}%
\pgfpathmoveto{\pgfqpoint{0.000000in}{0.000000in}}%
\pgfpathlineto{\pgfqpoint{2.037734in}{0.000000in}}%
\pgfpathlineto{\pgfqpoint{2.037734in}{1.945621in}}%
\pgfpathlineto{\pgfqpoint{0.000000in}{1.945621in}}%
\pgfpathclose%
\pgfusepath{fill}%
\end{pgfscope}%
\begin{pgfscope}%
\pgfsetbuttcap%
\pgfsetmiterjoin%
\definecolor{currentfill}{rgb}{1.000000,1.000000,1.000000}%
\pgfsetfillcolor{currentfill}%
\pgfsetlinewidth{0.000000pt}%
\definecolor{currentstroke}{rgb}{0.000000,0.000000,0.000000}%
\pgfsetstrokecolor{currentstroke}%
\pgfsetstrokeopacity{0.000000}%
\pgfsetdash{}{0pt}%
\pgfpathmoveto{\pgfqpoint{0.556596in}{0.520000in}}%
\pgfpathlineto{\pgfqpoint{1.874096in}{0.520000in}}%
\pgfpathlineto{\pgfqpoint{1.874096in}{1.803500in}}%
\pgfpathlineto{\pgfqpoint{0.556596in}{1.803500in}}%
\pgfpathclose%
\pgfusepath{fill}%
\end{pgfscope}%
\begin{pgfscope}%
\pgfpathrectangle{\pgfqpoint{0.556596in}{0.520000in}}{\pgfqpoint{1.317500in}{1.283500in}}%
\pgfusepath{clip}%
\pgfsetbuttcap%
\pgfsetmiterjoin%
\definecolor{currentfill}{rgb}{0.000000,0.000000,1.000000}%
\pgfsetfillcolor{currentfill}%
\pgfsetfillopacity{0.500000}%
\pgfsetlinewidth{0.000000pt}%
\definecolor{currentstroke}{rgb}{0.000000,0.000000,0.000000}%
\pgfsetstrokecolor{currentstroke}%
\pgfsetstrokeopacity{0.500000}%
\pgfsetdash{}{0pt}%
\pgfpathmoveto{\pgfqpoint{0.616482in}{0.520000in}}%
\pgfpathlineto{\pgfqpoint{0.665020in}{0.520000in}}%
\pgfpathlineto{\pgfqpoint{0.665020in}{0.528548in}}%
\pgfpathlineto{\pgfqpoint{0.616482in}{0.528548in}}%
\pgfpathclose%
\pgfusepath{fill}%
\end{pgfscope}%
\begin{pgfscope}%
\pgfpathrectangle{\pgfqpoint{0.556596in}{0.520000in}}{\pgfqpoint{1.317500in}{1.283500in}}%
\pgfusepath{clip}%
\pgfsetbuttcap%
\pgfsetmiterjoin%
\definecolor{currentfill}{rgb}{0.000000,0.000000,1.000000}%
\pgfsetfillcolor{currentfill}%
\pgfsetfillopacity{0.500000}%
\pgfsetlinewidth{0.000000pt}%
\definecolor{currentstroke}{rgb}{0.000000,0.000000,0.000000}%
\pgfsetstrokecolor{currentstroke}%
\pgfsetstrokeopacity{0.500000}%
\pgfsetdash{}{0pt}%
\pgfpathmoveto{\pgfqpoint{0.665020in}{0.520000in}}%
\pgfpathlineto{\pgfqpoint{0.713557in}{0.520000in}}%
\pgfpathlineto{\pgfqpoint{0.713557in}{0.545644in}}%
\pgfpathlineto{\pgfqpoint{0.665020in}{0.545644in}}%
\pgfpathclose%
\pgfusepath{fill}%
\end{pgfscope}%
\begin{pgfscope}%
\pgfpathrectangle{\pgfqpoint{0.556596in}{0.520000in}}{\pgfqpoint{1.317500in}{1.283500in}}%
\pgfusepath{clip}%
\pgfsetbuttcap%
\pgfsetmiterjoin%
\definecolor{currentfill}{rgb}{0.000000,0.000000,1.000000}%
\pgfsetfillcolor{currentfill}%
\pgfsetfillopacity{0.500000}%
\pgfsetlinewidth{0.000000pt}%
\definecolor{currentstroke}{rgb}{0.000000,0.000000,0.000000}%
\pgfsetstrokecolor{currentstroke}%
\pgfsetstrokeopacity{0.500000}%
\pgfsetdash{}{0pt}%
\pgfpathmoveto{\pgfqpoint{0.713557in}{0.520000in}}%
\pgfpathlineto{\pgfqpoint{0.762095in}{0.520000in}}%
\pgfpathlineto{\pgfqpoint{0.762095in}{0.562741in}}%
\pgfpathlineto{\pgfqpoint{0.713557in}{0.562741in}}%
\pgfpathclose%
\pgfusepath{fill}%
\end{pgfscope}%
\begin{pgfscope}%
\pgfpathrectangle{\pgfqpoint{0.556596in}{0.520000in}}{\pgfqpoint{1.317500in}{1.283500in}}%
\pgfusepath{clip}%
\pgfsetbuttcap%
\pgfsetmiterjoin%
\definecolor{currentfill}{rgb}{0.000000,0.000000,1.000000}%
\pgfsetfillcolor{currentfill}%
\pgfsetfillopacity{0.500000}%
\pgfsetlinewidth{0.000000pt}%
\definecolor{currentstroke}{rgb}{0.000000,0.000000,0.000000}%
\pgfsetstrokecolor{currentstroke}%
\pgfsetstrokeopacity{0.500000}%
\pgfsetdash{}{0pt}%
\pgfpathmoveto{\pgfqpoint{0.762095in}{0.520000in}}%
\pgfpathlineto{\pgfqpoint{0.810633in}{0.520000in}}%
\pgfpathlineto{\pgfqpoint{0.810633in}{0.596933in}}%
\pgfpathlineto{\pgfqpoint{0.762095in}{0.596933in}}%
\pgfpathclose%
\pgfusepath{fill}%
\end{pgfscope}%
\begin{pgfscope}%
\pgfpathrectangle{\pgfqpoint{0.556596in}{0.520000in}}{\pgfqpoint{1.317500in}{1.283500in}}%
\pgfusepath{clip}%
\pgfsetbuttcap%
\pgfsetmiterjoin%
\definecolor{currentfill}{rgb}{0.000000,0.000000,1.000000}%
\pgfsetfillcolor{currentfill}%
\pgfsetfillopacity{0.500000}%
\pgfsetlinewidth{0.000000pt}%
\definecolor{currentstroke}{rgb}{0.000000,0.000000,0.000000}%
\pgfsetstrokecolor{currentstroke}%
\pgfsetstrokeopacity{0.500000}%
\pgfsetdash{}{0pt}%
\pgfpathmoveto{\pgfqpoint{0.810633in}{0.520000in}}%
\pgfpathlineto{\pgfqpoint{0.859170in}{0.520000in}}%
\pgfpathlineto{\pgfqpoint{0.859170in}{0.699511in}}%
\pgfpathlineto{\pgfqpoint{0.810633in}{0.699511in}}%
\pgfpathclose%
\pgfusepath{fill}%
\end{pgfscope}%
\begin{pgfscope}%
\pgfpathrectangle{\pgfqpoint{0.556596in}{0.520000in}}{\pgfqpoint{1.317500in}{1.283500in}}%
\pgfusepath{clip}%
\pgfsetbuttcap%
\pgfsetmiterjoin%
\definecolor{currentfill}{rgb}{0.000000,0.000000,1.000000}%
\pgfsetfillcolor{currentfill}%
\pgfsetfillopacity{0.500000}%
\pgfsetlinewidth{0.000000pt}%
\definecolor{currentstroke}{rgb}{0.000000,0.000000,0.000000}%
\pgfsetstrokecolor{currentstroke}%
\pgfsetstrokeopacity{0.500000}%
\pgfsetdash{}{0pt}%
\pgfpathmoveto{\pgfqpoint{0.859170in}{0.520000in}}%
\pgfpathlineto{\pgfqpoint{0.907708in}{0.520000in}}%
\pgfpathlineto{\pgfqpoint{0.907708in}{0.810636in}}%
\pgfpathlineto{\pgfqpoint{0.859170in}{0.810636in}}%
\pgfpathclose%
\pgfusepath{fill}%
\end{pgfscope}%
\begin{pgfscope}%
\pgfpathrectangle{\pgfqpoint{0.556596in}{0.520000in}}{\pgfqpoint{1.317500in}{1.283500in}}%
\pgfusepath{clip}%
\pgfsetbuttcap%
\pgfsetmiterjoin%
\definecolor{currentfill}{rgb}{0.000000,0.000000,1.000000}%
\pgfsetfillcolor{currentfill}%
\pgfsetfillopacity{0.500000}%
\pgfsetlinewidth{0.000000pt}%
\definecolor{currentstroke}{rgb}{0.000000,0.000000,0.000000}%
\pgfsetstrokecolor{currentstroke}%
\pgfsetstrokeopacity{0.500000}%
\pgfsetdash{}{0pt}%
\pgfpathmoveto{\pgfqpoint{0.907708in}{0.520000in}}%
\pgfpathlineto{\pgfqpoint{0.956245in}{0.520000in}}%
\pgfpathlineto{\pgfqpoint{0.956245in}{0.981599in}}%
\pgfpathlineto{\pgfqpoint{0.907708in}{0.981599in}}%
\pgfpathclose%
\pgfusepath{fill}%
\end{pgfscope}%
\begin{pgfscope}%
\pgfpathrectangle{\pgfqpoint{0.556596in}{0.520000in}}{\pgfqpoint{1.317500in}{1.283500in}}%
\pgfusepath{clip}%
\pgfsetbuttcap%
\pgfsetmiterjoin%
\definecolor{currentfill}{rgb}{0.000000,0.000000,1.000000}%
\pgfsetfillcolor{currentfill}%
\pgfsetfillopacity{0.500000}%
\pgfsetlinewidth{0.000000pt}%
\definecolor{currentstroke}{rgb}{0.000000,0.000000,0.000000}%
\pgfsetstrokecolor{currentstroke}%
\pgfsetstrokeopacity{0.500000}%
\pgfsetdash{}{0pt}%
\pgfpathmoveto{\pgfqpoint{0.956245in}{0.520000in}}%
\pgfpathlineto{\pgfqpoint{1.004783in}{0.520000in}}%
\pgfpathlineto{\pgfqpoint{1.004783in}{1.323523in}}%
\pgfpathlineto{\pgfqpoint{0.956245in}{1.323523in}}%
\pgfpathclose%
\pgfusepath{fill}%
\end{pgfscope}%
\begin{pgfscope}%
\pgfpathrectangle{\pgfqpoint{0.556596in}{0.520000in}}{\pgfqpoint{1.317500in}{1.283500in}}%
\pgfusepath{clip}%
\pgfsetbuttcap%
\pgfsetmiterjoin%
\definecolor{currentfill}{rgb}{0.000000,0.000000,1.000000}%
\pgfsetfillcolor{currentfill}%
\pgfsetfillopacity{0.500000}%
\pgfsetlinewidth{0.000000pt}%
\definecolor{currentstroke}{rgb}{0.000000,0.000000,0.000000}%
\pgfsetstrokecolor{currentstroke}%
\pgfsetstrokeopacity{0.500000}%
\pgfsetdash{}{0pt}%
\pgfpathmoveto{\pgfqpoint{1.004783in}{0.520000in}}%
\pgfpathlineto{\pgfqpoint{1.053320in}{0.520000in}}%
\pgfpathlineto{\pgfqpoint{1.053320in}{1.220946in}}%
\pgfpathlineto{\pgfqpoint{1.004783in}{1.220946in}}%
\pgfpathclose%
\pgfusepath{fill}%
\end{pgfscope}%
\begin{pgfscope}%
\pgfpathrectangle{\pgfqpoint{0.556596in}{0.520000in}}{\pgfqpoint{1.317500in}{1.283500in}}%
\pgfusepath{clip}%
\pgfsetbuttcap%
\pgfsetmiterjoin%
\definecolor{currentfill}{rgb}{0.000000,0.000000,1.000000}%
\pgfsetfillcolor{currentfill}%
\pgfsetfillopacity{0.500000}%
\pgfsetlinewidth{0.000000pt}%
\definecolor{currentstroke}{rgb}{0.000000,0.000000,0.000000}%
\pgfsetstrokecolor{currentstroke}%
\pgfsetstrokeopacity{0.500000}%
\pgfsetdash{}{0pt}%
\pgfpathmoveto{\pgfqpoint{1.053320in}{0.520000in}}%
\pgfpathlineto{\pgfqpoint{1.101858in}{0.520000in}}%
\pgfpathlineto{\pgfqpoint{1.101858in}{1.742381in}}%
\pgfpathlineto{\pgfqpoint{1.053320in}{1.742381in}}%
\pgfpathclose%
\pgfusepath{fill}%
\end{pgfscope}%
\begin{pgfscope}%
\pgfpathrectangle{\pgfqpoint{0.556596in}{0.520000in}}{\pgfqpoint{1.317500in}{1.283500in}}%
\pgfusepath{clip}%
\pgfsetbuttcap%
\pgfsetmiterjoin%
\definecolor{currentfill}{rgb}{0.000000,0.000000,1.000000}%
\pgfsetfillcolor{currentfill}%
\pgfsetfillopacity{0.500000}%
\pgfsetlinewidth{0.000000pt}%
\definecolor{currentstroke}{rgb}{0.000000,0.000000,0.000000}%
\pgfsetstrokecolor{currentstroke}%
\pgfsetstrokeopacity{0.500000}%
\pgfsetdash{}{0pt}%
\pgfpathmoveto{\pgfqpoint{1.101858in}{0.520000in}}%
\pgfpathlineto{\pgfqpoint{1.150395in}{0.520000in}}%
\pgfpathlineto{\pgfqpoint{1.150395in}{1.682544in}}%
\pgfpathlineto{\pgfqpoint{1.101858in}{1.682544in}}%
\pgfpathclose%
\pgfusepath{fill}%
\end{pgfscope}%
\begin{pgfscope}%
\pgfpathrectangle{\pgfqpoint{0.556596in}{0.520000in}}{\pgfqpoint{1.317500in}{1.283500in}}%
\pgfusepath{clip}%
\pgfsetbuttcap%
\pgfsetmiterjoin%
\definecolor{currentfill}{rgb}{0.000000,0.000000,1.000000}%
\pgfsetfillcolor{currentfill}%
\pgfsetfillopacity{0.500000}%
\pgfsetlinewidth{0.000000pt}%
\definecolor{currentstroke}{rgb}{0.000000,0.000000,0.000000}%
\pgfsetstrokecolor{currentstroke}%
\pgfsetstrokeopacity{0.500000}%
\pgfsetdash{}{0pt}%
\pgfpathmoveto{\pgfqpoint{1.150395in}{0.520000in}}%
\pgfpathlineto{\pgfqpoint{1.198933in}{0.520000in}}%
\pgfpathlineto{\pgfqpoint{1.198933in}{1.597063in}}%
\pgfpathlineto{\pgfqpoint{1.150395in}{1.597063in}}%
\pgfpathclose%
\pgfusepath{fill}%
\end{pgfscope}%
\begin{pgfscope}%
\pgfpathrectangle{\pgfqpoint{0.556596in}{0.520000in}}{\pgfqpoint{1.317500in}{1.283500in}}%
\pgfusepath{clip}%
\pgfsetbuttcap%
\pgfsetmiterjoin%
\definecolor{currentfill}{rgb}{0.000000,0.000000,1.000000}%
\pgfsetfillcolor{currentfill}%
\pgfsetfillopacity{0.500000}%
\pgfsetlinewidth{0.000000pt}%
\definecolor{currentstroke}{rgb}{0.000000,0.000000,0.000000}%
\pgfsetstrokecolor{currentstroke}%
\pgfsetstrokeopacity{0.500000}%
\pgfsetdash{}{0pt}%
\pgfpathmoveto{\pgfqpoint{1.198933in}{0.520000in}}%
\pgfpathlineto{\pgfqpoint{1.247470in}{0.520000in}}%
\pgfpathlineto{\pgfqpoint{1.247470in}{1.366264in}}%
\pgfpathlineto{\pgfqpoint{1.198933in}{1.366264in}}%
\pgfpathclose%
\pgfusepath{fill}%
\end{pgfscope}%
\begin{pgfscope}%
\pgfpathrectangle{\pgfqpoint{0.556596in}{0.520000in}}{\pgfqpoint{1.317500in}{1.283500in}}%
\pgfusepath{clip}%
\pgfsetbuttcap%
\pgfsetmiterjoin%
\definecolor{currentfill}{rgb}{0.000000,0.000000,1.000000}%
\pgfsetfillcolor{currentfill}%
\pgfsetfillopacity{0.500000}%
\pgfsetlinewidth{0.000000pt}%
\definecolor{currentstroke}{rgb}{0.000000,0.000000,0.000000}%
\pgfsetstrokecolor{currentstroke}%
\pgfsetstrokeopacity{0.500000}%
\pgfsetdash{}{0pt}%
\pgfpathmoveto{\pgfqpoint{1.247470in}{0.520000in}}%
\pgfpathlineto{\pgfqpoint{1.296008in}{0.520000in}}%
\pgfpathlineto{\pgfqpoint{1.296008in}{1.118368in}}%
\pgfpathlineto{\pgfqpoint{1.247470in}{1.118368in}}%
\pgfpathclose%
\pgfusepath{fill}%
\end{pgfscope}%
\begin{pgfscope}%
\pgfpathrectangle{\pgfqpoint{0.556596in}{0.520000in}}{\pgfqpoint{1.317500in}{1.283500in}}%
\pgfusepath{clip}%
\pgfsetbuttcap%
\pgfsetmiterjoin%
\definecolor{currentfill}{rgb}{0.000000,0.000000,1.000000}%
\pgfsetfillcolor{currentfill}%
\pgfsetfillopacity{0.500000}%
\pgfsetlinewidth{0.000000pt}%
\definecolor{currentstroke}{rgb}{0.000000,0.000000,0.000000}%
\pgfsetstrokecolor{currentstroke}%
\pgfsetstrokeopacity{0.500000}%
\pgfsetdash{}{0pt}%
\pgfpathmoveto{\pgfqpoint{1.296008in}{0.520000in}}%
\pgfpathlineto{\pgfqpoint{1.344546in}{0.520000in}}%
\pgfpathlineto{\pgfqpoint{1.344546in}{1.007243in}}%
\pgfpathlineto{\pgfqpoint{1.296008in}{1.007243in}}%
\pgfpathclose%
\pgfusepath{fill}%
\end{pgfscope}%
\begin{pgfscope}%
\pgfpathrectangle{\pgfqpoint{0.556596in}{0.520000in}}{\pgfqpoint{1.317500in}{1.283500in}}%
\pgfusepath{clip}%
\pgfsetbuttcap%
\pgfsetmiterjoin%
\definecolor{currentfill}{rgb}{0.000000,0.000000,1.000000}%
\pgfsetfillcolor{currentfill}%
\pgfsetfillopacity{0.500000}%
\pgfsetlinewidth{0.000000pt}%
\definecolor{currentstroke}{rgb}{0.000000,0.000000,0.000000}%
\pgfsetstrokecolor{currentstroke}%
\pgfsetstrokeopacity{0.500000}%
\pgfsetdash{}{0pt}%
\pgfpathmoveto{\pgfqpoint{1.344546in}{0.520000in}}%
\pgfpathlineto{\pgfqpoint{1.393083in}{0.520000in}}%
\pgfpathlineto{\pgfqpoint{1.393083in}{0.784992in}}%
\pgfpathlineto{\pgfqpoint{1.344546in}{0.784992in}}%
\pgfpathclose%
\pgfusepath{fill}%
\end{pgfscope}%
\begin{pgfscope}%
\pgfpathrectangle{\pgfqpoint{0.556596in}{0.520000in}}{\pgfqpoint{1.317500in}{1.283500in}}%
\pgfusepath{clip}%
\pgfsetbuttcap%
\pgfsetmiterjoin%
\definecolor{currentfill}{rgb}{0.000000,0.000000,1.000000}%
\pgfsetfillcolor{currentfill}%
\pgfsetfillopacity{0.500000}%
\pgfsetlinewidth{0.000000pt}%
\definecolor{currentstroke}{rgb}{0.000000,0.000000,0.000000}%
\pgfsetstrokecolor{currentstroke}%
\pgfsetstrokeopacity{0.500000}%
\pgfsetdash{}{0pt}%
\pgfpathmoveto{\pgfqpoint{1.393083in}{0.520000in}}%
\pgfpathlineto{\pgfqpoint{1.441621in}{0.520000in}}%
\pgfpathlineto{\pgfqpoint{1.441621in}{0.673866in}}%
\pgfpathlineto{\pgfqpoint{1.393083in}{0.673866in}}%
\pgfpathclose%
\pgfusepath{fill}%
\end{pgfscope}%
\begin{pgfscope}%
\pgfpathrectangle{\pgfqpoint{0.556596in}{0.520000in}}{\pgfqpoint{1.317500in}{1.283500in}}%
\pgfusepath{clip}%
\pgfsetbuttcap%
\pgfsetmiterjoin%
\definecolor{currentfill}{rgb}{0.000000,0.000000,1.000000}%
\pgfsetfillcolor{currentfill}%
\pgfsetfillopacity{0.500000}%
\pgfsetlinewidth{0.000000pt}%
\definecolor{currentstroke}{rgb}{0.000000,0.000000,0.000000}%
\pgfsetstrokecolor{currentstroke}%
\pgfsetstrokeopacity{0.500000}%
\pgfsetdash{}{0pt}%
\pgfpathmoveto{\pgfqpoint{1.441621in}{0.520000in}}%
\pgfpathlineto{\pgfqpoint{1.490158in}{0.520000in}}%
\pgfpathlineto{\pgfqpoint{1.490158in}{0.579837in}}%
\pgfpathlineto{\pgfqpoint{1.441621in}{0.579837in}}%
\pgfpathclose%
\pgfusepath{fill}%
\end{pgfscope}%
\begin{pgfscope}%
\pgfpathrectangle{\pgfqpoint{0.556596in}{0.520000in}}{\pgfqpoint{1.317500in}{1.283500in}}%
\pgfusepath{clip}%
\pgfsetbuttcap%
\pgfsetmiterjoin%
\definecolor{currentfill}{rgb}{0.000000,0.000000,1.000000}%
\pgfsetfillcolor{currentfill}%
\pgfsetfillopacity{0.500000}%
\pgfsetlinewidth{0.000000pt}%
\definecolor{currentstroke}{rgb}{0.000000,0.000000,0.000000}%
\pgfsetstrokecolor{currentstroke}%
\pgfsetstrokeopacity{0.500000}%
\pgfsetdash{}{0pt}%
\pgfpathmoveto{\pgfqpoint{1.490158in}{0.520000in}}%
\pgfpathlineto{\pgfqpoint{1.538696in}{0.520000in}}%
\pgfpathlineto{\pgfqpoint{1.538696in}{0.571289in}}%
\pgfpathlineto{\pgfqpoint{1.490158in}{0.571289in}}%
\pgfpathclose%
\pgfusepath{fill}%
\end{pgfscope}%
\begin{pgfscope}%
\pgfpathrectangle{\pgfqpoint{0.556596in}{0.520000in}}{\pgfqpoint{1.317500in}{1.283500in}}%
\pgfusepath{clip}%
\pgfsetbuttcap%
\pgfsetmiterjoin%
\definecolor{currentfill}{rgb}{0.000000,0.000000,1.000000}%
\pgfsetfillcolor{currentfill}%
\pgfsetfillopacity{0.500000}%
\pgfsetlinewidth{0.000000pt}%
\definecolor{currentstroke}{rgb}{0.000000,0.000000,0.000000}%
\pgfsetstrokecolor{currentstroke}%
\pgfsetstrokeopacity{0.500000}%
\pgfsetdash{}{0pt}%
\pgfpathmoveto{\pgfqpoint{1.538696in}{0.520000in}}%
\pgfpathlineto{\pgfqpoint{1.587233in}{0.520000in}}%
\pgfpathlineto{\pgfqpoint{1.587233in}{0.554193in}}%
\pgfpathlineto{\pgfqpoint{1.538696in}{0.554193in}}%
\pgfpathclose%
\pgfusepath{fill}%
\end{pgfscope}%
\begin{pgfscope}%
\pgfpathrectangle{\pgfqpoint{0.556596in}{0.520000in}}{\pgfqpoint{1.317500in}{1.283500in}}%
\pgfusepath{clip}%
\pgfsetbuttcap%
\pgfsetmiterjoin%
\definecolor{currentfill}{rgb}{1.000000,0.000000,0.000000}%
\pgfsetfillcolor{currentfill}%
\pgfsetfillopacity{0.500000}%
\pgfsetlinewidth{0.000000pt}%
\definecolor{currentstroke}{rgb}{0.000000,0.000000,0.000000}%
\pgfsetstrokecolor{currentstroke}%
\pgfsetstrokeopacity{0.500000}%
\pgfsetdash{}{0pt}%
\pgfpathmoveto{\pgfqpoint{0.867487in}{0.520000in}}%
\pgfpathlineto{\pgfqpoint{0.914823in}{0.520000in}}%
\pgfpathlineto{\pgfqpoint{0.914823in}{0.571289in}}%
\pgfpathlineto{\pgfqpoint{0.867487in}{0.571289in}}%
\pgfpathclose%
\pgfusepath{fill}%
\end{pgfscope}%
\begin{pgfscope}%
\pgfpathrectangle{\pgfqpoint{0.556596in}{0.520000in}}{\pgfqpoint{1.317500in}{1.283500in}}%
\pgfusepath{clip}%
\pgfsetbuttcap%
\pgfsetmiterjoin%
\definecolor{currentfill}{rgb}{1.000000,0.000000,0.000000}%
\pgfsetfillcolor{currentfill}%
\pgfsetfillopacity{0.500000}%
\pgfsetlinewidth{0.000000pt}%
\definecolor{currentstroke}{rgb}{0.000000,0.000000,0.000000}%
\pgfsetstrokecolor{currentstroke}%
\pgfsetstrokeopacity{0.500000}%
\pgfsetdash{}{0pt}%
\pgfpathmoveto{\pgfqpoint{0.914823in}{0.520000in}}%
\pgfpathlineto{\pgfqpoint{0.962160in}{0.520000in}}%
\pgfpathlineto{\pgfqpoint{0.962160in}{0.554193in}}%
\pgfpathlineto{\pgfqpoint{0.914823in}{0.554193in}}%
\pgfpathclose%
\pgfusepath{fill}%
\end{pgfscope}%
\begin{pgfscope}%
\pgfpathrectangle{\pgfqpoint{0.556596in}{0.520000in}}{\pgfqpoint{1.317500in}{1.283500in}}%
\pgfusepath{clip}%
\pgfsetbuttcap%
\pgfsetmiterjoin%
\definecolor{currentfill}{rgb}{1.000000,0.000000,0.000000}%
\pgfsetfillcolor{currentfill}%
\pgfsetfillopacity{0.500000}%
\pgfsetlinewidth{0.000000pt}%
\definecolor{currentstroke}{rgb}{0.000000,0.000000,0.000000}%
\pgfsetstrokecolor{currentstroke}%
\pgfsetstrokeopacity{0.500000}%
\pgfsetdash{}{0pt}%
\pgfpathmoveto{\pgfqpoint{0.962160in}{0.520000in}}%
\pgfpathlineto{\pgfqpoint{1.009496in}{0.520000in}}%
\pgfpathlineto{\pgfqpoint{1.009496in}{0.588385in}}%
\pgfpathlineto{\pgfqpoint{0.962160in}{0.588385in}}%
\pgfpathclose%
\pgfusepath{fill}%
\end{pgfscope}%
\begin{pgfscope}%
\pgfpathrectangle{\pgfqpoint{0.556596in}{0.520000in}}{\pgfqpoint{1.317500in}{1.283500in}}%
\pgfusepath{clip}%
\pgfsetbuttcap%
\pgfsetmiterjoin%
\definecolor{currentfill}{rgb}{1.000000,0.000000,0.000000}%
\pgfsetfillcolor{currentfill}%
\pgfsetfillopacity{0.500000}%
\pgfsetlinewidth{0.000000pt}%
\definecolor{currentstroke}{rgb}{0.000000,0.000000,0.000000}%
\pgfsetstrokecolor{currentstroke}%
\pgfsetstrokeopacity{0.500000}%
\pgfsetdash{}{0pt}%
\pgfpathmoveto{\pgfqpoint{1.009496in}{0.520000in}}%
\pgfpathlineto{\pgfqpoint{1.056832in}{0.520000in}}%
\pgfpathlineto{\pgfqpoint{1.056832in}{0.708059in}}%
\pgfpathlineto{\pgfqpoint{1.009496in}{0.708059in}}%
\pgfpathclose%
\pgfusepath{fill}%
\end{pgfscope}%
\begin{pgfscope}%
\pgfpathrectangle{\pgfqpoint{0.556596in}{0.520000in}}{\pgfqpoint{1.317500in}{1.283500in}}%
\pgfusepath{clip}%
\pgfsetbuttcap%
\pgfsetmiterjoin%
\definecolor{currentfill}{rgb}{1.000000,0.000000,0.000000}%
\pgfsetfillcolor{currentfill}%
\pgfsetfillopacity{0.500000}%
\pgfsetlinewidth{0.000000pt}%
\definecolor{currentstroke}{rgb}{0.000000,0.000000,0.000000}%
\pgfsetstrokecolor{currentstroke}%
\pgfsetstrokeopacity{0.500000}%
\pgfsetdash{}{0pt}%
\pgfpathmoveto{\pgfqpoint{1.056832in}{0.520000in}}%
\pgfpathlineto{\pgfqpoint{1.104168in}{0.520000in}}%
\pgfpathlineto{\pgfqpoint{1.104168in}{0.887569in}}%
\pgfpathlineto{\pgfqpoint{1.056832in}{0.887569in}}%
\pgfpathclose%
\pgfusepath{fill}%
\end{pgfscope}%
\begin{pgfscope}%
\pgfpathrectangle{\pgfqpoint{0.556596in}{0.520000in}}{\pgfqpoint{1.317500in}{1.283500in}}%
\pgfusepath{clip}%
\pgfsetbuttcap%
\pgfsetmiterjoin%
\definecolor{currentfill}{rgb}{1.000000,0.000000,0.000000}%
\pgfsetfillcolor{currentfill}%
\pgfsetfillopacity{0.500000}%
\pgfsetlinewidth{0.000000pt}%
\definecolor{currentstroke}{rgb}{0.000000,0.000000,0.000000}%
\pgfsetstrokecolor{currentstroke}%
\pgfsetstrokeopacity{0.500000}%
\pgfsetdash{}{0pt}%
\pgfpathmoveto{\pgfqpoint{1.104168in}{0.520000in}}%
\pgfpathlineto{\pgfqpoint{1.151504in}{0.520000in}}%
\pgfpathlineto{\pgfqpoint{1.151504in}{0.938858in}}%
\pgfpathlineto{\pgfqpoint{1.104168in}{0.938858in}}%
\pgfpathclose%
\pgfusepath{fill}%
\end{pgfscope}%
\begin{pgfscope}%
\pgfpathrectangle{\pgfqpoint{0.556596in}{0.520000in}}{\pgfqpoint{1.317500in}{1.283500in}}%
\pgfusepath{clip}%
\pgfsetbuttcap%
\pgfsetmiterjoin%
\definecolor{currentfill}{rgb}{1.000000,0.000000,0.000000}%
\pgfsetfillcolor{currentfill}%
\pgfsetfillopacity{0.500000}%
\pgfsetlinewidth{0.000000pt}%
\definecolor{currentstroke}{rgb}{0.000000,0.000000,0.000000}%
\pgfsetstrokecolor{currentstroke}%
\pgfsetstrokeopacity{0.500000}%
\pgfsetdash{}{0pt}%
\pgfpathmoveto{\pgfqpoint{1.151504in}{0.520000in}}%
\pgfpathlineto{\pgfqpoint{1.198840in}{0.520000in}}%
\pgfpathlineto{\pgfqpoint{1.198840in}{1.195302in}}%
\pgfpathlineto{\pgfqpoint{1.151504in}{1.195302in}}%
\pgfpathclose%
\pgfusepath{fill}%
\end{pgfscope}%
\begin{pgfscope}%
\pgfpathrectangle{\pgfqpoint{0.556596in}{0.520000in}}{\pgfqpoint{1.317500in}{1.283500in}}%
\pgfusepath{clip}%
\pgfsetbuttcap%
\pgfsetmiterjoin%
\definecolor{currentfill}{rgb}{1.000000,0.000000,0.000000}%
\pgfsetfillcolor{currentfill}%
\pgfsetfillopacity{0.500000}%
\pgfsetlinewidth{0.000000pt}%
\definecolor{currentstroke}{rgb}{0.000000,0.000000,0.000000}%
\pgfsetstrokecolor{currentstroke}%
\pgfsetstrokeopacity{0.500000}%
\pgfsetdash{}{0pt}%
\pgfpathmoveto{\pgfqpoint{1.198840in}{0.520000in}}%
\pgfpathlineto{\pgfqpoint{1.246176in}{0.520000in}}%
\pgfpathlineto{\pgfqpoint{1.246176in}{1.562871in}}%
\pgfpathlineto{\pgfqpoint{1.198840in}{1.562871in}}%
\pgfpathclose%
\pgfusepath{fill}%
\end{pgfscope}%
\begin{pgfscope}%
\pgfpathrectangle{\pgfqpoint{0.556596in}{0.520000in}}{\pgfqpoint{1.317500in}{1.283500in}}%
\pgfusepath{clip}%
\pgfsetbuttcap%
\pgfsetmiterjoin%
\definecolor{currentfill}{rgb}{1.000000,0.000000,0.000000}%
\pgfsetfillcolor{currentfill}%
\pgfsetfillopacity{0.500000}%
\pgfsetlinewidth{0.000000pt}%
\definecolor{currentstroke}{rgb}{0.000000,0.000000,0.000000}%
\pgfsetstrokecolor{currentstroke}%
\pgfsetstrokeopacity{0.500000}%
\pgfsetdash{}{0pt}%
\pgfpathmoveto{\pgfqpoint{1.246176in}{0.520000in}}%
\pgfpathlineto{\pgfqpoint{1.293512in}{0.520000in}}%
\pgfpathlineto{\pgfqpoint{1.293512in}{1.665448in}}%
\pgfpathlineto{\pgfqpoint{1.246176in}{1.665448in}}%
\pgfpathclose%
\pgfusepath{fill}%
\end{pgfscope}%
\begin{pgfscope}%
\pgfpathrectangle{\pgfqpoint{0.556596in}{0.520000in}}{\pgfqpoint{1.317500in}{1.283500in}}%
\pgfusepath{clip}%
\pgfsetbuttcap%
\pgfsetmiterjoin%
\definecolor{currentfill}{rgb}{1.000000,0.000000,0.000000}%
\pgfsetfillcolor{currentfill}%
\pgfsetfillopacity{0.500000}%
\pgfsetlinewidth{0.000000pt}%
\definecolor{currentstroke}{rgb}{0.000000,0.000000,0.000000}%
\pgfsetstrokecolor{currentstroke}%
\pgfsetstrokeopacity{0.500000}%
\pgfsetdash{}{0pt}%
\pgfpathmoveto{\pgfqpoint{1.293512in}{0.520000in}}%
\pgfpathlineto{\pgfqpoint{1.340848in}{0.520000in}}%
\pgfpathlineto{\pgfqpoint{1.340848in}{1.571419in}}%
\pgfpathlineto{\pgfqpoint{1.293512in}{1.571419in}}%
\pgfpathclose%
\pgfusepath{fill}%
\end{pgfscope}%
\begin{pgfscope}%
\pgfpathrectangle{\pgfqpoint{0.556596in}{0.520000in}}{\pgfqpoint{1.317500in}{1.283500in}}%
\pgfusepath{clip}%
\pgfsetbuttcap%
\pgfsetmiterjoin%
\definecolor{currentfill}{rgb}{1.000000,0.000000,0.000000}%
\pgfsetfillcolor{currentfill}%
\pgfsetfillopacity{0.500000}%
\pgfsetlinewidth{0.000000pt}%
\definecolor{currentstroke}{rgb}{0.000000,0.000000,0.000000}%
\pgfsetstrokecolor{currentstroke}%
\pgfsetstrokeopacity{0.500000}%
\pgfsetdash{}{0pt}%
\pgfpathmoveto{\pgfqpoint{1.340848in}{0.520000in}}%
\pgfpathlineto{\pgfqpoint{1.388185in}{0.520000in}}%
\pgfpathlineto{\pgfqpoint{1.388185in}{1.639804in}}%
\pgfpathlineto{\pgfqpoint{1.340848in}{1.639804in}}%
\pgfpathclose%
\pgfusepath{fill}%
\end{pgfscope}%
\begin{pgfscope}%
\pgfpathrectangle{\pgfqpoint{0.556596in}{0.520000in}}{\pgfqpoint{1.317500in}{1.283500in}}%
\pgfusepath{clip}%
\pgfsetbuttcap%
\pgfsetmiterjoin%
\definecolor{currentfill}{rgb}{1.000000,0.000000,0.000000}%
\pgfsetfillcolor{currentfill}%
\pgfsetfillopacity{0.500000}%
\pgfsetlinewidth{0.000000pt}%
\definecolor{currentstroke}{rgb}{0.000000,0.000000,0.000000}%
\pgfsetstrokecolor{currentstroke}%
\pgfsetstrokeopacity{0.500000}%
\pgfsetdash{}{0pt}%
\pgfpathmoveto{\pgfqpoint{1.388185in}{0.520000in}}%
\pgfpathlineto{\pgfqpoint{1.435521in}{0.520000in}}%
\pgfpathlineto{\pgfqpoint{1.435521in}{1.289331in}}%
\pgfpathlineto{\pgfqpoint{1.388185in}{1.289331in}}%
\pgfpathclose%
\pgfusepath{fill}%
\end{pgfscope}%
\begin{pgfscope}%
\pgfpathrectangle{\pgfqpoint{0.556596in}{0.520000in}}{\pgfqpoint{1.317500in}{1.283500in}}%
\pgfusepath{clip}%
\pgfsetbuttcap%
\pgfsetmiterjoin%
\definecolor{currentfill}{rgb}{1.000000,0.000000,0.000000}%
\pgfsetfillcolor{currentfill}%
\pgfsetfillopacity{0.500000}%
\pgfsetlinewidth{0.000000pt}%
\definecolor{currentstroke}{rgb}{0.000000,0.000000,0.000000}%
\pgfsetstrokecolor{currentstroke}%
\pgfsetstrokeopacity{0.500000}%
\pgfsetdash{}{0pt}%
\pgfpathmoveto{\pgfqpoint{1.435521in}{0.520000in}}%
\pgfpathlineto{\pgfqpoint{1.482857in}{0.520000in}}%
\pgfpathlineto{\pgfqpoint{1.482857in}{1.058532in}}%
\pgfpathlineto{\pgfqpoint{1.435521in}{1.058532in}}%
\pgfpathclose%
\pgfusepath{fill}%
\end{pgfscope}%
\begin{pgfscope}%
\pgfpathrectangle{\pgfqpoint{0.556596in}{0.520000in}}{\pgfqpoint{1.317500in}{1.283500in}}%
\pgfusepath{clip}%
\pgfsetbuttcap%
\pgfsetmiterjoin%
\definecolor{currentfill}{rgb}{1.000000,0.000000,0.000000}%
\pgfsetfillcolor{currentfill}%
\pgfsetfillopacity{0.500000}%
\pgfsetlinewidth{0.000000pt}%
\definecolor{currentstroke}{rgb}{0.000000,0.000000,0.000000}%
\pgfsetstrokecolor{currentstroke}%
\pgfsetstrokeopacity{0.500000}%
\pgfsetdash{}{0pt}%
\pgfpathmoveto{\pgfqpoint{1.482857in}{0.520000in}}%
\pgfpathlineto{\pgfqpoint{1.530193in}{0.520000in}}%
\pgfpathlineto{\pgfqpoint{1.530193in}{0.955954in}}%
\pgfpathlineto{\pgfqpoint{1.482857in}{0.955954in}}%
\pgfpathclose%
\pgfusepath{fill}%
\end{pgfscope}%
\begin{pgfscope}%
\pgfpathrectangle{\pgfqpoint{0.556596in}{0.520000in}}{\pgfqpoint{1.317500in}{1.283500in}}%
\pgfusepath{clip}%
\pgfsetbuttcap%
\pgfsetmiterjoin%
\definecolor{currentfill}{rgb}{1.000000,0.000000,0.000000}%
\pgfsetfillcolor{currentfill}%
\pgfsetfillopacity{0.500000}%
\pgfsetlinewidth{0.000000pt}%
\definecolor{currentstroke}{rgb}{0.000000,0.000000,0.000000}%
\pgfsetstrokecolor{currentstroke}%
\pgfsetstrokeopacity{0.500000}%
\pgfsetdash{}{0pt}%
\pgfpathmoveto{\pgfqpoint{1.530193in}{0.520000in}}%
\pgfpathlineto{\pgfqpoint{1.577529in}{0.520000in}}%
\pgfpathlineto{\pgfqpoint{1.577529in}{0.793540in}}%
\pgfpathlineto{\pgfqpoint{1.530193in}{0.793540in}}%
\pgfpathclose%
\pgfusepath{fill}%
\end{pgfscope}%
\begin{pgfscope}%
\pgfpathrectangle{\pgfqpoint{0.556596in}{0.520000in}}{\pgfqpoint{1.317500in}{1.283500in}}%
\pgfusepath{clip}%
\pgfsetbuttcap%
\pgfsetmiterjoin%
\definecolor{currentfill}{rgb}{1.000000,0.000000,0.000000}%
\pgfsetfillcolor{currentfill}%
\pgfsetfillopacity{0.500000}%
\pgfsetlinewidth{0.000000pt}%
\definecolor{currentstroke}{rgb}{0.000000,0.000000,0.000000}%
\pgfsetstrokecolor{currentstroke}%
\pgfsetstrokeopacity{0.500000}%
\pgfsetdash{}{0pt}%
\pgfpathmoveto{\pgfqpoint{1.577529in}{0.520000in}}%
\pgfpathlineto{\pgfqpoint{1.624865in}{0.520000in}}%
\pgfpathlineto{\pgfqpoint{1.624865in}{0.725155in}}%
\pgfpathlineto{\pgfqpoint{1.577529in}{0.725155in}}%
\pgfpathclose%
\pgfusepath{fill}%
\end{pgfscope}%
\begin{pgfscope}%
\pgfpathrectangle{\pgfqpoint{0.556596in}{0.520000in}}{\pgfqpoint{1.317500in}{1.283500in}}%
\pgfusepath{clip}%
\pgfsetbuttcap%
\pgfsetmiterjoin%
\definecolor{currentfill}{rgb}{1.000000,0.000000,0.000000}%
\pgfsetfillcolor{currentfill}%
\pgfsetfillopacity{0.500000}%
\pgfsetlinewidth{0.000000pt}%
\definecolor{currentstroke}{rgb}{0.000000,0.000000,0.000000}%
\pgfsetstrokecolor{currentstroke}%
\pgfsetstrokeopacity{0.500000}%
\pgfsetdash{}{0pt}%
\pgfpathmoveto{\pgfqpoint{1.624865in}{0.520000in}}%
\pgfpathlineto{\pgfqpoint{1.672201in}{0.520000in}}%
\pgfpathlineto{\pgfqpoint{1.672201in}{0.622578in}}%
\pgfpathlineto{\pgfqpoint{1.624865in}{0.622578in}}%
\pgfpathclose%
\pgfusepath{fill}%
\end{pgfscope}%
\begin{pgfscope}%
\pgfpathrectangle{\pgfqpoint{0.556596in}{0.520000in}}{\pgfqpoint{1.317500in}{1.283500in}}%
\pgfusepath{clip}%
\pgfsetbuttcap%
\pgfsetmiterjoin%
\definecolor{currentfill}{rgb}{1.000000,0.000000,0.000000}%
\pgfsetfillcolor{currentfill}%
\pgfsetfillopacity{0.500000}%
\pgfsetlinewidth{0.000000pt}%
\definecolor{currentstroke}{rgb}{0.000000,0.000000,0.000000}%
\pgfsetstrokecolor{currentstroke}%
\pgfsetstrokeopacity{0.500000}%
\pgfsetdash{}{0pt}%
\pgfpathmoveto{\pgfqpoint{1.672201in}{0.520000in}}%
\pgfpathlineto{\pgfqpoint{1.719537in}{0.520000in}}%
\pgfpathlineto{\pgfqpoint{1.719537in}{0.545644in}}%
\pgfpathlineto{\pgfqpoint{1.672201in}{0.545644in}}%
\pgfpathclose%
\pgfusepath{fill}%
\end{pgfscope}%
\begin{pgfscope}%
\pgfpathrectangle{\pgfqpoint{0.556596in}{0.520000in}}{\pgfqpoint{1.317500in}{1.283500in}}%
\pgfusepath{clip}%
\pgfsetbuttcap%
\pgfsetmiterjoin%
\definecolor{currentfill}{rgb}{1.000000,0.000000,0.000000}%
\pgfsetfillcolor{currentfill}%
\pgfsetfillopacity{0.500000}%
\pgfsetlinewidth{0.000000pt}%
\definecolor{currentstroke}{rgb}{0.000000,0.000000,0.000000}%
\pgfsetstrokecolor{currentstroke}%
\pgfsetstrokeopacity{0.500000}%
\pgfsetdash{}{0pt}%
\pgfpathmoveto{\pgfqpoint{1.719537in}{0.520000in}}%
\pgfpathlineto{\pgfqpoint{1.766873in}{0.520000in}}%
\pgfpathlineto{\pgfqpoint{1.766873in}{0.520000in}}%
\pgfpathlineto{\pgfqpoint{1.719537in}{0.520000in}}%
\pgfpathclose%
\pgfusepath{fill}%
\end{pgfscope}%
\begin{pgfscope}%
\pgfpathrectangle{\pgfqpoint{0.556596in}{0.520000in}}{\pgfqpoint{1.317500in}{1.283500in}}%
\pgfusepath{clip}%
\pgfsetbuttcap%
\pgfsetmiterjoin%
\definecolor{currentfill}{rgb}{1.000000,0.000000,0.000000}%
\pgfsetfillcolor{currentfill}%
\pgfsetfillopacity{0.500000}%
\pgfsetlinewidth{0.000000pt}%
\definecolor{currentstroke}{rgb}{0.000000,0.000000,0.000000}%
\pgfsetstrokecolor{currentstroke}%
\pgfsetstrokeopacity{0.500000}%
\pgfsetdash{}{0pt}%
\pgfpathmoveto{\pgfqpoint{1.766873in}{0.520000in}}%
\pgfpathlineto{\pgfqpoint{1.814210in}{0.520000in}}%
\pgfpathlineto{\pgfqpoint{1.814210in}{0.554193in}}%
\pgfpathlineto{\pgfqpoint{1.766873in}{0.554193in}}%
\pgfpathclose%
\pgfusepath{fill}%
\end{pgfscope}%
\begin{pgfscope}%
\pgfsetbuttcap%
\pgfsetroundjoin%
\definecolor{currentfill}{rgb}{0.000000,0.000000,0.000000}%
\pgfsetfillcolor{currentfill}%
\pgfsetlinewidth{0.803000pt}%
\definecolor{currentstroke}{rgb}{0.000000,0.000000,0.000000}%
\pgfsetstrokecolor{currentstroke}%
\pgfsetdash{}{0pt}%
\pgfsys@defobject{currentmarker}{\pgfqpoint{0.000000in}{-0.048611in}}{\pgfqpoint{0.000000in}{0.000000in}}{%
\pgfpathmoveto{\pgfqpoint{0.000000in}{0.000000in}}%
\pgfpathlineto{\pgfqpoint{0.000000in}{-0.048611in}}%
\pgfusepath{stroke,fill}%
}%
\begin{pgfscope}%
\pgfsys@transformshift{1.118027in}{0.520000in}%
\pgfsys@useobject{currentmarker}{}%
\end{pgfscope}%
\end{pgfscope}%
\begin{pgfscope}%
\definecolor{textcolor}{rgb}{0.000000,0.000000,0.000000}%
\pgfsetstrokecolor{textcolor}%
\pgfsetfillcolor{textcolor}%
\pgftext[x=1.118027in,y=0.422778in,,top]{\color{textcolor}\rmfamily\fontsize{9.000000}{10.800000}\selectfont \(\displaystyle {0.000}\)}%
\end{pgfscope}%
\begin{pgfscope}%
\pgfsetbuttcap%
\pgfsetroundjoin%
\definecolor{currentfill}{rgb}{0.000000,0.000000,0.000000}%
\pgfsetfillcolor{currentfill}%
\pgfsetlinewidth{0.803000pt}%
\definecolor{currentstroke}{rgb}{0.000000,0.000000,0.000000}%
\pgfsetstrokecolor{currentstroke}%
\pgfsetdash{}{0pt}%
\pgfsys@defobject{currentmarker}{\pgfqpoint{0.000000in}{-0.048611in}}{\pgfqpoint{0.000000in}{0.000000in}}{%
\pgfpathmoveto{\pgfqpoint{0.000000in}{0.000000in}}%
\pgfpathlineto{\pgfqpoint{0.000000in}{-0.048611in}}%
\pgfusepath{stroke,fill}%
}%
\begin{pgfscope}%
\pgfsys@transformshift{1.791420in}{0.520000in}%
\pgfsys@useobject{currentmarker}{}%
\end{pgfscope}%
\end{pgfscope}%
\begin{pgfscope}%
\definecolor{textcolor}{rgb}{0.000000,0.000000,0.000000}%
\pgfsetstrokecolor{textcolor}%
\pgfsetfillcolor{textcolor}%
\pgftext[x=1.791420in,y=0.422778in,,top]{\color{textcolor}\rmfamily\fontsize{9.000000}{10.800000}\selectfont \(\displaystyle {0.005}\)}%
\end{pgfscope}%
\begin{pgfscope}%
\definecolor{textcolor}{rgb}{0.000000,0.000000,0.000000}%
\pgfsetstrokecolor{textcolor}%
\pgfsetfillcolor{textcolor}%
\pgftext[x=1.215346in,y=0.256111in,,top]{\color{textcolor}\rmfamily\fontsize{9.000000}{10.800000}\selectfont \(\displaystyle T_{m^*, n}^{\mathrm{d}}\)}%
\end{pgfscope}%
\begin{pgfscope}%
\pgfsetbuttcap%
\pgfsetroundjoin%
\definecolor{currentfill}{rgb}{0.000000,0.000000,0.000000}%
\pgfsetfillcolor{currentfill}%
\pgfsetlinewidth{0.803000pt}%
\definecolor{currentstroke}{rgb}{0.000000,0.000000,0.000000}%
\pgfsetstrokecolor{currentstroke}%
\pgfsetdash{}{0pt}%
\pgfsys@defobject{currentmarker}{\pgfqpoint{-0.048611in}{0.000000in}}{\pgfqpoint{0.000000in}{0.000000in}}{%
\pgfpathmoveto{\pgfqpoint{0.000000in}{0.000000in}}%
\pgfpathlineto{\pgfqpoint{-0.048611in}{0.000000in}}%
\pgfusepath{stroke,fill}%
}%
\begin{pgfscope}%
\pgfsys@transformshift{0.556596in}{0.520000in}%
\pgfsys@useobject{currentmarker}{}%
\end{pgfscope}%
\end{pgfscope}%
\begin{pgfscope}%
\definecolor{textcolor}{rgb}{0.000000,0.000000,0.000000}%
\pgfsetstrokecolor{textcolor}%
\pgfsetfillcolor{textcolor}%
\pgftext[x=0.395138in, y=0.476597in, left, base]{\color{textcolor}\rmfamily\fontsize{9.000000}{10.800000}\selectfont \(\displaystyle {0}\)}%
\end{pgfscope}%
\begin{pgfscope}%
\pgfsetbuttcap%
\pgfsetroundjoin%
\definecolor{currentfill}{rgb}{0.000000,0.000000,0.000000}%
\pgfsetfillcolor{currentfill}%
\pgfsetlinewidth{0.803000pt}%
\definecolor{currentstroke}{rgb}{0.000000,0.000000,0.000000}%
\pgfsetstrokecolor{currentstroke}%
\pgfsetdash{}{0pt}%
\pgfsys@defobject{currentmarker}{\pgfqpoint{-0.048611in}{0.000000in}}{\pgfqpoint{0.000000in}{0.000000in}}{%
\pgfpathmoveto{\pgfqpoint{0.000000in}{0.000000in}}%
\pgfpathlineto{\pgfqpoint{-0.048611in}{0.000000in}}%
\pgfusepath{stroke,fill}%
}%
\begin{pgfscope}%
\pgfsys@transformshift{0.556596in}{0.947406in}%
\pgfsys@useobject{currentmarker}{}%
\end{pgfscope}%
\end{pgfscope}%
\begin{pgfscope}%
\definecolor{textcolor}{rgb}{0.000000,0.000000,0.000000}%
\pgfsetstrokecolor{textcolor}%
\pgfsetfillcolor{textcolor}%
\pgftext[x=0.330902in, y=0.904003in, left, base]{\color{textcolor}\rmfamily\fontsize{9.000000}{10.800000}\selectfont \(\displaystyle {50}\)}%
\end{pgfscope}%
\begin{pgfscope}%
\pgfsetbuttcap%
\pgfsetroundjoin%
\definecolor{currentfill}{rgb}{0.000000,0.000000,0.000000}%
\pgfsetfillcolor{currentfill}%
\pgfsetlinewidth{0.803000pt}%
\definecolor{currentstroke}{rgb}{0.000000,0.000000,0.000000}%
\pgfsetstrokecolor{currentstroke}%
\pgfsetdash{}{0pt}%
\pgfsys@defobject{currentmarker}{\pgfqpoint{-0.048611in}{0.000000in}}{\pgfqpoint{0.000000in}{0.000000in}}{%
\pgfpathmoveto{\pgfqpoint{0.000000in}{0.000000in}}%
\pgfpathlineto{\pgfqpoint{-0.048611in}{0.000000in}}%
\pgfusepath{stroke,fill}%
}%
\begin{pgfscope}%
\pgfsys@transformshift{0.556596in}{1.374812in}%
\pgfsys@useobject{currentmarker}{}%
\end{pgfscope}%
\end{pgfscope}%
\begin{pgfscope}%
\definecolor{textcolor}{rgb}{0.000000,0.000000,0.000000}%
\pgfsetstrokecolor{textcolor}%
\pgfsetfillcolor{textcolor}%
\pgftext[x=0.266667in, y=1.331409in, left, base]{\color{textcolor}\rmfamily\fontsize{9.000000}{10.800000}\selectfont \(\displaystyle {100}\)}%
\end{pgfscope}%
\begin{pgfscope}%
\pgfsetbuttcap%
\pgfsetroundjoin%
\definecolor{currentfill}{rgb}{0.000000,0.000000,0.000000}%
\pgfsetfillcolor{currentfill}%
\pgfsetlinewidth{0.803000pt}%
\definecolor{currentstroke}{rgb}{0.000000,0.000000,0.000000}%
\pgfsetstrokecolor{currentstroke}%
\pgfsetdash{}{0pt}%
\pgfsys@defobject{currentmarker}{\pgfqpoint{-0.048611in}{0.000000in}}{\pgfqpoint{0.000000in}{0.000000in}}{%
\pgfpathmoveto{\pgfqpoint{0.000000in}{0.000000in}}%
\pgfpathlineto{\pgfqpoint{-0.048611in}{0.000000in}}%
\pgfusepath{stroke,fill}%
}%
\begin{pgfscope}%
\pgfsys@transformshift{0.556596in}{1.802218in}%
\pgfsys@useobject{currentmarker}{}%
\end{pgfscope}%
\end{pgfscope}%
\begin{pgfscope}%
\definecolor{textcolor}{rgb}{0.000000,0.000000,0.000000}%
\pgfsetstrokecolor{textcolor}%
\pgfsetfillcolor{textcolor}%
\pgftext[x=0.266667in, y=1.758815in, left, base]{\color{textcolor}\rmfamily\fontsize{9.000000}{10.800000}\selectfont \(\displaystyle {150}\)}%
\end{pgfscope}%
\begin{pgfscope}%
\definecolor{textcolor}{rgb}{0.000000,0.000000,0.000000}%
\pgfsetstrokecolor{textcolor}%
\pgfsetfillcolor{textcolor}%
\pgftext[x=0.211111in,y=1.161750in,,bottom,rotate=90.000000]{\color{textcolor}\rmfamily\fontsize{9.000000}{10.800000}\selectfont Count}%
\end{pgfscope}%
\begin{pgfscope}%
\pgfpathrectangle{\pgfqpoint{0.556596in}{0.520000in}}{\pgfqpoint{1.317500in}{1.283500in}}%
\pgfusepath{clip}%
\pgfsetbuttcap%
\pgfsetroundjoin%
\pgfsetlinewidth{1.505625pt}%
\definecolor{currentstroke}{rgb}{0.000000,0.000000,1.000000}%
\pgfsetstrokecolor{currentstroke}%
\pgfsetdash{{5.550000pt}{2.400000pt}}{0.000000pt}%
\pgfpathmoveto{\pgfqpoint{1.121090in}{0.520000in}}%
\pgfpathlineto{\pgfqpoint{1.121090in}{1.803500in}}%
\pgfusepath{stroke}%
\end{pgfscope}%
\begin{pgfscope}%
\pgfpathrectangle{\pgfqpoint{0.556596in}{0.520000in}}{\pgfqpoint{1.317500in}{1.283500in}}%
\pgfusepath{clip}%
\pgfsetbuttcap%
\pgfsetroundjoin%
\pgfsetlinewidth{1.505625pt}%
\definecolor{currentstroke}{rgb}{1.000000,0.000000,0.000000}%
\pgfsetstrokecolor{currentstroke}%
\pgfsetdash{{5.550000pt}{2.400000pt}}{0.000000pt}%
\pgfpathmoveto{\pgfqpoint{1.311437in}{0.520000in}}%
\pgfpathlineto{\pgfqpoint{1.311437in}{1.803500in}}%
\pgfusepath{stroke}%
\end{pgfscope}%
\begin{pgfscope}%
\pgfsetrectcap%
\pgfsetmiterjoin%
\pgfsetlinewidth{0.803000pt}%
\definecolor{currentstroke}{rgb}{0.000000,0.000000,0.000000}%
\pgfsetstrokecolor{currentstroke}%
\pgfsetdash{}{0pt}%
\pgfpathmoveto{\pgfqpoint{0.556596in}{0.520000in}}%
\pgfpathlineto{\pgfqpoint{0.556596in}{1.803500in}}%
\pgfusepath{stroke}%
\end{pgfscope}%
\begin{pgfscope}%
\pgfsetrectcap%
\pgfsetmiterjoin%
\pgfsetlinewidth{0.803000pt}%
\definecolor{currentstroke}{rgb}{0.000000,0.000000,0.000000}%
\pgfsetstrokecolor{currentstroke}%
\pgfsetdash{}{0pt}%
\pgfpathmoveto{\pgfqpoint{1.874096in}{0.520000in}}%
\pgfpathlineto{\pgfqpoint{1.874096in}{1.803500in}}%
\pgfusepath{stroke}%
\end{pgfscope}%
\begin{pgfscope}%
\pgfsetrectcap%
\pgfsetmiterjoin%
\pgfsetlinewidth{0.803000pt}%
\definecolor{currentstroke}{rgb}{0.000000,0.000000,0.000000}%
\pgfsetstrokecolor{currentstroke}%
\pgfsetdash{}{0pt}%
\pgfpathmoveto{\pgfqpoint{0.556596in}{0.520000in}}%
\pgfpathlineto{\pgfqpoint{1.874096in}{0.520000in}}%
\pgfusepath{stroke}%
\end{pgfscope}%
\begin{pgfscope}%
\pgfsetrectcap%
\pgfsetmiterjoin%
\pgfsetlinewidth{0.803000pt}%
\definecolor{currentstroke}{rgb}{0.000000,0.000000,0.000000}%
\pgfsetstrokecolor{currentstroke}%
\pgfsetdash{}{0pt}%
\pgfpathmoveto{\pgfqpoint{0.556596in}{1.803500in}}%
\pgfpathlineto{\pgfqpoint{1.874096in}{1.803500in}}%
\pgfusepath{stroke}%
\end{pgfscope}%
\begin{pgfscope}%
\pgfsetbuttcap%
\pgfsetmiterjoin%
\definecolor{currentfill}{rgb}{1.000000,1.000000,1.000000}%
\pgfsetfillcolor{currentfill}%
\pgfsetfillopacity{0.500000}%
\pgfsetlinewidth{1.003750pt}%
\definecolor{currentstroke}{rgb}{0.800000,0.800000,0.800000}%
\pgfsetstrokecolor{currentstroke}%
\pgfsetstrokeopacity{0.500000}%
\pgfsetdash{}{0pt}%
\pgfpathmoveto{\pgfqpoint{1.246700in}{1.354889in}}%
\pgfpathlineto{\pgfqpoint{1.786596in}{1.354889in}}%
\pgfpathquadraticcurveto{\pgfqpoint{1.811596in}{1.354889in}}{\pgfqpoint{1.811596in}{1.379889in}}%
\pgfpathlineto{\pgfqpoint{1.811596in}{1.716000in}}%
\pgfpathquadraticcurveto{\pgfqpoint{1.811596in}{1.741000in}}{\pgfqpoint{1.786596in}{1.741000in}}%
\pgfpathlineto{\pgfqpoint{1.246700in}{1.741000in}}%
\pgfpathquadraticcurveto{\pgfqpoint{1.221700in}{1.741000in}}{\pgfqpoint{1.221700in}{1.716000in}}%
\pgfpathlineto{\pgfqpoint{1.221700in}{1.379889in}}%
\pgfpathquadraticcurveto{\pgfqpoint{1.221700in}{1.354889in}}{\pgfqpoint{1.246700in}{1.354889in}}%
\pgfpathclose%
\pgfusepath{stroke,fill}%
\end{pgfscope}%
\begin{pgfscope}%
\pgfsetbuttcap%
\pgfsetmiterjoin%
\definecolor{currentfill}{rgb}{0.000000,0.000000,1.000000}%
\pgfsetfillcolor{currentfill}%
\pgfsetfillopacity{0.500000}%
\pgfsetlinewidth{0.000000pt}%
\definecolor{currentstroke}{rgb}{0.000000,0.000000,0.000000}%
\pgfsetstrokecolor{currentstroke}%
\pgfsetstrokeopacity{0.500000}%
\pgfsetdash{}{0pt}%
\pgfpathmoveto{\pgfqpoint{1.271700in}{1.603500in}}%
\pgfpathlineto{\pgfqpoint{1.521700in}{1.603500in}}%
\pgfpathlineto{\pgfqpoint{1.521700in}{1.691000in}}%
\pgfpathlineto{\pgfqpoint{1.271700in}{1.691000in}}%
\pgfpathclose%
\pgfusepath{fill}%
\end{pgfscope}%
\begin{pgfscope}%
\definecolor{textcolor}{rgb}{0.000000,0.000000,0.000000}%
\pgfsetstrokecolor{textcolor}%
\pgfsetfillcolor{textcolor}%
\pgftext[x=1.621700in,y=1.603500in,left,base]{\color{textcolor}\rmfamily\fontsize{9.000000}{10.800000}\selectfont \(\displaystyle P_0\)}%
\end{pgfscope}%
\begin{pgfscope}%
\pgfsetbuttcap%
\pgfsetmiterjoin%
\definecolor{currentfill}{rgb}{1.000000,0.000000,0.000000}%
\pgfsetfillcolor{currentfill}%
\pgfsetfillopacity{0.500000}%
\pgfsetlinewidth{0.000000pt}%
\definecolor{currentstroke}{rgb}{0.000000,0.000000,0.000000}%
\pgfsetstrokecolor{currentstroke}%
\pgfsetstrokeopacity{0.500000}%
\pgfsetdash{}{0pt}%
\pgfpathmoveto{\pgfqpoint{1.271700in}{1.429195in}}%
\pgfpathlineto{\pgfqpoint{1.521700in}{1.429195in}}%
\pgfpathlineto{\pgfqpoint{1.521700in}{1.516695in}}%
\pgfpathlineto{\pgfqpoint{1.271700in}{1.516695in}}%
\pgfpathclose%
\pgfusepath{fill}%
\end{pgfscope}%
\begin{pgfscope}%
\definecolor{textcolor}{rgb}{0.000000,0.000000,0.000000}%
\pgfsetstrokecolor{textcolor}%
\pgfsetfillcolor{textcolor}%
\pgftext[x=1.621700in,y=1.429195in,left,base]{\color{textcolor}\rmfamily\fontsize{9.000000}{10.800000}\selectfont \(\displaystyle P_1\)}%
\end{pgfscope}%
\end{pgfpicture}%
\makeatother%
\endgroup%

%% file: images/altdist2.pgf
\begingroup%
\makeatletter%
\begin{pgfpicture}%
\pgfpathrectangle{\pgfpointorigin}{\pgfqpoint{1.571936in}{1.497110in}}%
\pgfusepath{use as bounding box, clip}%
\begin{pgfscope}%
\pgfsetbuttcap%
\pgfsetmiterjoin%
\definecolor{currentfill}{rgb}{1.000000,1.000000,1.000000}%
\pgfsetfillcolor{currentfill}%
\pgfsetlinewidth{0.000000pt}%
\definecolor{currentstroke}{rgb}{1.000000,1.000000,1.000000}%
\pgfsetstrokecolor{currentstroke}%
\pgfsetdash{}{0pt}%
\pgfpathmoveto{\pgfqpoint{0.000000in}{0.000000in}}%
\pgfpathlineto{\pgfqpoint{1.571936in}{0.000000in}}%
\pgfpathlineto{\pgfqpoint{1.571936in}{1.497110in}}%
\pgfpathlineto{\pgfqpoint{0.000000in}{1.497110in}}%
\pgfpathclose%
\pgfusepath{fill}%
\end{pgfscope}%
\begin{pgfscope}%
\pgfsetbuttcap%
\pgfsetmiterjoin%
\definecolor{currentfill}{rgb}{1.000000,1.000000,1.000000}%
\pgfsetfillcolor{currentfill}%
\pgfsetlinewidth{0.000000pt}%
\definecolor{currentstroke}{rgb}{0.000000,0.000000,0.000000}%
\pgfsetstrokecolor{currentstroke}%
\pgfsetstrokeopacity{0.000000}%
\pgfsetdash{}{0pt}%
\pgfpathmoveto{\pgfqpoint{0.541936in}{0.488889in}}%
\pgfpathlineto{\pgfqpoint{1.471936in}{0.488889in}}%
\pgfpathlineto{\pgfqpoint{1.471936in}{1.394889in}}%
\pgfpathlineto{\pgfqpoint{0.541936in}{1.394889in}}%
\pgfpathclose%
\pgfusepath{fill}%
\end{pgfscope}%
\begin{pgfscope}%
\pgfsetbuttcap%
\pgfsetroundjoin%
\definecolor{currentfill}{rgb}{0.000000,0.000000,0.000000}%
\pgfsetfillcolor{currentfill}%
\pgfsetlinewidth{0.803000pt}%
\definecolor{currentstroke}{rgb}{0.000000,0.000000,0.000000}%
\pgfsetstrokecolor{currentstroke}%
\pgfsetdash{}{0pt}%
\pgfsys@defobject{currentmarker}{\pgfqpoint{0.000000in}{-0.048611in}}{\pgfqpoint{0.000000in}{0.000000in}}{%
\pgfpathmoveto{\pgfqpoint{0.000000in}{0.000000in}}%
\pgfpathlineto{\pgfqpoint{0.000000in}{-0.048611in}}%
\pgfusepath{stroke,fill}%
}%
\begin{pgfscope}%
\pgfsys@transformshift{0.584208in}{0.488889in}%
\pgfsys@useobject{currentmarker}{}%
\end{pgfscope}%
\end{pgfscope}%
\begin{pgfscope}%
\definecolor{textcolor}{rgb}{0.000000,0.000000,0.000000}%
\pgfsetstrokecolor{textcolor}%
\pgfsetfillcolor{textcolor}%
\pgftext[x=0.584208in,y=0.391667in,,top]{\color{textcolor}\rmfamily\fontsize{9.000000}{10.800000}\selectfont \(\displaystyle {0}\)}%
\end{pgfscope}%
\begin{pgfscope}%
\pgfsetbuttcap%
\pgfsetroundjoin%
\definecolor{currentfill}{rgb}{0.000000,0.000000,0.000000}%
\pgfsetfillcolor{currentfill}%
\pgfsetlinewidth{0.803000pt}%
\definecolor{currentstroke}{rgb}{0.000000,0.000000,0.000000}%
\pgfsetstrokecolor{currentstroke}%
\pgfsetdash{}{0pt}%
\pgfsys@defobject{currentmarker}{\pgfqpoint{0.000000in}{-0.048611in}}{\pgfqpoint{0.000000in}{0.000000in}}{%
\pgfpathmoveto{\pgfqpoint{0.000000in}{0.000000in}}%
\pgfpathlineto{\pgfqpoint{0.000000in}{-0.048611in}}%
\pgfusepath{stroke,fill}%
}%
\begin{pgfscope}%
\pgfsys@transformshift{1.429663in}{0.488889in}%
\pgfsys@useobject{currentmarker}{}%
\end{pgfscope}%
\end{pgfscope}%
\begin{pgfscope}%
\definecolor{textcolor}{rgb}{0.000000,0.000000,0.000000}%
\pgfsetstrokecolor{textcolor}%
\pgfsetfillcolor{textcolor}%
\pgftext[x=1.429663in,y=0.391667in,,top]{\color{textcolor}\rmfamily\fontsize{9.000000}{10.800000}\selectfont \(\displaystyle {1}\)}%
\end{pgfscope}%
\begin{pgfscope}%
\definecolor{textcolor}{rgb}{0.000000,0.000000,0.000000}%
\pgfsetstrokecolor{textcolor}%
\pgfsetfillcolor{textcolor}%
\pgftext[x=1.006936in,y=0.225000in,,top]{\color{textcolor}\rmfamily\fontsize{9.000000}{10.800000}\selectfont \(\displaystyle f(X)\)}%
\end{pgfscope}%
\begin{pgfscope}%
\pgfsetbuttcap%
\pgfsetroundjoin%
\definecolor{currentfill}{rgb}{0.000000,0.000000,0.000000}%
\pgfsetfillcolor{currentfill}%
\pgfsetlinewidth{0.803000pt}%
\definecolor{currentstroke}{rgb}{0.000000,0.000000,0.000000}%
\pgfsetstrokecolor{currentstroke}%
\pgfsetdash{}{0pt}%
\pgfsys@defobject{currentmarker}{\pgfqpoint{-0.048611in}{0.000000in}}{\pgfqpoint{0.000000in}{0.000000in}}{%
\pgfpathmoveto{\pgfqpoint{0.000000in}{0.000000in}}%
\pgfpathlineto{\pgfqpoint{-0.048611in}{0.000000in}}%
\pgfusepath{stroke,fill}%
}%
\begin{pgfscope}%
\pgfsys@transformshift{0.541936in}{0.530071in}%
\pgfsys@useobject{currentmarker}{}%
\end{pgfscope}%
\end{pgfscope}%
\begin{pgfscope}%
\definecolor{textcolor}{rgb}{0.000000,0.000000,0.000000}%
\pgfsetstrokecolor{textcolor}%
\pgfsetfillcolor{textcolor}%
\pgftext[x=0.280556in, y=0.486668in, left, base]{\color{textcolor}\rmfamily\fontsize{9.000000}{10.800000}\selectfont \(\displaystyle {0.0}\)}%
\end{pgfscope}%
\begin{pgfscope}%
\pgfsetbuttcap%
\pgfsetroundjoin%
\definecolor{currentfill}{rgb}{0.000000,0.000000,0.000000}%
\pgfsetfillcolor{currentfill}%
\pgfsetlinewidth{0.803000pt}%
\definecolor{currentstroke}{rgb}{0.000000,0.000000,0.000000}%
\pgfsetstrokecolor{currentstroke}%
\pgfsetdash{}{0pt}%
\pgfsys@defobject{currentmarker}{\pgfqpoint{-0.048611in}{0.000000in}}{\pgfqpoint{0.000000in}{0.000000in}}{%
\pgfpathmoveto{\pgfqpoint{0.000000in}{0.000000in}}%
\pgfpathlineto{\pgfqpoint{-0.048611in}{0.000000in}}%
\pgfusepath{stroke,fill}%
}%
\begin{pgfscope}%
\pgfsys@transformshift{0.541936in}{0.941889in}%
\pgfsys@useobject{currentmarker}{}%
\end{pgfscope}%
\end{pgfscope}%
\begin{pgfscope}%
\definecolor{textcolor}{rgb}{0.000000,0.000000,0.000000}%
\pgfsetstrokecolor{textcolor}%
\pgfsetfillcolor{textcolor}%
\pgftext[x=0.280556in, y=0.898486in, left, base]{\color{textcolor}\rmfamily\fontsize{9.000000}{10.800000}\selectfont \(\displaystyle {0.5}\)}%
\end{pgfscope}%
\begin{pgfscope}%
\pgfsetbuttcap%
\pgfsetroundjoin%
\definecolor{currentfill}{rgb}{0.000000,0.000000,0.000000}%
\pgfsetfillcolor{currentfill}%
\pgfsetlinewidth{0.803000pt}%
\definecolor{currentstroke}{rgb}{0.000000,0.000000,0.000000}%
\pgfsetstrokecolor{currentstroke}%
\pgfsetdash{}{0pt}%
\pgfsys@defobject{currentmarker}{\pgfqpoint{-0.048611in}{0.000000in}}{\pgfqpoint{0.000000in}{0.000000in}}{%
\pgfpathmoveto{\pgfqpoint{0.000000in}{0.000000in}}%
\pgfpathlineto{\pgfqpoint{-0.048611in}{0.000000in}}%
\pgfusepath{stroke,fill}%
}%
\begin{pgfscope}%
\pgfsys@transformshift{0.541936in}{1.353707in}%
\pgfsys@useobject{currentmarker}{}%
\end{pgfscope}%
\end{pgfscope}%
\begin{pgfscope}%
\definecolor{textcolor}{rgb}{0.000000,0.000000,0.000000}%
\pgfsetstrokecolor{textcolor}%
\pgfsetfillcolor{textcolor}%
\pgftext[x=0.280556in, y=1.310304in, left, base]{\color{textcolor}\rmfamily\fontsize{9.000000}{10.800000}\selectfont \(\displaystyle {1.0}\)}%
\end{pgfscope}%
\begin{pgfscope}%
\definecolor{textcolor}{rgb}{0.000000,0.000000,0.000000}%
\pgfsetstrokecolor{textcolor}%
\pgfsetfillcolor{textcolor}%
\pgftext[x=0.225000in,y=0.941889in,,bottom,rotate=90.000000]{\color{textcolor}\rmfamily\fontsize{9.000000}{10.800000}\selectfont \(\displaystyle \mathbb{E}[Y \mid  f(X)]\)}%
\end{pgfscope}%
\begin{pgfscope}%
\pgfpathrectangle{\pgfqpoint{0.541936in}{0.488889in}}{\pgfqpoint{0.930000in}{0.906000in}}%
\pgfusepath{clip}%
\pgfsetrectcap%
\pgfsetroundjoin%
\pgfsetlinewidth{0.501875pt}%
\definecolor{currentstroke}{rgb}{0.000000,0.000000,0.000000}%
\pgfsetstrokecolor{currentstroke}%
\pgfsetdash{}{0pt}%
\pgfpathmoveto{\pgfqpoint{0.584208in}{0.530071in}}%
\pgfpathlineto{\pgfqpoint{0.799483in}{0.741917in}}%
\pgfpathlineto{\pgfqpoint{0.801258in}{0.757388in}}%
\pgfpathlineto{\pgfqpoint{0.808699in}{0.857055in}}%
\pgfpathlineto{\pgfqpoint{0.808784in}{0.640547in}}%
\pgfpathlineto{\pgfqpoint{0.810306in}{0.647713in}}%
\pgfpathlineto{\pgfqpoint{0.812927in}{0.682904in}}%
\pgfpathlineto{\pgfqpoint{0.818761in}{0.758038in}}%
\pgfpathlineto{\pgfqpoint{0.822228in}{0.761948in}}%
\pgfpathlineto{\pgfqpoint{0.825863in}{0.767479in}}%
\pgfpathlineto{\pgfqpoint{0.827639in}{0.782616in}}%
\pgfpathlineto{\pgfqpoint{0.835164in}{0.882851in}}%
\pgfpathlineto{\pgfqpoint{0.835249in}{0.666334in}}%
\pgfpathlineto{\pgfqpoint{0.836771in}{0.673830in}}%
\pgfpathlineto{\pgfqpoint{0.839477in}{0.710902in}}%
\pgfpathlineto{\pgfqpoint{0.845142in}{0.783684in}}%
\pgfpathlineto{\pgfqpoint{0.848439in}{0.787483in}}%
\pgfpathlineto{\pgfqpoint{0.852329in}{0.793419in}}%
\pgfpathlineto{\pgfqpoint{0.854105in}{0.808933in}}%
\pgfpathlineto{\pgfqpoint{0.861545in}{0.908539in}}%
\pgfpathlineto{\pgfqpoint{0.861630in}{0.692029in}}%
\pgfpathlineto{\pgfqpoint{0.863152in}{0.699234in}}%
\pgfpathlineto{\pgfqpoint{0.865773in}{0.734473in}}%
\pgfpathlineto{\pgfqpoint{0.871607in}{0.809527in}}%
\pgfpathlineto{\pgfqpoint{0.875074in}{0.813430in}}%
\pgfpathlineto{\pgfqpoint{0.878710in}{0.818980in}}%
\pgfpathlineto{\pgfqpoint{0.880485in}{0.834161in}}%
\pgfpathlineto{\pgfqpoint{0.888011in}{0.934334in}}%
\pgfpathlineto{\pgfqpoint{0.888095in}{0.717818in}}%
\pgfpathlineto{\pgfqpoint{0.889617in}{0.725353in}}%
\pgfpathlineto{\pgfqpoint{0.892323in}{0.762472in}}%
\pgfpathlineto{\pgfqpoint{0.897988in}{0.835173in}}%
\pgfpathlineto{\pgfqpoint{0.901370in}{0.839048in}}%
\pgfpathlineto{\pgfqpoint{0.905091in}{0.844547in}}%
\pgfpathlineto{\pgfqpoint{0.906782in}{0.858343in}}%
\pgfpathlineto{\pgfqpoint{0.911347in}{0.934324in}}%
\pgfpathlineto{\pgfqpoint{0.914391in}{0.960024in}}%
\pgfpathlineto{\pgfqpoint{0.914476in}{0.743512in}}%
\pgfpathlineto{\pgfqpoint{0.915998in}{0.750756in}}%
\pgfpathlineto{\pgfqpoint{0.918619in}{0.786042in}}%
\pgfpathlineto{\pgfqpoint{0.924453in}{0.861016in}}%
\pgfpathlineto{\pgfqpoint{0.927920in}{0.864912in}}%
\pgfpathlineto{\pgfqpoint{0.931556in}{0.870480in}}%
\pgfpathlineto{\pgfqpoint{0.933332in}{0.885706in}}%
\pgfpathlineto{\pgfqpoint{0.940857in}{0.985817in}}%
\pgfpathlineto{\pgfqpoint{0.940941in}{0.769302in}}%
\pgfpathlineto{\pgfqpoint{0.942463in}{0.776875in}}%
\pgfpathlineto{\pgfqpoint{0.945169in}{0.814042in}}%
\pgfpathlineto{\pgfqpoint{0.950750in}{0.886454in}}%
\pgfpathlineto{\pgfqpoint{0.953878in}{0.890201in}}%
\pgfpathlineto{\pgfqpoint{0.957937in}{0.896047in}}%
\pgfpathlineto{\pgfqpoint{0.959628in}{0.909885in}}%
\pgfpathlineto{\pgfqpoint{0.964278in}{0.987107in}}%
\pgfpathlineto{\pgfqpoint{0.967238in}{1.011508in}}%
\pgfpathlineto{\pgfqpoint{0.967322in}{0.794995in}}%
\pgfpathlineto{\pgfqpoint{0.968844in}{0.802277in}}%
\pgfpathlineto{\pgfqpoint{0.971465in}{0.837611in}}%
\pgfpathlineto{\pgfqpoint{0.977300in}{0.912506in}}%
\pgfpathlineto{\pgfqpoint{0.980766in}{0.916395in}}%
\pgfpathlineto{\pgfqpoint{0.984402in}{0.921981in}}%
\pgfpathlineto{\pgfqpoint{0.986178in}{0.937251in}}%
\pgfpathlineto{\pgfqpoint{0.993703in}{1.037300in}}%
\pgfpathlineto{\pgfqpoint{0.993788in}{0.820786in}}%
\pgfpathlineto{\pgfqpoint{0.995310in}{0.828398in}}%
\pgfpathlineto{\pgfqpoint{0.998015in}{0.865612in}}%
\pgfpathlineto{\pgfqpoint{1.003596in}{0.937945in}}%
\pgfpathlineto{\pgfqpoint{1.006809in}{0.941765in}}%
\pgfpathlineto{\pgfqpoint{1.010783in}{0.947547in}}%
\pgfpathlineto{\pgfqpoint{1.012559in}{0.962483in}}%
\pgfpathlineto{\pgfqpoint{1.020084in}{1.062992in}}%
\pgfpathlineto{\pgfqpoint{1.020168in}{0.846478in}}%
\pgfpathlineto{\pgfqpoint{1.021690in}{0.853799in}}%
\pgfpathlineto{\pgfqpoint{1.024312in}{0.889180in}}%
\pgfpathlineto{\pgfqpoint{1.030146in}{0.963995in}}%
\pgfpathlineto{\pgfqpoint{1.033612in}{0.967877in}}%
\pgfpathlineto{\pgfqpoint{1.037248in}{0.973482in}}%
\pgfpathlineto{\pgfqpoint{1.039024in}{0.988796in}}%
\pgfpathlineto{\pgfqpoint{1.046549in}{1.088783in}}%
\pgfpathlineto{\pgfqpoint{1.046634in}{0.872270in}}%
\pgfpathlineto{\pgfqpoint{1.048156in}{0.879921in}}%
\pgfpathlineto{\pgfqpoint{1.050861in}{0.917182in}}%
\pgfpathlineto{\pgfqpoint{1.056442in}{0.989436in}}%
\pgfpathlineto{\pgfqpoint{1.059655in}{0.993248in}}%
\pgfpathlineto{\pgfqpoint{1.063629in}{0.999047in}}%
\pgfpathlineto{\pgfqpoint{1.065405in}{1.014028in}}%
\pgfpathlineto{\pgfqpoint{1.072930in}{1.114476in}}%
\pgfpathlineto{\pgfqpoint{1.073015in}{0.897961in}}%
\pgfpathlineto{\pgfqpoint{1.074537in}{0.905321in}}%
\pgfpathlineto{\pgfqpoint{1.077158in}{0.940749in}}%
\pgfpathlineto{\pgfqpoint{1.082907in}{1.015287in}}%
\pgfpathlineto{\pgfqpoint{1.086205in}{1.019113in}}%
\pgfpathlineto{\pgfqpoint{1.090094in}{1.024983in}}%
\pgfpathlineto{\pgfqpoint{1.091870in}{1.040342in}}%
\pgfpathlineto{\pgfqpoint{1.099395in}{1.140266in}}%
\pgfpathlineto{\pgfqpoint{1.099480in}{0.923754in}}%
\pgfpathlineto{\pgfqpoint{1.101002in}{0.931444in}}%
\pgfpathlineto{\pgfqpoint{1.103708in}{0.968752in}}%
\pgfpathlineto{\pgfqpoint{1.109288in}{1.040926in}}%
\pgfpathlineto{\pgfqpoint{1.112501in}{1.044730in}}%
\pgfpathlineto{\pgfqpoint{1.116475in}{1.050547in}}%
\pgfpathlineto{\pgfqpoint{1.118251in}{1.065572in}}%
\pgfpathlineto{\pgfqpoint{1.125776in}{1.165960in}}%
\pgfpathlineto{\pgfqpoint{1.125861in}{0.949444in}}%
\pgfpathlineto{\pgfqpoint{1.127383in}{0.956843in}}%
\pgfpathlineto{\pgfqpoint{1.130004in}{0.992318in}}%
\pgfpathlineto{\pgfqpoint{1.135754in}{1.066777in}}%
\pgfpathlineto{\pgfqpoint{1.139051in}{1.070595in}}%
\pgfpathlineto{\pgfqpoint{1.142941in}{1.076484in}}%
\pgfpathlineto{\pgfqpoint{1.144716in}{1.091887in}}%
\pgfpathlineto{\pgfqpoint{1.152242in}{1.191749in}}%
\pgfpathlineto{\pgfqpoint{1.152326in}{0.975239in}}%
\pgfpathlineto{\pgfqpoint{1.153848in}{0.982967in}}%
\pgfpathlineto{\pgfqpoint{1.156554in}{1.020323in}}%
\pgfpathlineto{\pgfqpoint{1.162134in}{1.092417in}}%
\pgfpathlineto{\pgfqpoint{1.165347in}{1.096213in}}%
\pgfpathlineto{\pgfqpoint{1.169321in}{1.102047in}}%
\pgfpathlineto{\pgfqpoint{1.171097in}{1.117117in}}%
\pgfpathlineto{\pgfqpoint{1.178622in}{1.217443in}}%
\pgfpathlineto{\pgfqpoint{1.178707in}{1.000927in}}%
\pgfpathlineto{\pgfqpoint{1.180229in}{1.008365in}}%
\pgfpathlineto{\pgfqpoint{1.182935in}{1.045365in}}%
\pgfpathlineto{\pgfqpoint{1.188600in}{1.118267in}}%
\pgfpathlineto{\pgfqpoint{1.191897in}{1.122077in}}%
\pgfpathlineto{\pgfqpoint{1.195787in}{1.127985in}}%
\pgfpathlineto{\pgfqpoint{1.197562in}{1.143433in}}%
\pgfpathlineto{\pgfqpoint{1.205088in}{1.243231in}}%
\pgfpathlineto{\pgfqpoint{1.205172in}{1.026723in}}%
\pgfpathlineto{\pgfqpoint{1.206694in}{1.034490in}}%
\pgfpathlineto{\pgfqpoint{1.209400in}{1.071893in}}%
\pgfpathlineto{\pgfqpoint{1.214981in}{1.143908in}}%
\pgfpathlineto{\pgfqpoint{1.218194in}{1.147695in}}%
\pgfpathlineto{\pgfqpoint{1.429663in}{1.353707in}}%
\pgfpathlineto{\pgfqpoint{1.429663in}{1.353707in}}%
\pgfusepath{stroke}%
\end{pgfscope}%
\begin{pgfscope}%
\pgfsetrectcap%
\pgfsetmiterjoin%
\pgfsetlinewidth{0.803000pt}%
\definecolor{currentstroke}{rgb}{0.000000,0.000000,0.000000}%
\pgfsetstrokecolor{currentstroke}%
\pgfsetdash{}{0pt}%
\pgfpathmoveto{\pgfqpoint{0.541936in}{0.488889in}}%
\pgfpathlineto{\pgfqpoint{0.541936in}{1.394889in}}%
\pgfusepath{stroke}%
\end{pgfscope}%
\begin{pgfscope}%
\pgfsetrectcap%
\pgfsetmiterjoin%
\pgfsetlinewidth{0.803000pt}%
\definecolor{currentstroke}{rgb}{0.000000,0.000000,0.000000}%
\pgfsetstrokecolor{currentstroke}%
\pgfsetdash{}{0pt}%
\pgfpathmoveto{\pgfqpoint{1.471936in}{0.488889in}}%
\pgfpathlineto{\pgfqpoint{1.471936in}{1.394889in}}%
\pgfusepath{stroke}%
\end{pgfscope}%
\begin{pgfscope}%
\pgfsetrectcap%
\pgfsetmiterjoin%
\pgfsetlinewidth{0.803000pt}%
\definecolor{currentstroke}{rgb}{0.000000,0.000000,0.000000}%
\pgfsetstrokecolor{currentstroke}%
\pgfsetdash{}{0pt}%
\pgfpathmoveto{\pgfqpoint{0.541936in}{0.488889in}}%
\pgfpathlineto{\pgfqpoint{1.471936in}{0.488889in}}%
\pgfusepath{stroke}%
\end{pgfscope}%
\begin{pgfscope}%
\pgfsetrectcap%
\pgfsetmiterjoin%
\pgfsetlinewidth{0.803000pt}%
\definecolor{currentstroke}{rgb}{0.000000,0.000000,0.000000}%
\pgfsetstrokecolor{currentstroke}%
\pgfsetdash{}{0pt}%
\pgfpathmoveto{\pgfqpoint{0.541936in}{1.394889in}}%
\pgfpathlineto{\pgfqpoint{1.471936in}{1.394889in}}%
\pgfusepath{stroke}%
\end{pgfscope}%
\end{pgfpicture}%
\makeatother%
\endgroup%

%% file: images/detection_rate0.pgf
\begingroup%
\makeatletter%
\begin{pgfpicture}%
\pgfpathrectangle{\pgfpointorigin}{\pgfqpoint{2.052413in}{1.865444in}}%
\pgfusepath{use as bounding box, clip}%
\begin{pgfscope}%
\pgfsetbuttcap%
\pgfsetmiterjoin%
\definecolor{currentfill}{rgb}{1.000000,1.000000,1.000000}%
\pgfsetfillcolor{currentfill}%
\pgfsetlinewidth{0.000000pt}%
\definecolor{currentstroke}{rgb}{1.000000,1.000000,1.000000}%
\pgfsetstrokecolor{currentstroke}%
\pgfsetdash{}{0pt}%
\pgfpathmoveto{\pgfqpoint{0.000000in}{-0.000000in}}%
\pgfpathlineto{\pgfqpoint{2.052413in}{-0.000000in}}%
\pgfpathlineto{\pgfqpoint{2.052413in}{1.865444in}}%
\pgfpathlineto{\pgfqpoint{0.000000in}{1.865444in}}%
\pgfpathclose%
\pgfusepath{fill}%
\end{pgfscope}%
\begin{pgfscope}%
\pgfsetbuttcap%
\pgfsetmiterjoin%
\definecolor{currentfill}{rgb}{1.000000,1.000000,1.000000}%
\pgfsetfillcolor{currentfill}%
\pgfsetlinewidth{0.000000pt}%
\definecolor{currentstroke}{rgb}{0.000000,0.000000,0.000000}%
\pgfsetstrokecolor{currentstroke}%
\pgfsetstrokeopacity{0.000000}%
\pgfsetdash{}{0pt}%
\pgfpathmoveto{\pgfqpoint{0.634913in}{0.481944in}}%
\pgfpathlineto{\pgfqpoint{1.952413in}{0.481944in}}%
\pgfpathlineto{\pgfqpoint{1.952413in}{1.765444in}}%
\pgfpathlineto{\pgfqpoint{0.634913in}{1.765444in}}%
\pgfpathclose%
\pgfusepath{fill}%
\end{pgfscope}%
\begin{pgfscope}%
\pgfpathrectangle{\pgfqpoint{0.634913in}{0.481944in}}{\pgfqpoint{1.317500in}{1.283500in}}%
\pgfusepath{clip}%
\pgfsetbuttcap%
\pgfsetroundjoin%
\definecolor{currentfill}{rgb}{0.000000,0.000000,0.000000}%
\pgfsetfillcolor{currentfill}%
\pgfsetlinewidth{1.003750pt}%
\definecolor{currentstroke}{rgb}{0.000000,0.000000,0.000000}%
\pgfsetstrokecolor{currentstroke}%
\pgfsetdash{}{0pt}%
\pgfsys@defobject{currentmarker}{\pgfqpoint{-0.012028in}{-0.012028in}}{\pgfqpoint{0.012028in}{0.012028in}}{%
\pgfpathmoveto{\pgfqpoint{0.000000in}{-0.012028in}}%
\pgfpathcurveto{\pgfqpoint{0.003190in}{-0.012028in}}{\pgfqpoint{0.006250in}{-0.010761in}}{\pgfqpoint{0.008505in}{-0.008505in}}%
\pgfpathcurveto{\pgfqpoint{0.010761in}{-0.006250in}}{\pgfqpoint{0.012028in}{-0.003190in}}{\pgfqpoint{0.012028in}{0.000000in}}%
\pgfpathcurveto{\pgfqpoint{0.012028in}{0.003190in}}{\pgfqpoint{0.010761in}{0.006250in}}{\pgfqpoint{0.008505in}{0.008505in}}%
\pgfpathcurveto{\pgfqpoint{0.006250in}{0.010761in}}{\pgfqpoint{0.003190in}{0.012028in}}{\pgfqpoint{0.000000in}{0.012028in}}%
\pgfpathcurveto{\pgfqpoint{-0.003190in}{0.012028in}}{\pgfqpoint{-0.006250in}{0.010761in}}{\pgfqpoint{-0.008505in}{0.008505in}}%
\pgfpathcurveto{\pgfqpoint{-0.010761in}{0.006250in}}{\pgfqpoint{-0.012028in}{0.003190in}}{\pgfqpoint{-0.012028in}{0.000000in}}%
\pgfpathcurveto{\pgfqpoint{-0.012028in}{-0.003190in}}{\pgfqpoint{-0.010761in}{-0.006250in}}{\pgfqpoint{-0.008505in}{-0.008505in}}%
\pgfpathcurveto{\pgfqpoint{-0.006250in}{-0.010761in}}{\pgfqpoint{-0.003190in}{-0.012028in}}{\pgfqpoint{0.000000in}{-0.012028in}}%
\pgfpathclose%
\pgfusepath{stroke,fill}%
}%
\begin{pgfscope}%
\pgfsys@transformshift{0.694799in}{1.707103in}%
\pgfsys@useobject{currentmarker}{}%
\end{pgfscope}%
\begin{pgfscope}%
\pgfsys@transformshift{0.907017in}{1.503137in}%
\pgfsys@useobject{currentmarker}{}%
\end{pgfscope}%
\begin{pgfscope}%
\pgfsys@transformshift{1.187554in}{1.222021in}%
\pgfsys@useobject{currentmarker}{}%
\end{pgfscope}%
\begin{pgfscope}%
\pgfsys@transformshift{1.399772in}{1.040493in}%
\pgfsys@useobject{currentmarker}{}%
\end{pgfscope}%
\begin{pgfscope}%
\pgfsys@transformshift{1.611990in}{0.874287in}%
\pgfsys@useobject{currentmarker}{}%
\end{pgfscope}%
\begin{pgfscope}%
\pgfsys@transformshift{1.892527in}{0.582121in}%
\pgfsys@useobject{currentmarker}{}%
\end{pgfscope}%
\end{pgfscope}%
\begin{pgfscope}%
\pgfsetbuttcap%
\pgfsetroundjoin%
\definecolor{currentfill}{rgb}{0.000000,0.000000,0.000000}%
\pgfsetfillcolor{currentfill}%
\pgfsetlinewidth{0.803000pt}%
\definecolor{currentstroke}{rgb}{0.000000,0.000000,0.000000}%
\pgfsetstrokecolor{currentstroke}%
\pgfsetdash{}{0pt}%
\pgfsys@defobject{currentmarker}{\pgfqpoint{0.000000in}{-0.048611in}}{\pgfqpoint{0.000000in}{0.000000in}}{%
\pgfpathmoveto{\pgfqpoint{0.000000in}{0.000000in}}%
\pgfpathlineto{\pgfqpoint{0.000000in}{-0.048611in}}%
\pgfusepath{stroke,fill}%
}%
\begin{pgfscope}%
\pgfsys@transformshift{0.694799in}{0.481944in}%
\pgfsys@useobject{currentmarker}{}%
\end{pgfscope}%
\end{pgfscope}%
\begin{pgfscope}%
\definecolor{textcolor}{rgb}{0.000000,0.000000,0.000000}%
\pgfsetstrokecolor{textcolor}%
\pgfsetfillcolor{textcolor}%
\pgftext[x=0.694799in,y=0.384722in,,top]{\color{textcolor}\rmfamily\fontsize{9.000000}{10.800000}\selectfont \(\displaystyle {3}\)}%
\end{pgfscope}%
\begin{pgfscope}%
\pgfsetbuttcap%
\pgfsetroundjoin%
\definecolor{currentfill}{rgb}{0.000000,0.000000,0.000000}%
\pgfsetfillcolor{currentfill}%
\pgfsetlinewidth{0.803000pt}%
\definecolor{currentstroke}{rgb}{0.000000,0.000000,0.000000}%
\pgfsetstrokecolor{currentstroke}%
\pgfsetdash{}{0pt}%
\pgfsys@defobject{currentmarker}{\pgfqpoint{0.000000in}{-0.048611in}}{\pgfqpoint{0.000000in}{0.000000in}}{%
\pgfpathmoveto{\pgfqpoint{0.000000in}{0.000000in}}%
\pgfpathlineto{\pgfqpoint{0.000000in}{-0.048611in}}%
\pgfusepath{stroke,fill}%
}%
\begin{pgfscope}%
\pgfsys@transformshift{1.399772in}{0.481944in}%
\pgfsys@useobject{currentmarker}{}%
\end{pgfscope}%
\end{pgfscope}%
\begin{pgfscope}%
\definecolor{textcolor}{rgb}{0.000000,0.000000,0.000000}%
\pgfsetstrokecolor{textcolor}%
\pgfsetfillcolor{textcolor}%
\pgftext[x=1.399772in,y=0.384722in,,top]{\color{textcolor}\rmfamily\fontsize{9.000000}{10.800000}\selectfont \(\displaystyle {4}\)}%
\end{pgfscope}%
\begin{pgfscope}%
\definecolor{textcolor}{rgb}{0.000000,0.000000,0.000000}%
\pgfsetstrokecolor{textcolor}%
\pgfsetfillcolor{textcolor}%
\pgftext[x=1.293663in,y=0.218055in,,top]{\color{textcolor}\rmfamily\fontsize{9.000000}{10.800000}\selectfont \(\displaystyle \log_{10} n\)}%
\end{pgfscope}%
\begin{pgfscope}%
\pgfsetbuttcap%
\pgfsetroundjoin%
\definecolor{currentfill}{rgb}{0.000000,0.000000,0.000000}%
\pgfsetfillcolor{currentfill}%
\pgfsetlinewidth{0.803000pt}%
\definecolor{currentstroke}{rgb}{0.000000,0.000000,0.000000}%
\pgfsetstrokecolor{currentstroke}%
\pgfsetdash{}{0pt}%
\pgfsys@defobject{currentmarker}{\pgfqpoint{-0.048611in}{0.000000in}}{\pgfqpoint{0.000000in}{0.000000in}}{%
\pgfpathmoveto{\pgfqpoint{0.000000in}{0.000000in}}%
\pgfpathlineto{\pgfqpoint{-0.048611in}{0.000000in}}%
\pgfusepath{stroke,fill}%
}%
\begin{pgfscope}%
\pgfsys@transformshift{0.634913in}{0.963473in}%
\pgfsys@useobject{currentmarker}{}%
\end{pgfscope}%
\end{pgfscope}%
\begin{pgfscope}%
\definecolor{textcolor}{rgb}{0.000000,0.000000,0.000000}%
\pgfsetstrokecolor{textcolor}%
\pgfsetfillcolor{textcolor}%
\pgftext[x=0.273611in, y=0.920070in, left, base]{\color{textcolor}\rmfamily\fontsize{9.000000}{10.800000}\selectfont \(\displaystyle {-1.0}\)}%
\end{pgfscope}%
\begin{pgfscope}%
\pgfsetbuttcap%
\pgfsetroundjoin%
\definecolor{currentfill}{rgb}{0.000000,0.000000,0.000000}%
\pgfsetfillcolor{currentfill}%
\pgfsetlinewidth{0.803000pt}%
\definecolor{currentstroke}{rgb}{0.000000,0.000000,0.000000}%
\pgfsetstrokecolor{currentstroke}%
\pgfsetdash{}{0pt}%
\pgfsys@defobject{currentmarker}{\pgfqpoint{-0.048611in}{0.000000in}}{\pgfqpoint{0.000000in}{0.000000in}}{%
\pgfpathmoveto{\pgfqpoint{0.000000in}{0.000000in}}%
\pgfpathlineto{\pgfqpoint{-0.048611in}{0.000000in}}%
\pgfusepath{stroke,fill}%
}%
\begin{pgfscope}%
\pgfsys@transformshift{0.634913in}{1.461122in}%
\pgfsys@useobject{currentmarker}{}%
\end{pgfscope}%
\end{pgfscope}%
\begin{pgfscope}%
\definecolor{textcolor}{rgb}{0.000000,0.000000,0.000000}%
\pgfsetstrokecolor{textcolor}%
\pgfsetfillcolor{textcolor}%
\pgftext[x=0.273611in, y=1.417720in, left, base]{\color{textcolor}\rmfamily\fontsize{9.000000}{10.800000}\selectfont \(\displaystyle {-0.8}\)}%
\end{pgfscope}%
\begin{pgfscope}%
\definecolor{textcolor}{rgb}{0.000000,0.000000,0.000000}%
\pgfsetstrokecolor{textcolor}%
\pgfsetfillcolor{textcolor}%
\pgftext[x=0.218055in,y=1.123694in,,bottom,rotate=90.000000]{\color{textcolor}\rmfamily\fontsize{9.000000}{10.800000}\selectfont \(\displaystyle \log_{10} \varepsilon_n\)}%
\end{pgfscope}%
\begin{pgfscope}%
\pgfpathrectangle{\pgfqpoint{0.634913in}{0.481944in}}{\pgfqpoint{1.317500in}{1.283500in}}%
\pgfusepath{clip}%
\pgfsetbuttcap%
\pgfsetroundjoin%
\pgfsetlinewidth{0.501875pt}%
\definecolor{currentstroke}{rgb}{0.000000,0.000000,0.000000}%
\pgfsetstrokecolor{currentstroke}%
\pgfsetdash{}{0pt}%
\pgfpathmoveto{\pgfqpoint{0.694799in}{1.707103in}}%
\pgfpathlineto{\pgfqpoint{0.694799in}{1.707103in}}%
\pgfusepath{stroke}%
\end{pgfscope}%
\begin{pgfscope}%
\pgfpathrectangle{\pgfqpoint{0.634913in}{0.481944in}}{\pgfqpoint{1.317500in}{1.283500in}}%
\pgfusepath{clip}%
\pgfsetbuttcap%
\pgfsetroundjoin%
\pgfsetlinewidth{0.501875pt}%
\definecolor{currentstroke}{rgb}{0.000000,0.000000,0.000000}%
\pgfsetstrokecolor{currentstroke}%
\pgfsetdash{}{0pt}%
\pgfpathmoveto{\pgfqpoint{0.907017in}{1.451153in}}%
\pgfpathlineto{\pgfqpoint{0.907017in}{1.555122in}}%
\pgfusepath{stroke}%
\end{pgfscope}%
\begin{pgfscope}%
\pgfpathrectangle{\pgfqpoint{0.634913in}{0.481944in}}{\pgfqpoint{1.317500in}{1.283500in}}%
\pgfusepath{clip}%
\pgfsetbuttcap%
\pgfsetroundjoin%
\pgfsetlinewidth{0.501875pt}%
\definecolor{currentstroke}{rgb}{0.000000,0.000000,0.000000}%
\pgfsetstrokecolor{currentstroke}%
\pgfsetdash{}{0pt}%
\pgfpathmoveto{\pgfqpoint{1.187554in}{1.181055in}}%
\pgfpathlineto{\pgfqpoint{1.187554in}{1.262986in}}%
\pgfusepath{stroke}%
\end{pgfscope}%
\begin{pgfscope}%
\pgfpathrectangle{\pgfqpoint{0.634913in}{0.481944in}}{\pgfqpoint{1.317500in}{1.283500in}}%
\pgfusepath{clip}%
\pgfsetbuttcap%
\pgfsetroundjoin%
\pgfsetlinewidth{0.501875pt}%
\definecolor{currentstroke}{rgb}{0.000000,0.000000,0.000000}%
\pgfsetstrokecolor{currentstroke}%
\pgfsetdash{}{0pt}%
\pgfpathmoveto{\pgfqpoint{1.399772in}{0.992923in}}%
\pgfpathlineto{\pgfqpoint{1.399772in}{1.088064in}}%
\pgfusepath{stroke}%
\end{pgfscope}%
\begin{pgfscope}%
\pgfpathrectangle{\pgfqpoint{0.634913in}{0.481944in}}{\pgfqpoint{1.317500in}{1.283500in}}%
\pgfusepath{clip}%
\pgfsetbuttcap%
\pgfsetroundjoin%
\pgfsetlinewidth{0.501875pt}%
\definecolor{currentstroke}{rgb}{0.000000,0.000000,0.000000}%
\pgfsetstrokecolor{currentstroke}%
\pgfsetdash{}{0pt}%
\pgfpathmoveto{\pgfqpoint{1.611990in}{0.834727in}}%
\pgfpathlineto{\pgfqpoint{1.611990in}{0.913848in}}%
\pgfusepath{stroke}%
\end{pgfscope}%
\begin{pgfscope}%
\pgfpathrectangle{\pgfqpoint{0.634913in}{0.481944in}}{\pgfqpoint{1.317500in}{1.283500in}}%
\pgfusepath{clip}%
\pgfsetbuttcap%
\pgfsetroundjoin%
\pgfsetlinewidth{0.501875pt}%
\definecolor{currentstroke}{rgb}{0.000000,0.000000,0.000000}%
\pgfsetstrokecolor{currentstroke}%
\pgfsetdash{}{0pt}%
\pgfpathmoveto{\pgfqpoint{1.892527in}{0.540285in}}%
\pgfpathlineto{\pgfqpoint{1.892527in}{0.623958in}}%
\pgfusepath{stroke}%
\end{pgfscope}%
\begin{pgfscope}%
\pgfpathrectangle{\pgfqpoint{0.634913in}{0.481944in}}{\pgfqpoint{1.317500in}{1.283500in}}%
\pgfusepath{clip}%
\pgfsetrectcap%
\pgfsetroundjoin%
\pgfsetlinewidth{0.501875pt}%
\definecolor{currentstroke}{rgb}{1.000000,0.000000,0.000000}%
\pgfsetstrokecolor{currentstroke}%
\pgfsetdash{}{0pt}%
\pgfpathmoveto{\pgfqpoint{0.694799in}{1.707103in}}%
\pgfpathlineto{\pgfqpoint{0.907017in}{1.502820in}}%
\pgfpathlineto{\pgfqpoint{1.187554in}{1.232773in}}%
\pgfpathlineto{\pgfqpoint{1.399772in}{1.028490in}}%
\pgfpathlineto{\pgfqpoint{1.611990in}{0.824207in}}%
\pgfpathlineto{\pgfqpoint{1.892527in}{0.554160in}}%
\pgfusepath{stroke}%
\end{pgfscope}%
\begin{pgfscope}%
\pgfsetrectcap%
\pgfsetmiterjoin%
\pgfsetlinewidth{0.803000pt}%
\definecolor{currentstroke}{rgb}{0.000000,0.000000,0.000000}%
\pgfsetstrokecolor{currentstroke}%
\pgfsetdash{}{0pt}%
\pgfpathmoveto{\pgfqpoint{0.634913in}{0.481944in}}%
\pgfpathlineto{\pgfqpoint{0.634913in}{1.765444in}}%
\pgfusepath{stroke}%
\end{pgfscope}%
\begin{pgfscope}%
\pgfsetrectcap%
\pgfsetmiterjoin%
\pgfsetlinewidth{0.803000pt}%
\definecolor{currentstroke}{rgb}{0.000000,0.000000,0.000000}%
\pgfsetstrokecolor{currentstroke}%
\pgfsetdash{}{0pt}%
\pgfpathmoveto{\pgfqpoint{1.952413in}{0.481944in}}%
\pgfpathlineto{\pgfqpoint{1.952413in}{1.765444in}}%
\pgfusepath{stroke}%
\end{pgfscope}%
\begin{pgfscope}%
\pgfsetrectcap%
\pgfsetmiterjoin%
\pgfsetlinewidth{0.803000pt}%
\definecolor{currentstroke}{rgb}{0.000000,0.000000,0.000000}%
\pgfsetstrokecolor{currentstroke}%
\pgfsetdash{}{0pt}%
\pgfpathmoveto{\pgfqpoint{0.634913in}{0.481944in}}%
\pgfpathlineto{\pgfqpoint{1.952413in}{0.481944in}}%
\pgfusepath{stroke}%
\end{pgfscope}%
\begin{pgfscope}%
\pgfsetrectcap%
\pgfsetmiterjoin%
\pgfsetlinewidth{0.803000pt}%
\definecolor{currentstroke}{rgb}{0.000000,0.000000,0.000000}%
\pgfsetstrokecolor{currentstroke}%
\pgfsetdash{}{0pt}%
\pgfpathmoveto{\pgfqpoint{0.634913in}{1.765444in}}%
\pgfpathlineto{\pgfqpoint{1.952413in}{1.765444in}}%
\pgfusepath{stroke}%
\end{pgfscope}%
\begin{pgfscope}%
\pgfsetbuttcap%
\pgfsetmiterjoin%
\definecolor{currentfill}{rgb}{1.000000,1.000000,1.000000}%
\pgfsetfillcolor{currentfill}%
\pgfsetfillopacity{0.500000}%
\pgfsetlinewidth{1.003750pt}%
\definecolor{currentstroke}{rgb}{0.800000,0.800000,0.800000}%
\pgfsetstrokecolor{currentstroke}%
\pgfsetstrokeopacity{0.500000}%
\pgfsetdash{}{0pt}%
\pgfpathmoveto{\pgfqpoint{0.722413in}{0.544444in}}%
\pgfpathlineto{\pgfqpoint{1.758018in}{0.544444in}}%
\pgfpathquadraticcurveto{\pgfqpoint{1.783018in}{0.544444in}}{\pgfqpoint{1.783018in}{0.569444in}}%
\pgfpathlineto{\pgfqpoint{1.783018in}{0.905555in}}%
\pgfpathquadraticcurveto{\pgfqpoint{1.783018in}{0.930555in}}{\pgfqpoint{1.758018in}{0.930555in}}%
\pgfpathlineto{\pgfqpoint{0.722413in}{0.930555in}}%
\pgfpathquadraticcurveto{\pgfqpoint{0.697413in}{0.930555in}}{\pgfqpoint{0.697413in}{0.905555in}}%
\pgfpathlineto{\pgfqpoint{0.697413in}{0.569444in}}%
\pgfpathquadraticcurveto{\pgfqpoint{0.697413in}{0.544444in}}{\pgfqpoint{0.722413in}{0.544444in}}%
\pgfpathclose%
\pgfusepath{stroke,fill}%
\end{pgfscope}%
\begin{pgfscope}%
\pgfsetrectcap%
\pgfsetroundjoin%
\pgfsetlinewidth{0.501875pt}%
\definecolor{currentstroke}{rgb}{1.000000,0.000000,0.000000}%
\pgfsetstrokecolor{currentstroke}%
\pgfsetdash{}{0pt}%
\pgfpathmoveto{\pgfqpoint{0.747413in}{0.836805in}}%
\pgfpathlineto{\pgfqpoint{0.997413in}{0.836805in}}%
\pgfusepath{stroke}%
\end{pgfscope}%
\begin{pgfscope}%
\definecolor{textcolor}{rgb}{0.000000,0.000000,0.000000}%
\pgfsetstrokecolor{textcolor}%
\pgfsetfillcolor{textcolor}%
\pgftext[x=1.097413in,y=0.793055in,left,base]{\color{textcolor}\rmfamily\fontsize{9.000000}{10.800000}\selectfont Theoretical}%
\end{pgfscope}%
\begin{pgfscope}%
\pgfsetbuttcap%
\pgfsetroundjoin%
\definecolor{currentfill}{rgb}{0.000000,0.000000,0.000000}%
\pgfsetfillcolor{currentfill}%
\pgfsetlinewidth{1.003750pt}%
\definecolor{currentstroke}{rgb}{0.000000,0.000000,0.000000}%
\pgfsetstrokecolor{currentstroke}%
\pgfsetdash{}{0pt}%
\pgfsys@defobject{currentmarker}{\pgfqpoint{-0.012028in}{-0.012028in}}{\pgfqpoint{0.012028in}{0.012028in}}{%
\pgfpathmoveto{\pgfqpoint{0.000000in}{-0.012028in}}%
\pgfpathcurveto{\pgfqpoint{0.003190in}{-0.012028in}}{\pgfqpoint{0.006250in}{-0.010761in}}{\pgfqpoint{0.008505in}{-0.008505in}}%
\pgfpathcurveto{\pgfqpoint{0.010761in}{-0.006250in}}{\pgfqpoint{0.012028in}{-0.003190in}}{\pgfqpoint{0.012028in}{0.000000in}}%
\pgfpathcurveto{\pgfqpoint{0.012028in}{0.003190in}}{\pgfqpoint{0.010761in}{0.006250in}}{\pgfqpoint{0.008505in}{0.008505in}}%
\pgfpathcurveto{\pgfqpoint{0.006250in}{0.010761in}}{\pgfqpoint{0.003190in}{0.012028in}}{\pgfqpoint{0.000000in}{0.012028in}}%
\pgfpathcurveto{\pgfqpoint{-0.003190in}{0.012028in}}{\pgfqpoint{-0.006250in}{0.010761in}}{\pgfqpoint{-0.008505in}{0.008505in}}%
\pgfpathcurveto{\pgfqpoint{-0.010761in}{0.006250in}}{\pgfqpoint{-0.012028in}{0.003190in}}{\pgfqpoint{-0.012028in}{0.000000in}}%
\pgfpathcurveto{\pgfqpoint{-0.012028in}{-0.003190in}}{\pgfqpoint{-0.010761in}{-0.006250in}}{\pgfqpoint{-0.008505in}{-0.008505in}}%
\pgfpathcurveto{\pgfqpoint{-0.006250in}{-0.010761in}}{\pgfqpoint{-0.003190in}{-0.012028in}}{\pgfqpoint{0.000000in}{-0.012028in}}%
\pgfpathclose%
\pgfusepath{stroke,fill}%
}%
\begin{pgfscope}%
\pgfsys@transformshift{0.872413in}{0.651562in}%
\pgfsys@useobject{currentmarker}{}%
\end{pgfscope}%
\end{pgfscope}%
\begin{pgfscope}%
\definecolor{textcolor}{rgb}{0.000000,0.000000,0.000000}%
\pgfsetstrokecolor{textcolor}%
\pgfsetfillcolor{textcolor}%
\pgftext[x=1.097413in,y=0.618750in,left,base]{\color{textcolor}\rmfamily\fontsize{9.000000}{10.800000}\selectfont Empirical}%
\end{pgfscope}%
\end{pgfpicture}%
\makeatother%
\endgroup%

%% file: images/detection_rate1.pgf
\begingroup%
\makeatletter%
\begin{pgfpicture}%
\pgfpathrectangle{\pgfpointorigin}{\pgfqpoint{2.052413in}{1.871943in}}%
\pgfusepath{use as bounding box, clip}%
\begin{pgfscope}%
\pgfsetbuttcap%
\pgfsetmiterjoin%
\definecolor{currentfill}{rgb}{1.000000,1.000000,1.000000}%
\pgfsetfillcolor{currentfill}%
\pgfsetlinewidth{0.000000pt}%
\definecolor{currentstroke}{rgb}{1.000000,1.000000,1.000000}%
\pgfsetstrokecolor{currentstroke}%
\pgfsetdash{}{0pt}%
\pgfpathmoveto{\pgfqpoint{0.000000in}{-0.000000in}}%
\pgfpathlineto{\pgfqpoint{2.052413in}{-0.000000in}}%
\pgfpathlineto{\pgfqpoint{2.052413in}{1.871943in}}%
\pgfpathlineto{\pgfqpoint{0.000000in}{1.871943in}}%
\pgfpathclose%
\pgfusepath{fill}%
\end{pgfscope}%
\begin{pgfscope}%
\pgfsetbuttcap%
\pgfsetmiterjoin%
\definecolor{currentfill}{rgb}{1.000000,1.000000,1.000000}%
\pgfsetfillcolor{currentfill}%
\pgfsetlinewidth{0.000000pt}%
\definecolor{currentstroke}{rgb}{0.000000,0.000000,0.000000}%
\pgfsetstrokecolor{currentstroke}%
\pgfsetstrokeopacity{0.000000}%
\pgfsetdash{}{0pt}%
\pgfpathmoveto{\pgfqpoint{0.634913in}{0.481944in}}%
\pgfpathlineto{\pgfqpoint{1.952413in}{0.481944in}}%
\pgfpathlineto{\pgfqpoint{1.952413in}{1.765444in}}%
\pgfpathlineto{\pgfqpoint{0.634913in}{1.765444in}}%
\pgfpathclose%
\pgfusepath{fill}%
\end{pgfscope}%
\begin{pgfscope}%
\pgfpathrectangle{\pgfqpoint{0.634913in}{0.481944in}}{\pgfqpoint{1.317500in}{1.283500in}}%
\pgfusepath{clip}%
\pgfsetbuttcap%
\pgfsetroundjoin%
\definecolor{currentfill}{rgb}{0.000000,0.000000,0.000000}%
\pgfsetfillcolor{currentfill}%
\pgfsetlinewidth{1.003750pt}%
\definecolor{currentstroke}{rgb}{0.000000,0.000000,0.000000}%
\pgfsetstrokecolor{currentstroke}%
\pgfsetdash{}{0pt}%
\pgfsys@defobject{currentmarker}{\pgfqpoint{-0.012028in}{-0.012028in}}{\pgfqpoint{0.012028in}{0.012028in}}{%
\pgfpathmoveto{\pgfqpoint{0.000000in}{-0.012028in}}%
\pgfpathcurveto{\pgfqpoint{0.003190in}{-0.012028in}}{\pgfqpoint{0.006250in}{-0.010761in}}{\pgfqpoint{0.008505in}{-0.008505in}}%
\pgfpathcurveto{\pgfqpoint{0.010761in}{-0.006250in}}{\pgfqpoint{0.012028in}{-0.003190in}}{\pgfqpoint{0.012028in}{0.000000in}}%
\pgfpathcurveto{\pgfqpoint{0.012028in}{0.003190in}}{\pgfqpoint{0.010761in}{0.006250in}}{\pgfqpoint{0.008505in}{0.008505in}}%
\pgfpathcurveto{\pgfqpoint{0.006250in}{0.010761in}}{\pgfqpoint{0.003190in}{0.012028in}}{\pgfqpoint{0.000000in}{0.012028in}}%
\pgfpathcurveto{\pgfqpoint{-0.003190in}{0.012028in}}{\pgfqpoint{-0.006250in}{0.010761in}}{\pgfqpoint{-0.008505in}{0.008505in}}%
\pgfpathcurveto{\pgfqpoint{-0.010761in}{0.006250in}}{\pgfqpoint{-0.012028in}{0.003190in}}{\pgfqpoint{-0.012028in}{0.000000in}}%
\pgfpathcurveto{\pgfqpoint{-0.012028in}{-0.003190in}}{\pgfqpoint{-0.010761in}{-0.006250in}}{\pgfqpoint{-0.008505in}{-0.008505in}}%
\pgfpathcurveto{\pgfqpoint{-0.006250in}{-0.010761in}}{\pgfqpoint{-0.003190in}{-0.012028in}}{\pgfqpoint{0.000000in}{-0.012028in}}%
\pgfpathclose%
\pgfusepath{stroke,fill}%
}%
\begin{pgfscope}%
\pgfsys@transformshift{0.694799in}{1.676136in}%
\pgfsys@useobject{currentmarker}{}%
\end{pgfscope}%
\begin{pgfscope}%
\pgfsys@transformshift{0.907017in}{1.468469in}%
\pgfsys@useobject{currentmarker}{}%
\end{pgfscope}%
\begin{pgfscope}%
\pgfsys@transformshift{1.187554in}{1.198509in}%
\pgfsys@useobject{currentmarker}{}%
\end{pgfscope}%
\begin{pgfscope}%
\pgfsys@transformshift{1.399772in}{1.013689in}%
\pgfsys@useobject{currentmarker}{}%
\end{pgfscope}%
\begin{pgfscope}%
\pgfsys@transformshift{1.611990in}{0.808372in}%
\pgfsys@useobject{currentmarker}{}%
\end{pgfscope}%
\begin{pgfscope}%
\pgfsys@transformshift{1.892527in}{0.558906in}%
\pgfsys@useobject{currentmarker}{}%
\end{pgfscope}%
\end{pgfscope}%
\begin{pgfscope}%
\pgfsetbuttcap%
\pgfsetroundjoin%
\definecolor{currentfill}{rgb}{0.000000,0.000000,0.000000}%
\pgfsetfillcolor{currentfill}%
\pgfsetlinewidth{0.803000pt}%
\definecolor{currentstroke}{rgb}{0.000000,0.000000,0.000000}%
\pgfsetstrokecolor{currentstroke}%
\pgfsetdash{}{0pt}%
\pgfsys@defobject{currentmarker}{\pgfqpoint{0.000000in}{-0.048611in}}{\pgfqpoint{0.000000in}{0.000000in}}{%
\pgfpathmoveto{\pgfqpoint{0.000000in}{0.000000in}}%
\pgfpathlineto{\pgfqpoint{0.000000in}{-0.048611in}}%
\pgfusepath{stroke,fill}%
}%
\begin{pgfscope}%
\pgfsys@transformshift{0.694799in}{0.481944in}%
\pgfsys@useobject{currentmarker}{}%
\end{pgfscope}%
\end{pgfscope}%
\begin{pgfscope}%
\definecolor{textcolor}{rgb}{0.000000,0.000000,0.000000}%
\pgfsetstrokecolor{textcolor}%
\pgfsetfillcolor{textcolor}%
\pgftext[x=0.694799in,y=0.384722in,,top]{\color{textcolor}\rmfamily\fontsize{9.000000}{10.800000}\selectfont \(\displaystyle {3}\)}%
\end{pgfscope}%
\begin{pgfscope}%
\pgfsetbuttcap%
\pgfsetroundjoin%
\definecolor{currentfill}{rgb}{0.000000,0.000000,0.000000}%
\pgfsetfillcolor{currentfill}%
\pgfsetlinewidth{0.803000pt}%
\definecolor{currentstroke}{rgb}{0.000000,0.000000,0.000000}%
\pgfsetstrokecolor{currentstroke}%
\pgfsetdash{}{0pt}%
\pgfsys@defobject{currentmarker}{\pgfqpoint{0.000000in}{-0.048611in}}{\pgfqpoint{0.000000in}{0.000000in}}{%
\pgfpathmoveto{\pgfqpoint{0.000000in}{0.000000in}}%
\pgfpathlineto{\pgfqpoint{0.000000in}{-0.048611in}}%
\pgfusepath{stroke,fill}%
}%
\begin{pgfscope}%
\pgfsys@transformshift{1.399772in}{0.481944in}%
\pgfsys@useobject{currentmarker}{}%
\end{pgfscope}%
\end{pgfscope}%
\begin{pgfscope}%
\definecolor{textcolor}{rgb}{0.000000,0.000000,0.000000}%
\pgfsetstrokecolor{textcolor}%
\pgfsetfillcolor{textcolor}%
\pgftext[x=1.399772in,y=0.384722in,,top]{\color{textcolor}\rmfamily\fontsize{9.000000}{10.800000}\selectfont \(\displaystyle {4}\)}%
\end{pgfscope}%
\begin{pgfscope}%
\definecolor{textcolor}{rgb}{0.000000,0.000000,0.000000}%
\pgfsetstrokecolor{textcolor}%
\pgfsetfillcolor{textcolor}%
\pgftext[x=1.293663in,y=0.218055in,,top]{\color{textcolor}\rmfamily\fontsize{9.000000}{10.800000}\selectfont \(\displaystyle \log_{10} n\)}%
\end{pgfscope}%
\begin{pgfscope}%
\pgfsetbuttcap%
\pgfsetroundjoin%
\definecolor{currentfill}{rgb}{0.000000,0.000000,0.000000}%
\pgfsetfillcolor{currentfill}%
\pgfsetlinewidth{0.803000pt}%
\definecolor{currentstroke}{rgb}{0.000000,0.000000,0.000000}%
\pgfsetstrokecolor{currentstroke}%
\pgfsetdash{}{0pt}%
\pgfsys@defobject{currentmarker}{\pgfqpoint{-0.048611in}{0.000000in}}{\pgfqpoint{0.000000in}{0.000000in}}{%
\pgfpathmoveto{\pgfqpoint{0.000000in}{0.000000in}}%
\pgfpathlineto{\pgfqpoint{-0.048611in}{0.000000in}}%
\pgfusepath{stroke,fill}%
}%
\begin{pgfscope}%
\pgfsys@transformshift{0.634913in}{0.601168in}%
\pgfsys@useobject{currentmarker}{}%
\end{pgfscope}%
\end{pgfscope}%
\begin{pgfscope}%
\definecolor{textcolor}{rgb}{0.000000,0.000000,0.000000}%
\pgfsetstrokecolor{textcolor}%
\pgfsetfillcolor{textcolor}%
\pgftext[x=0.273611in, y=0.557765in, left, base]{\color{textcolor}\rmfamily\fontsize{9.000000}{10.800000}\selectfont \(\displaystyle {-1.4}\)}%
\end{pgfscope}%
\begin{pgfscope}%
\pgfsetbuttcap%
\pgfsetroundjoin%
\definecolor{currentfill}{rgb}{0.000000,0.000000,0.000000}%
\pgfsetfillcolor{currentfill}%
\pgfsetlinewidth{0.803000pt}%
\definecolor{currentstroke}{rgb}{0.000000,0.000000,0.000000}%
\pgfsetstrokecolor{currentstroke}%
\pgfsetdash{}{0pt}%
\pgfsys@defobject{currentmarker}{\pgfqpoint{-0.048611in}{0.000000in}}{\pgfqpoint{0.000000in}{0.000000in}}{%
\pgfpathmoveto{\pgfqpoint{0.000000in}{0.000000in}}%
\pgfpathlineto{\pgfqpoint{-0.048611in}{0.000000in}}%
\pgfusepath{stroke,fill}%
}%
\begin{pgfscope}%
\pgfsys@transformshift{0.634913in}{0.976959in}%
\pgfsys@useobject{currentmarker}{}%
\end{pgfscope}%
\end{pgfscope}%
\begin{pgfscope}%
\definecolor{textcolor}{rgb}{0.000000,0.000000,0.000000}%
\pgfsetstrokecolor{textcolor}%
\pgfsetfillcolor{textcolor}%
\pgftext[x=0.273611in, y=0.933556in, left, base]{\color{textcolor}\rmfamily\fontsize{9.000000}{10.800000}\selectfont \(\displaystyle {-1.2}\)}%
\end{pgfscope}%
\begin{pgfscope}%
\pgfsetbuttcap%
\pgfsetroundjoin%
\definecolor{currentfill}{rgb}{0.000000,0.000000,0.000000}%
\pgfsetfillcolor{currentfill}%
\pgfsetlinewidth{0.803000pt}%
\definecolor{currentstroke}{rgb}{0.000000,0.000000,0.000000}%
\pgfsetstrokecolor{currentstroke}%
\pgfsetdash{}{0pt}%
\pgfsys@defobject{currentmarker}{\pgfqpoint{-0.048611in}{0.000000in}}{\pgfqpoint{0.000000in}{0.000000in}}{%
\pgfpathmoveto{\pgfqpoint{0.000000in}{0.000000in}}%
\pgfpathlineto{\pgfqpoint{-0.048611in}{0.000000in}}%
\pgfusepath{stroke,fill}%
}%
\begin{pgfscope}%
\pgfsys@transformshift{0.634913in}{1.352750in}%
\pgfsys@useobject{currentmarker}{}%
\end{pgfscope}%
\end{pgfscope}%
\begin{pgfscope}%
\definecolor{textcolor}{rgb}{0.000000,0.000000,0.000000}%
\pgfsetstrokecolor{textcolor}%
\pgfsetfillcolor{textcolor}%
\pgftext[x=0.273611in, y=1.309347in, left, base]{\color{textcolor}\rmfamily\fontsize{9.000000}{10.800000}\selectfont \(\displaystyle {-1.0}\)}%
\end{pgfscope}%
\begin{pgfscope}%
\pgfsetbuttcap%
\pgfsetroundjoin%
\definecolor{currentfill}{rgb}{0.000000,0.000000,0.000000}%
\pgfsetfillcolor{currentfill}%
\pgfsetlinewidth{0.803000pt}%
\definecolor{currentstroke}{rgb}{0.000000,0.000000,0.000000}%
\pgfsetstrokecolor{currentstroke}%
\pgfsetdash{}{0pt}%
\pgfsys@defobject{currentmarker}{\pgfqpoint{-0.048611in}{0.000000in}}{\pgfqpoint{0.000000in}{0.000000in}}{%
\pgfpathmoveto{\pgfqpoint{0.000000in}{0.000000in}}%
\pgfpathlineto{\pgfqpoint{-0.048611in}{0.000000in}}%
\pgfusepath{stroke,fill}%
}%
\begin{pgfscope}%
\pgfsys@transformshift{0.634913in}{1.728540in}%
\pgfsys@useobject{currentmarker}{}%
\end{pgfscope}%
\end{pgfscope}%
\begin{pgfscope}%
\definecolor{textcolor}{rgb}{0.000000,0.000000,0.000000}%
\pgfsetstrokecolor{textcolor}%
\pgfsetfillcolor{textcolor}%
\pgftext[x=0.273611in, y=1.685138in, left, base]{\color{textcolor}\rmfamily\fontsize{9.000000}{10.800000}\selectfont \(\displaystyle {-0.8}\)}%
\end{pgfscope}%
\begin{pgfscope}%
\definecolor{textcolor}{rgb}{0.000000,0.000000,0.000000}%
\pgfsetstrokecolor{textcolor}%
\pgfsetfillcolor{textcolor}%
\pgftext[x=0.218055in,y=1.123694in,,bottom,rotate=90.000000]{\color{textcolor}\rmfamily\fontsize{9.000000}{10.800000}\selectfont \(\displaystyle \log_{10} \varepsilon_n\)}%
\end{pgfscope}%
\begin{pgfscope}%
\pgfpathrectangle{\pgfqpoint{0.634913in}{0.481944in}}{\pgfqpoint{1.317500in}{1.283500in}}%
\pgfusepath{clip}%
\pgfsetbuttcap%
\pgfsetroundjoin%
\pgfsetlinewidth{0.501875pt}%
\definecolor{currentstroke}{rgb}{0.000000,0.000000,0.000000}%
\pgfsetstrokecolor{currentstroke}%
\pgfsetdash{}{0pt}%
\pgfpathmoveto{\pgfqpoint{0.694799in}{1.645170in}}%
\pgfpathlineto{\pgfqpoint{0.694799in}{1.707103in}}%
\pgfusepath{stroke}%
\end{pgfscope}%
\begin{pgfscope}%
\pgfpathrectangle{\pgfqpoint{0.634913in}{0.481944in}}{\pgfqpoint{1.317500in}{1.283500in}}%
\pgfusepath{clip}%
\pgfsetbuttcap%
\pgfsetroundjoin%
\pgfsetlinewidth{0.501875pt}%
\definecolor{currentstroke}{rgb}{0.000000,0.000000,0.000000}%
\pgfsetstrokecolor{currentstroke}%
\pgfsetdash{}{0pt}%
\pgfpathmoveto{\pgfqpoint{0.907017in}{1.449641in}}%
\pgfpathlineto{\pgfqpoint{0.907017in}{1.487297in}}%
\pgfusepath{stroke}%
\end{pgfscope}%
\begin{pgfscope}%
\pgfpathrectangle{\pgfqpoint{0.634913in}{0.481944in}}{\pgfqpoint{1.317500in}{1.283500in}}%
\pgfusepath{clip}%
\pgfsetbuttcap%
\pgfsetroundjoin%
\pgfsetlinewidth{0.501875pt}%
\definecolor{currentstroke}{rgb}{0.000000,0.000000,0.000000}%
\pgfsetstrokecolor{currentstroke}%
\pgfsetdash{}{0pt}%
\pgfpathmoveto{\pgfqpoint{1.187554in}{1.174584in}}%
\pgfpathlineto{\pgfqpoint{1.187554in}{1.222434in}}%
\pgfusepath{stroke}%
\end{pgfscope}%
\begin{pgfscope}%
\pgfpathrectangle{\pgfqpoint{0.634913in}{0.481944in}}{\pgfqpoint{1.317500in}{1.283500in}}%
\pgfusepath{clip}%
\pgfsetbuttcap%
\pgfsetroundjoin%
\pgfsetlinewidth{0.501875pt}%
\definecolor{currentstroke}{rgb}{0.000000,0.000000,0.000000}%
\pgfsetstrokecolor{currentstroke}%
\pgfsetdash{}{0pt}%
\pgfpathmoveto{\pgfqpoint{1.399772in}{0.991572in}}%
\pgfpathlineto{\pgfqpoint{1.399772in}{1.035806in}}%
\pgfusepath{stroke}%
\end{pgfscope}%
\begin{pgfscope}%
\pgfpathrectangle{\pgfqpoint{0.634913in}{0.481944in}}{\pgfqpoint{1.317500in}{1.283500in}}%
\pgfusepath{clip}%
\pgfsetbuttcap%
\pgfsetroundjoin%
\pgfsetlinewidth{0.501875pt}%
\definecolor{currentstroke}{rgb}{0.000000,0.000000,0.000000}%
\pgfsetstrokecolor{currentstroke}%
\pgfsetdash{}{0pt}%
\pgfpathmoveto{\pgfqpoint{1.611990in}{0.786181in}}%
\pgfpathlineto{\pgfqpoint{1.611990in}{0.830563in}}%
\pgfusepath{stroke}%
\end{pgfscope}%
\begin{pgfscope}%
\pgfpathrectangle{\pgfqpoint{0.634913in}{0.481944in}}{\pgfqpoint{1.317500in}{1.283500in}}%
\pgfusepath{clip}%
\pgfsetbuttcap%
\pgfsetroundjoin%
\pgfsetlinewidth{0.501875pt}%
\definecolor{currentstroke}{rgb}{0.000000,0.000000,0.000000}%
\pgfsetstrokecolor{currentstroke}%
\pgfsetdash{}{0pt}%
\pgfpathmoveto{\pgfqpoint{1.892527in}{0.540285in}}%
\pgfpathlineto{\pgfqpoint{1.892527in}{0.577526in}}%
\pgfusepath{stroke}%
\end{pgfscope}%
\begin{pgfscope}%
\pgfpathrectangle{\pgfqpoint{0.634913in}{0.481944in}}{\pgfqpoint{1.317500in}{1.283500in}}%
\pgfusepath{clip}%
\pgfsetrectcap%
\pgfsetroundjoin%
\pgfsetlinewidth{0.501875pt}%
\definecolor{currentstroke}{rgb}{1.000000,0.000000,0.000000}%
\pgfsetstrokecolor{currentstroke}%
\pgfsetdash{}{0pt}%
\pgfpathmoveto{\pgfqpoint{0.694799in}{1.636286in}}%
\pgfpathlineto{\pgfqpoint{0.907017in}{1.447746in}}%
\pgfpathlineto{\pgfqpoint{1.187554in}{1.198509in}}%
\pgfpathlineto{\pgfqpoint{1.399772in}{1.009968in}}%
\pgfpathlineto{\pgfqpoint{1.611990in}{0.821428in}}%
\pgfpathlineto{\pgfqpoint{1.892527in}{0.572191in}}%
\pgfusepath{stroke}%
\end{pgfscope}%
\begin{pgfscope}%
\pgfsetrectcap%
\pgfsetmiterjoin%
\pgfsetlinewidth{0.803000pt}%
\definecolor{currentstroke}{rgb}{0.000000,0.000000,0.000000}%
\pgfsetstrokecolor{currentstroke}%
\pgfsetdash{}{0pt}%
\pgfpathmoveto{\pgfqpoint{0.634913in}{0.481944in}}%
\pgfpathlineto{\pgfqpoint{0.634913in}{1.765444in}}%
\pgfusepath{stroke}%
\end{pgfscope}%
\begin{pgfscope}%
\pgfsetrectcap%
\pgfsetmiterjoin%
\pgfsetlinewidth{0.803000pt}%
\definecolor{currentstroke}{rgb}{0.000000,0.000000,0.000000}%
\pgfsetstrokecolor{currentstroke}%
\pgfsetdash{}{0pt}%
\pgfpathmoveto{\pgfqpoint{1.952413in}{0.481944in}}%
\pgfpathlineto{\pgfqpoint{1.952413in}{1.765444in}}%
\pgfusepath{stroke}%
\end{pgfscope}%
\begin{pgfscope}%
\pgfsetrectcap%
\pgfsetmiterjoin%
\pgfsetlinewidth{0.803000pt}%
\definecolor{currentstroke}{rgb}{0.000000,0.000000,0.000000}%
\pgfsetstrokecolor{currentstroke}%
\pgfsetdash{}{0pt}%
\pgfpathmoveto{\pgfqpoint{0.634913in}{0.481944in}}%
\pgfpathlineto{\pgfqpoint{1.952413in}{0.481944in}}%
\pgfusepath{stroke}%
\end{pgfscope}%
\begin{pgfscope}%
\pgfsetrectcap%
\pgfsetmiterjoin%
\pgfsetlinewidth{0.803000pt}%
\definecolor{currentstroke}{rgb}{0.000000,0.000000,0.000000}%
\pgfsetstrokecolor{currentstroke}%
\pgfsetdash{}{0pt}%
\pgfpathmoveto{\pgfqpoint{0.634913in}{1.765444in}}%
\pgfpathlineto{\pgfqpoint{1.952413in}{1.765444in}}%
\pgfusepath{stroke}%
\end{pgfscope}%
\begin{pgfscope}%
\pgfsetbuttcap%
\pgfsetmiterjoin%
\definecolor{currentfill}{rgb}{1.000000,1.000000,1.000000}%
\pgfsetfillcolor{currentfill}%
\pgfsetfillopacity{0.500000}%
\pgfsetlinewidth{1.003750pt}%
\definecolor{currentstroke}{rgb}{0.800000,0.800000,0.800000}%
\pgfsetstrokecolor{currentstroke}%
\pgfsetstrokeopacity{0.500000}%
\pgfsetdash{}{0pt}%
\pgfpathmoveto{\pgfqpoint{0.722413in}{0.544444in}}%
\pgfpathlineto{\pgfqpoint{1.758018in}{0.544444in}}%
\pgfpathquadraticcurveto{\pgfqpoint{1.783018in}{0.544444in}}{\pgfqpoint{1.783018in}{0.569444in}}%
\pgfpathlineto{\pgfqpoint{1.783018in}{0.905555in}}%
\pgfpathquadraticcurveto{\pgfqpoint{1.783018in}{0.930555in}}{\pgfqpoint{1.758018in}{0.930555in}}%
\pgfpathlineto{\pgfqpoint{0.722413in}{0.930555in}}%
\pgfpathquadraticcurveto{\pgfqpoint{0.697413in}{0.930555in}}{\pgfqpoint{0.697413in}{0.905555in}}%
\pgfpathlineto{\pgfqpoint{0.697413in}{0.569444in}}%
\pgfpathquadraticcurveto{\pgfqpoint{0.697413in}{0.544444in}}{\pgfqpoint{0.722413in}{0.544444in}}%
\pgfpathclose%
\pgfusepath{stroke,fill}%
\end{pgfscope}%
\begin{pgfscope}%
\pgfsetrectcap%
\pgfsetroundjoin%
\pgfsetlinewidth{0.501875pt}%
\definecolor{currentstroke}{rgb}{1.000000,0.000000,0.000000}%
\pgfsetstrokecolor{currentstroke}%
\pgfsetdash{}{0pt}%
\pgfpathmoveto{\pgfqpoint{0.747413in}{0.836805in}}%
\pgfpathlineto{\pgfqpoint{0.997413in}{0.836805in}}%
\pgfusepath{stroke}%
\end{pgfscope}%
\begin{pgfscope}%
\definecolor{textcolor}{rgb}{0.000000,0.000000,0.000000}%
\pgfsetstrokecolor{textcolor}%
\pgfsetfillcolor{textcolor}%
\pgftext[x=1.097413in,y=0.793055in,left,base]{\color{textcolor}\rmfamily\fontsize{9.000000}{10.800000}\selectfont Theoretical}%
\end{pgfscope}%
\begin{pgfscope}%
\pgfsetbuttcap%
\pgfsetroundjoin%
\definecolor{currentfill}{rgb}{0.000000,0.000000,0.000000}%
\pgfsetfillcolor{currentfill}%
\pgfsetlinewidth{1.003750pt}%
\definecolor{currentstroke}{rgb}{0.000000,0.000000,0.000000}%
\pgfsetstrokecolor{currentstroke}%
\pgfsetdash{}{0pt}%
\pgfsys@defobject{currentmarker}{\pgfqpoint{-0.012028in}{-0.012028in}}{\pgfqpoint{0.012028in}{0.012028in}}{%
\pgfpathmoveto{\pgfqpoint{0.000000in}{-0.012028in}}%
\pgfpathcurveto{\pgfqpoint{0.003190in}{-0.012028in}}{\pgfqpoint{0.006250in}{-0.010761in}}{\pgfqpoint{0.008505in}{-0.008505in}}%
\pgfpathcurveto{\pgfqpoint{0.010761in}{-0.006250in}}{\pgfqpoint{0.012028in}{-0.003190in}}{\pgfqpoint{0.012028in}{0.000000in}}%
\pgfpathcurveto{\pgfqpoint{0.012028in}{0.003190in}}{\pgfqpoint{0.010761in}{0.006250in}}{\pgfqpoint{0.008505in}{0.008505in}}%
\pgfpathcurveto{\pgfqpoint{0.006250in}{0.010761in}}{\pgfqpoint{0.003190in}{0.012028in}}{\pgfqpoint{0.000000in}{0.012028in}}%
\pgfpathcurveto{\pgfqpoint{-0.003190in}{0.012028in}}{\pgfqpoint{-0.006250in}{0.010761in}}{\pgfqpoint{-0.008505in}{0.008505in}}%
\pgfpathcurveto{\pgfqpoint{-0.010761in}{0.006250in}}{\pgfqpoint{-0.012028in}{0.003190in}}{\pgfqpoint{-0.012028in}{0.000000in}}%
\pgfpathcurveto{\pgfqpoint{-0.012028in}{-0.003190in}}{\pgfqpoint{-0.010761in}{-0.006250in}}{\pgfqpoint{-0.008505in}{-0.008505in}}%
\pgfpathcurveto{\pgfqpoint{-0.006250in}{-0.010761in}}{\pgfqpoint{-0.003190in}{-0.012028in}}{\pgfqpoint{0.000000in}{-0.012028in}}%
\pgfpathclose%
\pgfusepath{stroke,fill}%
}%
\begin{pgfscope}%
\pgfsys@transformshift{0.872413in}{0.651562in}%
\pgfsys@useobject{currentmarker}{}%
\end{pgfscope}%
\end{pgfscope}%
\begin{pgfscope}%
\definecolor{textcolor}{rgb}{0.000000,0.000000,0.000000}%
\pgfsetstrokecolor{textcolor}%
\pgfsetfillcolor{textcolor}%
\pgftext[x=1.097413in,y=0.618750in,left,base]{\color{textcolor}\rmfamily\fontsize{9.000000}{10.800000}\selectfont Empirical}%
\end{pgfscope}%
\end{pgfpicture}%
\makeatother%
\endgroup%

%% file: images/diagramplatt.pgf
\begingroup%
\makeatletter%
\begin{pgfpicture}%
\pgfpathrectangle{\pgfpointorigin}{\pgfqpoint{2.027626in}{1.927403in}}%
\pgfusepath{use as bounding box, clip}%
\begin{pgfscope}%
\pgfsetbuttcap%
\pgfsetmiterjoin%
\definecolor{currentfill}{rgb}{1.000000,1.000000,1.000000}%
\pgfsetfillcolor{currentfill}%
\pgfsetlinewidth{0.000000pt}%
\definecolor{currentstroke}{rgb}{1.000000,1.000000,1.000000}%
\pgfsetstrokecolor{currentstroke}%
\pgfsetdash{}{0pt}%
\pgfpathmoveto{\pgfqpoint{0.000000in}{0.000000in}}%
\pgfpathlineto{\pgfqpoint{2.027626in}{0.000000in}}%
\pgfpathlineto{\pgfqpoint{2.027626in}{1.927403in}}%
\pgfpathlineto{\pgfqpoint{0.000000in}{1.927403in}}%
\pgfpathclose%
\pgfusepath{fill}%
\end{pgfscope}%
\begin{pgfscope}%
\pgfsetbuttcap%
\pgfsetmiterjoin%
\definecolor{currentfill}{rgb}{1.000000,1.000000,1.000000}%
\pgfsetfillcolor{currentfill}%
\pgfsetlinewidth{0.000000pt}%
\definecolor{currentstroke}{rgb}{0.000000,0.000000,0.000000}%
\pgfsetstrokecolor{currentstroke}%
\pgfsetstrokeopacity{0.000000}%
\pgfsetdash{}{0pt}%
\pgfpathmoveto{\pgfqpoint{0.528047in}{0.475000in}}%
\pgfpathlineto{\pgfqpoint{1.845547in}{0.475000in}}%
\pgfpathlineto{\pgfqpoint{1.845547in}{1.784000in}}%
\pgfpathlineto{\pgfqpoint{0.528047in}{1.784000in}}%
\pgfpathclose%
\pgfusepath{fill}%
\end{pgfscope}%
\begin{pgfscope}%
\pgfpathrectangle{\pgfqpoint{0.528047in}{0.475000in}}{\pgfqpoint{1.317500in}{1.309000in}}%
\pgfusepath{clip}%
\pgfsetbuttcap%
\pgfsetmiterjoin%
\definecolor{currentfill}{rgb}{0.000000,0.000000,1.000000}%
\pgfsetfillcolor{currentfill}%
\pgfsetfillopacity{0.500000}%
\pgfsetlinewidth{0.000000pt}%
\definecolor{currentstroke}{rgb}{0.000000,0.000000,0.000000}%
\pgfsetstrokecolor{currentstroke}%
\pgfsetstrokeopacity{0.500000}%
\pgfsetdash{}{0pt}%
\pgfpathmoveto{\pgfqpoint{0.791547in}{0.475000in}}%
\pgfpathlineto{\pgfqpoint{0.923297in}{0.475000in}}%
\pgfpathlineto{\pgfqpoint{0.923297in}{0.820417in}}%
\pgfpathlineto{\pgfqpoint{0.791547in}{0.820417in}}%
\pgfpathclose%
\pgfusepath{fill}%
\end{pgfscope}%
\begin{pgfscope}%
\pgfpathrectangle{\pgfqpoint{0.528047in}{0.475000in}}{\pgfqpoint{1.317500in}{1.309000in}}%
\pgfusepath{clip}%
\pgfsetbuttcap%
\pgfsetmiterjoin%
\definecolor{currentfill}{rgb}{0.000000,0.000000,1.000000}%
\pgfsetfillcolor{currentfill}%
\pgfsetfillopacity{0.500000}%
\pgfsetlinewidth{0.000000pt}%
\definecolor{currentstroke}{rgb}{0.000000,0.000000,0.000000}%
\pgfsetstrokecolor{currentstroke}%
\pgfsetstrokeopacity{0.500000}%
\pgfsetdash{}{0pt}%
\pgfpathmoveto{\pgfqpoint{0.923297in}{0.475000in}}%
\pgfpathlineto{\pgfqpoint{1.055047in}{0.475000in}}%
\pgfpathlineto{\pgfqpoint{1.055047in}{0.930217in}}%
\pgfpathlineto{\pgfqpoint{0.923297in}{0.930217in}}%
\pgfpathclose%
\pgfusepath{fill}%
\end{pgfscope}%
\begin{pgfscope}%
\pgfpathrectangle{\pgfqpoint{0.528047in}{0.475000in}}{\pgfqpoint{1.317500in}{1.309000in}}%
\pgfusepath{clip}%
\pgfsetbuttcap%
\pgfsetmiterjoin%
\definecolor{currentfill}{rgb}{0.000000,0.000000,1.000000}%
\pgfsetfillcolor{currentfill}%
\pgfsetfillopacity{0.500000}%
\pgfsetlinewidth{0.000000pt}%
\definecolor{currentstroke}{rgb}{0.000000,0.000000,0.000000}%
\pgfsetstrokecolor{currentstroke}%
\pgfsetstrokeopacity{0.500000}%
\pgfsetdash{}{0pt}%
\pgfpathmoveto{\pgfqpoint{1.055047in}{0.475000in}}%
\pgfpathlineto{\pgfqpoint{1.186797in}{0.475000in}}%
\pgfpathlineto{\pgfqpoint{1.186797in}{1.059677in}}%
\pgfpathlineto{\pgfqpoint{1.055047in}{1.059677in}}%
\pgfpathclose%
\pgfusepath{fill}%
\end{pgfscope}%
\begin{pgfscope}%
\pgfpathrectangle{\pgfqpoint{0.528047in}{0.475000in}}{\pgfqpoint{1.317500in}{1.309000in}}%
\pgfusepath{clip}%
\pgfsetbuttcap%
\pgfsetmiterjoin%
\definecolor{currentfill}{rgb}{0.000000,0.000000,1.000000}%
\pgfsetfillcolor{currentfill}%
\pgfsetfillopacity{0.500000}%
\pgfsetlinewidth{0.000000pt}%
\definecolor{currentstroke}{rgb}{0.000000,0.000000,0.000000}%
\pgfsetstrokecolor{currentstroke}%
\pgfsetstrokeopacity{0.500000}%
\pgfsetdash{}{0pt}%
\pgfpathmoveto{\pgfqpoint{1.186797in}{0.475000in}}%
\pgfpathlineto{\pgfqpoint{1.318547in}{0.475000in}}%
\pgfpathlineto{\pgfqpoint{1.318547in}{1.192306in}}%
\pgfpathlineto{\pgfqpoint{1.186797in}{1.192306in}}%
\pgfpathclose%
\pgfusepath{fill}%
\end{pgfscope}%
\begin{pgfscope}%
\pgfpathrectangle{\pgfqpoint{0.528047in}{0.475000in}}{\pgfqpoint{1.317500in}{1.309000in}}%
\pgfusepath{clip}%
\pgfsetbuttcap%
\pgfsetmiterjoin%
\definecolor{currentfill}{rgb}{0.000000,0.000000,1.000000}%
\pgfsetfillcolor{currentfill}%
\pgfsetfillopacity{0.500000}%
\pgfsetlinewidth{0.000000pt}%
\definecolor{currentstroke}{rgb}{0.000000,0.000000,0.000000}%
\pgfsetstrokecolor{currentstroke}%
\pgfsetstrokeopacity{0.500000}%
\pgfsetdash{}{0pt}%
\pgfpathmoveto{\pgfqpoint{1.318547in}{0.475000in}}%
\pgfpathlineto{\pgfqpoint{1.450297in}{0.475000in}}%
\pgfpathlineto{\pgfqpoint{1.450297in}{1.332576in}}%
\pgfpathlineto{\pgfqpoint{1.318547in}{1.332576in}}%
\pgfpathclose%
\pgfusepath{fill}%
\end{pgfscope}%
\begin{pgfscope}%
\pgfpathrectangle{\pgfqpoint{0.528047in}{0.475000in}}{\pgfqpoint{1.317500in}{1.309000in}}%
\pgfusepath{clip}%
\pgfsetbuttcap%
\pgfsetmiterjoin%
\definecolor{currentfill}{rgb}{0.000000,0.000000,1.000000}%
\pgfsetfillcolor{currentfill}%
\pgfsetfillopacity{0.500000}%
\pgfsetlinewidth{0.000000pt}%
\definecolor{currentstroke}{rgb}{0.000000,0.000000,0.000000}%
\pgfsetstrokecolor{currentstroke}%
\pgfsetstrokeopacity{0.500000}%
\pgfsetdash{}{0pt}%
\pgfpathmoveto{\pgfqpoint{1.450297in}{0.475000in}}%
\pgfpathlineto{\pgfqpoint{1.582047in}{0.475000in}}%
\pgfpathlineto{\pgfqpoint{1.582047in}{1.461292in}}%
\pgfpathlineto{\pgfqpoint{1.450297in}{1.461292in}}%
\pgfpathclose%
\pgfusepath{fill}%
\end{pgfscope}%
\begin{pgfscope}%
\pgfpathrectangle{\pgfqpoint{0.528047in}{0.475000in}}{\pgfqpoint{1.317500in}{1.309000in}}%
\pgfusepath{clip}%
\pgfsetbuttcap%
\pgfsetmiterjoin%
\definecolor{currentfill}{rgb}{0.000000,0.000000,1.000000}%
\pgfsetfillcolor{currentfill}%
\pgfsetfillopacity{0.500000}%
\pgfsetlinewidth{0.000000pt}%
\definecolor{currentstroke}{rgb}{0.000000,0.000000,0.000000}%
\pgfsetstrokecolor{currentstroke}%
\pgfsetstrokeopacity{0.500000}%
\pgfsetdash{}{0pt}%
\pgfpathmoveto{\pgfqpoint{1.582047in}{0.475000in}}%
\pgfpathlineto{\pgfqpoint{1.713797in}{0.475000in}}%
\pgfpathlineto{\pgfqpoint{1.713797in}{1.595070in}}%
\pgfpathlineto{\pgfqpoint{1.582047in}{1.595070in}}%
\pgfpathclose%
\pgfusepath{fill}%
\end{pgfscope}%
\begin{pgfscope}%
\pgfpathrectangle{\pgfqpoint{0.528047in}{0.475000in}}{\pgfqpoint{1.317500in}{1.309000in}}%
\pgfusepath{clip}%
\pgfsetbuttcap%
\pgfsetmiterjoin%
\definecolor{currentfill}{rgb}{0.000000,0.000000,1.000000}%
\pgfsetfillcolor{currentfill}%
\pgfsetfillopacity{0.500000}%
\pgfsetlinewidth{0.000000pt}%
\definecolor{currentstroke}{rgb}{0.000000,0.000000,0.000000}%
\pgfsetstrokecolor{currentstroke}%
\pgfsetstrokeopacity{0.500000}%
\pgfsetdash{}{0pt}%
\pgfpathmoveto{\pgfqpoint{1.713797in}{0.475000in}}%
\pgfpathlineto{\pgfqpoint{1.845547in}{0.475000in}}%
\pgfpathlineto{\pgfqpoint{1.845547in}{1.752920in}}%
\pgfpathlineto{\pgfqpoint{1.713797in}{1.752920in}}%
\pgfpathclose%
\pgfusepath{fill}%
\end{pgfscope}%
\begin{pgfscope}%
\pgfpathrectangle{\pgfqpoint{0.528047in}{0.475000in}}{\pgfqpoint{1.317500in}{1.309000in}}%
\pgfusepath{clip}%
\pgfsetbuttcap%
\pgfsetmiterjoin%
\definecolor{currentfill}{rgb}{1.000000,0.000000,0.000000}%
\pgfsetfillcolor{currentfill}%
\pgfsetfillopacity{0.500000}%
\pgfsetlinewidth{0.000000pt}%
\definecolor{currentstroke}{rgb}{0.000000,0.000000,0.000000}%
\pgfsetstrokecolor{currentstroke}%
\pgfsetstrokeopacity{0.500000}%
\pgfsetdash{}{0pt}%
\pgfpathmoveto{\pgfqpoint{0.791547in}{0.475000in}}%
\pgfpathlineto{\pgfqpoint{0.923297in}{0.475000in}}%
\pgfpathlineto{\pgfqpoint{0.923297in}{0.969511in}}%
\pgfpathlineto{\pgfqpoint{0.791547in}{0.969511in}}%
\pgfpathclose%
\pgfusepath{fill}%
\end{pgfscope}%
\begin{pgfscope}%
\pgfpathrectangle{\pgfqpoint{0.528047in}{0.475000in}}{\pgfqpoint{1.317500in}{1.309000in}}%
\pgfusepath{clip}%
\pgfsetbuttcap%
\pgfsetmiterjoin%
\definecolor{currentfill}{rgb}{1.000000,0.000000,0.000000}%
\pgfsetfillcolor{currentfill}%
\pgfsetfillopacity{0.500000}%
\pgfsetlinewidth{0.000000pt}%
\definecolor{currentstroke}{rgb}{0.000000,0.000000,0.000000}%
\pgfsetstrokecolor{currentstroke}%
\pgfsetstrokeopacity{0.500000}%
\pgfsetdash{}{0pt}%
\pgfpathmoveto{\pgfqpoint{0.923297in}{0.475000in}}%
\pgfpathlineto{\pgfqpoint{1.055047in}{0.475000in}}%
\pgfpathlineto{\pgfqpoint{1.055047in}{1.089429in}}%
\pgfpathlineto{\pgfqpoint{0.923297in}{1.089429in}}%
\pgfpathclose%
\pgfusepath{fill}%
\end{pgfscope}%
\begin{pgfscope}%
\pgfpathrectangle{\pgfqpoint{0.528047in}{0.475000in}}{\pgfqpoint{1.317500in}{1.309000in}}%
\pgfusepath{clip}%
\pgfsetbuttcap%
\pgfsetmiterjoin%
\definecolor{currentfill}{rgb}{1.000000,0.000000,0.000000}%
\pgfsetfillcolor{currentfill}%
\pgfsetfillopacity{0.500000}%
\pgfsetlinewidth{0.000000pt}%
\definecolor{currentstroke}{rgb}{0.000000,0.000000,0.000000}%
\pgfsetstrokecolor{currentstroke}%
\pgfsetstrokeopacity{0.500000}%
\pgfsetdash{}{0pt}%
\pgfpathmoveto{\pgfqpoint{1.055047in}{0.475000in}}%
\pgfpathlineto{\pgfqpoint{1.186797in}{0.475000in}}%
\pgfpathlineto{\pgfqpoint{1.186797in}{1.182161in}}%
\pgfpathlineto{\pgfqpoint{1.055047in}{1.182161in}}%
\pgfpathclose%
\pgfusepath{fill}%
\end{pgfscope}%
\begin{pgfscope}%
\pgfpathrectangle{\pgfqpoint{0.528047in}{0.475000in}}{\pgfqpoint{1.317500in}{1.309000in}}%
\pgfusepath{clip}%
\pgfsetbuttcap%
\pgfsetmiterjoin%
\definecolor{currentfill}{rgb}{1.000000,0.000000,0.000000}%
\pgfsetfillcolor{currentfill}%
\pgfsetfillopacity{0.500000}%
\pgfsetlinewidth{0.000000pt}%
\definecolor{currentstroke}{rgb}{0.000000,0.000000,0.000000}%
\pgfsetstrokecolor{currentstroke}%
\pgfsetstrokeopacity{0.500000}%
\pgfsetdash{}{0pt}%
\pgfpathmoveto{\pgfqpoint{1.186797in}{0.475000in}}%
\pgfpathlineto{\pgfqpoint{1.318547in}{0.475000in}}%
\pgfpathlineto{\pgfqpoint{1.318547in}{1.228123in}}%
\pgfpathlineto{\pgfqpoint{1.186797in}{1.228123in}}%
\pgfpathclose%
\pgfusepath{fill}%
\end{pgfscope}%
\begin{pgfscope}%
\pgfpathrectangle{\pgfqpoint{0.528047in}{0.475000in}}{\pgfqpoint{1.317500in}{1.309000in}}%
\pgfusepath{clip}%
\pgfsetbuttcap%
\pgfsetmiterjoin%
\definecolor{currentfill}{rgb}{1.000000,0.000000,0.000000}%
\pgfsetfillcolor{currentfill}%
\pgfsetfillopacity{0.500000}%
\pgfsetlinewidth{0.000000pt}%
\definecolor{currentstroke}{rgb}{0.000000,0.000000,0.000000}%
\pgfsetstrokecolor{currentstroke}%
\pgfsetstrokeopacity{0.500000}%
\pgfsetdash{}{0pt}%
\pgfpathmoveto{\pgfqpoint{1.318547in}{0.475000in}}%
\pgfpathlineto{\pgfqpoint{1.450297in}{0.475000in}}%
\pgfpathlineto{\pgfqpoint{1.450297in}{1.252991in}}%
\pgfpathlineto{\pgfqpoint{1.318547in}{1.252991in}}%
\pgfpathclose%
\pgfusepath{fill}%
\end{pgfscope}%
\begin{pgfscope}%
\pgfpathrectangle{\pgfqpoint{0.528047in}{0.475000in}}{\pgfqpoint{1.317500in}{1.309000in}}%
\pgfusepath{clip}%
\pgfsetbuttcap%
\pgfsetmiterjoin%
\definecolor{currentfill}{rgb}{1.000000,0.000000,0.000000}%
\pgfsetfillcolor{currentfill}%
\pgfsetfillopacity{0.500000}%
\pgfsetlinewidth{0.000000pt}%
\definecolor{currentstroke}{rgb}{0.000000,0.000000,0.000000}%
\pgfsetstrokecolor{currentstroke}%
\pgfsetstrokeopacity{0.500000}%
\pgfsetdash{}{0pt}%
\pgfpathmoveto{\pgfqpoint{1.450297in}{0.475000in}}%
\pgfpathlineto{\pgfqpoint{1.582047in}{0.475000in}}%
\pgfpathlineto{\pgfqpoint{1.582047in}{1.246974in}}%
\pgfpathlineto{\pgfqpoint{1.450297in}{1.246974in}}%
\pgfpathclose%
\pgfusepath{fill}%
\end{pgfscope}%
\begin{pgfscope}%
\pgfpathrectangle{\pgfqpoint{0.528047in}{0.475000in}}{\pgfqpoint{1.317500in}{1.309000in}}%
\pgfusepath{clip}%
\pgfsetbuttcap%
\pgfsetmiterjoin%
\definecolor{currentfill}{rgb}{1.000000,0.000000,0.000000}%
\pgfsetfillcolor{currentfill}%
\pgfsetfillopacity{0.500000}%
\pgfsetlinewidth{0.000000pt}%
\definecolor{currentstroke}{rgb}{0.000000,0.000000,0.000000}%
\pgfsetstrokecolor{currentstroke}%
\pgfsetstrokeopacity{0.500000}%
\pgfsetdash{}{0pt}%
\pgfpathmoveto{\pgfqpoint{1.582047in}{0.475000in}}%
\pgfpathlineto{\pgfqpoint{1.713797in}{0.475000in}}%
\pgfpathlineto{\pgfqpoint{1.713797in}{1.349623in}}%
\pgfpathlineto{\pgfqpoint{1.582047in}{1.349623in}}%
\pgfpathclose%
\pgfusepath{fill}%
\end{pgfscope}%
\begin{pgfscope}%
\pgfpathrectangle{\pgfqpoint{0.528047in}{0.475000in}}{\pgfqpoint{1.317500in}{1.309000in}}%
\pgfusepath{clip}%
\pgfsetbuttcap%
\pgfsetmiterjoin%
\definecolor{currentfill}{rgb}{1.000000,0.000000,0.000000}%
\pgfsetfillcolor{currentfill}%
\pgfsetfillopacity{0.500000}%
\pgfsetlinewidth{0.000000pt}%
\definecolor{currentstroke}{rgb}{0.000000,0.000000,0.000000}%
\pgfsetstrokecolor{currentstroke}%
\pgfsetstrokeopacity{0.500000}%
\pgfsetdash{}{0pt}%
\pgfpathmoveto{\pgfqpoint{1.713797in}{0.475000in}}%
\pgfpathlineto{\pgfqpoint{1.845547in}{0.475000in}}%
\pgfpathlineto{\pgfqpoint{1.845547in}{1.753466in}}%
\pgfpathlineto{\pgfqpoint{1.713797in}{1.753466in}}%
\pgfpathclose%
\pgfusepath{fill}%
\end{pgfscope}%
\begin{pgfscope}%
\pgfsetbuttcap%
\pgfsetroundjoin%
\definecolor{currentfill}{rgb}{0.000000,0.000000,0.000000}%
\pgfsetfillcolor{currentfill}%
\pgfsetlinewidth{0.803000pt}%
\definecolor{currentstroke}{rgb}{0.000000,0.000000,0.000000}%
\pgfsetstrokecolor{currentstroke}%
\pgfsetdash{}{0pt}%
\pgfsys@defobject{currentmarker}{\pgfqpoint{0.000000in}{-0.048611in}}{\pgfqpoint{0.000000in}{0.000000in}}{%
\pgfpathmoveto{\pgfqpoint{0.000000in}{0.000000in}}%
\pgfpathlineto{\pgfqpoint{0.000000in}{-0.048611in}}%
\pgfusepath{stroke,fill}%
}%
\begin{pgfscope}%
\pgfsys@transformshift{0.528047in}{0.475000in}%
\pgfsys@useobject{currentmarker}{}%
\end{pgfscope}%
\end{pgfscope}%
\begin{pgfscope}%
\definecolor{textcolor}{rgb}{0.000000,0.000000,0.000000}%
\pgfsetstrokecolor{textcolor}%
\pgfsetfillcolor{textcolor}%
\pgftext[x=0.528047in,y=0.377778in,,top]{\color{textcolor}\rmfamily\fontsize{9.000000}{10.800000}\selectfont \(\displaystyle {0.0}\)}%
\end{pgfscope}%
\begin{pgfscope}%
\pgfsetbuttcap%
\pgfsetroundjoin%
\definecolor{currentfill}{rgb}{0.000000,0.000000,0.000000}%
\pgfsetfillcolor{currentfill}%
\pgfsetlinewidth{0.803000pt}%
\definecolor{currentstroke}{rgb}{0.000000,0.000000,0.000000}%
\pgfsetstrokecolor{currentstroke}%
\pgfsetdash{}{0pt}%
\pgfsys@defobject{currentmarker}{\pgfqpoint{0.000000in}{-0.048611in}}{\pgfqpoint{0.000000in}{0.000000in}}{%
\pgfpathmoveto{\pgfqpoint{0.000000in}{0.000000in}}%
\pgfpathlineto{\pgfqpoint{0.000000in}{-0.048611in}}%
\pgfusepath{stroke,fill}%
}%
\begin{pgfscope}%
\pgfsys@transformshift{1.186797in}{0.475000in}%
\pgfsys@useobject{currentmarker}{}%
\end{pgfscope}%
\end{pgfscope}%
\begin{pgfscope}%
\definecolor{textcolor}{rgb}{0.000000,0.000000,0.000000}%
\pgfsetstrokecolor{textcolor}%
\pgfsetfillcolor{textcolor}%
\pgftext[x=1.186797in,y=0.377778in,,top]{\color{textcolor}\rmfamily\fontsize{9.000000}{10.800000}\selectfont \(\displaystyle {0.5}\)}%
\end{pgfscope}%
\begin{pgfscope}%
\pgfsetbuttcap%
\pgfsetroundjoin%
\definecolor{currentfill}{rgb}{0.000000,0.000000,0.000000}%
\pgfsetfillcolor{currentfill}%
\pgfsetlinewidth{0.803000pt}%
\definecolor{currentstroke}{rgb}{0.000000,0.000000,0.000000}%
\pgfsetstrokecolor{currentstroke}%
\pgfsetdash{}{0pt}%
\pgfsys@defobject{currentmarker}{\pgfqpoint{0.000000in}{-0.048611in}}{\pgfqpoint{0.000000in}{0.000000in}}{%
\pgfpathmoveto{\pgfqpoint{0.000000in}{0.000000in}}%
\pgfpathlineto{\pgfqpoint{0.000000in}{-0.048611in}}%
\pgfusepath{stroke,fill}%
}%
\begin{pgfscope}%
\pgfsys@transformshift{1.845547in}{0.475000in}%
\pgfsys@useobject{currentmarker}{}%
\end{pgfscope}%
\end{pgfscope}%
\begin{pgfscope}%
\definecolor{textcolor}{rgb}{0.000000,0.000000,0.000000}%
\pgfsetstrokecolor{textcolor}%
\pgfsetfillcolor{textcolor}%
\pgftext[x=1.845547in,y=0.377778in,,top]{\color{textcolor}\rmfamily\fontsize{9.000000}{10.800000}\selectfont \(\displaystyle {1.0}\)}%
\end{pgfscope}%
\begin{pgfscope}%
\definecolor{textcolor}{rgb}{0.000000,0.000000,0.000000}%
\pgfsetstrokecolor{textcolor}%
\pgfsetfillcolor{textcolor}%
\pgftext[x=1.186797in,y=0.211111in,,top]{\color{textcolor}\rmfamily\fontsize{9.000000}{10.800000}\selectfont Predicted Confidence}%
\end{pgfscope}%
\begin{pgfscope}%
\pgfsetbuttcap%
\pgfsetroundjoin%
\definecolor{currentfill}{rgb}{0.000000,0.000000,0.000000}%
\pgfsetfillcolor{currentfill}%
\pgfsetlinewidth{0.803000pt}%
\definecolor{currentstroke}{rgb}{0.000000,0.000000,0.000000}%
\pgfsetstrokecolor{currentstroke}%
\pgfsetdash{}{0pt}%
\pgfsys@defobject{currentmarker}{\pgfqpoint{-0.048611in}{0.000000in}}{\pgfqpoint{-0.000000in}{0.000000in}}{%
\pgfpathmoveto{\pgfqpoint{-0.000000in}{0.000000in}}%
\pgfpathlineto{\pgfqpoint{-0.048611in}{0.000000in}}%
\pgfusepath{stroke,fill}%
}%
\begin{pgfscope}%
\pgfsys@transformshift{0.528047in}{0.475000in}%
\pgfsys@useobject{currentmarker}{}%
\end{pgfscope}%
\end{pgfscope}%
\begin{pgfscope}%
\definecolor{textcolor}{rgb}{0.000000,0.000000,0.000000}%
\pgfsetstrokecolor{textcolor}%
\pgfsetfillcolor{textcolor}%
\pgftext[x=0.266667in, y=0.431597in, left, base]{\color{textcolor}\rmfamily\fontsize{9.000000}{10.800000}\selectfont \(\displaystyle {0.0}\)}%
\end{pgfscope}%
\begin{pgfscope}%
\pgfsetbuttcap%
\pgfsetroundjoin%
\definecolor{currentfill}{rgb}{0.000000,0.000000,0.000000}%
\pgfsetfillcolor{currentfill}%
\pgfsetlinewidth{0.803000pt}%
\definecolor{currentstroke}{rgb}{0.000000,0.000000,0.000000}%
\pgfsetstrokecolor{currentstroke}%
\pgfsetdash{}{0pt}%
\pgfsys@defobject{currentmarker}{\pgfqpoint{-0.048611in}{0.000000in}}{\pgfqpoint{-0.000000in}{0.000000in}}{%
\pgfpathmoveto{\pgfqpoint{-0.000000in}{0.000000in}}%
\pgfpathlineto{\pgfqpoint{-0.048611in}{0.000000in}}%
\pgfusepath{stroke,fill}%
}%
\begin{pgfscope}%
\pgfsys@transformshift{0.528047in}{1.129500in}%
\pgfsys@useobject{currentmarker}{}%
\end{pgfscope}%
\end{pgfscope}%
\begin{pgfscope}%
\definecolor{textcolor}{rgb}{0.000000,0.000000,0.000000}%
\pgfsetstrokecolor{textcolor}%
\pgfsetfillcolor{textcolor}%
\pgftext[x=0.266667in, y=1.086097in, left, base]{\color{textcolor}\rmfamily\fontsize{9.000000}{10.800000}\selectfont \(\displaystyle {0.5}\)}%
\end{pgfscope}%
\begin{pgfscope}%
\pgfsetbuttcap%
\pgfsetroundjoin%
\definecolor{currentfill}{rgb}{0.000000,0.000000,0.000000}%
\pgfsetfillcolor{currentfill}%
\pgfsetlinewidth{0.803000pt}%
\definecolor{currentstroke}{rgb}{0.000000,0.000000,0.000000}%
\pgfsetstrokecolor{currentstroke}%
\pgfsetdash{}{0pt}%
\pgfsys@defobject{currentmarker}{\pgfqpoint{-0.048611in}{0.000000in}}{\pgfqpoint{-0.000000in}{0.000000in}}{%
\pgfpathmoveto{\pgfqpoint{-0.000000in}{0.000000in}}%
\pgfpathlineto{\pgfqpoint{-0.048611in}{0.000000in}}%
\pgfusepath{stroke,fill}%
}%
\begin{pgfscope}%
\pgfsys@transformshift{0.528047in}{1.784000in}%
\pgfsys@useobject{currentmarker}{}%
\end{pgfscope}%
\end{pgfscope}%
\begin{pgfscope}%
\definecolor{textcolor}{rgb}{0.000000,0.000000,0.000000}%
\pgfsetstrokecolor{textcolor}%
\pgfsetfillcolor{textcolor}%
\pgftext[x=0.266667in, y=1.740597in, left, base]{\color{textcolor}\rmfamily\fontsize{9.000000}{10.800000}\selectfont \(\displaystyle {1.0}\)}%
\end{pgfscope}%
\begin{pgfscope}%
\definecolor{textcolor}{rgb}{0.000000,0.000000,0.000000}%
\pgfsetstrokecolor{textcolor}%
\pgfsetfillcolor{textcolor}%
\pgftext[x=0.211111in,y=1.129500in,,bottom,rotate=90.000000]{\color{textcolor}\rmfamily\fontsize{9.000000}{10.800000}\selectfont Empirical Accuracy}%
\end{pgfscope}%
\begin{pgfscope}%
\pgfpathrectangle{\pgfqpoint{0.528047in}{0.475000in}}{\pgfqpoint{1.317500in}{1.309000in}}%
\pgfusepath{clip}%
\pgfsetbuttcap%
\pgfsetroundjoin%
\pgfsetlinewidth{1.505625pt}%
\definecolor{currentstroke}{rgb}{0.000000,0.000000,0.000000}%
\pgfsetstrokecolor{currentstroke}%
\pgfsetdash{{5.550000pt}{2.400000pt}}{0.000000pt}%
\pgfpathmoveto{\pgfqpoint{0.528047in}{0.475000in}}%
\pgfpathlineto{\pgfqpoint{0.541222in}{0.488090in}}%
\pgfpathlineto{\pgfqpoint{0.554397in}{0.501180in}}%
\pgfpathlineto{\pgfqpoint{0.567572in}{0.514270in}}%
\pgfpathlineto{\pgfqpoint{0.580747in}{0.527360in}}%
\pgfpathlineto{\pgfqpoint{0.593922in}{0.540450in}}%
\pgfpathlineto{\pgfqpoint{0.607097in}{0.553540in}}%
\pgfpathlineto{\pgfqpoint{0.620272in}{0.566630in}}%
\pgfpathlineto{\pgfqpoint{0.633447in}{0.579720in}}%
\pgfpathlineto{\pgfqpoint{0.646622in}{0.592810in}}%
\pgfpathlineto{\pgfqpoint{0.659797in}{0.605900in}}%
\pgfpathlineto{\pgfqpoint{0.672972in}{0.618990in}}%
\pgfpathlineto{\pgfqpoint{0.686147in}{0.632080in}}%
\pgfpathlineto{\pgfqpoint{0.699322in}{0.645170in}}%
\pgfpathlineto{\pgfqpoint{0.712497in}{0.658260in}}%
\pgfpathlineto{\pgfqpoint{0.725672in}{0.671350in}}%
\pgfpathlineto{\pgfqpoint{0.738847in}{0.684440in}}%
\pgfpathlineto{\pgfqpoint{0.752022in}{0.697530in}}%
\pgfpathlineto{\pgfqpoint{0.765197in}{0.710620in}}%
\pgfpathlineto{\pgfqpoint{0.778372in}{0.723710in}}%
\pgfpathlineto{\pgfqpoint{0.791547in}{0.736800in}}%
\pgfpathlineto{\pgfqpoint{0.804722in}{0.749890in}}%
\pgfpathlineto{\pgfqpoint{0.817897in}{0.762980in}}%
\pgfpathlineto{\pgfqpoint{0.831072in}{0.776070in}}%
\pgfpathlineto{\pgfqpoint{0.844247in}{0.789160in}}%
\pgfpathlineto{\pgfqpoint{0.857422in}{0.802250in}}%
\pgfpathlineto{\pgfqpoint{0.870597in}{0.815340in}}%
\pgfpathlineto{\pgfqpoint{0.883772in}{0.828430in}}%
\pgfpathlineto{\pgfqpoint{0.896947in}{0.841520in}}%
\pgfpathlineto{\pgfqpoint{0.910122in}{0.854610in}}%
\pgfpathlineto{\pgfqpoint{0.923297in}{0.867700in}}%
\pgfpathlineto{\pgfqpoint{0.936472in}{0.880790in}}%
\pgfpathlineto{\pgfqpoint{0.949647in}{0.893880in}}%
\pgfpathlineto{\pgfqpoint{0.962822in}{0.906970in}}%
\pgfpathlineto{\pgfqpoint{0.975997in}{0.920060in}}%
\pgfpathlineto{\pgfqpoint{0.989172in}{0.933150in}}%
\pgfpathlineto{\pgfqpoint{1.002347in}{0.946240in}}%
\pgfpathlineto{\pgfqpoint{1.015522in}{0.959330in}}%
\pgfpathlineto{\pgfqpoint{1.028697in}{0.972420in}}%
\pgfpathlineto{\pgfqpoint{1.041872in}{0.985510in}}%
\pgfpathlineto{\pgfqpoint{1.055047in}{0.998600in}}%
\pgfpathlineto{\pgfqpoint{1.068222in}{1.011690in}}%
\pgfpathlineto{\pgfqpoint{1.081397in}{1.024780in}}%
\pgfpathlineto{\pgfqpoint{1.094572in}{1.037870in}}%
\pgfpathlineto{\pgfqpoint{1.107747in}{1.050960in}}%
\pgfpathlineto{\pgfqpoint{1.120922in}{1.064050in}}%
\pgfpathlineto{\pgfqpoint{1.134097in}{1.077140in}}%
\pgfpathlineto{\pgfqpoint{1.147272in}{1.090230in}}%
\pgfpathlineto{\pgfqpoint{1.160447in}{1.103320in}}%
\pgfpathlineto{\pgfqpoint{1.173622in}{1.116410in}}%
\pgfpathlineto{\pgfqpoint{1.186797in}{1.129500in}}%
\pgfpathlineto{\pgfqpoint{1.199972in}{1.142590in}}%
\pgfpathlineto{\pgfqpoint{1.213147in}{1.155680in}}%
\pgfpathlineto{\pgfqpoint{1.226322in}{1.168770in}}%
\pgfpathlineto{\pgfqpoint{1.239497in}{1.181860in}}%
\pgfpathlineto{\pgfqpoint{1.252672in}{1.194950in}}%
\pgfpathlineto{\pgfqpoint{1.265847in}{1.208040in}}%
\pgfpathlineto{\pgfqpoint{1.279022in}{1.221130in}}%
\pgfpathlineto{\pgfqpoint{1.292197in}{1.234220in}}%
\pgfpathlineto{\pgfqpoint{1.305372in}{1.247310in}}%
\pgfpathlineto{\pgfqpoint{1.318547in}{1.260400in}}%
\pgfpathlineto{\pgfqpoint{1.331722in}{1.273490in}}%
\pgfpathlineto{\pgfqpoint{1.344897in}{1.286580in}}%
\pgfpathlineto{\pgfqpoint{1.358072in}{1.299670in}}%
\pgfpathlineto{\pgfqpoint{1.371247in}{1.312760in}}%
\pgfpathlineto{\pgfqpoint{1.384422in}{1.325850in}}%
\pgfpathlineto{\pgfqpoint{1.397597in}{1.338940in}}%
\pgfpathlineto{\pgfqpoint{1.410772in}{1.352030in}}%
\pgfpathlineto{\pgfqpoint{1.423947in}{1.365120in}}%
\pgfpathlineto{\pgfqpoint{1.437122in}{1.378210in}}%
\pgfpathlineto{\pgfqpoint{1.450297in}{1.391300in}}%
\pgfpathlineto{\pgfqpoint{1.463472in}{1.404390in}}%
\pgfpathlineto{\pgfqpoint{1.476647in}{1.417480in}}%
\pgfpathlineto{\pgfqpoint{1.489822in}{1.430570in}}%
\pgfpathlineto{\pgfqpoint{1.502997in}{1.443660in}}%
\pgfpathlineto{\pgfqpoint{1.516172in}{1.456750in}}%
\pgfpathlineto{\pgfqpoint{1.529347in}{1.469840in}}%
\pgfpathlineto{\pgfqpoint{1.542522in}{1.482930in}}%
\pgfpathlineto{\pgfqpoint{1.555697in}{1.496020in}}%
\pgfpathlineto{\pgfqpoint{1.568872in}{1.509110in}}%
\pgfpathlineto{\pgfqpoint{1.582047in}{1.522200in}}%
\pgfpathlineto{\pgfqpoint{1.595222in}{1.535290in}}%
\pgfpathlineto{\pgfqpoint{1.608397in}{1.548380in}}%
\pgfpathlineto{\pgfqpoint{1.621572in}{1.561470in}}%
\pgfpathlineto{\pgfqpoint{1.634747in}{1.574560in}}%
\pgfpathlineto{\pgfqpoint{1.647922in}{1.587650in}}%
\pgfpathlineto{\pgfqpoint{1.661097in}{1.600740in}}%
\pgfpathlineto{\pgfqpoint{1.674272in}{1.613830in}}%
\pgfpathlineto{\pgfqpoint{1.687447in}{1.626920in}}%
\pgfpathlineto{\pgfqpoint{1.700622in}{1.640010in}}%
\pgfpathlineto{\pgfqpoint{1.713797in}{1.653100in}}%
\pgfpathlineto{\pgfqpoint{1.726972in}{1.666190in}}%
\pgfpathlineto{\pgfqpoint{1.740147in}{1.679280in}}%
\pgfpathlineto{\pgfqpoint{1.753322in}{1.692370in}}%
\pgfpathlineto{\pgfqpoint{1.766497in}{1.705460in}}%
\pgfpathlineto{\pgfqpoint{1.779672in}{1.718550in}}%
\pgfpathlineto{\pgfqpoint{1.792847in}{1.731640in}}%
\pgfpathlineto{\pgfqpoint{1.806022in}{1.744730in}}%
\pgfpathlineto{\pgfqpoint{1.819197in}{1.757820in}}%
\pgfpathlineto{\pgfqpoint{1.832372in}{1.770910in}}%
\pgfusepath{stroke}%
\end{pgfscope}%
\begin{pgfscope}%
\pgfsetrectcap%
\pgfsetmiterjoin%
\pgfsetlinewidth{0.803000pt}%
\definecolor{currentstroke}{rgb}{0.000000,0.000000,0.000000}%
\pgfsetstrokecolor{currentstroke}%
\pgfsetdash{}{0pt}%
\pgfpathmoveto{\pgfqpoint{0.528047in}{0.475000in}}%
\pgfpathlineto{\pgfqpoint{0.528047in}{1.784000in}}%
\pgfusepath{stroke}%
\end{pgfscope}%
\begin{pgfscope}%
\pgfsetrectcap%
\pgfsetmiterjoin%
\pgfsetlinewidth{0.803000pt}%
\definecolor{currentstroke}{rgb}{0.000000,0.000000,0.000000}%
\pgfsetstrokecolor{currentstroke}%
\pgfsetdash{}{0pt}%
\pgfpathmoveto{\pgfqpoint{1.845547in}{0.475000in}}%
\pgfpathlineto{\pgfqpoint{1.845547in}{1.784000in}}%
\pgfusepath{stroke}%
\end{pgfscope}%
\begin{pgfscope}%
\pgfsetrectcap%
\pgfsetmiterjoin%
\pgfsetlinewidth{0.803000pt}%
\definecolor{currentstroke}{rgb}{0.000000,0.000000,0.000000}%
\pgfsetstrokecolor{currentstroke}%
\pgfsetdash{}{0pt}%
\pgfpathmoveto{\pgfqpoint{0.528047in}{0.475000in}}%
\pgfpathlineto{\pgfqpoint{1.845547in}{0.475000in}}%
\pgfusepath{stroke}%
\end{pgfscope}%
\begin{pgfscope}%
\pgfsetrectcap%
\pgfsetmiterjoin%
\pgfsetlinewidth{0.803000pt}%
\definecolor{currentstroke}{rgb}{0.000000,0.000000,0.000000}%
\pgfsetstrokecolor{currentstroke}%
\pgfsetdash{}{0pt}%
\pgfpathmoveto{\pgfqpoint{0.528047in}{1.784000in}}%
\pgfpathlineto{\pgfqpoint{1.845547in}{1.784000in}}%
\pgfusepath{stroke}%
\end{pgfscope}%
\begin{pgfscope}%
\pgfsetbuttcap%
\pgfsetmiterjoin%
\definecolor{currentfill}{rgb}{1.000000,1.000000,1.000000}%
\pgfsetfillcolor{currentfill}%
\pgfsetfillopacity{0.500000}%
\pgfsetlinewidth{1.003750pt}%
\definecolor{currentstroke}{rgb}{0.800000,0.800000,0.800000}%
\pgfsetstrokecolor{currentstroke}%
\pgfsetstrokeopacity{0.500000}%
\pgfsetdash{}{0pt}%
\pgfpathmoveto{\pgfqpoint{0.615547in}{1.335389in}}%
\pgfpathlineto{\pgfqpoint{1.318882in}{1.335389in}}%
\pgfpathquadraticcurveto{\pgfqpoint{1.343882in}{1.335389in}}{\pgfqpoint{1.343882in}{1.360389in}}%
\pgfpathlineto{\pgfqpoint{1.343882in}{1.696500in}}%
\pgfpathquadraticcurveto{\pgfqpoint{1.343882in}{1.721500in}}{\pgfqpoint{1.318882in}{1.721500in}}%
\pgfpathlineto{\pgfqpoint{0.615547in}{1.721500in}}%
\pgfpathquadraticcurveto{\pgfqpoint{0.590547in}{1.721500in}}{\pgfqpoint{0.590547in}{1.696500in}}%
\pgfpathlineto{\pgfqpoint{0.590547in}{1.360389in}}%
\pgfpathquadraticcurveto{\pgfqpoint{0.590547in}{1.335389in}}{\pgfqpoint{0.615547in}{1.335389in}}%
\pgfpathclose%
\pgfusepath{stroke,fill}%
\end{pgfscope}%
\begin{pgfscope}%
\pgfsetbuttcap%
\pgfsetmiterjoin%
\definecolor{currentfill}{rgb}{0.000000,0.000000,1.000000}%
\pgfsetfillcolor{currentfill}%
\pgfsetfillopacity{0.500000}%
\pgfsetlinewidth{0.000000pt}%
\definecolor{currentstroke}{rgb}{0.000000,0.000000,0.000000}%
\pgfsetstrokecolor{currentstroke}%
\pgfsetstrokeopacity{0.500000}%
\pgfsetdash{}{0pt}%
\pgfpathmoveto{\pgfqpoint{0.640547in}{1.584000in}}%
\pgfpathlineto{\pgfqpoint{0.890547in}{1.584000in}}%
\pgfpathlineto{\pgfqpoint{0.890547in}{1.671500in}}%
\pgfpathlineto{\pgfqpoint{0.640547in}{1.671500in}}%
\pgfpathclose%
\pgfusepath{fill}%
\end{pgfscope}%
\begin{pgfscope}%
\definecolor{textcolor}{rgb}{0.000000,0.000000,0.000000}%
\pgfsetstrokecolor{textcolor}%
\pgfsetfillcolor{textcolor}%
\pgftext[x=0.990547in,y=1.584000in,left,base]{\color{textcolor}\rmfamily\fontsize{9.000000}{10.800000}\selectfont Conf.}%
\end{pgfscope}%
\begin{pgfscope}%
\pgfsetbuttcap%
\pgfsetmiterjoin%
\definecolor{currentfill}{rgb}{1.000000,0.000000,0.000000}%
\pgfsetfillcolor{currentfill}%
\pgfsetfillopacity{0.500000}%
\pgfsetlinewidth{0.000000pt}%
\definecolor{currentstroke}{rgb}{0.000000,0.000000,0.000000}%
\pgfsetstrokecolor{currentstroke}%
\pgfsetstrokeopacity{0.500000}%
\pgfsetdash{}{0pt}%
\pgfpathmoveto{\pgfqpoint{0.640547in}{1.409694in}}%
\pgfpathlineto{\pgfqpoint{0.890547in}{1.409694in}}%
\pgfpathlineto{\pgfqpoint{0.890547in}{1.497194in}}%
\pgfpathlineto{\pgfqpoint{0.640547in}{1.497194in}}%
\pgfpathclose%
\pgfusepath{fill}%
\end{pgfscope}%
\begin{pgfscope}%
\definecolor{textcolor}{rgb}{0.000000,0.000000,0.000000}%
\pgfsetstrokecolor{textcolor}%
\pgfsetfillcolor{textcolor}%
\pgftext[x=0.990547in,y=1.409694in,left,base]{\color{textcolor}\rmfamily\fontsize{9.000000}{10.800000}\selectfont Acc.}%
\end{pgfscope}%
\end{pgfpicture}%
\makeatother%
\endgroup%

%% file: images/diagrampoly.pgf
\begingroup%
\makeatletter%
\begin{pgfpicture}%
\pgfpathrectangle{\pgfpointorigin}{\pgfqpoint{1.860959in}{1.927403in}}%
\pgfusepath{use as bounding box, clip}%
\begin{pgfscope}%
\pgfsetbuttcap%
\pgfsetmiterjoin%
\definecolor{currentfill}{rgb}{1.000000,1.000000,1.000000}%
\pgfsetfillcolor{currentfill}%
\pgfsetlinewidth{0.000000pt}%
\definecolor{currentstroke}{rgb}{1.000000,1.000000,1.000000}%
\pgfsetstrokecolor{currentstroke}%
\pgfsetdash{}{0pt}%
\pgfpathmoveto{\pgfqpoint{0.000000in}{0.000000in}}%
\pgfpathlineto{\pgfqpoint{1.860959in}{0.000000in}}%
\pgfpathlineto{\pgfqpoint{1.860959in}{1.927403in}}%
\pgfpathlineto{\pgfqpoint{0.000000in}{1.927403in}}%
\pgfpathclose%
\pgfusepath{fill}%
\end{pgfscope}%
\begin{pgfscope}%
\pgfsetbuttcap%
\pgfsetmiterjoin%
\definecolor{currentfill}{rgb}{1.000000,1.000000,1.000000}%
\pgfsetfillcolor{currentfill}%
\pgfsetlinewidth{0.000000pt}%
\definecolor{currentstroke}{rgb}{0.000000,0.000000,0.000000}%
\pgfsetstrokecolor{currentstroke}%
\pgfsetstrokeopacity{0.000000}%
\pgfsetdash{}{0pt}%
\pgfpathmoveto{\pgfqpoint{0.361380in}{0.475000in}}%
\pgfpathlineto{\pgfqpoint{1.678880in}{0.475000in}}%
\pgfpathlineto{\pgfqpoint{1.678880in}{1.784000in}}%
\pgfpathlineto{\pgfqpoint{0.361380in}{1.784000in}}%
\pgfpathclose%
\pgfusepath{fill}%
\end{pgfscope}%
\begin{pgfscope}%
\pgfpathrectangle{\pgfqpoint{0.361380in}{0.475000in}}{\pgfqpoint{1.317500in}{1.309000in}}%
\pgfusepath{clip}%
\pgfsetbuttcap%
\pgfsetmiterjoin%
\definecolor{currentfill}{rgb}{0.000000,0.000000,1.000000}%
\pgfsetfillcolor{currentfill}%
\pgfsetfillopacity{0.500000}%
\pgfsetlinewidth{0.000000pt}%
\definecolor{currentstroke}{rgb}{0.000000,0.000000,0.000000}%
\pgfsetstrokecolor{currentstroke}%
\pgfsetstrokeopacity{0.500000}%
\pgfsetdash{}{0pt}%
\pgfpathmoveto{\pgfqpoint{1.020130in}{0.475000in}}%
\pgfpathlineto{\pgfqpoint{1.151880in}{0.475000in}}%
\pgfpathlineto{\pgfqpoint{1.151880in}{1.224481in}}%
\pgfpathlineto{\pgfqpoint{1.020130in}{1.224481in}}%
\pgfpathclose%
\pgfusepath{fill}%
\end{pgfscope}%
\begin{pgfscope}%
\pgfpathrectangle{\pgfqpoint{0.361380in}{0.475000in}}{\pgfqpoint{1.317500in}{1.309000in}}%
\pgfusepath{clip}%
\pgfsetbuttcap%
\pgfsetmiterjoin%
\definecolor{currentfill}{rgb}{0.000000,0.000000,1.000000}%
\pgfsetfillcolor{currentfill}%
\pgfsetfillopacity{0.500000}%
\pgfsetlinewidth{0.000000pt}%
\definecolor{currentstroke}{rgb}{0.000000,0.000000,0.000000}%
\pgfsetstrokecolor{currentstroke}%
\pgfsetstrokeopacity{0.500000}%
\pgfsetdash{}{0pt}%
\pgfpathmoveto{\pgfqpoint{1.151880in}{0.475000in}}%
\pgfpathlineto{\pgfqpoint{1.283630in}{0.475000in}}%
\pgfpathlineto{\pgfqpoint{1.283630in}{1.326226in}}%
\pgfpathlineto{\pgfqpoint{1.151880in}{1.326226in}}%
\pgfpathclose%
\pgfusepath{fill}%
\end{pgfscope}%
\begin{pgfscope}%
\pgfpathrectangle{\pgfqpoint{0.361380in}{0.475000in}}{\pgfqpoint{1.317500in}{1.309000in}}%
\pgfusepath{clip}%
\pgfsetbuttcap%
\pgfsetmiterjoin%
\definecolor{currentfill}{rgb}{0.000000,0.000000,1.000000}%
\pgfsetfillcolor{currentfill}%
\pgfsetfillopacity{0.500000}%
\pgfsetlinewidth{0.000000pt}%
\definecolor{currentstroke}{rgb}{0.000000,0.000000,0.000000}%
\pgfsetstrokecolor{currentstroke}%
\pgfsetstrokeopacity{0.500000}%
\pgfsetdash{}{0pt}%
\pgfpathmoveto{\pgfqpoint{1.283630in}{0.475000in}}%
\pgfpathlineto{\pgfqpoint{1.415380in}{0.475000in}}%
\pgfpathlineto{\pgfqpoint{1.415380in}{1.456148in}}%
\pgfpathlineto{\pgfqpoint{1.283630in}{1.456148in}}%
\pgfpathclose%
\pgfusepath{fill}%
\end{pgfscope}%
\begin{pgfscope}%
\pgfpathrectangle{\pgfqpoint{0.361380in}{0.475000in}}{\pgfqpoint{1.317500in}{1.309000in}}%
\pgfusepath{clip}%
\pgfsetbuttcap%
\pgfsetmiterjoin%
\definecolor{currentfill}{rgb}{0.000000,0.000000,1.000000}%
\pgfsetfillcolor{currentfill}%
\pgfsetfillopacity{0.500000}%
\pgfsetlinewidth{0.000000pt}%
\definecolor{currentstroke}{rgb}{0.000000,0.000000,0.000000}%
\pgfsetstrokecolor{currentstroke}%
\pgfsetstrokeopacity{0.500000}%
\pgfsetdash{}{0pt}%
\pgfpathmoveto{\pgfqpoint{1.415380in}{0.475000in}}%
\pgfpathlineto{\pgfqpoint{1.547130in}{0.475000in}}%
\pgfpathlineto{\pgfqpoint{1.547130in}{1.592008in}}%
\pgfpathlineto{\pgfqpoint{1.415380in}{1.592008in}}%
\pgfpathclose%
\pgfusepath{fill}%
\end{pgfscope}%
\begin{pgfscope}%
\pgfpathrectangle{\pgfqpoint{0.361380in}{0.475000in}}{\pgfqpoint{1.317500in}{1.309000in}}%
\pgfusepath{clip}%
\pgfsetbuttcap%
\pgfsetmiterjoin%
\definecolor{currentfill}{rgb}{0.000000,0.000000,1.000000}%
\pgfsetfillcolor{currentfill}%
\pgfsetfillopacity{0.500000}%
\pgfsetlinewidth{0.000000pt}%
\definecolor{currentstroke}{rgb}{0.000000,0.000000,0.000000}%
\pgfsetstrokecolor{currentstroke}%
\pgfsetstrokeopacity{0.500000}%
\pgfsetdash{}{0pt}%
\pgfpathmoveto{\pgfqpoint{1.547130in}{0.475000in}}%
\pgfpathlineto{\pgfqpoint{1.678880in}{0.475000in}}%
\pgfpathlineto{\pgfqpoint{1.678880in}{1.760825in}}%
\pgfpathlineto{\pgfqpoint{1.547130in}{1.760825in}}%
\pgfpathclose%
\pgfusepath{fill}%
\end{pgfscope}%
\begin{pgfscope}%
\pgfpathrectangle{\pgfqpoint{0.361380in}{0.475000in}}{\pgfqpoint{1.317500in}{1.309000in}}%
\pgfusepath{clip}%
\pgfsetbuttcap%
\pgfsetmiterjoin%
\definecolor{currentfill}{rgb}{1.000000,0.000000,0.000000}%
\pgfsetfillcolor{currentfill}%
\pgfsetfillopacity{0.500000}%
\pgfsetlinewidth{0.000000pt}%
\definecolor{currentstroke}{rgb}{0.000000,0.000000,0.000000}%
\pgfsetstrokecolor{currentstroke}%
\pgfsetstrokeopacity{0.500000}%
\pgfsetdash{}{0pt}%
\pgfpathmoveto{\pgfqpoint{1.020130in}{0.475000in}}%
\pgfpathlineto{\pgfqpoint{1.151880in}{0.475000in}}%
\pgfpathlineto{\pgfqpoint{1.151880in}{1.185265in}}%
\pgfpathlineto{\pgfqpoint{1.020130in}{1.185265in}}%
\pgfpathclose%
\pgfusepath{fill}%
\end{pgfscope}%
\begin{pgfscope}%
\pgfpathrectangle{\pgfqpoint{0.361380in}{0.475000in}}{\pgfqpoint{1.317500in}{1.309000in}}%
\pgfusepath{clip}%
\pgfsetbuttcap%
\pgfsetmiterjoin%
\definecolor{currentfill}{rgb}{1.000000,0.000000,0.000000}%
\pgfsetfillcolor{currentfill}%
\pgfsetfillopacity{0.500000}%
\pgfsetlinewidth{0.000000pt}%
\definecolor{currentstroke}{rgb}{0.000000,0.000000,0.000000}%
\pgfsetstrokecolor{currentstroke}%
\pgfsetstrokeopacity{0.500000}%
\pgfsetdash{}{0pt}%
\pgfpathmoveto{\pgfqpoint{1.151880in}{0.475000in}}%
\pgfpathlineto{\pgfqpoint{1.283630in}{0.475000in}}%
\pgfpathlineto{\pgfqpoint{1.283630in}{1.274389in}}%
\pgfpathlineto{\pgfqpoint{1.151880in}{1.274389in}}%
\pgfpathclose%
\pgfusepath{fill}%
\end{pgfscope}%
\begin{pgfscope}%
\pgfpathrectangle{\pgfqpoint{0.361380in}{0.475000in}}{\pgfqpoint{1.317500in}{1.309000in}}%
\pgfusepath{clip}%
\pgfsetbuttcap%
\pgfsetmiterjoin%
\definecolor{currentfill}{rgb}{1.000000,0.000000,0.000000}%
\pgfsetfillcolor{currentfill}%
\pgfsetfillopacity{0.500000}%
\pgfsetlinewidth{0.000000pt}%
\definecolor{currentstroke}{rgb}{0.000000,0.000000,0.000000}%
\pgfsetstrokecolor{currentstroke}%
\pgfsetstrokeopacity{0.500000}%
\pgfsetdash{}{0pt}%
\pgfpathmoveto{\pgfqpoint{1.283630in}{0.475000in}}%
\pgfpathlineto{\pgfqpoint{1.415380in}{0.475000in}}%
\pgfpathlineto{\pgfqpoint{1.415380in}{1.356393in}}%
\pgfpathlineto{\pgfqpoint{1.283630in}{1.356393in}}%
\pgfpathclose%
\pgfusepath{fill}%
\end{pgfscope}%
\begin{pgfscope}%
\pgfpathrectangle{\pgfqpoint{0.361380in}{0.475000in}}{\pgfqpoint{1.317500in}{1.309000in}}%
\pgfusepath{clip}%
\pgfsetbuttcap%
\pgfsetmiterjoin%
\definecolor{currentfill}{rgb}{1.000000,0.000000,0.000000}%
\pgfsetfillcolor{currentfill}%
\pgfsetfillopacity{0.500000}%
\pgfsetlinewidth{0.000000pt}%
\definecolor{currentstroke}{rgb}{0.000000,0.000000,0.000000}%
\pgfsetstrokecolor{currentstroke}%
\pgfsetstrokeopacity{0.500000}%
\pgfsetdash{}{0pt}%
\pgfpathmoveto{\pgfqpoint{1.415380in}{0.475000in}}%
\pgfpathlineto{\pgfqpoint{1.547130in}{0.475000in}}%
\pgfpathlineto{\pgfqpoint{1.547130in}{1.502727in}}%
\pgfpathlineto{\pgfqpoint{1.415380in}{1.502727in}}%
\pgfpathclose%
\pgfusepath{fill}%
\end{pgfscope}%
\begin{pgfscope}%
\pgfpathrectangle{\pgfqpoint{0.361380in}{0.475000in}}{\pgfqpoint{1.317500in}{1.309000in}}%
\pgfusepath{clip}%
\pgfsetbuttcap%
\pgfsetmiterjoin%
\definecolor{currentfill}{rgb}{1.000000,0.000000,0.000000}%
\pgfsetfillcolor{currentfill}%
\pgfsetfillopacity{0.500000}%
\pgfsetlinewidth{0.000000pt}%
\definecolor{currentstroke}{rgb}{0.000000,0.000000,0.000000}%
\pgfsetstrokecolor{currentstroke}%
\pgfsetstrokeopacity{0.500000}%
\pgfsetdash{}{0pt}%
\pgfpathmoveto{\pgfqpoint{1.547130in}{0.475000in}}%
\pgfpathlineto{\pgfqpoint{1.678880in}{0.475000in}}%
\pgfpathlineto{\pgfqpoint{1.678880in}{1.761238in}}%
\pgfpathlineto{\pgfqpoint{1.547130in}{1.761238in}}%
\pgfpathclose%
\pgfusepath{fill}%
\end{pgfscope}%
\begin{pgfscope}%
\pgfsetbuttcap%
\pgfsetroundjoin%
\definecolor{currentfill}{rgb}{0.000000,0.000000,0.000000}%
\pgfsetfillcolor{currentfill}%
\pgfsetlinewidth{0.803000pt}%
\definecolor{currentstroke}{rgb}{0.000000,0.000000,0.000000}%
\pgfsetstrokecolor{currentstroke}%
\pgfsetdash{}{0pt}%
\pgfsys@defobject{currentmarker}{\pgfqpoint{0.000000in}{-0.048611in}}{\pgfqpoint{0.000000in}{0.000000in}}{%
\pgfpathmoveto{\pgfqpoint{0.000000in}{0.000000in}}%
\pgfpathlineto{\pgfqpoint{0.000000in}{-0.048611in}}%
\pgfusepath{stroke,fill}%
}%
\begin{pgfscope}%
\pgfsys@transformshift{0.361380in}{0.475000in}%
\pgfsys@useobject{currentmarker}{}%
\end{pgfscope}%
\end{pgfscope}%
\begin{pgfscope}%
\definecolor{textcolor}{rgb}{0.000000,0.000000,0.000000}%
\pgfsetstrokecolor{textcolor}%
\pgfsetfillcolor{textcolor}%
\pgftext[x=0.361380in,y=0.377778in,,top]{\color{textcolor}\rmfamily\fontsize{9.000000}{10.800000}\selectfont \(\displaystyle {0.0}\)}%
\end{pgfscope}%
\begin{pgfscope}%
\pgfsetbuttcap%
\pgfsetroundjoin%
\definecolor{currentfill}{rgb}{0.000000,0.000000,0.000000}%
\pgfsetfillcolor{currentfill}%
\pgfsetlinewidth{0.803000pt}%
\definecolor{currentstroke}{rgb}{0.000000,0.000000,0.000000}%
\pgfsetstrokecolor{currentstroke}%
\pgfsetdash{}{0pt}%
\pgfsys@defobject{currentmarker}{\pgfqpoint{0.000000in}{-0.048611in}}{\pgfqpoint{0.000000in}{0.000000in}}{%
\pgfpathmoveto{\pgfqpoint{0.000000in}{0.000000in}}%
\pgfpathlineto{\pgfqpoint{0.000000in}{-0.048611in}}%
\pgfusepath{stroke,fill}%
}%
\begin{pgfscope}%
\pgfsys@transformshift{1.020130in}{0.475000in}%
\pgfsys@useobject{currentmarker}{}%
\end{pgfscope}%
\end{pgfscope}%
\begin{pgfscope}%
\definecolor{textcolor}{rgb}{0.000000,0.000000,0.000000}%
\pgfsetstrokecolor{textcolor}%
\pgfsetfillcolor{textcolor}%
\pgftext[x=1.020130in,y=0.377778in,,top]{\color{textcolor}\rmfamily\fontsize{9.000000}{10.800000}\selectfont \(\displaystyle {0.5}\)}%
\end{pgfscope}%
\begin{pgfscope}%
\pgfsetbuttcap%
\pgfsetroundjoin%
\definecolor{currentfill}{rgb}{0.000000,0.000000,0.000000}%
\pgfsetfillcolor{currentfill}%
\pgfsetlinewidth{0.803000pt}%
\definecolor{currentstroke}{rgb}{0.000000,0.000000,0.000000}%
\pgfsetstrokecolor{currentstroke}%
\pgfsetdash{}{0pt}%
\pgfsys@defobject{currentmarker}{\pgfqpoint{0.000000in}{-0.048611in}}{\pgfqpoint{0.000000in}{0.000000in}}{%
\pgfpathmoveto{\pgfqpoint{0.000000in}{0.000000in}}%
\pgfpathlineto{\pgfqpoint{0.000000in}{-0.048611in}}%
\pgfusepath{stroke,fill}%
}%
\begin{pgfscope}%
\pgfsys@transformshift{1.678880in}{0.475000in}%
\pgfsys@useobject{currentmarker}{}%
\end{pgfscope}%
\end{pgfscope}%
\begin{pgfscope}%
\definecolor{textcolor}{rgb}{0.000000,0.000000,0.000000}%
\pgfsetstrokecolor{textcolor}%
\pgfsetfillcolor{textcolor}%
\pgftext[x=1.678880in,y=0.377778in,,top]{\color{textcolor}\rmfamily\fontsize{9.000000}{10.800000}\selectfont \(\displaystyle {1.0}\)}%
\end{pgfscope}%
\begin{pgfscope}%
\definecolor{textcolor}{rgb}{0.000000,0.000000,0.000000}%
\pgfsetstrokecolor{textcolor}%
\pgfsetfillcolor{textcolor}%
\pgftext[x=1.020130in,y=0.211111in,,top]{\color{textcolor}\rmfamily\fontsize{9.000000}{10.800000}\selectfont Predicted Confidence}%
\end{pgfscope}%
\begin{pgfscope}%
\pgfsetbuttcap%
\pgfsetroundjoin%
\definecolor{currentfill}{rgb}{0.000000,0.000000,0.000000}%
\pgfsetfillcolor{currentfill}%
\pgfsetlinewidth{0.803000pt}%
\definecolor{currentstroke}{rgb}{0.000000,0.000000,0.000000}%
\pgfsetstrokecolor{currentstroke}%
\pgfsetdash{}{0pt}%
\pgfsys@defobject{currentmarker}{\pgfqpoint{-0.048611in}{0.000000in}}{\pgfqpoint{-0.000000in}{0.000000in}}{%
\pgfpathmoveto{\pgfqpoint{-0.000000in}{0.000000in}}%
\pgfpathlineto{\pgfqpoint{-0.048611in}{0.000000in}}%
\pgfusepath{stroke,fill}%
}%
\begin{pgfscope}%
\pgfsys@transformshift{0.361380in}{0.475000in}%
\pgfsys@useobject{currentmarker}{}%
\end{pgfscope}%
\end{pgfscope}%
\begin{pgfscope}%
\definecolor{textcolor}{rgb}{0.000000,0.000000,0.000000}%
\pgfsetstrokecolor{textcolor}%
\pgfsetfillcolor{textcolor}%
\pgftext[x=0.100000in, y=0.431597in, left, base]{\color{textcolor}\rmfamily\fontsize{9.000000}{10.800000}\selectfont \(\displaystyle {0.0}\)}%
\end{pgfscope}%
\begin{pgfscope}%
\pgfsetbuttcap%
\pgfsetroundjoin%
\definecolor{currentfill}{rgb}{0.000000,0.000000,0.000000}%
\pgfsetfillcolor{currentfill}%
\pgfsetlinewidth{0.803000pt}%
\definecolor{currentstroke}{rgb}{0.000000,0.000000,0.000000}%
\pgfsetstrokecolor{currentstroke}%
\pgfsetdash{}{0pt}%
\pgfsys@defobject{currentmarker}{\pgfqpoint{-0.048611in}{0.000000in}}{\pgfqpoint{-0.000000in}{0.000000in}}{%
\pgfpathmoveto{\pgfqpoint{-0.000000in}{0.000000in}}%
\pgfpathlineto{\pgfqpoint{-0.048611in}{0.000000in}}%
\pgfusepath{stroke,fill}%
}%
\begin{pgfscope}%
\pgfsys@transformshift{0.361380in}{1.129500in}%
\pgfsys@useobject{currentmarker}{}%
\end{pgfscope}%
\end{pgfscope}%
\begin{pgfscope}%
\definecolor{textcolor}{rgb}{0.000000,0.000000,0.000000}%
\pgfsetstrokecolor{textcolor}%
\pgfsetfillcolor{textcolor}%
\pgftext[x=0.100000in, y=1.086097in, left, base]{\color{textcolor}\rmfamily\fontsize{9.000000}{10.800000}\selectfont \(\displaystyle {0.5}\)}%
\end{pgfscope}%
\begin{pgfscope}%
\pgfsetbuttcap%
\pgfsetroundjoin%
\definecolor{currentfill}{rgb}{0.000000,0.000000,0.000000}%
\pgfsetfillcolor{currentfill}%
\pgfsetlinewidth{0.803000pt}%
\definecolor{currentstroke}{rgb}{0.000000,0.000000,0.000000}%
\pgfsetstrokecolor{currentstroke}%
\pgfsetdash{}{0pt}%
\pgfsys@defobject{currentmarker}{\pgfqpoint{-0.048611in}{0.000000in}}{\pgfqpoint{-0.000000in}{0.000000in}}{%
\pgfpathmoveto{\pgfqpoint{-0.000000in}{0.000000in}}%
\pgfpathlineto{\pgfqpoint{-0.048611in}{0.000000in}}%
\pgfusepath{stroke,fill}%
}%
\begin{pgfscope}%
\pgfsys@transformshift{0.361380in}{1.784000in}%
\pgfsys@useobject{currentmarker}{}%
\end{pgfscope}%
\end{pgfscope}%
\begin{pgfscope}%
\definecolor{textcolor}{rgb}{0.000000,0.000000,0.000000}%
\pgfsetstrokecolor{textcolor}%
\pgfsetfillcolor{textcolor}%
\pgftext[x=0.100000in, y=1.740597in, left, base]{\color{textcolor}\rmfamily\fontsize{9.000000}{10.800000}\selectfont \(\displaystyle {1.0}\)}%
\end{pgfscope}%
\begin{pgfscope}%
\pgfpathrectangle{\pgfqpoint{0.361380in}{0.475000in}}{\pgfqpoint{1.317500in}{1.309000in}}%
\pgfusepath{clip}%
\pgfsetbuttcap%
\pgfsetroundjoin%
\pgfsetlinewidth{1.505625pt}%
\definecolor{currentstroke}{rgb}{0.000000,0.000000,0.000000}%
\pgfsetstrokecolor{currentstroke}%
\pgfsetdash{{5.550000pt}{2.400000pt}}{0.000000pt}%
\pgfpathmoveto{\pgfqpoint{0.361380in}{0.475000in}}%
\pgfpathlineto{\pgfqpoint{0.374555in}{0.488090in}}%
\pgfpathlineto{\pgfqpoint{0.387730in}{0.501180in}}%
\pgfpathlineto{\pgfqpoint{0.400905in}{0.514270in}}%
\pgfpathlineto{\pgfqpoint{0.414080in}{0.527360in}}%
\pgfpathlineto{\pgfqpoint{0.427255in}{0.540450in}}%
\pgfpathlineto{\pgfqpoint{0.440430in}{0.553540in}}%
\pgfpathlineto{\pgfqpoint{0.453605in}{0.566630in}}%
\pgfpathlineto{\pgfqpoint{0.466780in}{0.579720in}}%
\pgfpathlineto{\pgfqpoint{0.479955in}{0.592810in}}%
\pgfpathlineto{\pgfqpoint{0.493130in}{0.605900in}}%
\pgfpathlineto{\pgfqpoint{0.506305in}{0.618990in}}%
\pgfpathlineto{\pgfqpoint{0.519480in}{0.632080in}}%
\pgfpathlineto{\pgfqpoint{0.532655in}{0.645170in}}%
\pgfpathlineto{\pgfqpoint{0.545830in}{0.658260in}}%
\pgfpathlineto{\pgfqpoint{0.559005in}{0.671350in}}%
\pgfpathlineto{\pgfqpoint{0.572180in}{0.684440in}}%
\pgfpathlineto{\pgfqpoint{0.585355in}{0.697530in}}%
\pgfpathlineto{\pgfqpoint{0.598530in}{0.710620in}}%
\pgfpathlineto{\pgfqpoint{0.611705in}{0.723710in}}%
\pgfpathlineto{\pgfqpoint{0.624880in}{0.736800in}}%
\pgfpathlineto{\pgfqpoint{0.638055in}{0.749890in}}%
\pgfpathlineto{\pgfqpoint{0.651230in}{0.762980in}}%
\pgfpathlineto{\pgfqpoint{0.664405in}{0.776070in}}%
\pgfpathlineto{\pgfqpoint{0.677580in}{0.789160in}}%
\pgfpathlineto{\pgfqpoint{0.690755in}{0.802250in}}%
\pgfpathlineto{\pgfqpoint{0.703930in}{0.815340in}}%
\pgfpathlineto{\pgfqpoint{0.717105in}{0.828430in}}%
\pgfpathlineto{\pgfqpoint{0.730280in}{0.841520in}}%
\pgfpathlineto{\pgfqpoint{0.743455in}{0.854610in}}%
\pgfpathlineto{\pgfqpoint{0.756630in}{0.867700in}}%
\pgfpathlineto{\pgfqpoint{0.769805in}{0.880790in}}%
\pgfpathlineto{\pgfqpoint{0.782980in}{0.893880in}}%
\pgfpathlineto{\pgfqpoint{0.796155in}{0.906970in}}%
\pgfpathlineto{\pgfqpoint{0.809330in}{0.920060in}}%
\pgfpathlineto{\pgfqpoint{0.822505in}{0.933150in}}%
\pgfpathlineto{\pgfqpoint{0.835680in}{0.946240in}}%
\pgfpathlineto{\pgfqpoint{0.848855in}{0.959330in}}%
\pgfpathlineto{\pgfqpoint{0.862030in}{0.972420in}}%
\pgfpathlineto{\pgfqpoint{0.875205in}{0.985510in}}%
\pgfpathlineto{\pgfqpoint{0.888380in}{0.998600in}}%
\pgfpathlineto{\pgfqpoint{0.901555in}{1.011690in}}%
\pgfpathlineto{\pgfqpoint{0.914730in}{1.024780in}}%
\pgfpathlineto{\pgfqpoint{0.927905in}{1.037870in}}%
\pgfpathlineto{\pgfqpoint{0.941080in}{1.050960in}}%
\pgfpathlineto{\pgfqpoint{0.954255in}{1.064050in}}%
\pgfpathlineto{\pgfqpoint{0.967430in}{1.077140in}}%
\pgfpathlineto{\pgfqpoint{0.980605in}{1.090230in}}%
\pgfpathlineto{\pgfqpoint{0.993780in}{1.103320in}}%
\pgfpathlineto{\pgfqpoint{1.006955in}{1.116410in}}%
\pgfpathlineto{\pgfqpoint{1.020130in}{1.129500in}}%
\pgfpathlineto{\pgfqpoint{1.033305in}{1.142590in}}%
\pgfpathlineto{\pgfqpoint{1.046480in}{1.155680in}}%
\pgfpathlineto{\pgfqpoint{1.059655in}{1.168770in}}%
\pgfpathlineto{\pgfqpoint{1.072830in}{1.181860in}}%
\pgfpathlineto{\pgfqpoint{1.086005in}{1.194950in}}%
\pgfpathlineto{\pgfqpoint{1.099180in}{1.208040in}}%
\pgfpathlineto{\pgfqpoint{1.112355in}{1.221130in}}%
\pgfpathlineto{\pgfqpoint{1.125530in}{1.234220in}}%
\pgfpathlineto{\pgfqpoint{1.138705in}{1.247310in}}%
\pgfpathlineto{\pgfqpoint{1.151880in}{1.260400in}}%
\pgfpathlineto{\pgfqpoint{1.165055in}{1.273490in}}%
\pgfpathlineto{\pgfqpoint{1.178230in}{1.286580in}}%
\pgfpathlineto{\pgfqpoint{1.191405in}{1.299670in}}%
\pgfpathlineto{\pgfqpoint{1.204580in}{1.312760in}}%
\pgfpathlineto{\pgfqpoint{1.217755in}{1.325850in}}%
\pgfpathlineto{\pgfqpoint{1.230930in}{1.338940in}}%
\pgfpathlineto{\pgfqpoint{1.244105in}{1.352030in}}%
\pgfpathlineto{\pgfqpoint{1.257280in}{1.365120in}}%
\pgfpathlineto{\pgfqpoint{1.270455in}{1.378210in}}%
\pgfpathlineto{\pgfqpoint{1.283630in}{1.391300in}}%
\pgfpathlineto{\pgfqpoint{1.296805in}{1.404390in}}%
\pgfpathlineto{\pgfqpoint{1.309980in}{1.417480in}}%
\pgfpathlineto{\pgfqpoint{1.323155in}{1.430570in}}%
\pgfpathlineto{\pgfqpoint{1.336330in}{1.443660in}}%
\pgfpathlineto{\pgfqpoint{1.349505in}{1.456750in}}%
\pgfpathlineto{\pgfqpoint{1.362680in}{1.469840in}}%
\pgfpathlineto{\pgfqpoint{1.375855in}{1.482930in}}%
\pgfpathlineto{\pgfqpoint{1.389030in}{1.496020in}}%
\pgfpathlineto{\pgfqpoint{1.402205in}{1.509110in}}%
\pgfpathlineto{\pgfqpoint{1.415380in}{1.522200in}}%
\pgfpathlineto{\pgfqpoint{1.428555in}{1.535290in}}%
\pgfpathlineto{\pgfqpoint{1.441730in}{1.548380in}}%
\pgfpathlineto{\pgfqpoint{1.454905in}{1.561470in}}%
\pgfpathlineto{\pgfqpoint{1.468080in}{1.574560in}}%
\pgfpathlineto{\pgfqpoint{1.481255in}{1.587650in}}%
\pgfpathlineto{\pgfqpoint{1.494430in}{1.600740in}}%
\pgfpathlineto{\pgfqpoint{1.507605in}{1.613830in}}%
\pgfpathlineto{\pgfqpoint{1.520780in}{1.626920in}}%
\pgfpathlineto{\pgfqpoint{1.533955in}{1.640010in}}%
\pgfpathlineto{\pgfqpoint{1.547130in}{1.653100in}}%
\pgfpathlineto{\pgfqpoint{1.560305in}{1.666190in}}%
\pgfpathlineto{\pgfqpoint{1.573480in}{1.679280in}}%
\pgfpathlineto{\pgfqpoint{1.586655in}{1.692370in}}%
\pgfpathlineto{\pgfqpoint{1.599830in}{1.705460in}}%
\pgfpathlineto{\pgfqpoint{1.613005in}{1.718550in}}%
\pgfpathlineto{\pgfqpoint{1.626180in}{1.731640in}}%
\pgfpathlineto{\pgfqpoint{1.639355in}{1.744730in}}%
\pgfpathlineto{\pgfqpoint{1.652530in}{1.757820in}}%
\pgfpathlineto{\pgfqpoint{1.665705in}{1.770910in}}%
\pgfusepath{stroke}%
\end{pgfscope}%
\begin{pgfscope}%
\pgfsetrectcap%
\pgfsetmiterjoin%
\pgfsetlinewidth{0.803000pt}%
\definecolor{currentstroke}{rgb}{0.000000,0.000000,0.000000}%
\pgfsetstrokecolor{currentstroke}%
\pgfsetdash{}{0pt}%
\pgfpathmoveto{\pgfqpoint{0.361380in}{0.475000in}}%
\pgfpathlineto{\pgfqpoint{0.361380in}{1.784000in}}%
\pgfusepath{stroke}%
\end{pgfscope}%
\begin{pgfscope}%
\pgfsetrectcap%
\pgfsetmiterjoin%
\pgfsetlinewidth{0.803000pt}%
\definecolor{currentstroke}{rgb}{0.000000,0.000000,0.000000}%
\pgfsetstrokecolor{currentstroke}%
\pgfsetdash{}{0pt}%
\pgfpathmoveto{\pgfqpoint{1.678880in}{0.475000in}}%
\pgfpathlineto{\pgfqpoint{1.678880in}{1.784000in}}%
\pgfusepath{stroke}%
\end{pgfscope}%
\begin{pgfscope}%
\pgfsetrectcap%
\pgfsetmiterjoin%
\pgfsetlinewidth{0.803000pt}%
\definecolor{currentstroke}{rgb}{0.000000,0.000000,0.000000}%
\pgfsetstrokecolor{currentstroke}%
\pgfsetdash{}{0pt}%
\pgfpathmoveto{\pgfqpoint{0.361380in}{0.475000in}}%
\pgfpathlineto{\pgfqpoint{1.678880in}{0.475000in}}%
\pgfusepath{stroke}%
\end{pgfscope}%
\begin{pgfscope}%
\pgfsetrectcap%
\pgfsetmiterjoin%
\pgfsetlinewidth{0.803000pt}%
\definecolor{currentstroke}{rgb}{0.000000,0.000000,0.000000}%
\pgfsetstrokecolor{currentstroke}%
\pgfsetdash{}{0pt}%
\pgfpathmoveto{\pgfqpoint{0.361380in}{1.784000in}}%
\pgfpathlineto{\pgfqpoint{1.678880in}{1.784000in}}%
\pgfusepath{stroke}%
\end{pgfscope}%
\begin{pgfscope}%
\pgfsetbuttcap%
\pgfsetmiterjoin%
\definecolor{currentfill}{rgb}{1.000000,1.000000,1.000000}%
\pgfsetfillcolor{currentfill}%
\pgfsetfillopacity{0.500000}%
\pgfsetlinewidth{1.003750pt}%
\definecolor{currentstroke}{rgb}{0.800000,0.800000,0.800000}%
\pgfsetstrokecolor{currentstroke}%
\pgfsetstrokeopacity{0.500000}%
\pgfsetdash{}{0pt}%
\pgfpathmoveto{\pgfqpoint{0.448880in}{1.335389in}}%
\pgfpathlineto{\pgfqpoint{1.152215in}{1.335389in}}%
\pgfpathquadraticcurveto{\pgfqpoint{1.177215in}{1.335389in}}{\pgfqpoint{1.177215in}{1.360389in}}%
\pgfpathlineto{\pgfqpoint{1.177215in}{1.696500in}}%
\pgfpathquadraticcurveto{\pgfqpoint{1.177215in}{1.721500in}}{\pgfqpoint{1.152215in}{1.721500in}}%
\pgfpathlineto{\pgfqpoint{0.448880in}{1.721500in}}%
\pgfpathquadraticcurveto{\pgfqpoint{0.423880in}{1.721500in}}{\pgfqpoint{0.423880in}{1.696500in}}%
\pgfpathlineto{\pgfqpoint{0.423880in}{1.360389in}}%
\pgfpathquadraticcurveto{\pgfqpoint{0.423880in}{1.335389in}}{\pgfqpoint{0.448880in}{1.335389in}}%
\pgfpathclose%
\pgfusepath{stroke,fill}%
\end{pgfscope}%
\begin{pgfscope}%
\pgfsetbuttcap%
\pgfsetmiterjoin%
\definecolor{currentfill}{rgb}{0.000000,0.000000,1.000000}%
\pgfsetfillcolor{currentfill}%
\pgfsetfillopacity{0.500000}%
\pgfsetlinewidth{0.000000pt}%
\definecolor{currentstroke}{rgb}{0.000000,0.000000,0.000000}%
\pgfsetstrokecolor{currentstroke}%
\pgfsetstrokeopacity{0.500000}%
\pgfsetdash{}{0pt}%
\pgfpathmoveto{\pgfqpoint{0.473880in}{1.584000in}}%
\pgfpathlineto{\pgfqpoint{0.723880in}{1.584000in}}%
\pgfpathlineto{\pgfqpoint{0.723880in}{1.671500in}}%
\pgfpathlineto{\pgfqpoint{0.473880in}{1.671500in}}%
\pgfpathclose%
\pgfusepath{fill}%
\end{pgfscope}%
\begin{pgfscope}%
\definecolor{textcolor}{rgb}{0.000000,0.000000,0.000000}%
\pgfsetstrokecolor{textcolor}%
\pgfsetfillcolor{textcolor}%
\pgftext[x=0.823880in,y=1.584000in,left,base]{\color{textcolor}\rmfamily\fontsize{9.000000}{10.800000}\selectfont Conf.}%
\end{pgfscope}%
\begin{pgfscope}%
\pgfsetbuttcap%
\pgfsetmiterjoin%
\definecolor{currentfill}{rgb}{1.000000,0.000000,0.000000}%
\pgfsetfillcolor{currentfill}%
\pgfsetfillopacity{0.500000}%
\pgfsetlinewidth{0.000000pt}%
\definecolor{currentstroke}{rgb}{0.000000,0.000000,0.000000}%
\pgfsetstrokecolor{currentstroke}%
\pgfsetstrokeopacity{0.500000}%
\pgfsetdash{}{0pt}%
\pgfpathmoveto{\pgfqpoint{0.473880in}{1.409694in}}%
\pgfpathlineto{\pgfqpoint{0.723880in}{1.409694in}}%
\pgfpathlineto{\pgfqpoint{0.723880in}{1.497194in}}%
\pgfpathlineto{\pgfqpoint{0.473880in}{1.497194in}}%
\pgfpathclose%
\pgfusepath{fill}%
\end{pgfscope}%
\begin{pgfscope}%
\definecolor{textcolor}{rgb}{0.000000,0.000000,0.000000}%
\pgfsetstrokecolor{textcolor}%
\pgfsetfillcolor{textcolor}%
\pgftext[x=0.823880in,y=1.409694in,left,base]{\color{textcolor}\rmfamily\fontsize{9.000000}{10.800000}\selectfont Acc.}%
\end{pgfscope}%
\end{pgfpicture}%
\makeatother%
\endgroup%

%% file: images/diagramhistogram.pgf
\begingroup%
\makeatletter%
\begin{pgfpicture}%
\pgfpathrectangle{\pgfpointorigin}{\pgfqpoint{1.860959in}{1.927403in}}%
\pgfusepath{use as bounding box, clip}%
\begin{pgfscope}%
\pgfsetbuttcap%
\pgfsetmiterjoin%
\definecolor{currentfill}{rgb}{1.000000,1.000000,1.000000}%
\pgfsetfillcolor{currentfill}%
\pgfsetlinewidth{0.000000pt}%
\definecolor{currentstroke}{rgb}{1.000000,1.000000,1.000000}%
\pgfsetstrokecolor{currentstroke}%
\pgfsetdash{}{0pt}%
\pgfpathmoveto{\pgfqpoint{0.000000in}{0.000000in}}%
\pgfpathlineto{\pgfqpoint{1.860959in}{0.000000in}}%
\pgfpathlineto{\pgfqpoint{1.860959in}{1.927403in}}%
\pgfpathlineto{\pgfqpoint{0.000000in}{1.927403in}}%
\pgfpathclose%
\pgfusepath{fill}%
\end{pgfscope}%
\begin{pgfscope}%
\pgfsetbuttcap%
\pgfsetmiterjoin%
\definecolor{currentfill}{rgb}{1.000000,1.000000,1.000000}%
\pgfsetfillcolor{currentfill}%
\pgfsetlinewidth{0.000000pt}%
\definecolor{currentstroke}{rgb}{0.000000,0.000000,0.000000}%
\pgfsetstrokecolor{currentstroke}%
\pgfsetstrokeopacity{0.000000}%
\pgfsetdash{}{0pt}%
\pgfpathmoveto{\pgfqpoint{0.361380in}{0.475000in}}%
\pgfpathlineto{\pgfqpoint{1.678880in}{0.475000in}}%
\pgfpathlineto{\pgfqpoint{1.678880in}{1.784000in}}%
\pgfpathlineto{\pgfqpoint{0.361380in}{1.784000in}}%
\pgfpathclose%
\pgfusepath{fill}%
\end{pgfscope}%
\begin{pgfscope}%
\pgfpathrectangle{\pgfqpoint{0.361380in}{0.475000in}}{\pgfqpoint{1.317500in}{1.309000in}}%
\pgfusepath{clip}%
\pgfsetbuttcap%
\pgfsetmiterjoin%
\definecolor{currentfill}{rgb}{0.000000,0.000000,1.000000}%
\pgfsetfillcolor{currentfill}%
\pgfsetfillopacity{0.500000}%
\pgfsetlinewidth{0.000000pt}%
\definecolor{currentstroke}{rgb}{0.000000,0.000000,0.000000}%
\pgfsetstrokecolor{currentstroke}%
\pgfsetstrokeopacity{0.500000}%
\pgfsetdash{}{0pt}%
\pgfpathmoveto{\pgfqpoint{1.020130in}{0.475000in}}%
\pgfpathlineto{\pgfqpoint{1.151880in}{0.475000in}}%
\pgfpathlineto{\pgfqpoint{1.151880in}{1.197881in}}%
\pgfpathlineto{\pgfqpoint{1.020130in}{1.197881in}}%
\pgfpathclose%
\pgfusepath{fill}%
\end{pgfscope}%
\begin{pgfscope}%
\pgfpathrectangle{\pgfqpoint{0.361380in}{0.475000in}}{\pgfqpoint{1.317500in}{1.309000in}}%
\pgfusepath{clip}%
\pgfsetbuttcap%
\pgfsetmiterjoin%
\definecolor{currentfill}{rgb}{0.000000,0.000000,1.000000}%
\pgfsetfillcolor{currentfill}%
\pgfsetfillopacity{0.500000}%
\pgfsetlinewidth{0.000000pt}%
\definecolor{currentstroke}{rgb}{0.000000,0.000000,0.000000}%
\pgfsetstrokecolor{currentstroke}%
\pgfsetstrokeopacity{0.500000}%
\pgfsetdash{}{0pt}%
\pgfpathmoveto{\pgfqpoint{1.415380in}{0.475000in}}%
\pgfpathlineto{\pgfqpoint{1.547130in}{0.475000in}}%
\pgfpathlineto{\pgfqpoint{1.547130in}{1.577111in}}%
\pgfpathlineto{\pgfqpoint{1.415380in}{1.577111in}}%
\pgfpathclose%
\pgfusepath{fill}%
\end{pgfscope}%
\begin{pgfscope}%
\pgfpathrectangle{\pgfqpoint{0.361380in}{0.475000in}}{\pgfqpoint{1.317500in}{1.309000in}}%
\pgfusepath{clip}%
\pgfsetbuttcap%
\pgfsetmiterjoin%
\definecolor{currentfill}{rgb}{0.000000,0.000000,1.000000}%
\pgfsetfillcolor{currentfill}%
\pgfsetfillopacity{0.500000}%
\pgfsetlinewidth{0.000000pt}%
\definecolor{currentstroke}{rgb}{0.000000,0.000000,0.000000}%
\pgfsetstrokecolor{currentstroke}%
\pgfsetstrokeopacity{0.500000}%
\pgfsetdash{}{0pt}%
\pgfpathmoveto{\pgfqpoint{1.547130in}{0.475000in}}%
\pgfpathlineto{\pgfqpoint{1.678880in}{0.475000in}}%
\pgfpathlineto{\pgfqpoint{1.678880in}{1.772377in}}%
\pgfpathlineto{\pgfqpoint{1.547130in}{1.772377in}}%
\pgfpathclose%
\pgfusepath{fill}%
\end{pgfscope}%
\begin{pgfscope}%
\pgfpathrectangle{\pgfqpoint{0.361380in}{0.475000in}}{\pgfqpoint{1.317500in}{1.309000in}}%
\pgfusepath{clip}%
\pgfsetbuttcap%
\pgfsetmiterjoin%
\definecolor{currentfill}{rgb}{1.000000,0.000000,0.000000}%
\pgfsetfillcolor{currentfill}%
\pgfsetfillopacity{0.500000}%
\pgfsetlinewidth{0.000000pt}%
\definecolor{currentstroke}{rgb}{0.000000,0.000000,0.000000}%
\pgfsetstrokecolor{currentstroke}%
\pgfsetstrokeopacity{0.500000}%
\pgfsetdash{}{0pt}%
\pgfpathmoveto{\pgfqpoint{1.020130in}{0.475000in}}%
\pgfpathlineto{\pgfqpoint{1.151880in}{0.475000in}}%
\pgfpathlineto{\pgfqpoint{1.151880in}{1.176847in}}%
\pgfpathlineto{\pgfqpoint{1.020130in}{1.176847in}}%
\pgfpathclose%
\pgfusepath{fill}%
\end{pgfscope}%
\begin{pgfscope}%
\pgfpathrectangle{\pgfqpoint{0.361380in}{0.475000in}}{\pgfqpoint{1.317500in}{1.309000in}}%
\pgfusepath{clip}%
\pgfsetbuttcap%
\pgfsetmiterjoin%
\definecolor{currentfill}{rgb}{1.000000,0.000000,0.000000}%
\pgfsetfillcolor{currentfill}%
\pgfsetfillopacity{0.500000}%
\pgfsetlinewidth{0.000000pt}%
\definecolor{currentstroke}{rgb}{0.000000,0.000000,0.000000}%
\pgfsetstrokecolor{currentstroke}%
\pgfsetstrokeopacity{0.500000}%
\pgfsetdash{}{0pt}%
\pgfpathmoveto{\pgfqpoint{1.415380in}{0.475000in}}%
\pgfpathlineto{\pgfqpoint{1.547130in}{0.475000in}}%
\pgfpathlineto{\pgfqpoint{1.547130in}{1.548661in}}%
\pgfpathlineto{\pgfqpoint{1.415380in}{1.548661in}}%
\pgfpathclose%
\pgfusepath{fill}%
\end{pgfscope}%
\begin{pgfscope}%
\pgfpathrectangle{\pgfqpoint{0.361380in}{0.475000in}}{\pgfqpoint{1.317500in}{1.309000in}}%
\pgfusepath{clip}%
\pgfsetbuttcap%
\pgfsetmiterjoin%
\definecolor{currentfill}{rgb}{1.000000,0.000000,0.000000}%
\pgfsetfillcolor{currentfill}%
\pgfsetfillopacity{0.500000}%
\pgfsetlinewidth{0.000000pt}%
\definecolor{currentstroke}{rgb}{0.000000,0.000000,0.000000}%
\pgfsetstrokecolor{currentstroke}%
\pgfsetstrokeopacity{0.500000}%
\pgfsetdash{}{0pt}%
\pgfpathmoveto{\pgfqpoint{1.547130in}{0.475000in}}%
\pgfpathlineto{\pgfqpoint{1.678880in}{0.475000in}}%
\pgfpathlineto{\pgfqpoint{1.678880in}{1.769710in}}%
\pgfpathlineto{\pgfqpoint{1.547130in}{1.769710in}}%
\pgfpathclose%
\pgfusepath{fill}%
\end{pgfscope}%
\begin{pgfscope}%
\pgfsetbuttcap%
\pgfsetroundjoin%
\definecolor{currentfill}{rgb}{0.000000,0.000000,0.000000}%
\pgfsetfillcolor{currentfill}%
\pgfsetlinewidth{0.803000pt}%
\definecolor{currentstroke}{rgb}{0.000000,0.000000,0.000000}%
\pgfsetstrokecolor{currentstroke}%
\pgfsetdash{}{0pt}%
\pgfsys@defobject{currentmarker}{\pgfqpoint{0.000000in}{-0.048611in}}{\pgfqpoint{0.000000in}{0.000000in}}{%
\pgfpathmoveto{\pgfqpoint{0.000000in}{0.000000in}}%
\pgfpathlineto{\pgfqpoint{0.000000in}{-0.048611in}}%
\pgfusepath{stroke,fill}%
}%
\begin{pgfscope}%
\pgfsys@transformshift{0.361380in}{0.475000in}%
\pgfsys@useobject{currentmarker}{}%
\end{pgfscope}%
\end{pgfscope}%
\begin{pgfscope}%
\definecolor{textcolor}{rgb}{0.000000,0.000000,0.000000}%
\pgfsetstrokecolor{textcolor}%
\pgfsetfillcolor{textcolor}%
\pgftext[x=0.361380in,y=0.377778in,,top]{\color{textcolor}\rmfamily\fontsize{9.000000}{10.800000}\selectfont \(\displaystyle {0.0}\)}%
\end{pgfscope}%
\begin{pgfscope}%
\pgfsetbuttcap%
\pgfsetroundjoin%
\definecolor{currentfill}{rgb}{0.000000,0.000000,0.000000}%
\pgfsetfillcolor{currentfill}%
\pgfsetlinewidth{0.803000pt}%
\definecolor{currentstroke}{rgb}{0.000000,0.000000,0.000000}%
\pgfsetstrokecolor{currentstroke}%
\pgfsetdash{}{0pt}%
\pgfsys@defobject{currentmarker}{\pgfqpoint{0.000000in}{-0.048611in}}{\pgfqpoint{0.000000in}{0.000000in}}{%
\pgfpathmoveto{\pgfqpoint{0.000000in}{0.000000in}}%
\pgfpathlineto{\pgfqpoint{0.000000in}{-0.048611in}}%
\pgfusepath{stroke,fill}%
}%
\begin{pgfscope}%
\pgfsys@transformshift{1.020130in}{0.475000in}%
\pgfsys@useobject{currentmarker}{}%
\end{pgfscope}%
\end{pgfscope}%
\begin{pgfscope}%
\definecolor{textcolor}{rgb}{0.000000,0.000000,0.000000}%
\pgfsetstrokecolor{textcolor}%
\pgfsetfillcolor{textcolor}%
\pgftext[x=1.020130in,y=0.377778in,,top]{\color{textcolor}\rmfamily\fontsize{9.000000}{10.800000}\selectfont \(\displaystyle {0.5}\)}%
\end{pgfscope}%
\begin{pgfscope}%
\pgfsetbuttcap%
\pgfsetroundjoin%
\definecolor{currentfill}{rgb}{0.000000,0.000000,0.000000}%
\pgfsetfillcolor{currentfill}%
\pgfsetlinewidth{0.803000pt}%
\definecolor{currentstroke}{rgb}{0.000000,0.000000,0.000000}%
\pgfsetstrokecolor{currentstroke}%
\pgfsetdash{}{0pt}%
\pgfsys@defobject{currentmarker}{\pgfqpoint{0.000000in}{-0.048611in}}{\pgfqpoint{0.000000in}{0.000000in}}{%
\pgfpathmoveto{\pgfqpoint{0.000000in}{0.000000in}}%
\pgfpathlineto{\pgfqpoint{0.000000in}{-0.048611in}}%
\pgfusepath{stroke,fill}%
}%
\begin{pgfscope}%
\pgfsys@transformshift{1.678880in}{0.475000in}%
\pgfsys@useobject{currentmarker}{}%
\end{pgfscope}%
\end{pgfscope}%
\begin{pgfscope}%
\definecolor{textcolor}{rgb}{0.000000,0.000000,0.000000}%
\pgfsetstrokecolor{textcolor}%
\pgfsetfillcolor{textcolor}%
\pgftext[x=1.678880in,y=0.377778in,,top]{\color{textcolor}\rmfamily\fontsize{9.000000}{10.800000}\selectfont \(\displaystyle {1.0}\)}%
\end{pgfscope}%
\begin{pgfscope}%
\definecolor{textcolor}{rgb}{0.000000,0.000000,0.000000}%
\pgfsetstrokecolor{textcolor}%
\pgfsetfillcolor{textcolor}%
\pgftext[x=1.020130in,y=0.211111in,,top]{\color{textcolor}\rmfamily\fontsize{9.000000}{10.800000}\selectfont Predicted Confidence}%
\end{pgfscope}%
\begin{pgfscope}%
\pgfsetbuttcap%
\pgfsetroundjoin%
\definecolor{currentfill}{rgb}{0.000000,0.000000,0.000000}%
\pgfsetfillcolor{currentfill}%
\pgfsetlinewidth{0.803000pt}%
\definecolor{currentstroke}{rgb}{0.000000,0.000000,0.000000}%
\pgfsetstrokecolor{currentstroke}%
\pgfsetdash{}{0pt}%
\pgfsys@defobject{currentmarker}{\pgfqpoint{-0.048611in}{0.000000in}}{\pgfqpoint{-0.000000in}{0.000000in}}{%
\pgfpathmoveto{\pgfqpoint{-0.000000in}{0.000000in}}%
\pgfpathlineto{\pgfqpoint{-0.048611in}{0.000000in}}%
\pgfusepath{stroke,fill}%
}%
\begin{pgfscope}%
\pgfsys@transformshift{0.361380in}{0.475000in}%
\pgfsys@useobject{currentmarker}{}%
\end{pgfscope}%
\end{pgfscope}%
\begin{pgfscope}%
\definecolor{textcolor}{rgb}{0.000000,0.000000,0.000000}%
\pgfsetstrokecolor{textcolor}%
\pgfsetfillcolor{textcolor}%
\pgftext[x=0.100000in, y=0.431597in, left, base]{\color{textcolor}\rmfamily\fontsize{9.000000}{10.800000}\selectfont \(\displaystyle {0.0}\)}%
\end{pgfscope}%
\begin{pgfscope}%
\pgfsetbuttcap%
\pgfsetroundjoin%
\definecolor{currentfill}{rgb}{0.000000,0.000000,0.000000}%
\pgfsetfillcolor{currentfill}%
\pgfsetlinewidth{0.803000pt}%
\definecolor{currentstroke}{rgb}{0.000000,0.000000,0.000000}%
\pgfsetstrokecolor{currentstroke}%
\pgfsetdash{}{0pt}%
\pgfsys@defobject{currentmarker}{\pgfqpoint{-0.048611in}{0.000000in}}{\pgfqpoint{-0.000000in}{0.000000in}}{%
\pgfpathmoveto{\pgfqpoint{-0.000000in}{0.000000in}}%
\pgfpathlineto{\pgfqpoint{-0.048611in}{0.000000in}}%
\pgfusepath{stroke,fill}%
}%
\begin{pgfscope}%
\pgfsys@transformshift{0.361380in}{1.129500in}%
\pgfsys@useobject{currentmarker}{}%
\end{pgfscope}%
\end{pgfscope}%
\begin{pgfscope}%
\definecolor{textcolor}{rgb}{0.000000,0.000000,0.000000}%
\pgfsetstrokecolor{textcolor}%
\pgfsetfillcolor{textcolor}%
\pgftext[x=0.100000in, y=1.086097in, left, base]{\color{textcolor}\rmfamily\fontsize{9.000000}{10.800000}\selectfont \(\displaystyle {0.5}\)}%
\end{pgfscope}%
\begin{pgfscope}%
\pgfsetbuttcap%
\pgfsetroundjoin%
\definecolor{currentfill}{rgb}{0.000000,0.000000,0.000000}%
\pgfsetfillcolor{currentfill}%
\pgfsetlinewidth{0.803000pt}%
\definecolor{currentstroke}{rgb}{0.000000,0.000000,0.000000}%
\pgfsetstrokecolor{currentstroke}%
\pgfsetdash{}{0pt}%
\pgfsys@defobject{currentmarker}{\pgfqpoint{-0.048611in}{0.000000in}}{\pgfqpoint{-0.000000in}{0.000000in}}{%
\pgfpathmoveto{\pgfqpoint{-0.000000in}{0.000000in}}%
\pgfpathlineto{\pgfqpoint{-0.048611in}{0.000000in}}%
\pgfusepath{stroke,fill}%
}%
\begin{pgfscope}%
\pgfsys@transformshift{0.361380in}{1.784000in}%
\pgfsys@useobject{currentmarker}{}%
\end{pgfscope}%
\end{pgfscope}%
\begin{pgfscope}%
\definecolor{textcolor}{rgb}{0.000000,0.000000,0.000000}%
\pgfsetstrokecolor{textcolor}%
\pgfsetfillcolor{textcolor}%
\pgftext[x=0.100000in, y=1.740597in, left, base]{\color{textcolor}\rmfamily\fontsize{9.000000}{10.800000}\selectfont \(\displaystyle {1.0}\)}%
\end{pgfscope}%
\begin{pgfscope}%
\pgfpathrectangle{\pgfqpoint{0.361380in}{0.475000in}}{\pgfqpoint{1.317500in}{1.309000in}}%
\pgfusepath{clip}%
\pgfsetbuttcap%
\pgfsetroundjoin%
\pgfsetlinewidth{1.505625pt}%
\definecolor{currentstroke}{rgb}{0.000000,0.000000,0.000000}%
\pgfsetstrokecolor{currentstroke}%
\pgfsetdash{{5.550000pt}{2.400000pt}}{0.000000pt}%
\pgfpathmoveto{\pgfqpoint{0.361380in}{0.475000in}}%
\pgfpathlineto{\pgfqpoint{0.374555in}{0.488090in}}%
\pgfpathlineto{\pgfqpoint{0.387730in}{0.501180in}}%
\pgfpathlineto{\pgfqpoint{0.400905in}{0.514270in}}%
\pgfpathlineto{\pgfqpoint{0.414080in}{0.527360in}}%
\pgfpathlineto{\pgfqpoint{0.427255in}{0.540450in}}%
\pgfpathlineto{\pgfqpoint{0.440430in}{0.553540in}}%
\pgfpathlineto{\pgfqpoint{0.453605in}{0.566630in}}%
\pgfpathlineto{\pgfqpoint{0.466780in}{0.579720in}}%
\pgfpathlineto{\pgfqpoint{0.479955in}{0.592810in}}%
\pgfpathlineto{\pgfqpoint{0.493130in}{0.605900in}}%
\pgfpathlineto{\pgfqpoint{0.506305in}{0.618990in}}%
\pgfpathlineto{\pgfqpoint{0.519480in}{0.632080in}}%
\pgfpathlineto{\pgfqpoint{0.532655in}{0.645170in}}%
\pgfpathlineto{\pgfqpoint{0.545830in}{0.658260in}}%
\pgfpathlineto{\pgfqpoint{0.559005in}{0.671350in}}%
\pgfpathlineto{\pgfqpoint{0.572180in}{0.684440in}}%
\pgfpathlineto{\pgfqpoint{0.585355in}{0.697530in}}%
\pgfpathlineto{\pgfqpoint{0.598530in}{0.710620in}}%
\pgfpathlineto{\pgfqpoint{0.611705in}{0.723710in}}%
\pgfpathlineto{\pgfqpoint{0.624880in}{0.736800in}}%
\pgfpathlineto{\pgfqpoint{0.638055in}{0.749890in}}%
\pgfpathlineto{\pgfqpoint{0.651230in}{0.762980in}}%
\pgfpathlineto{\pgfqpoint{0.664405in}{0.776070in}}%
\pgfpathlineto{\pgfqpoint{0.677580in}{0.789160in}}%
\pgfpathlineto{\pgfqpoint{0.690755in}{0.802250in}}%
\pgfpathlineto{\pgfqpoint{0.703930in}{0.815340in}}%
\pgfpathlineto{\pgfqpoint{0.717105in}{0.828430in}}%
\pgfpathlineto{\pgfqpoint{0.730280in}{0.841520in}}%
\pgfpathlineto{\pgfqpoint{0.743455in}{0.854610in}}%
\pgfpathlineto{\pgfqpoint{0.756630in}{0.867700in}}%
\pgfpathlineto{\pgfqpoint{0.769805in}{0.880790in}}%
\pgfpathlineto{\pgfqpoint{0.782980in}{0.893880in}}%
\pgfpathlineto{\pgfqpoint{0.796155in}{0.906970in}}%
\pgfpathlineto{\pgfqpoint{0.809330in}{0.920060in}}%
\pgfpathlineto{\pgfqpoint{0.822505in}{0.933150in}}%
\pgfpathlineto{\pgfqpoint{0.835680in}{0.946240in}}%
\pgfpathlineto{\pgfqpoint{0.848855in}{0.959330in}}%
\pgfpathlineto{\pgfqpoint{0.862030in}{0.972420in}}%
\pgfpathlineto{\pgfqpoint{0.875205in}{0.985510in}}%
\pgfpathlineto{\pgfqpoint{0.888380in}{0.998600in}}%
\pgfpathlineto{\pgfqpoint{0.901555in}{1.011690in}}%
\pgfpathlineto{\pgfqpoint{0.914730in}{1.024780in}}%
\pgfpathlineto{\pgfqpoint{0.927905in}{1.037870in}}%
\pgfpathlineto{\pgfqpoint{0.941080in}{1.050960in}}%
\pgfpathlineto{\pgfqpoint{0.954255in}{1.064050in}}%
\pgfpathlineto{\pgfqpoint{0.967430in}{1.077140in}}%
\pgfpathlineto{\pgfqpoint{0.980605in}{1.090230in}}%
\pgfpathlineto{\pgfqpoint{0.993780in}{1.103320in}}%
\pgfpathlineto{\pgfqpoint{1.006955in}{1.116410in}}%
\pgfpathlineto{\pgfqpoint{1.020130in}{1.129500in}}%
\pgfpathlineto{\pgfqpoint{1.033305in}{1.142590in}}%
\pgfpathlineto{\pgfqpoint{1.046480in}{1.155680in}}%
\pgfpathlineto{\pgfqpoint{1.059655in}{1.168770in}}%
\pgfpathlineto{\pgfqpoint{1.072830in}{1.181860in}}%
\pgfpathlineto{\pgfqpoint{1.086005in}{1.194950in}}%
\pgfpathlineto{\pgfqpoint{1.099180in}{1.208040in}}%
\pgfpathlineto{\pgfqpoint{1.112355in}{1.221130in}}%
\pgfpathlineto{\pgfqpoint{1.125530in}{1.234220in}}%
\pgfpathlineto{\pgfqpoint{1.138705in}{1.247310in}}%
\pgfpathlineto{\pgfqpoint{1.151880in}{1.260400in}}%
\pgfpathlineto{\pgfqpoint{1.165055in}{1.273490in}}%
\pgfpathlineto{\pgfqpoint{1.178230in}{1.286580in}}%
\pgfpathlineto{\pgfqpoint{1.191405in}{1.299670in}}%
\pgfpathlineto{\pgfqpoint{1.204580in}{1.312760in}}%
\pgfpathlineto{\pgfqpoint{1.217755in}{1.325850in}}%
\pgfpathlineto{\pgfqpoint{1.230930in}{1.338940in}}%
\pgfpathlineto{\pgfqpoint{1.244105in}{1.352030in}}%
\pgfpathlineto{\pgfqpoint{1.257280in}{1.365120in}}%
\pgfpathlineto{\pgfqpoint{1.270455in}{1.378210in}}%
\pgfpathlineto{\pgfqpoint{1.283630in}{1.391300in}}%
\pgfpathlineto{\pgfqpoint{1.296805in}{1.404390in}}%
\pgfpathlineto{\pgfqpoint{1.309980in}{1.417480in}}%
\pgfpathlineto{\pgfqpoint{1.323155in}{1.430570in}}%
\pgfpathlineto{\pgfqpoint{1.336330in}{1.443660in}}%
\pgfpathlineto{\pgfqpoint{1.349505in}{1.456750in}}%
\pgfpathlineto{\pgfqpoint{1.362680in}{1.469840in}}%
\pgfpathlineto{\pgfqpoint{1.375855in}{1.482930in}}%
\pgfpathlineto{\pgfqpoint{1.389030in}{1.496020in}}%
\pgfpathlineto{\pgfqpoint{1.402205in}{1.509110in}}%
\pgfpathlineto{\pgfqpoint{1.415380in}{1.522200in}}%
\pgfpathlineto{\pgfqpoint{1.428555in}{1.535290in}}%
\pgfpathlineto{\pgfqpoint{1.441730in}{1.548380in}}%
\pgfpathlineto{\pgfqpoint{1.454905in}{1.561470in}}%
\pgfpathlineto{\pgfqpoint{1.468080in}{1.574560in}}%
\pgfpathlineto{\pgfqpoint{1.481255in}{1.587650in}}%
\pgfpathlineto{\pgfqpoint{1.494430in}{1.600740in}}%
\pgfpathlineto{\pgfqpoint{1.507605in}{1.613830in}}%
\pgfpathlineto{\pgfqpoint{1.520780in}{1.626920in}}%
\pgfpathlineto{\pgfqpoint{1.533955in}{1.640010in}}%
\pgfpathlineto{\pgfqpoint{1.547130in}{1.653100in}}%
\pgfpathlineto{\pgfqpoint{1.560305in}{1.666190in}}%
\pgfpathlineto{\pgfqpoint{1.573480in}{1.679280in}}%
\pgfpathlineto{\pgfqpoint{1.586655in}{1.692370in}}%
\pgfpathlineto{\pgfqpoint{1.599830in}{1.705460in}}%
\pgfpathlineto{\pgfqpoint{1.613005in}{1.718550in}}%
\pgfpathlineto{\pgfqpoint{1.626180in}{1.731640in}}%
\pgfpathlineto{\pgfqpoint{1.639355in}{1.744730in}}%
\pgfpathlineto{\pgfqpoint{1.652530in}{1.757820in}}%
\pgfpathlineto{\pgfqpoint{1.665705in}{1.770910in}}%
\pgfusepath{stroke}%
\end{pgfscope}%
\begin{pgfscope}%
\pgfsetrectcap%
\pgfsetmiterjoin%
\pgfsetlinewidth{0.803000pt}%
\definecolor{currentstroke}{rgb}{0.000000,0.000000,0.000000}%
\pgfsetstrokecolor{currentstroke}%
\pgfsetdash{}{0pt}%
\pgfpathmoveto{\pgfqpoint{0.361380in}{0.475000in}}%
\pgfpathlineto{\pgfqpoint{0.361380in}{1.784000in}}%
\pgfusepath{stroke}%
\end{pgfscope}%
\begin{pgfscope}%
\pgfsetrectcap%
\pgfsetmiterjoin%
\pgfsetlinewidth{0.803000pt}%
\definecolor{currentstroke}{rgb}{0.000000,0.000000,0.000000}%
\pgfsetstrokecolor{currentstroke}%
\pgfsetdash{}{0pt}%
\pgfpathmoveto{\pgfqpoint{1.678880in}{0.475000in}}%
\pgfpathlineto{\pgfqpoint{1.678880in}{1.784000in}}%
\pgfusepath{stroke}%
\end{pgfscope}%
\begin{pgfscope}%
\pgfsetrectcap%
\pgfsetmiterjoin%
\pgfsetlinewidth{0.803000pt}%
\definecolor{currentstroke}{rgb}{0.000000,0.000000,0.000000}%
\pgfsetstrokecolor{currentstroke}%
\pgfsetdash{}{0pt}%
\pgfpathmoveto{\pgfqpoint{0.361380in}{0.475000in}}%
\pgfpathlineto{\pgfqpoint{1.678880in}{0.475000in}}%
\pgfusepath{stroke}%
\end{pgfscope}%
\begin{pgfscope}%
\pgfsetrectcap%
\pgfsetmiterjoin%
\pgfsetlinewidth{0.803000pt}%
\definecolor{currentstroke}{rgb}{0.000000,0.000000,0.000000}%
\pgfsetstrokecolor{currentstroke}%
\pgfsetdash{}{0pt}%
\pgfpathmoveto{\pgfqpoint{0.361380in}{1.784000in}}%
\pgfpathlineto{\pgfqpoint{1.678880in}{1.784000in}}%
\pgfusepath{stroke}%
\end{pgfscope}%
\begin{pgfscope}%
\pgfsetbuttcap%
\pgfsetmiterjoin%
\definecolor{currentfill}{rgb}{1.000000,1.000000,1.000000}%
\pgfsetfillcolor{currentfill}%
\pgfsetfillopacity{0.500000}%
\pgfsetlinewidth{1.003750pt}%
\definecolor{currentstroke}{rgb}{0.800000,0.800000,0.800000}%
\pgfsetstrokecolor{currentstroke}%
\pgfsetstrokeopacity{0.500000}%
\pgfsetdash{}{0pt}%
\pgfpathmoveto{\pgfqpoint{0.448880in}{1.335389in}}%
\pgfpathlineto{\pgfqpoint{1.152215in}{1.335389in}}%
\pgfpathquadraticcurveto{\pgfqpoint{1.177215in}{1.335389in}}{\pgfqpoint{1.177215in}{1.360389in}}%
\pgfpathlineto{\pgfqpoint{1.177215in}{1.696500in}}%
\pgfpathquadraticcurveto{\pgfqpoint{1.177215in}{1.721500in}}{\pgfqpoint{1.152215in}{1.721500in}}%
\pgfpathlineto{\pgfqpoint{0.448880in}{1.721500in}}%
\pgfpathquadraticcurveto{\pgfqpoint{0.423880in}{1.721500in}}{\pgfqpoint{0.423880in}{1.696500in}}%
\pgfpathlineto{\pgfqpoint{0.423880in}{1.360389in}}%
\pgfpathquadraticcurveto{\pgfqpoint{0.423880in}{1.335389in}}{\pgfqpoint{0.448880in}{1.335389in}}%
\pgfpathclose%
\pgfusepath{stroke,fill}%
\end{pgfscope}%
\begin{pgfscope}%
\pgfsetbuttcap%
\pgfsetmiterjoin%
\definecolor{currentfill}{rgb}{0.000000,0.000000,1.000000}%
\pgfsetfillcolor{currentfill}%
\pgfsetfillopacity{0.500000}%
\pgfsetlinewidth{0.000000pt}%
\definecolor{currentstroke}{rgb}{0.000000,0.000000,0.000000}%
\pgfsetstrokecolor{currentstroke}%
\pgfsetstrokeopacity{0.500000}%
\pgfsetdash{}{0pt}%
\pgfpathmoveto{\pgfqpoint{0.473880in}{1.584000in}}%
\pgfpathlineto{\pgfqpoint{0.723880in}{1.584000in}}%
\pgfpathlineto{\pgfqpoint{0.723880in}{1.671500in}}%
\pgfpathlineto{\pgfqpoint{0.473880in}{1.671500in}}%
\pgfpathclose%
\pgfusepath{fill}%
\end{pgfscope}%
\begin{pgfscope}%
\definecolor{textcolor}{rgb}{0.000000,0.000000,0.000000}%
\pgfsetstrokecolor{textcolor}%
\pgfsetfillcolor{textcolor}%
\pgftext[x=0.823880in,y=1.584000in,left,base]{\color{textcolor}\rmfamily\fontsize{9.000000}{10.800000}\selectfont Conf.}%
\end{pgfscope}%
\begin{pgfscope}%
\pgfsetbuttcap%
\pgfsetmiterjoin%
\definecolor{currentfill}{rgb}{1.000000,0.000000,0.000000}%
\pgfsetfillcolor{currentfill}%
\pgfsetfillopacity{0.500000}%
\pgfsetlinewidth{0.000000pt}%
\definecolor{currentstroke}{rgb}{0.000000,0.000000,0.000000}%
\pgfsetstrokecolor{currentstroke}%
\pgfsetstrokeopacity{0.500000}%
\pgfsetdash{}{0pt}%
\pgfpathmoveto{\pgfqpoint{0.473880in}{1.409694in}}%
\pgfpathlineto{\pgfqpoint{0.723880in}{1.409694in}}%
\pgfpathlineto{\pgfqpoint{0.723880in}{1.497194in}}%
\pgfpathlineto{\pgfqpoint{0.473880in}{1.497194in}}%
\pgfpathclose%
\pgfusepath{fill}%
\end{pgfscope}%
\begin{pgfscope}%
\definecolor{textcolor}{rgb}{0.000000,0.000000,0.000000}%
\pgfsetstrokecolor{textcolor}%
\pgfsetfillcolor{textcolor}%
\pgftext[x=0.823880in,y=1.409694in,left,base]{\color{textcolor}\rmfamily\fontsize{9.000000}{10.800000}\selectfont Acc.}%
\end{pgfscope}%
\end{pgfpicture}%
\makeatother%
\endgroup%

%% file: images/plugin_vs_splitting100.pgf
\begingroup%
\makeatletter%
\begin{pgfpicture}%
\pgfpathrectangle{\pgfpointorigin}{\pgfqpoint{2.009783in}{1.906472in}}%
\pgfusepath{use as bounding box, clip}%
\begin{pgfscope}%
\pgfsetbuttcap%
\pgfsetmiterjoin%
\definecolor{currentfill}{rgb}{1.000000,1.000000,1.000000}%
\pgfsetfillcolor{currentfill}%
\pgfsetlinewidth{0.000000pt}%
\definecolor{currentstroke}{rgb}{1.000000,1.000000,1.000000}%
\pgfsetstrokecolor{currentstroke}%
\pgfsetdash{}{0pt}%
\pgfpathmoveto{\pgfqpoint{0.000000in}{0.000000in}}%
\pgfpathlineto{\pgfqpoint{2.009783in}{0.000000in}}%
\pgfpathlineto{\pgfqpoint{2.009783in}{1.906472in}}%
\pgfpathlineto{\pgfqpoint{0.000000in}{1.906472in}}%
\pgfpathclose%
\pgfusepath{fill}%
\end{pgfscope}%
\begin{pgfscope}%
\pgfsetbuttcap%
\pgfsetmiterjoin%
\definecolor{currentfill}{rgb}{1.000000,1.000000,1.000000}%
\pgfsetfillcolor{currentfill}%
\pgfsetlinewidth{0.000000pt}%
\definecolor{currentstroke}{rgb}{0.000000,0.000000,0.000000}%
\pgfsetstrokecolor{currentstroke}%
\pgfsetstrokeopacity{0.000000}%
\pgfsetdash{}{0pt}%
\pgfpathmoveto{\pgfqpoint{0.592283in}{0.475000in}}%
\pgfpathlineto{\pgfqpoint{1.909783in}{0.475000in}}%
\pgfpathlineto{\pgfqpoint{1.909783in}{1.784000in}}%
\pgfpathlineto{\pgfqpoint{0.592283in}{1.784000in}}%
\pgfpathclose%
\pgfusepath{fill}%
\end{pgfscope}%
\begin{pgfscope}%
\pgfsetbuttcap%
\pgfsetroundjoin%
\definecolor{currentfill}{rgb}{0.000000,0.000000,0.000000}%
\pgfsetfillcolor{currentfill}%
\pgfsetlinewidth{0.803000pt}%
\definecolor{currentstroke}{rgb}{0.000000,0.000000,0.000000}%
\pgfsetstrokecolor{currentstroke}%
\pgfsetdash{}{0pt}%
\pgfsys@defobject{currentmarker}{\pgfqpoint{0.000000in}{-0.048611in}}{\pgfqpoint{0.000000in}{0.000000in}}{%
\pgfpathmoveto{\pgfqpoint{0.000000in}{0.000000in}}%
\pgfpathlineto{\pgfqpoint{0.000000in}{-0.048611in}}%
\pgfusepath{stroke,fill}%
}%
\begin{pgfscope}%
\pgfsys@transformshift{0.880104in}{0.475000in}%
\pgfsys@useobject{currentmarker}{}%
\end{pgfscope}%
\end{pgfscope}%
\begin{pgfscope}%
\definecolor{textcolor}{rgb}{0.000000,0.000000,0.000000}%
\pgfsetstrokecolor{textcolor}%
\pgfsetfillcolor{textcolor}%
\pgftext[x=0.880104in,y=0.377778in,,top]{\color{textcolor}\rmfamily\fontsize{9.000000}{10.800000}\selectfont \(\displaystyle {0.05}\)}%
\end{pgfscope}%
\begin{pgfscope}%
\pgfsetbuttcap%
\pgfsetroundjoin%
\definecolor{currentfill}{rgb}{0.000000,0.000000,0.000000}%
\pgfsetfillcolor{currentfill}%
\pgfsetlinewidth{0.803000pt}%
\definecolor{currentstroke}{rgb}{0.000000,0.000000,0.000000}%
\pgfsetstrokecolor{currentstroke}%
\pgfsetdash{}{0pt}%
\pgfsys@defobject{currentmarker}{\pgfqpoint{0.000000in}{-0.048611in}}{\pgfqpoint{0.000000in}{0.000000in}}{%
\pgfpathmoveto{\pgfqpoint{0.000000in}{0.000000in}}%
\pgfpathlineto{\pgfqpoint{0.000000in}{-0.048611in}}%
\pgfusepath{stroke,fill}%
}%
\begin{pgfscope}%
\pgfsys@transformshift{1.728574in}{0.475000in}%
\pgfsys@useobject{currentmarker}{}%
\end{pgfscope}%
\end{pgfscope}%
\begin{pgfscope}%
\definecolor{textcolor}{rgb}{0.000000,0.000000,0.000000}%
\pgfsetstrokecolor{textcolor}%
\pgfsetfillcolor{textcolor}%
\pgftext[x=1.728574in,y=0.377778in,,top]{\color{textcolor}\rmfamily\fontsize{9.000000}{10.800000}\selectfont \(\displaystyle {0.10}\)}%
\end{pgfscope}%
\begin{pgfscope}%
\definecolor{textcolor}{rgb}{0.000000,0.000000,0.000000}%
\pgfsetstrokecolor{textcolor}%
\pgfsetfillcolor{textcolor}%
\pgftext[x=1.251032in,y=0.211111in,,top]{\color{textcolor}\rmfamily\fontsize{9.000000}{10.800000}\selectfont \(\displaystyle \ell_2\)-ECE}%
\end{pgfscope}%
\begin{pgfscope}%
\pgfsetbuttcap%
\pgfsetroundjoin%
\definecolor{currentfill}{rgb}{0.000000,0.000000,0.000000}%
\pgfsetfillcolor{currentfill}%
\pgfsetlinewidth{0.803000pt}%
\definecolor{currentstroke}{rgb}{0.000000,0.000000,0.000000}%
\pgfsetstrokecolor{currentstroke}%
\pgfsetdash{}{0pt}%
\pgfsys@defobject{currentmarker}{\pgfqpoint{-0.048611in}{0.000000in}}{\pgfqpoint{-0.000000in}{0.000000in}}{%
\pgfpathmoveto{\pgfqpoint{-0.000000in}{0.000000in}}%
\pgfpathlineto{\pgfqpoint{-0.048611in}{0.000000in}}%
\pgfusepath{stroke,fill}%
}%
\begin{pgfscope}%
\pgfsys@transformshift{0.592283in}{0.534746in}%
\pgfsys@useobject{currentmarker}{}%
\end{pgfscope}%
\end{pgfscope}%
\begin{pgfscope}%
\definecolor{textcolor}{rgb}{0.000000,0.000000,0.000000}%
\pgfsetstrokecolor{textcolor}%
\pgfsetfillcolor{textcolor}%
\pgftext[x=0.266667in, y=0.491343in, left, base]{\color{textcolor}\rmfamily\fontsize{9.000000}{10.800000}\selectfont \(\displaystyle {0.00}\)}%
\end{pgfscope}%
\begin{pgfscope}%
\pgfsetbuttcap%
\pgfsetroundjoin%
\definecolor{currentfill}{rgb}{0.000000,0.000000,0.000000}%
\pgfsetfillcolor{currentfill}%
\pgfsetlinewidth{0.803000pt}%
\definecolor{currentstroke}{rgb}{0.000000,0.000000,0.000000}%
\pgfsetstrokecolor{currentstroke}%
\pgfsetdash{}{0pt}%
\pgfsys@defobject{currentmarker}{\pgfqpoint{-0.048611in}{0.000000in}}{\pgfqpoint{-0.000000in}{0.000000in}}{%
\pgfpathmoveto{\pgfqpoint{-0.000000in}{0.000000in}}%
\pgfpathlineto{\pgfqpoint{-0.048611in}{0.000000in}}%
\pgfusepath{stroke,fill}%
}%
\begin{pgfscope}%
\pgfsys@transformshift{0.592283in}{0.841827in}%
\pgfsys@useobject{currentmarker}{}%
\end{pgfscope}%
\end{pgfscope}%
\begin{pgfscope}%
\definecolor{textcolor}{rgb}{0.000000,0.000000,0.000000}%
\pgfsetstrokecolor{textcolor}%
\pgfsetfillcolor{textcolor}%
\pgftext[x=0.266667in, y=0.798424in, left, base]{\color{textcolor}\rmfamily\fontsize{9.000000}{10.800000}\selectfont \(\displaystyle {0.25}\)}%
\end{pgfscope}%
\begin{pgfscope}%
\pgfsetbuttcap%
\pgfsetroundjoin%
\definecolor{currentfill}{rgb}{0.000000,0.000000,0.000000}%
\pgfsetfillcolor{currentfill}%
\pgfsetlinewidth{0.803000pt}%
\definecolor{currentstroke}{rgb}{0.000000,0.000000,0.000000}%
\pgfsetstrokecolor{currentstroke}%
\pgfsetdash{}{0pt}%
\pgfsys@defobject{currentmarker}{\pgfqpoint{-0.048611in}{0.000000in}}{\pgfqpoint{-0.000000in}{0.000000in}}{%
\pgfpathmoveto{\pgfqpoint{-0.000000in}{0.000000in}}%
\pgfpathlineto{\pgfqpoint{-0.048611in}{0.000000in}}%
\pgfusepath{stroke,fill}%
}%
\begin{pgfscope}%
\pgfsys@transformshift{0.592283in}{1.148908in}%
\pgfsys@useobject{currentmarker}{}%
\end{pgfscope}%
\end{pgfscope}%
\begin{pgfscope}%
\definecolor{textcolor}{rgb}{0.000000,0.000000,0.000000}%
\pgfsetstrokecolor{textcolor}%
\pgfsetfillcolor{textcolor}%
\pgftext[x=0.266667in, y=1.105505in, left, base]{\color{textcolor}\rmfamily\fontsize{9.000000}{10.800000}\selectfont \(\displaystyle {0.50}\)}%
\end{pgfscope}%
\begin{pgfscope}%
\pgfsetbuttcap%
\pgfsetroundjoin%
\definecolor{currentfill}{rgb}{0.000000,0.000000,0.000000}%
\pgfsetfillcolor{currentfill}%
\pgfsetlinewidth{0.803000pt}%
\definecolor{currentstroke}{rgb}{0.000000,0.000000,0.000000}%
\pgfsetstrokecolor{currentstroke}%
\pgfsetdash{}{0pt}%
\pgfsys@defobject{currentmarker}{\pgfqpoint{-0.048611in}{0.000000in}}{\pgfqpoint{-0.000000in}{0.000000in}}{%
\pgfpathmoveto{\pgfqpoint{-0.000000in}{0.000000in}}%
\pgfpathlineto{\pgfqpoint{-0.048611in}{0.000000in}}%
\pgfusepath{stroke,fill}%
}%
\begin{pgfscope}%
\pgfsys@transformshift{0.592283in}{1.455988in}%
\pgfsys@useobject{currentmarker}{}%
\end{pgfscope}%
\end{pgfscope}%
\begin{pgfscope}%
\definecolor{textcolor}{rgb}{0.000000,0.000000,0.000000}%
\pgfsetstrokecolor{textcolor}%
\pgfsetfillcolor{textcolor}%
\pgftext[x=0.266667in, y=1.412586in, left, base]{\color{textcolor}\rmfamily\fontsize{9.000000}{10.800000}\selectfont \(\displaystyle {0.75}\)}%
\end{pgfscope}%
\begin{pgfscope}%
\pgfsetbuttcap%
\pgfsetroundjoin%
\definecolor{currentfill}{rgb}{0.000000,0.000000,0.000000}%
\pgfsetfillcolor{currentfill}%
\pgfsetlinewidth{0.803000pt}%
\definecolor{currentstroke}{rgb}{0.000000,0.000000,0.000000}%
\pgfsetstrokecolor{currentstroke}%
\pgfsetdash{}{0pt}%
\pgfsys@defobject{currentmarker}{\pgfqpoint{-0.048611in}{0.000000in}}{\pgfqpoint{-0.000000in}{0.000000in}}{%
\pgfpathmoveto{\pgfqpoint{-0.000000in}{0.000000in}}%
\pgfpathlineto{\pgfqpoint{-0.048611in}{0.000000in}}%
\pgfusepath{stroke,fill}%
}%
\begin{pgfscope}%
\pgfsys@transformshift{0.592283in}{1.763069in}%
\pgfsys@useobject{currentmarker}{}%
\end{pgfscope}%
\end{pgfscope}%
\begin{pgfscope}%
\definecolor{textcolor}{rgb}{0.000000,0.000000,0.000000}%
\pgfsetstrokecolor{textcolor}%
\pgfsetfillcolor{textcolor}%
\pgftext[x=0.266667in, y=1.719667in, left, base]{\color{textcolor}\rmfamily\fontsize{9.000000}{10.800000}\selectfont \(\displaystyle {1.00}\)}%
\end{pgfscope}%
\begin{pgfscope}%
\definecolor{textcolor}{rgb}{0.000000,0.000000,0.000000}%
\pgfsetstrokecolor{textcolor}%
\pgfsetfillcolor{textcolor}%
\pgftext[x=0.211111in,y=1.129500in,,bottom,rotate=90.000000]{\color{textcolor}\rmfamily\fontsize{9.000000}{10.800000}\selectfont Type II error}%
\end{pgfscope}%
\begin{pgfscope}%
\pgfpathrectangle{\pgfqpoint{0.592283in}{0.475000in}}{\pgfqpoint{1.317500in}{1.309000in}}%
\pgfusepath{clip}%
\pgfsetbuttcap%
\pgfsetroundjoin%
\pgfsetlinewidth{0.501875pt}%
\definecolor{currentstroke}{rgb}{0.000000,0.000000,1.000000}%
\pgfsetstrokecolor{currentstroke}%
\pgfsetdash{}{0pt}%
\pgfpathmoveto{\pgfqpoint{1.849896in}{0.534746in}}%
\pgfpathlineto{\pgfqpoint{1.849896in}{0.534746in}}%
\pgfusepath{stroke}%
\end{pgfscope}%
\begin{pgfscope}%
\pgfpathrectangle{\pgfqpoint{0.592283in}{0.475000in}}{\pgfqpoint{1.317500in}{1.309000in}}%
\pgfusepath{clip}%
\pgfsetbuttcap%
\pgfsetroundjoin%
\pgfsetlinewidth{0.501875pt}%
\definecolor{currentstroke}{rgb}{0.000000,0.000000,1.000000}%
\pgfsetstrokecolor{currentstroke}%
\pgfsetdash{}{0pt}%
\pgfpathmoveto{\pgfqpoint{1.622049in}{0.534746in}}%
\pgfpathlineto{\pgfqpoint{1.622049in}{0.534746in}}%
\pgfusepath{stroke}%
\end{pgfscope}%
\begin{pgfscope}%
\pgfpathrectangle{\pgfqpoint{0.592283in}{0.475000in}}{\pgfqpoint{1.317500in}{1.309000in}}%
\pgfusepath{clip}%
\pgfsetbuttcap%
\pgfsetroundjoin%
\pgfsetlinewidth{0.501875pt}%
\definecolor{currentstroke}{rgb}{0.000000,0.000000,1.000000}%
\pgfsetstrokecolor{currentstroke}%
\pgfsetdash{}{0pt}%
\pgfpathmoveto{\pgfqpoint{1.457248in}{0.534746in}}%
\pgfpathlineto{\pgfqpoint{1.457248in}{0.534746in}}%
\pgfusepath{stroke}%
\end{pgfscope}%
\begin{pgfscope}%
\pgfpathrectangle{\pgfqpoint{0.592283in}{0.475000in}}{\pgfqpoint{1.317500in}{1.309000in}}%
\pgfusepath{clip}%
\pgfsetbuttcap%
\pgfsetroundjoin%
\pgfsetlinewidth{0.501875pt}%
\definecolor{currentstroke}{rgb}{0.000000,0.000000,1.000000}%
\pgfsetstrokecolor{currentstroke}%
\pgfsetdash{}{0pt}%
\pgfpathmoveto{\pgfqpoint{1.331306in}{0.534746in}}%
\pgfpathlineto{\pgfqpoint{1.331306in}{0.534746in}}%
\pgfusepath{stroke}%
\end{pgfscope}%
\begin{pgfscope}%
\pgfpathrectangle{\pgfqpoint{0.592283in}{0.475000in}}{\pgfqpoint{1.317500in}{1.309000in}}%
\pgfusepath{clip}%
\pgfsetbuttcap%
\pgfsetroundjoin%
\pgfsetlinewidth{0.501875pt}%
\definecolor{currentstroke}{rgb}{0.000000,0.000000,1.000000}%
\pgfsetstrokecolor{currentstroke}%
\pgfsetdash{}{0pt}%
\pgfpathmoveto{\pgfqpoint{1.231240in}{0.534746in}}%
\pgfpathlineto{\pgfqpoint{1.231240in}{0.534746in}}%
\pgfusepath{stroke}%
\end{pgfscope}%
\begin{pgfscope}%
\pgfpathrectangle{\pgfqpoint{0.592283in}{0.475000in}}{\pgfqpoint{1.317500in}{1.309000in}}%
\pgfusepath{clip}%
\pgfsetbuttcap%
\pgfsetroundjoin%
\pgfsetlinewidth{0.501875pt}%
\definecolor{currentstroke}{rgb}{0.000000,0.000000,1.000000}%
\pgfsetstrokecolor{currentstroke}%
\pgfsetdash{}{0pt}%
\pgfpathmoveto{\pgfqpoint{1.149390in}{0.534500in}}%
\pgfpathlineto{\pgfqpoint{1.149390in}{0.535483in}}%
\pgfusepath{stroke}%
\end{pgfscope}%
\begin{pgfscope}%
\pgfpathrectangle{\pgfqpoint{0.592283in}{0.475000in}}{\pgfqpoint{1.317500in}{1.309000in}}%
\pgfusepath{clip}%
\pgfsetbuttcap%
\pgfsetroundjoin%
\pgfsetlinewidth{0.501875pt}%
\definecolor{currentstroke}{rgb}{0.000000,0.000000,1.000000}%
\pgfsetstrokecolor{currentstroke}%
\pgfsetdash{}{0pt}%
\pgfpathmoveto{\pgfqpoint{1.080917in}{0.537154in}}%
\pgfpathlineto{\pgfqpoint{1.080917in}{0.543393in}}%
\pgfusepath{stroke}%
\end{pgfscope}%
\begin{pgfscope}%
\pgfpathrectangle{\pgfqpoint{0.592283in}{0.475000in}}{\pgfqpoint{1.317500in}{1.309000in}}%
\pgfusepath{clip}%
\pgfsetbuttcap%
\pgfsetroundjoin%
\pgfsetlinewidth{0.501875pt}%
\definecolor{currentstroke}{rgb}{0.000000,0.000000,1.000000}%
\pgfsetstrokecolor{currentstroke}%
\pgfsetdash{}{0pt}%
\pgfpathmoveto{\pgfqpoint{1.022596in}{0.553828in}}%
\pgfpathlineto{\pgfqpoint{1.022596in}{0.565288in}}%
\pgfusepath{stroke}%
\end{pgfscope}%
\begin{pgfscope}%
\pgfpathrectangle{\pgfqpoint{0.592283in}{0.475000in}}{\pgfqpoint{1.317500in}{1.309000in}}%
\pgfusepath{clip}%
\pgfsetbuttcap%
\pgfsetroundjoin%
\pgfsetlinewidth{0.501875pt}%
\definecolor{currentstroke}{rgb}{0.000000,0.000000,1.000000}%
\pgfsetstrokecolor{currentstroke}%
\pgfsetdash{}{0pt}%
\pgfpathmoveto{\pgfqpoint{0.972188in}{0.610299in}}%
\pgfpathlineto{\pgfqpoint{0.972188in}{0.626735in}}%
\pgfusepath{stroke}%
\end{pgfscope}%
\begin{pgfscope}%
\pgfpathrectangle{\pgfqpoint{0.592283in}{0.475000in}}{\pgfqpoint{1.317500in}{1.309000in}}%
\pgfusepath{clip}%
\pgfsetbuttcap%
\pgfsetroundjoin%
\pgfsetlinewidth{0.501875pt}%
\definecolor{currentstroke}{rgb}{0.000000,0.000000,1.000000}%
\pgfsetstrokecolor{currentstroke}%
\pgfsetdash{}{0pt}%
\pgfpathmoveto{\pgfqpoint{0.928085in}{0.695780in}}%
\pgfpathlineto{\pgfqpoint{0.928085in}{0.716660in}}%
\pgfusepath{stroke}%
\end{pgfscope}%
\begin{pgfscope}%
\pgfpathrectangle{\pgfqpoint{0.592283in}{0.475000in}}{\pgfqpoint{1.317500in}{1.309000in}}%
\pgfusepath{clip}%
\pgfsetbuttcap%
\pgfsetroundjoin%
\pgfsetlinewidth{0.501875pt}%
\definecolor{currentstroke}{rgb}{0.000000,0.000000,1.000000}%
\pgfsetstrokecolor{currentstroke}%
\pgfsetdash{}{0pt}%
\pgfpathmoveto{\pgfqpoint{0.889098in}{0.796383in}}%
\pgfpathlineto{\pgfqpoint{0.889098in}{0.822906in}}%
\pgfusepath{stroke}%
\end{pgfscope}%
\begin{pgfscope}%
\pgfpathrectangle{\pgfqpoint{0.592283in}{0.475000in}}{\pgfqpoint{1.317500in}{1.309000in}}%
\pgfusepath{clip}%
\pgfsetbuttcap%
\pgfsetroundjoin%
\pgfsetlinewidth{0.501875pt}%
\definecolor{currentstroke}{rgb}{0.000000,0.000000,1.000000}%
\pgfsetstrokecolor{currentstroke}%
\pgfsetdash{}{0pt}%
\pgfpathmoveto{\pgfqpoint{0.854327in}{0.942424in}}%
\pgfpathlineto{\pgfqpoint{0.854327in}{0.977313in}}%
\pgfusepath{stroke}%
\end{pgfscope}%
\begin{pgfscope}%
\pgfpathrectangle{\pgfqpoint{0.592283in}{0.475000in}}{\pgfqpoint{1.317500in}{1.309000in}}%
\pgfusepath{clip}%
\pgfsetbuttcap%
\pgfsetroundjoin%
\pgfsetlinewidth{0.501875pt}%
\definecolor{currentstroke}{rgb}{0.000000,0.000000,1.000000}%
\pgfsetstrokecolor{currentstroke}%
\pgfsetdash{}{0pt}%
\pgfpathmoveto{\pgfqpoint{0.823079in}{1.071619in}}%
\pgfpathlineto{\pgfqpoint{0.823079in}{1.114173in}}%
\pgfusepath{stroke}%
\end{pgfscope}%
\begin{pgfscope}%
\pgfpathrectangle{\pgfqpoint{0.592283in}{0.475000in}}{\pgfqpoint{1.317500in}{1.309000in}}%
\pgfusepath{clip}%
\pgfsetbuttcap%
\pgfsetroundjoin%
\pgfsetlinewidth{0.501875pt}%
\definecolor{currentstroke}{rgb}{0.000000,0.000000,1.000000}%
\pgfsetstrokecolor{currentstroke}%
\pgfsetdash{}{0pt}%
\pgfpathmoveto{\pgfqpoint{0.794808in}{1.218817in}}%
\pgfpathlineto{\pgfqpoint{0.794808in}{1.244822in}}%
\pgfusepath{stroke}%
\end{pgfscope}%
\begin{pgfscope}%
\pgfpathrectangle{\pgfqpoint{0.592283in}{0.475000in}}{\pgfqpoint{1.317500in}{1.309000in}}%
\pgfusepath{clip}%
\pgfsetbuttcap%
\pgfsetroundjoin%
\pgfsetlinewidth{0.501875pt}%
\definecolor{currentstroke}{rgb}{0.000000,0.000000,1.000000}%
\pgfsetstrokecolor{currentstroke}%
\pgfsetdash{}{0pt}%
\pgfpathmoveto{\pgfqpoint{0.769078in}{1.325620in}}%
\pgfpathlineto{\pgfqpoint{0.769078in}{1.352976in}}%
\pgfusepath{stroke}%
\end{pgfscope}%
\begin{pgfscope}%
\pgfpathrectangle{\pgfqpoint{0.592283in}{0.475000in}}{\pgfqpoint{1.317500in}{1.309000in}}%
\pgfusepath{clip}%
\pgfsetbuttcap%
\pgfsetroundjoin%
\pgfsetlinewidth{0.501875pt}%
\definecolor{currentstroke}{rgb}{0.000000,0.000000,1.000000}%
\pgfsetstrokecolor{currentstroke}%
\pgfsetdash{}{0pt}%
\pgfpathmoveto{\pgfqpoint{0.745539in}{1.403810in}}%
\pgfpathlineto{\pgfqpoint{0.745539in}{1.432257in}}%
\pgfusepath{stroke}%
\end{pgfscope}%
\begin{pgfscope}%
\pgfpathrectangle{\pgfqpoint{0.592283in}{0.475000in}}{\pgfqpoint{1.317500in}{1.309000in}}%
\pgfusepath{clip}%
\pgfsetbuttcap%
\pgfsetroundjoin%
\pgfsetlinewidth{0.501875pt}%
\definecolor{currentstroke}{rgb}{0.000000,0.000000,1.000000}%
\pgfsetstrokecolor{currentstroke}%
\pgfsetdash{}{0pt}%
\pgfpathmoveto{\pgfqpoint{0.723903in}{1.481063in}}%
\pgfpathlineto{\pgfqpoint{0.723903in}{1.512966in}}%
\pgfusepath{stroke}%
\end{pgfscope}%
\begin{pgfscope}%
\pgfpathrectangle{\pgfqpoint{0.592283in}{0.475000in}}{\pgfqpoint{1.317500in}{1.309000in}}%
\pgfusepath{clip}%
\pgfsetbuttcap%
\pgfsetroundjoin%
\pgfsetlinewidth{0.501875pt}%
\definecolor{currentstroke}{rgb}{0.000000,0.000000,1.000000}%
\pgfsetstrokecolor{currentstroke}%
\pgfsetdash{}{0pt}%
\pgfpathmoveto{\pgfqpoint{0.703931in}{1.528446in}}%
\pgfpathlineto{\pgfqpoint{0.703931in}{1.563358in}}%
\pgfusepath{stroke}%
\end{pgfscope}%
\begin{pgfscope}%
\pgfpathrectangle{\pgfqpoint{0.592283in}{0.475000in}}{\pgfqpoint{1.317500in}{1.309000in}}%
\pgfusepath{clip}%
\pgfsetbuttcap%
\pgfsetroundjoin%
\pgfsetlinewidth{0.501875pt}%
\definecolor{currentstroke}{rgb}{0.000000,0.000000,1.000000}%
\pgfsetstrokecolor{currentstroke}%
\pgfsetdash{}{0pt}%
\pgfpathmoveto{\pgfqpoint{0.685426in}{1.570097in}}%
\pgfpathlineto{\pgfqpoint{0.685426in}{1.600073in}}%
\pgfusepath{stroke}%
\end{pgfscope}%
\begin{pgfscope}%
\pgfpathrectangle{\pgfqpoint{0.592283in}{0.475000in}}{\pgfqpoint{1.317500in}{1.309000in}}%
\pgfusepath{clip}%
\pgfsetbuttcap%
\pgfsetroundjoin%
\pgfsetlinewidth{0.501875pt}%
\definecolor{currentstroke}{rgb}{0.000000,0.000000,1.000000}%
\pgfsetstrokecolor{currentstroke}%
\pgfsetdash{}{0pt}%
\pgfpathmoveto{\pgfqpoint{0.668219in}{1.603541in}}%
\pgfpathlineto{\pgfqpoint{0.668219in}{1.629766in}}%
\pgfusepath{stroke}%
\end{pgfscope}%
\begin{pgfscope}%
\pgfpathrectangle{\pgfqpoint{0.592283in}{0.475000in}}{\pgfqpoint{1.317500in}{1.309000in}}%
\pgfusepath{clip}%
\pgfsetbuttcap%
\pgfsetroundjoin%
\pgfsetlinewidth{0.501875pt}%
\definecolor{currentstroke}{rgb}{0.000000,0.000000,1.000000}%
\pgfsetstrokecolor{currentstroke}%
\pgfsetdash{}{0pt}%
\pgfpathmoveto{\pgfqpoint{0.652169in}{1.624673in}}%
\pgfpathlineto{\pgfqpoint{0.652169in}{1.652362in}}%
\pgfusepath{stroke}%
\end{pgfscope}%
\begin{pgfscope}%
\pgfpathrectangle{\pgfqpoint{0.592283in}{0.475000in}}{\pgfqpoint{1.317500in}{1.309000in}}%
\pgfusepath{clip}%
\pgfsetbuttcap%
\pgfsetroundjoin%
\pgfsetlinewidth{0.501875pt}%
\definecolor{currentstroke}{rgb}{1.000000,0.000000,0.000000}%
\pgfsetstrokecolor{currentstroke}%
\pgfsetdash{}{0pt}%
\pgfpathmoveto{\pgfqpoint{1.849896in}{1.187913in}}%
\pgfpathlineto{\pgfqpoint{1.849896in}{1.231752in}}%
\pgfusepath{stroke}%
\end{pgfscope}%
\begin{pgfscope}%
\pgfpathrectangle{\pgfqpoint{0.592283in}{0.475000in}}{\pgfqpoint{1.317500in}{1.309000in}}%
\pgfusepath{clip}%
\pgfsetbuttcap%
\pgfsetroundjoin%
\pgfsetlinewidth{0.501875pt}%
\definecolor{currentstroke}{rgb}{1.000000,0.000000,0.000000}%
\pgfsetstrokecolor{currentstroke}%
\pgfsetdash{}{0pt}%
\pgfpathmoveto{\pgfqpoint{1.622049in}{1.395014in}}%
\pgfpathlineto{\pgfqpoint{1.622049in}{1.422874in}}%
\pgfusepath{stroke}%
\end{pgfscope}%
\begin{pgfscope}%
\pgfpathrectangle{\pgfqpoint{0.592283in}{0.475000in}}{\pgfqpoint{1.317500in}{1.309000in}}%
\pgfusepath{clip}%
\pgfsetbuttcap%
\pgfsetroundjoin%
\pgfsetlinewidth{0.501875pt}%
\definecolor{currentstroke}{rgb}{1.000000,0.000000,0.000000}%
\pgfsetstrokecolor{currentstroke}%
\pgfsetdash{}{0pt}%
\pgfpathmoveto{\pgfqpoint{1.457248in}{1.506046in}}%
\pgfpathlineto{\pgfqpoint{1.457248in}{1.542520in}}%
\pgfusepath{stroke}%
\end{pgfscope}%
\begin{pgfscope}%
\pgfpathrectangle{\pgfqpoint{0.592283in}{0.475000in}}{\pgfqpoint{1.317500in}{1.309000in}}%
\pgfusepath{clip}%
\pgfsetbuttcap%
\pgfsetroundjoin%
\pgfsetlinewidth{0.501875pt}%
\definecolor{currentstroke}{rgb}{1.000000,0.000000,0.000000}%
\pgfsetstrokecolor{currentstroke}%
\pgfsetdash{}{0pt}%
\pgfpathmoveto{\pgfqpoint{1.331306in}{1.572324in}}%
\pgfpathlineto{\pgfqpoint{1.331306in}{1.593916in}}%
\pgfusepath{stroke}%
\end{pgfscope}%
\begin{pgfscope}%
\pgfpathrectangle{\pgfqpoint{0.592283in}{0.475000in}}{\pgfqpoint{1.317500in}{1.309000in}}%
\pgfusepath{clip}%
\pgfsetbuttcap%
\pgfsetroundjoin%
\pgfsetlinewidth{0.501875pt}%
\definecolor{currentstroke}{rgb}{1.000000,0.000000,0.000000}%
\pgfsetstrokecolor{currentstroke}%
\pgfsetdash{}{0pt}%
\pgfpathmoveto{\pgfqpoint{1.231240in}{1.610112in}}%
\pgfpathlineto{\pgfqpoint{1.231240in}{1.633021in}}%
\pgfusepath{stroke}%
\end{pgfscope}%
\begin{pgfscope}%
\pgfpathrectangle{\pgfqpoint{0.592283in}{0.475000in}}{\pgfqpoint{1.317500in}{1.309000in}}%
\pgfusepath{clip}%
\pgfsetbuttcap%
\pgfsetroundjoin%
\pgfsetlinewidth{0.501875pt}%
\definecolor{currentstroke}{rgb}{1.000000,0.000000,0.000000}%
\pgfsetstrokecolor{currentstroke}%
\pgfsetdash{}{0pt}%
\pgfpathmoveto{\pgfqpoint{1.149390in}{1.645997in}}%
\pgfpathlineto{\pgfqpoint{1.149390in}{1.670098in}}%
\pgfusepath{stroke}%
\end{pgfscope}%
\begin{pgfscope}%
\pgfpathrectangle{\pgfqpoint{0.592283in}{0.475000in}}{\pgfqpoint{1.317500in}{1.309000in}}%
\pgfusepath{clip}%
\pgfsetbuttcap%
\pgfsetroundjoin%
\pgfsetlinewidth{0.501875pt}%
\definecolor{currentstroke}{rgb}{1.000000,0.000000,0.000000}%
\pgfsetstrokecolor{currentstroke}%
\pgfsetdash{}{0pt}%
\pgfpathmoveto{\pgfqpoint{1.080917in}{1.655115in}}%
\pgfpathlineto{\pgfqpoint{1.080917in}{1.675966in}}%
\pgfusepath{stroke}%
\end{pgfscope}%
\begin{pgfscope}%
\pgfpathrectangle{\pgfqpoint{0.592283in}{0.475000in}}{\pgfqpoint{1.317500in}{1.309000in}}%
\pgfusepath{clip}%
\pgfsetbuttcap%
\pgfsetroundjoin%
\pgfsetlinewidth{0.501875pt}%
\definecolor{currentstroke}{rgb}{1.000000,0.000000,0.000000}%
\pgfsetstrokecolor{currentstroke}%
\pgfsetdash{}{0pt}%
\pgfpathmoveto{\pgfqpoint{1.022596in}{1.666566in}}%
\pgfpathlineto{\pgfqpoint{1.022596in}{1.686379in}}%
\pgfusepath{stroke}%
\end{pgfscope}%
\begin{pgfscope}%
\pgfpathrectangle{\pgfqpoint{0.592283in}{0.475000in}}{\pgfqpoint{1.317500in}{1.309000in}}%
\pgfusepath{clip}%
\pgfsetbuttcap%
\pgfsetroundjoin%
\pgfsetlinewidth{0.501875pt}%
\definecolor{currentstroke}{rgb}{1.000000,0.000000,0.000000}%
\pgfsetstrokecolor{currentstroke}%
\pgfsetdash{}{0pt}%
\pgfpathmoveto{\pgfqpoint{0.972188in}{1.669466in}}%
\pgfpathlineto{\pgfqpoint{0.972188in}{1.693060in}}%
\pgfusepath{stroke}%
\end{pgfscope}%
\begin{pgfscope}%
\pgfpathrectangle{\pgfqpoint{0.592283in}{0.475000in}}{\pgfqpoint{1.317500in}{1.309000in}}%
\pgfusepath{clip}%
\pgfsetbuttcap%
\pgfsetroundjoin%
\pgfsetlinewidth{0.501875pt}%
\definecolor{currentstroke}{rgb}{1.000000,0.000000,0.000000}%
\pgfsetstrokecolor{currentstroke}%
\pgfsetdash{}{0pt}%
\pgfpathmoveto{\pgfqpoint{0.928085in}{1.682136in}}%
\pgfpathlineto{\pgfqpoint{0.928085in}{1.698814in}}%
\pgfusepath{stroke}%
\end{pgfscope}%
\begin{pgfscope}%
\pgfpathrectangle{\pgfqpoint{0.592283in}{0.475000in}}{\pgfqpoint{1.317500in}{1.309000in}}%
\pgfusepath{clip}%
\pgfsetbuttcap%
\pgfsetroundjoin%
\pgfsetlinewidth{0.501875pt}%
\definecolor{currentstroke}{rgb}{1.000000,0.000000,0.000000}%
\pgfsetstrokecolor{currentstroke}%
\pgfsetdash{}{0pt}%
\pgfpathmoveto{\pgfqpoint{0.889098in}{1.685134in}}%
\pgfpathlineto{\pgfqpoint{0.889098in}{1.696308in}}%
\pgfusepath{stroke}%
\end{pgfscope}%
\begin{pgfscope}%
\pgfpathrectangle{\pgfqpoint{0.592283in}{0.475000in}}{\pgfqpoint{1.317500in}{1.309000in}}%
\pgfusepath{clip}%
\pgfsetbuttcap%
\pgfsetroundjoin%
\pgfsetlinewidth{0.501875pt}%
\definecolor{currentstroke}{rgb}{1.000000,0.000000,0.000000}%
\pgfsetstrokecolor{currentstroke}%
\pgfsetdash{}{0pt}%
\pgfpathmoveto{\pgfqpoint{0.854327in}{1.685727in}}%
\pgfpathlineto{\pgfqpoint{0.854327in}{1.704313in}}%
\pgfusepath{stroke}%
\end{pgfscope}%
\begin{pgfscope}%
\pgfpathrectangle{\pgfqpoint{0.592283in}{0.475000in}}{\pgfqpoint{1.317500in}{1.309000in}}%
\pgfusepath{clip}%
\pgfsetbuttcap%
\pgfsetroundjoin%
\pgfsetlinewidth{0.501875pt}%
\definecolor{currentstroke}{rgb}{1.000000,0.000000,0.000000}%
\pgfsetstrokecolor{currentstroke}%
\pgfsetdash{}{0pt}%
\pgfpathmoveto{\pgfqpoint{0.823079in}{1.699538in}}%
\pgfpathlineto{\pgfqpoint{0.823079in}{1.705734in}}%
\pgfusepath{stroke}%
\end{pgfscope}%
\begin{pgfscope}%
\pgfpathrectangle{\pgfqpoint{0.592283in}{0.475000in}}{\pgfqpoint{1.317500in}{1.309000in}}%
\pgfusepath{clip}%
\pgfsetbuttcap%
\pgfsetroundjoin%
\pgfsetlinewidth{0.501875pt}%
\definecolor{currentstroke}{rgb}{1.000000,0.000000,0.000000}%
\pgfsetstrokecolor{currentstroke}%
\pgfsetdash{}{0pt}%
\pgfpathmoveto{\pgfqpoint{0.794808in}{1.695550in}}%
\pgfpathlineto{\pgfqpoint{0.794808in}{1.714389in}}%
\pgfusepath{stroke}%
\end{pgfscope}%
\begin{pgfscope}%
\pgfpathrectangle{\pgfqpoint{0.592283in}{0.475000in}}{\pgfqpoint{1.317500in}{1.309000in}}%
\pgfusepath{clip}%
\pgfsetbuttcap%
\pgfsetroundjoin%
\pgfsetlinewidth{0.501875pt}%
\definecolor{currentstroke}{rgb}{1.000000,0.000000,0.000000}%
\pgfsetstrokecolor{currentstroke}%
\pgfsetdash{}{0pt}%
\pgfpathmoveto{\pgfqpoint{0.769078in}{1.698742in}}%
\pgfpathlineto{\pgfqpoint{0.769078in}{1.710215in}}%
\pgfusepath{stroke}%
\end{pgfscope}%
\begin{pgfscope}%
\pgfpathrectangle{\pgfqpoint{0.592283in}{0.475000in}}{\pgfqpoint{1.317500in}{1.309000in}}%
\pgfusepath{clip}%
\pgfsetbuttcap%
\pgfsetroundjoin%
\pgfsetlinewidth{0.501875pt}%
\definecolor{currentstroke}{rgb}{1.000000,0.000000,0.000000}%
\pgfsetstrokecolor{currentstroke}%
\pgfsetdash{}{0pt}%
\pgfpathmoveto{\pgfqpoint{0.745539in}{1.705067in}}%
\pgfpathlineto{\pgfqpoint{0.745539in}{1.717155in}}%
\pgfusepath{stroke}%
\end{pgfscope}%
\begin{pgfscope}%
\pgfpathrectangle{\pgfqpoint{0.592283in}{0.475000in}}{\pgfqpoint{1.317500in}{1.309000in}}%
\pgfusepath{clip}%
\pgfsetbuttcap%
\pgfsetroundjoin%
\pgfsetlinewidth{0.501875pt}%
\definecolor{currentstroke}{rgb}{1.000000,0.000000,0.000000}%
\pgfsetstrokecolor{currentstroke}%
\pgfsetdash{}{0pt}%
\pgfpathmoveto{\pgfqpoint{0.723903in}{1.702451in}}%
\pgfpathlineto{\pgfqpoint{0.723903in}{1.717315in}}%
\pgfusepath{stroke}%
\end{pgfscope}%
\begin{pgfscope}%
\pgfpathrectangle{\pgfqpoint{0.592283in}{0.475000in}}{\pgfqpoint{1.317500in}{1.309000in}}%
\pgfusepath{clip}%
\pgfsetbuttcap%
\pgfsetroundjoin%
\pgfsetlinewidth{0.501875pt}%
\definecolor{currentstroke}{rgb}{1.000000,0.000000,0.000000}%
\pgfsetstrokecolor{currentstroke}%
\pgfsetdash{}{0pt}%
\pgfpathmoveto{\pgfqpoint{0.703931in}{1.706152in}}%
\pgfpathlineto{\pgfqpoint{0.703931in}{1.717790in}}%
\pgfusepath{stroke}%
\end{pgfscope}%
\begin{pgfscope}%
\pgfpathrectangle{\pgfqpoint{0.592283in}{0.475000in}}{\pgfqpoint{1.317500in}{1.309000in}}%
\pgfusepath{clip}%
\pgfsetbuttcap%
\pgfsetroundjoin%
\pgfsetlinewidth{0.501875pt}%
\definecolor{currentstroke}{rgb}{1.000000,0.000000,0.000000}%
\pgfsetstrokecolor{currentstroke}%
\pgfsetdash{}{0pt}%
\pgfpathmoveto{\pgfqpoint{0.685426in}{1.711084in}}%
\pgfpathlineto{\pgfqpoint{0.685426in}{1.722930in}}%
\pgfusepath{stroke}%
\end{pgfscope}%
\begin{pgfscope}%
\pgfpathrectangle{\pgfqpoint{0.592283in}{0.475000in}}{\pgfqpoint{1.317500in}{1.309000in}}%
\pgfusepath{clip}%
\pgfsetbuttcap%
\pgfsetroundjoin%
\pgfsetlinewidth{0.501875pt}%
\definecolor{currentstroke}{rgb}{1.000000,0.000000,0.000000}%
\pgfsetstrokecolor{currentstroke}%
\pgfsetdash{}{0pt}%
\pgfpathmoveto{\pgfqpoint{0.668219in}{1.708706in}}%
\pgfpathlineto{\pgfqpoint{0.668219in}{1.718921in}}%
\pgfusepath{stroke}%
\end{pgfscope}%
\begin{pgfscope}%
\pgfpathrectangle{\pgfqpoint{0.592283in}{0.475000in}}{\pgfqpoint{1.317500in}{1.309000in}}%
\pgfusepath{clip}%
\pgfsetbuttcap%
\pgfsetroundjoin%
\pgfsetlinewidth{0.501875pt}%
\definecolor{currentstroke}{rgb}{1.000000,0.000000,0.000000}%
\pgfsetstrokecolor{currentstroke}%
\pgfsetdash{}{0pt}%
\pgfpathmoveto{\pgfqpoint{0.652169in}{1.711725in}}%
\pgfpathlineto{\pgfqpoint{0.652169in}{1.724500in}}%
\pgfusepath{stroke}%
\end{pgfscope}%
\begin{pgfscope}%
\pgfpathrectangle{\pgfqpoint{0.592283in}{0.475000in}}{\pgfqpoint{1.317500in}{1.309000in}}%
\pgfusepath{clip}%
\pgfsetbuttcap%
\pgfsetroundjoin%
\pgfsetlinewidth{0.501875pt}%
\definecolor{currentstroke}{rgb}{0.000000,0.000000,0.000000}%
\pgfsetstrokecolor{currentstroke}%
\pgfsetdash{{1.850000pt}{0.800000pt}}{0.000000pt}%
\pgfpathmoveto{\pgfqpoint{0.592283in}{1.701653in}}%
\pgfpathlineto{\pgfqpoint{1.909783in}{1.701653in}}%
\pgfusepath{stroke}%
\end{pgfscope}%
\begin{pgfscope}%
\pgfpathrectangle{\pgfqpoint{0.592283in}{0.475000in}}{\pgfqpoint{1.317500in}{1.309000in}}%
\pgfusepath{clip}%
\pgfsetrectcap%
\pgfsetroundjoin%
\pgfsetlinewidth{0.501875pt}%
\definecolor{currentstroke}{rgb}{0.000000,0.000000,1.000000}%
\pgfsetstrokecolor{currentstroke}%
\pgfsetdash{}{0pt}%
\pgfpathmoveto{\pgfqpoint{1.849896in}{0.534746in}}%
\pgfpathlineto{\pgfqpoint{1.622049in}{0.534746in}}%
\pgfpathlineto{\pgfqpoint{1.457248in}{0.534746in}}%
\pgfpathlineto{\pgfqpoint{1.331306in}{0.534746in}}%
\pgfpathlineto{\pgfqpoint{1.231240in}{0.534746in}}%
\pgfpathlineto{\pgfqpoint{1.149390in}{0.534991in}}%
\pgfpathlineto{\pgfqpoint{1.080917in}{0.540273in}}%
\pgfpathlineto{\pgfqpoint{1.022596in}{0.559558in}}%
\pgfpathlineto{\pgfqpoint{0.972188in}{0.618517in}}%
\pgfpathlineto{\pgfqpoint{0.928085in}{0.706220in}}%
\pgfpathlineto{\pgfqpoint{0.889098in}{0.809645in}}%
\pgfpathlineto{\pgfqpoint{0.854327in}{0.959868in}}%
\pgfpathlineto{\pgfqpoint{0.823079in}{1.092896in}}%
\pgfpathlineto{\pgfqpoint{0.794808in}{1.231819in}}%
\pgfpathlineto{\pgfqpoint{0.769078in}{1.339298in}}%
\pgfpathlineto{\pgfqpoint{0.745539in}{1.418033in}}%
\pgfpathlineto{\pgfqpoint{0.723903in}{1.497014in}}%
\pgfpathlineto{\pgfqpoint{0.703931in}{1.545902in}}%
\pgfpathlineto{\pgfqpoint{0.685426in}{1.585085in}}%
\pgfpathlineto{\pgfqpoint{0.668219in}{1.616653in}}%
\pgfpathlineto{\pgfqpoint{0.652169in}{1.638517in}}%
\pgfusepath{stroke}%
\end{pgfscope}%
\begin{pgfscope}%
\pgfpathrectangle{\pgfqpoint{0.592283in}{0.475000in}}{\pgfqpoint{1.317500in}{1.309000in}}%
\pgfusepath{clip}%
\pgfsetrectcap%
\pgfsetroundjoin%
\pgfsetlinewidth{0.501875pt}%
\definecolor{currentstroke}{rgb}{1.000000,0.000000,0.000000}%
\pgfsetstrokecolor{currentstroke}%
\pgfsetdash{}{0pt}%
\pgfpathmoveto{\pgfqpoint{1.849896in}{1.209832in}}%
\pgfpathlineto{\pgfqpoint{1.622049in}{1.408944in}}%
\pgfpathlineto{\pgfqpoint{1.457248in}{1.524283in}}%
\pgfpathlineto{\pgfqpoint{1.331306in}{1.583120in}}%
\pgfpathlineto{\pgfqpoint{1.231240in}{1.621566in}}%
\pgfpathlineto{\pgfqpoint{1.149390in}{1.658048in}}%
\pgfpathlineto{\pgfqpoint{1.080917in}{1.665540in}}%
\pgfpathlineto{\pgfqpoint{1.022596in}{1.676473in}}%
\pgfpathlineto{\pgfqpoint{0.972188in}{1.681263in}}%
\pgfpathlineto{\pgfqpoint{0.928085in}{1.690475in}}%
\pgfpathlineto{\pgfqpoint{0.889098in}{1.690721in}}%
\pgfpathlineto{\pgfqpoint{0.854327in}{1.695020in}}%
\pgfpathlineto{\pgfqpoint{0.823079in}{1.702636in}}%
\pgfpathlineto{\pgfqpoint{0.794808in}{1.704970in}}%
\pgfpathlineto{\pgfqpoint{0.769078in}{1.704478in}}%
\pgfpathlineto{\pgfqpoint{0.745539in}{1.711111in}}%
\pgfpathlineto{\pgfqpoint{0.723903in}{1.709883in}}%
\pgfpathlineto{\pgfqpoint{0.703931in}{1.711971in}}%
\pgfpathlineto{\pgfqpoint{0.685426in}{1.717007in}}%
\pgfpathlineto{\pgfqpoint{0.668219in}{1.713814in}}%
\pgfpathlineto{\pgfqpoint{0.652169in}{1.718113in}}%
\pgfusepath{stroke}%
\end{pgfscope}%
\begin{pgfscope}%
\pgfsetrectcap%
\pgfsetmiterjoin%
\pgfsetlinewidth{0.803000pt}%
\definecolor{currentstroke}{rgb}{0.000000,0.000000,0.000000}%
\pgfsetstrokecolor{currentstroke}%
\pgfsetdash{}{0pt}%
\pgfpathmoveto{\pgfqpoint{0.592283in}{0.475000in}}%
\pgfpathlineto{\pgfqpoint{0.592283in}{1.784000in}}%
\pgfusepath{stroke}%
\end{pgfscope}%
\begin{pgfscope}%
\pgfsetrectcap%
\pgfsetmiterjoin%
\pgfsetlinewidth{0.803000pt}%
\definecolor{currentstroke}{rgb}{0.000000,0.000000,0.000000}%
\pgfsetstrokecolor{currentstroke}%
\pgfsetdash{}{0pt}%
\pgfpathmoveto{\pgfqpoint{1.909783in}{0.475000in}}%
\pgfpathlineto{\pgfqpoint{1.909783in}{1.784000in}}%
\pgfusepath{stroke}%
\end{pgfscope}%
\begin{pgfscope}%
\pgfsetrectcap%
\pgfsetmiterjoin%
\pgfsetlinewidth{0.803000pt}%
\definecolor{currentstroke}{rgb}{0.000000,0.000000,0.000000}%
\pgfsetstrokecolor{currentstroke}%
\pgfsetdash{}{0pt}%
\pgfpathmoveto{\pgfqpoint{0.592283in}{0.475000in}}%
\pgfpathlineto{\pgfqpoint{1.909783in}{0.475000in}}%
\pgfusepath{stroke}%
\end{pgfscope}%
\begin{pgfscope}%
\pgfsetrectcap%
\pgfsetmiterjoin%
\pgfsetlinewidth{0.803000pt}%
\definecolor{currentstroke}{rgb}{0.000000,0.000000,0.000000}%
\pgfsetstrokecolor{currentstroke}%
\pgfsetdash{}{0pt}%
\pgfpathmoveto{\pgfqpoint{0.592283in}{1.784000in}}%
\pgfpathlineto{\pgfqpoint{1.909783in}{1.784000in}}%
\pgfusepath{stroke}%
\end{pgfscope}%
\begin{pgfscope}%
\pgfsetbuttcap%
\pgfsetmiterjoin%
\definecolor{currentfill}{rgb}{1.000000,1.000000,1.000000}%
\pgfsetfillcolor{currentfill}%
\pgfsetfillopacity{0.500000}%
\pgfsetlinewidth{1.003750pt}%
\definecolor{currentstroke}{rgb}{0.800000,0.800000,0.800000}%
\pgfsetstrokecolor{currentstroke}%
\pgfsetstrokeopacity{0.500000}%
\pgfsetdash{}{0pt}%
\pgfpathmoveto{\pgfqpoint{0.679783in}{0.537500in}}%
\pgfpathlineto{\pgfqpoint{1.565119in}{0.537500in}}%
\pgfpathquadraticcurveto{\pgfqpoint{1.590119in}{0.537500in}}{\pgfqpoint{1.590119in}{0.562500in}}%
\pgfpathlineto{\pgfqpoint{1.590119in}{0.898611in}}%
\pgfpathquadraticcurveto{\pgfqpoint{1.590119in}{0.923611in}}{\pgfqpoint{1.565119in}{0.923611in}}%
\pgfpathlineto{\pgfqpoint{0.679783in}{0.923611in}}%
\pgfpathquadraticcurveto{\pgfqpoint{0.654783in}{0.923611in}}{\pgfqpoint{0.654783in}{0.898611in}}%
\pgfpathlineto{\pgfqpoint{0.654783in}{0.562500in}}%
\pgfpathquadraticcurveto{\pgfqpoint{0.654783in}{0.537500in}}{\pgfqpoint{0.679783in}{0.537500in}}%
\pgfpathclose%
\pgfusepath{stroke,fill}%
\end{pgfscope}%
\begin{pgfscope}%
\pgfsetbuttcap%
\pgfsetroundjoin%
\pgfsetlinewidth{0.501875pt}%
\definecolor{currentstroke}{rgb}{0.000000,0.000000,1.000000}%
\pgfsetstrokecolor{currentstroke}%
\pgfsetdash{}{0pt}%
\pgfpathmoveto{\pgfqpoint{0.829783in}{0.767361in}}%
\pgfpathlineto{\pgfqpoint{0.829783in}{0.892361in}}%
\pgfusepath{stroke}%
\end{pgfscope}%
\begin{pgfscope}%
\pgfsetrectcap%
\pgfsetroundjoin%
\pgfsetlinewidth{0.501875pt}%
\definecolor{currentstroke}{rgb}{0.000000,0.000000,1.000000}%
\pgfsetstrokecolor{currentstroke}%
\pgfsetdash{}{0pt}%
\pgfpathmoveto{\pgfqpoint{0.704783in}{0.829861in}}%
\pgfpathlineto{\pgfqpoint{0.954783in}{0.829861in}}%
\pgfusepath{stroke}%
\end{pgfscope}%
\begin{pgfscope}%
\definecolor{textcolor}{rgb}{0.000000,0.000000,0.000000}%
\pgfsetstrokecolor{textcolor}%
\pgfsetfillcolor{textcolor}%
\pgftext[x=1.054782in,y=0.786111in,left,base]{\color{textcolor}\rmfamily\fontsize{9.000000}{10.800000}\selectfont Plug-in}%
\end{pgfscope}%
\begin{pgfscope}%
\pgfsetbuttcap%
\pgfsetroundjoin%
\pgfsetlinewidth{0.501875pt}%
\definecolor{currentstroke}{rgb}{1.000000,0.000000,0.000000}%
\pgfsetstrokecolor{currentstroke}%
\pgfsetdash{}{0pt}%
\pgfpathmoveto{\pgfqpoint{0.829783in}{0.593056in}}%
\pgfpathlineto{\pgfqpoint{0.829783in}{0.718056in}}%
\pgfusepath{stroke}%
\end{pgfscope}%
\begin{pgfscope}%
\pgfsetrectcap%
\pgfsetroundjoin%
\pgfsetlinewidth{0.501875pt}%
\definecolor{currentstroke}{rgb}{1.000000,0.000000,0.000000}%
\pgfsetstrokecolor{currentstroke}%
\pgfsetdash{}{0pt}%
\pgfpathmoveto{\pgfqpoint{0.704783in}{0.655556in}}%
\pgfpathlineto{\pgfqpoint{0.954783in}{0.655556in}}%
\pgfusepath{stroke}%
\end{pgfscope}%
\begin{pgfscope}%
\definecolor{textcolor}{rgb}{0.000000,0.000000,0.000000}%
\pgfsetstrokecolor{textcolor}%
\pgfsetfillcolor{textcolor}%
\pgftext[x=1.054782in,y=0.611806in,left,base]{\color{textcolor}\rmfamily\fontsize{9.000000}{10.800000}\selectfont Splitting}%
\end{pgfscope}%
\end{pgfpicture}%
\makeatother%
\endgroup%

%% file: images/plugin_vs_splitting200.pgf
\begingroup%
\makeatletter%
\begin{pgfpicture}%
\pgfpathrectangle{\pgfpointorigin}{\pgfqpoint{2.026140in}{1.920593in}}%
\pgfusepath{use as bounding box, clip}%
\begin{pgfscope}%
\pgfsetbuttcap%
\pgfsetmiterjoin%
\definecolor{currentfill}{rgb}{1.000000,1.000000,1.000000}%
\pgfsetfillcolor{currentfill}%
\pgfsetlinewidth{0.000000pt}%
\definecolor{currentstroke}{rgb}{1.000000,1.000000,1.000000}%
\pgfsetstrokecolor{currentstroke}%
\pgfsetdash{}{0pt}%
\pgfpathmoveto{\pgfqpoint{0.000000in}{0.000000in}}%
\pgfpathlineto{\pgfqpoint{2.026140in}{0.000000in}}%
\pgfpathlineto{\pgfqpoint{2.026140in}{1.920593in}}%
\pgfpathlineto{\pgfqpoint{0.000000in}{1.920593in}}%
\pgfpathclose%
\pgfusepath{fill}%
\end{pgfscope}%
\begin{pgfscope}%
\pgfsetbuttcap%
\pgfsetmiterjoin%
\definecolor{currentfill}{rgb}{1.000000,1.000000,1.000000}%
\pgfsetfillcolor{currentfill}%
\pgfsetlinewidth{0.000000pt}%
\definecolor{currentstroke}{rgb}{0.000000,0.000000,0.000000}%
\pgfsetstrokecolor{currentstroke}%
\pgfsetstrokeopacity{0.000000}%
\pgfsetdash{}{0pt}%
\pgfpathmoveto{\pgfqpoint{0.592283in}{0.475000in}}%
\pgfpathlineto{\pgfqpoint{1.909783in}{0.475000in}}%
\pgfpathlineto{\pgfqpoint{1.909783in}{1.784000in}}%
\pgfpathlineto{\pgfqpoint{0.592283in}{1.784000in}}%
\pgfpathclose%
\pgfusepath{fill}%
\end{pgfscope}%
\begin{pgfscope}%
\pgfsetbuttcap%
\pgfsetroundjoin%
\definecolor{currentfill}{rgb}{0.000000,0.000000,0.000000}%
\pgfsetfillcolor{currentfill}%
\pgfsetlinewidth{0.803000pt}%
\definecolor{currentstroke}{rgb}{0.000000,0.000000,0.000000}%
\pgfsetstrokecolor{currentstroke}%
\pgfsetdash{}{0pt}%
\pgfsys@defobject{currentmarker}{\pgfqpoint{0.000000in}{-0.048611in}}{\pgfqpoint{0.000000in}{0.000000in}}{%
\pgfpathmoveto{\pgfqpoint{0.000000in}{0.000000in}}%
\pgfpathlineto{\pgfqpoint{0.000000in}{-0.048611in}}%
\pgfusepath{stroke,fill}%
}%
\begin{pgfscope}%
\pgfsys@transformshift{1.044499in}{0.475000in}%
\pgfsys@useobject{currentmarker}{}%
\end{pgfscope}%
\end{pgfscope}%
\begin{pgfscope}%
\definecolor{textcolor}{rgb}{0.000000,0.000000,0.000000}%
\pgfsetstrokecolor{textcolor}%
\pgfsetfillcolor{textcolor}%
\pgftext[x=1.044499in,y=0.377778in,,top]{\color{textcolor}\rmfamily\fontsize{9.000000}{10.800000}\selectfont \(\displaystyle {0.05}\)}%
\end{pgfscope}%
\begin{pgfscope}%
\pgfsetbuttcap%
\pgfsetroundjoin%
\definecolor{currentfill}{rgb}{0.000000,0.000000,0.000000}%
\pgfsetfillcolor{currentfill}%
\pgfsetlinewidth{0.803000pt}%
\definecolor{currentstroke}{rgb}{0.000000,0.000000,0.000000}%
\pgfsetstrokecolor{currentstroke}%
\pgfsetdash{}{0pt}%
\pgfsys@defobject{currentmarker}{\pgfqpoint{0.000000in}{-0.048611in}}{\pgfqpoint{0.000000in}{0.000000in}}{%
\pgfpathmoveto{\pgfqpoint{0.000000in}{0.000000in}}%
\pgfpathlineto{\pgfqpoint{0.000000in}{-0.048611in}}%
\pgfusepath{stroke,fill}%
}%
\begin{pgfscope}%
\pgfsys@transformshift{1.811943in}{0.475000in}%
\pgfsys@useobject{currentmarker}{}%
\end{pgfscope}%
\end{pgfscope}%
\begin{pgfscope}%
\definecolor{textcolor}{rgb}{0.000000,0.000000,0.000000}%
\pgfsetstrokecolor{textcolor}%
\pgfsetfillcolor{textcolor}%
\pgftext[x=1.811943in,y=0.377778in,,top]{\color{textcolor}\rmfamily\fontsize{9.000000}{10.800000}\selectfont \(\displaystyle {0.10}\)}%
\end{pgfscope}%
\begin{pgfscope}%
\definecolor{textcolor}{rgb}{0.000000,0.000000,0.000000}%
\pgfsetstrokecolor{textcolor}%
\pgfsetfillcolor{textcolor}%
\pgftext[x=1.251032in,y=0.211111in,,top]{\color{textcolor}\rmfamily\fontsize{9.000000}{10.800000}\selectfont \(\displaystyle \ell_2\)-ECE}%
\end{pgfscope}%
\begin{pgfscope}%
\pgfsetbuttcap%
\pgfsetroundjoin%
\definecolor{currentfill}{rgb}{0.000000,0.000000,0.000000}%
\pgfsetfillcolor{currentfill}%
\pgfsetlinewidth{0.803000pt}%
\definecolor{currentstroke}{rgb}{0.000000,0.000000,0.000000}%
\pgfsetstrokecolor{currentstroke}%
\pgfsetdash{}{0pt}%
\pgfsys@defobject{currentmarker}{\pgfqpoint{-0.048611in}{0.000000in}}{\pgfqpoint{-0.000000in}{0.000000in}}{%
\pgfpathmoveto{\pgfqpoint{-0.000000in}{0.000000in}}%
\pgfpathlineto{\pgfqpoint{-0.048611in}{0.000000in}}%
\pgfusepath{stroke,fill}%
}%
\begin{pgfscope}%
\pgfsys@transformshift{0.592283in}{0.534500in}%
\pgfsys@useobject{currentmarker}{}%
\end{pgfscope}%
\end{pgfscope}%
\begin{pgfscope}%
\definecolor{textcolor}{rgb}{0.000000,0.000000,0.000000}%
\pgfsetstrokecolor{textcolor}%
\pgfsetfillcolor{textcolor}%
\pgftext[x=0.266667in, y=0.491097in, left, base]{\color{textcolor}\rmfamily\fontsize{9.000000}{10.800000}\selectfont \(\displaystyle {0.00}\)}%
\end{pgfscope}%
\begin{pgfscope}%
\pgfsetbuttcap%
\pgfsetroundjoin%
\definecolor{currentfill}{rgb}{0.000000,0.000000,0.000000}%
\pgfsetfillcolor{currentfill}%
\pgfsetlinewidth{0.803000pt}%
\definecolor{currentstroke}{rgb}{0.000000,0.000000,0.000000}%
\pgfsetstrokecolor{currentstroke}%
\pgfsetdash{}{0pt}%
\pgfsys@defobject{currentmarker}{\pgfqpoint{-0.048611in}{0.000000in}}{\pgfqpoint{-0.000000in}{0.000000in}}{%
\pgfpathmoveto{\pgfqpoint{-0.000000in}{0.000000in}}%
\pgfpathlineto{\pgfqpoint{-0.048611in}{0.000000in}}%
\pgfusepath{stroke,fill}%
}%
\begin{pgfscope}%
\pgfsys@transformshift{0.592283in}{0.845173in}%
\pgfsys@useobject{currentmarker}{}%
\end{pgfscope}%
\end{pgfscope}%
\begin{pgfscope}%
\definecolor{textcolor}{rgb}{0.000000,0.000000,0.000000}%
\pgfsetstrokecolor{textcolor}%
\pgfsetfillcolor{textcolor}%
\pgftext[x=0.266667in, y=0.801770in, left, base]{\color{textcolor}\rmfamily\fontsize{9.000000}{10.800000}\selectfont \(\displaystyle {0.25}\)}%
\end{pgfscope}%
\begin{pgfscope}%
\pgfsetbuttcap%
\pgfsetroundjoin%
\definecolor{currentfill}{rgb}{0.000000,0.000000,0.000000}%
\pgfsetfillcolor{currentfill}%
\pgfsetlinewidth{0.803000pt}%
\definecolor{currentstroke}{rgb}{0.000000,0.000000,0.000000}%
\pgfsetstrokecolor{currentstroke}%
\pgfsetdash{}{0pt}%
\pgfsys@defobject{currentmarker}{\pgfqpoint{-0.048611in}{0.000000in}}{\pgfqpoint{-0.000000in}{0.000000in}}{%
\pgfpathmoveto{\pgfqpoint{-0.000000in}{0.000000in}}%
\pgfpathlineto{\pgfqpoint{-0.048611in}{0.000000in}}%
\pgfusepath{stroke,fill}%
}%
\begin{pgfscope}%
\pgfsys@transformshift{0.592283in}{1.155845in}%
\pgfsys@useobject{currentmarker}{}%
\end{pgfscope}%
\end{pgfscope}%
\begin{pgfscope}%
\definecolor{textcolor}{rgb}{0.000000,0.000000,0.000000}%
\pgfsetstrokecolor{textcolor}%
\pgfsetfillcolor{textcolor}%
\pgftext[x=0.266667in, y=1.112442in, left, base]{\color{textcolor}\rmfamily\fontsize{9.000000}{10.800000}\selectfont \(\displaystyle {0.50}\)}%
\end{pgfscope}%
\begin{pgfscope}%
\pgfsetbuttcap%
\pgfsetroundjoin%
\definecolor{currentfill}{rgb}{0.000000,0.000000,0.000000}%
\pgfsetfillcolor{currentfill}%
\pgfsetlinewidth{0.803000pt}%
\definecolor{currentstroke}{rgb}{0.000000,0.000000,0.000000}%
\pgfsetstrokecolor{currentstroke}%
\pgfsetdash{}{0pt}%
\pgfsys@defobject{currentmarker}{\pgfqpoint{-0.048611in}{0.000000in}}{\pgfqpoint{-0.000000in}{0.000000in}}{%
\pgfpathmoveto{\pgfqpoint{-0.000000in}{0.000000in}}%
\pgfpathlineto{\pgfqpoint{-0.048611in}{0.000000in}}%
\pgfusepath{stroke,fill}%
}%
\begin{pgfscope}%
\pgfsys@transformshift{0.592283in}{1.466518in}%
\pgfsys@useobject{currentmarker}{}%
\end{pgfscope}%
\end{pgfscope}%
\begin{pgfscope}%
\definecolor{textcolor}{rgb}{0.000000,0.000000,0.000000}%
\pgfsetstrokecolor{textcolor}%
\pgfsetfillcolor{textcolor}%
\pgftext[x=0.266667in, y=1.423115in, left, base]{\color{textcolor}\rmfamily\fontsize{9.000000}{10.800000}\selectfont \(\displaystyle {0.75}\)}%
\end{pgfscope}%
\begin{pgfscope}%
\pgfsetbuttcap%
\pgfsetroundjoin%
\definecolor{currentfill}{rgb}{0.000000,0.000000,0.000000}%
\pgfsetfillcolor{currentfill}%
\pgfsetlinewidth{0.803000pt}%
\definecolor{currentstroke}{rgb}{0.000000,0.000000,0.000000}%
\pgfsetstrokecolor{currentstroke}%
\pgfsetdash{}{0pt}%
\pgfsys@defobject{currentmarker}{\pgfqpoint{-0.048611in}{0.000000in}}{\pgfqpoint{-0.000000in}{0.000000in}}{%
\pgfpathmoveto{\pgfqpoint{-0.000000in}{0.000000in}}%
\pgfpathlineto{\pgfqpoint{-0.048611in}{0.000000in}}%
\pgfusepath{stroke,fill}%
}%
\begin{pgfscope}%
\pgfsys@transformshift{0.592283in}{1.777190in}%
\pgfsys@useobject{currentmarker}{}%
\end{pgfscope}%
\end{pgfscope}%
\begin{pgfscope}%
\definecolor{textcolor}{rgb}{0.000000,0.000000,0.000000}%
\pgfsetstrokecolor{textcolor}%
\pgfsetfillcolor{textcolor}%
\pgftext[x=0.266667in, y=1.733787in, left, base]{\color{textcolor}\rmfamily\fontsize{9.000000}{10.800000}\selectfont \(\displaystyle {1.00}\)}%
\end{pgfscope}%
\begin{pgfscope}%
\definecolor{textcolor}{rgb}{0.000000,0.000000,0.000000}%
\pgfsetstrokecolor{textcolor}%
\pgfsetfillcolor{textcolor}%
\pgftext[x=0.211111in,y=1.129500in,,bottom,rotate=90.000000]{\color{textcolor}\rmfamily\fontsize{9.000000}{10.800000}\selectfont Type II error}%
\end{pgfscope}%
\begin{pgfscope}%
\pgfpathrectangle{\pgfqpoint{0.592283in}{0.475000in}}{\pgfqpoint{1.317500in}{1.309000in}}%
\pgfusepath{clip}%
\pgfsetbuttcap%
\pgfsetroundjoin%
\pgfsetlinewidth{0.501875pt}%
\definecolor{currentstroke}{rgb}{0.000000,0.000000,1.000000}%
\pgfsetstrokecolor{currentstroke}%
\pgfsetdash{}{0pt}%
\pgfpathmoveto{\pgfqpoint{1.849896in}{0.534500in}}%
\pgfpathlineto{\pgfqpoint{1.849896in}{0.534500in}}%
\pgfusepath{stroke}%
\end{pgfscope}%
\begin{pgfscope}%
\pgfpathrectangle{\pgfqpoint{0.592283in}{0.475000in}}{\pgfqpoint{1.317500in}{1.309000in}}%
\pgfusepath{clip}%
\pgfsetbuttcap%
\pgfsetroundjoin%
\pgfsetlinewidth{0.501875pt}%
\definecolor{currentstroke}{rgb}{0.000000,0.000000,1.000000}%
\pgfsetstrokecolor{currentstroke}%
\pgfsetdash{}{0pt}%
\pgfpathmoveto{\pgfqpoint{1.592755in}{0.534500in}}%
\pgfpathlineto{\pgfqpoint{1.592755in}{0.534500in}}%
\pgfusepath{stroke}%
\end{pgfscope}%
\begin{pgfscope}%
\pgfpathrectangle{\pgfqpoint{0.592283in}{0.475000in}}{\pgfqpoint{1.317500in}{1.309000in}}%
\pgfusepath{clip}%
\pgfsetbuttcap%
\pgfsetroundjoin%
\pgfsetlinewidth{0.501875pt}%
\definecolor{currentstroke}{rgb}{0.000000,0.000000,1.000000}%
\pgfsetstrokecolor{currentstroke}%
\pgfsetdash{}{0pt}%
\pgfpathmoveto{\pgfqpoint{1.414189in}{0.538106in}}%
\pgfpathlineto{\pgfqpoint{1.414189in}{0.544315in}}%
\pgfusepath{stroke}%
\end{pgfscope}%
\begin{pgfscope}%
\pgfpathrectangle{\pgfqpoint{0.592283in}{0.475000in}}{\pgfqpoint{1.317500in}{1.309000in}}%
\pgfusepath{clip}%
\pgfsetbuttcap%
\pgfsetroundjoin%
\pgfsetlinewidth{0.501875pt}%
\definecolor{currentstroke}{rgb}{0.000000,0.000000,1.000000}%
\pgfsetstrokecolor{currentstroke}%
\pgfsetdash{}{0pt}%
\pgfpathmoveto{\pgfqpoint{1.282259in}{0.780817in}}%
\pgfpathlineto{\pgfqpoint{1.282259in}{0.803899in}}%
\pgfusepath{stroke}%
\end{pgfscope}%
\begin{pgfscope}%
\pgfpathrectangle{\pgfqpoint{0.592283in}{0.475000in}}{\pgfqpoint{1.317500in}{1.309000in}}%
\pgfusepath{clip}%
\pgfsetbuttcap%
\pgfsetroundjoin%
\pgfsetlinewidth{0.501875pt}%
\definecolor{currentstroke}{rgb}{0.000000,0.000000,1.000000}%
\pgfsetstrokecolor{currentstroke}%
\pgfsetdash{}{0pt}%
\pgfpathmoveto{\pgfqpoint{1.180415in}{1.320838in}}%
\pgfpathlineto{\pgfqpoint{1.180415in}{1.349741in}}%
\pgfusepath{stroke}%
\end{pgfscope}%
\begin{pgfscope}%
\pgfpathrectangle{\pgfqpoint{0.592283in}{0.475000in}}{\pgfqpoint{1.317500in}{1.309000in}}%
\pgfusepath{clip}%
\pgfsetbuttcap%
\pgfsetroundjoin%
\pgfsetlinewidth{0.501875pt}%
\definecolor{currentstroke}{rgb}{0.000000,0.000000,1.000000}%
\pgfsetstrokecolor{currentstroke}%
\pgfsetdash{}{0pt}%
\pgfpathmoveto{\pgfqpoint{1.099181in}{1.593109in}}%
\pgfpathlineto{\pgfqpoint{1.099181in}{1.620526in}}%
\pgfusepath{stroke}%
\end{pgfscope}%
\begin{pgfscope}%
\pgfpathrectangle{\pgfqpoint{0.592283in}{0.475000in}}{\pgfqpoint{1.317500in}{1.309000in}}%
\pgfusepath{clip}%
\pgfsetbuttcap%
\pgfsetroundjoin%
\pgfsetlinewidth{0.501875pt}%
\definecolor{currentstroke}{rgb}{0.000000,0.000000,1.000000}%
\pgfsetstrokecolor{currentstroke}%
\pgfsetdash{}{0pt}%
\pgfpathmoveto{\pgfqpoint{1.032726in}{1.677811in}}%
\pgfpathlineto{\pgfqpoint{1.032726in}{1.689668in}}%
\pgfusepath{stroke}%
\end{pgfscope}%
\begin{pgfscope}%
\pgfpathrectangle{\pgfqpoint{0.592283in}{0.475000in}}{\pgfqpoint{1.317500in}{1.309000in}}%
\pgfusepath{clip}%
\pgfsetbuttcap%
\pgfsetroundjoin%
\pgfsetlinewidth{0.501875pt}%
\definecolor{currentstroke}{rgb}{0.000000,0.000000,1.000000}%
\pgfsetstrokecolor{currentstroke}%
\pgfsetdash{}{0pt}%
\pgfpathmoveto{\pgfqpoint{0.977249in}{1.692791in}}%
\pgfpathlineto{\pgfqpoint{0.977249in}{1.707495in}}%
\pgfusepath{stroke}%
\end{pgfscope}%
\begin{pgfscope}%
\pgfpathrectangle{\pgfqpoint{0.592283in}{0.475000in}}{\pgfqpoint{1.317500in}{1.309000in}}%
\pgfusepath{clip}%
\pgfsetbuttcap%
\pgfsetroundjoin%
\pgfsetlinewidth{0.501875pt}%
\definecolor{currentstroke}{rgb}{0.000000,0.000000,1.000000}%
\pgfsetstrokecolor{currentstroke}%
\pgfsetdash{}{0pt}%
\pgfpathmoveto{\pgfqpoint{0.930167in}{1.696797in}}%
\pgfpathlineto{\pgfqpoint{0.930167in}{1.704980in}}%
\pgfusepath{stroke}%
\end{pgfscope}%
\begin{pgfscope}%
\pgfpathrectangle{\pgfqpoint{0.592283in}{0.475000in}}{\pgfqpoint{1.317500in}{1.309000in}}%
\pgfusepath{clip}%
\pgfsetbuttcap%
\pgfsetroundjoin%
\pgfsetlinewidth{0.501875pt}%
\definecolor{currentstroke}{rgb}{0.000000,0.000000,1.000000}%
\pgfsetstrokecolor{currentstroke}%
\pgfsetdash{}{0pt}%
\pgfpathmoveto{\pgfqpoint{0.889656in}{1.673847in}}%
\pgfpathlineto{\pgfqpoint{0.889656in}{1.699349in}}%
\pgfusepath{stroke}%
\end{pgfscope}%
\begin{pgfscope}%
\pgfpathrectangle{\pgfqpoint{0.592283in}{0.475000in}}{\pgfqpoint{1.317500in}{1.309000in}}%
\pgfusepath{clip}%
\pgfsetbuttcap%
\pgfsetroundjoin%
\pgfsetlinewidth{0.501875pt}%
\definecolor{currentstroke}{rgb}{0.000000,0.000000,1.000000}%
\pgfsetstrokecolor{currentstroke}%
\pgfsetdash{}{0pt}%
\pgfpathmoveto{\pgfqpoint{0.854393in}{1.671775in}}%
\pgfpathlineto{\pgfqpoint{0.854393in}{1.689491in}}%
\pgfusepath{stroke}%
\end{pgfscope}%
\begin{pgfscope}%
\pgfpathrectangle{\pgfqpoint{0.592283in}{0.475000in}}{\pgfqpoint{1.317500in}{1.309000in}}%
\pgfusepath{clip}%
\pgfsetbuttcap%
\pgfsetroundjoin%
\pgfsetlinewidth{0.501875pt}%
\definecolor{currentstroke}{rgb}{0.000000,0.000000,1.000000}%
\pgfsetstrokecolor{currentstroke}%
\pgfsetdash{}{0pt}%
\pgfpathmoveto{\pgfqpoint{0.823390in}{1.666506in}}%
\pgfpathlineto{\pgfqpoint{0.823390in}{1.682582in}}%
\pgfusepath{stroke}%
\end{pgfscope}%
\begin{pgfscope}%
\pgfpathrectangle{\pgfqpoint{0.592283in}{0.475000in}}{\pgfqpoint{1.317500in}{1.309000in}}%
\pgfusepath{clip}%
\pgfsetbuttcap%
\pgfsetroundjoin%
\pgfsetlinewidth{0.501875pt}%
\definecolor{currentstroke}{rgb}{0.000000,0.000000,1.000000}%
\pgfsetstrokecolor{currentstroke}%
\pgfsetdash{}{0pt}%
\pgfpathmoveto{\pgfqpoint{0.795899in}{1.664306in}}%
\pgfpathlineto{\pgfqpoint{0.795899in}{1.683540in}}%
\pgfusepath{stroke}%
\end{pgfscope}%
\begin{pgfscope}%
\pgfpathrectangle{\pgfqpoint{0.592283in}{0.475000in}}{\pgfqpoint{1.317500in}{1.309000in}}%
\pgfusepath{clip}%
\pgfsetbuttcap%
\pgfsetroundjoin%
\pgfsetlinewidth{0.501875pt}%
\definecolor{currentstroke}{rgb}{0.000000,0.000000,1.000000}%
\pgfsetstrokecolor{currentstroke}%
\pgfsetdash{}{0pt}%
\pgfpathmoveto{\pgfqpoint{0.771335in}{1.669565in}}%
\pgfpathlineto{\pgfqpoint{0.771335in}{1.691453in}}%
\pgfusepath{stroke}%
\end{pgfscope}%
\begin{pgfscope}%
\pgfpathrectangle{\pgfqpoint{0.592283in}{0.475000in}}{\pgfqpoint{1.317500in}{1.309000in}}%
\pgfusepath{clip}%
\pgfsetbuttcap%
\pgfsetroundjoin%
\pgfsetlinewidth{0.501875pt}%
\definecolor{currentstroke}{rgb}{0.000000,0.000000,1.000000}%
\pgfsetstrokecolor{currentstroke}%
\pgfsetdash{}{0pt}%
\pgfpathmoveto{\pgfqpoint{0.749242in}{1.679009in}}%
\pgfpathlineto{\pgfqpoint{0.749242in}{1.693690in}}%
\pgfusepath{stroke}%
\end{pgfscope}%
\begin{pgfscope}%
\pgfpathrectangle{\pgfqpoint{0.592283in}{0.475000in}}{\pgfqpoint{1.317500in}{1.309000in}}%
\pgfusepath{clip}%
\pgfsetbuttcap%
\pgfsetroundjoin%
\pgfsetlinewidth{0.501875pt}%
\definecolor{currentstroke}{rgb}{0.000000,0.000000,1.000000}%
\pgfsetstrokecolor{currentstroke}%
\pgfsetdash{}{0pt}%
\pgfpathmoveto{\pgfqpoint{0.729254in}{1.682415in}}%
\pgfpathlineto{\pgfqpoint{0.729254in}{1.700474in}}%
\pgfusepath{stroke}%
\end{pgfscope}%
\begin{pgfscope}%
\pgfpathrectangle{\pgfqpoint{0.592283in}{0.475000in}}{\pgfqpoint{1.317500in}{1.309000in}}%
\pgfusepath{clip}%
\pgfsetbuttcap%
\pgfsetroundjoin%
\pgfsetlinewidth{0.501875pt}%
\definecolor{currentstroke}{rgb}{0.000000,0.000000,1.000000}%
\pgfsetstrokecolor{currentstroke}%
\pgfsetdash{}{0pt}%
\pgfpathmoveto{\pgfqpoint{0.711074in}{1.695822in}}%
\pgfpathlineto{\pgfqpoint{0.711074in}{1.709435in}}%
\pgfusepath{stroke}%
\end{pgfscope}%
\begin{pgfscope}%
\pgfpathrectangle{\pgfqpoint{0.592283in}{0.475000in}}{\pgfqpoint{1.317500in}{1.309000in}}%
\pgfusepath{clip}%
\pgfsetbuttcap%
\pgfsetroundjoin%
\pgfsetlinewidth{0.501875pt}%
\definecolor{currentstroke}{rgb}{0.000000,0.000000,1.000000}%
\pgfsetstrokecolor{currentstroke}%
\pgfsetdash{}{0pt}%
\pgfpathmoveto{\pgfqpoint{0.694459in}{1.697048in}}%
\pgfpathlineto{\pgfqpoint{0.694459in}{1.715169in}}%
\pgfusepath{stroke}%
\end{pgfscope}%
\begin{pgfscope}%
\pgfpathrectangle{\pgfqpoint{0.592283in}{0.475000in}}{\pgfqpoint{1.317500in}{1.309000in}}%
\pgfusepath{clip}%
\pgfsetbuttcap%
\pgfsetroundjoin%
\pgfsetlinewidth{0.501875pt}%
\definecolor{currentstroke}{rgb}{0.000000,0.000000,1.000000}%
\pgfsetstrokecolor{currentstroke}%
\pgfsetdash{}{0pt}%
\pgfpathmoveto{\pgfqpoint{0.679211in}{1.695246in}}%
\pgfpathlineto{\pgfqpoint{0.679211in}{1.710260in}}%
\pgfusepath{stroke}%
\end{pgfscope}%
\begin{pgfscope}%
\pgfpathrectangle{\pgfqpoint{0.592283in}{0.475000in}}{\pgfqpoint{1.317500in}{1.309000in}}%
\pgfusepath{clip}%
\pgfsetbuttcap%
\pgfsetroundjoin%
\pgfsetlinewidth{0.501875pt}%
\definecolor{currentstroke}{rgb}{0.000000,0.000000,1.000000}%
\pgfsetstrokecolor{currentstroke}%
\pgfsetdash{}{0pt}%
\pgfpathmoveto{\pgfqpoint{0.665161in}{1.698043in}}%
\pgfpathlineto{\pgfqpoint{0.665161in}{1.709700in}}%
\pgfusepath{stroke}%
\end{pgfscope}%
\begin{pgfscope}%
\pgfpathrectangle{\pgfqpoint{0.592283in}{0.475000in}}{\pgfqpoint{1.317500in}{1.309000in}}%
\pgfusepath{clip}%
\pgfsetbuttcap%
\pgfsetroundjoin%
\pgfsetlinewidth{0.501875pt}%
\definecolor{currentstroke}{rgb}{0.000000,0.000000,1.000000}%
\pgfsetstrokecolor{currentstroke}%
\pgfsetdash{}{0pt}%
\pgfpathmoveto{\pgfqpoint{0.652169in}{1.696089in}}%
\pgfpathlineto{\pgfqpoint{0.652169in}{1.712150in}}%
\pgfusepath{stroke}%
\end{pgfscope}%
\begin{pgfscope}%
\pgfpathrectangle{\pgfqpoint{0.592283in}{0.475000in}}{\pgfqpoint{1.317500in}{1.309000in}}%
\pgfusepath{clip}%
\pgfsetbuttcap%
\pgfsetroundjoin%
\pgfsetlinewidth{0.501875pt}%
\definecolor{currentstroke}{rgb}{1.000000,0.000000,0.000000}%
\pgfsetstrokecolor{currentstroke}%
\pgfsetdash{}{0pt}%
\pgfpathmoveto{\pgfqpoint{1.849896in}{1.173087in}}%
\pgfpathlineto{\pgfqpoint{1.849896in}{1.206205in}}%
\pgfusepath{stroke}%
\end{pgfscope}%
\begin{pgfscope}%
\pgfpathrectangle{\pgfqpoint{0.592283in}{0.475000in}}{\pgfqpoint{1.317500in}{1.309000in}}%
\pgfusepath{clip}%
\pgfsetbuttcap%
\pgfsetroundjoin%
\pgfsetlinewidth{0.501875pt}%
\definecolor{currentstroke}{rgb}{1.000000,0.000000,0.000000}%
\pgfsetstrokecolor{currentstroke}%
\pgfsetdash{}{0pt}%
\pgfpathmoveto{\pgfqpoint{1.592755in}{1.471254in}}%
\pgfpathlineto{\pgfqpoint{1.592755in}{1.493346in}}%
\pgfusepath{stroke}%
\end{pgfscope}%
\begin{pgfscope}%
\pgfpathrectangle{\pgfqpoint{0.592283in}{0.475000in}}{\pgfqpoint{1.317500in}{1.309000in}}%
\pgfusepath{clip}%
\pgfsetbuttcap%
\pgfsetroundjoin%
\pgfsetlinewidth{0.501875pt}%
\definecolor{currentstroke}{rgb}{1.000000,0.000000,0.000000}%
\pgfsetstrokecolor{currentstroke}%
\pgfsetdash{}{0pt}%
\pgfpathmoveto{\pgfqpoint{1.414189in}{1.611669in}}%
\pgfpathlineto{\pgfqpoint{1.414189in}{1.644217in}}%
\pgfusepath{stroke}%
\end{pgfscope}%
\begin{pgfscope}%
\pgfpathrectangle{\pgfqpoint{0.592283in}{0.475000in}}{\pgfqpoint{1.317500in}{1.309000in}}%
\pgfusepath{clip}%
\pgfsetbuttcap%
\pgfsetroundjoin%
\pgfsetlinewidth{0.501875pt}%
\definecolor{currentstroke}{rgb}{1.000000,0.000000,0.000000}%
\pgfsetstrokecolor{currentstroke}%
\pgfsetdash{}{0pt}%
\pgfpathmoveto{\pgfqpoint{1.282259in}{1.670793in}}%
\pgfpathlineto{\pgfqpoint{1.282259in}{1.687242in}}%
\pgfusepath{stroke}%
\end{pgfscope}%
\begin{pgfscope}%
\pgfpathrectangle{\pgfqpoint{0.592283in}{0.475000in}}{\pgfqpoint{1.317500in}{1.309000in}}%
\pgfusepath{clip}%
\pgfsetbuttcap%
\pgfsetroundjoin%
\pgfsetlinewidth{0.501875pt}%
\definecolor{currentstroke}{rgb}{1.000000,0.000000,0.000000}%
\pgfsetstrokecolor{currentstroke}%
\pgfsetdash{}{0pt}%
\pgfpathmoveto{\pgfqpoint{1.180415in}{1.691135in}}%
\pgfpathlineto{\pgfqpoint{1.180415in}{1.710394in}}%
\pgfusepath{stroke}%
\end{pgfscope}%
\begin{pgfscope}%
\pgfpathrectangle{\pgfqpoint{0.592283in}{0.475000in}}{\pgfqpoint{1.317500in}{1.309000in}}%
\pgfusepath{clip}%
\pgfsetbuttcap%
\pgfsetroundjoin%
\pgfsetlinewidth{0.501875pt}%
\definecolor{currentstroke}{rgb}{1.000000,0.000000,0.000000}%
\pgfsetstrokecolor{currentstroke}%
\pgfsetdash{}{0pt}%
\pgfpathmoveto{\pgfqpoint{1.099181in}{1.698461in}}%
\pgfpathlineto{\pgfqpoint{1.099181in}{1.711270in}}%
\pgfusepath{stroke}%
\end{pgfscope}%
\begin{pgfscope}%
\pgfpathrectangle{\pgfqpoint{0.592283in}{0.475000in}}{\pgfqpoint{1.317500in}{1.309000in}}%
\pgfusepath{clip}%
\pgfsetbuttcap%
\pgfsetroundjoin%
\pgfsetlinewidth{0.501875pt}%
\definecolor{currentstroke}{rgb}{1.000000,0.000000,0.000000}%
\pgfsetstrokecolor{currentstroke}%
\pgfsetdash{}{0pt}%
\pgfpathmoveto{\pgfqpoint{1.032726in}{1.703538in}}%
\pgfpathlineto{\pgfqpoint{1.032726in}{1.718123in}}%
\pgfusepath{stroke}%
\end{pgfscope}%
\begin{pgfscope}%
\pgfpathrectangle{\pgfqpoint{0.592283in}{0.475000in}}{\pgfqpoint{1.317500in}{1.309000in}}%
\pgfusepath{clip}%
\pgfsetbuttcap%
\pgfsetroundjoin%
\pgfsetlinewidth{0.501875pt}%
\definecolor{currentstroke}{rgb}{1.000000,0.000000,0.000000}%
\pgfsetstrokecolor{currentstroke}%
\pgfsetdash{}{0pt}%
\pgfpathmoveto{\pgfqpoint{0.977249in}{1.708827in}}%
\pgfpathlineto{\pgfqpoint{0.977249in}{1.720787in}}%
\pgfusepath{stroke}%
\end{pgfscope}%
\begin{pgfscope}%
\pgfpathrectangle{\pgfqpoint{0.592283in}{0.475000in}}{\pgfqpoint{1.317500in}{1.309000in}}%
\pgfusepath{clip}%
\pgfsetbuttcap%
\pgfsetroundjoin%
\pgfsetlinewidth{0.501875pt}%
\definecolor{currentstroke}{rgb}{1.000000,0.000000,0.000000}%
\pgfsetstrokecolor{currentstroke}%
\pgfsetdash{}{0pt}%
\pgfpathmoveto{\pgfqpoint{0.930167in}{1.703606in}}%
\pgfpathlineto{\pgfqpoint{0.930167in}{1.722031in}}%
\pgfusepath{stroke}%
\end{pgfscope}%
\begin{pgfscope}%
\pgfpathrectangle{\pgfqpoint{0.592283in}{0.475000in}}{\pgfqpoint{1.317500in}{1.309000in}}%
\pgfusepath{clip}%
\pgfsetbuttcap%
\pgfsetroundjoin%
\pgfsetlinewidth{0.501875pt}%
\definecolor{currentstroke}{rgb}{1.000000,0.000000,0.000000}%
\pgfsetstrokecolor{currentstroke}%
\pgfsetdash{}{0pt}%
\pgfpathmoveto{\pgfqpoint{0.889656in}{1.702587in}}%
\pgfpathlineto{\pgfqpoint{0.889656in}{1.717582in}}%
\pgfusepath{stroke}%
\end{pgfscope}%
\begin{pgfscope}%
\pgfpathrectangle{\pgfqpoint{0.592283in}{0.475000in}}{\pgfqpoint{1.317500in}{1.309000in}}%
\pgfusepath{clip}%
\pgfsetbuttcap%
\pgfsetroundjoin%
\pgfsetlinewidth{0.501875pt}%
\definecolor{currentstroke}{rgb}{1.000000,0.000000,0.000000}%
\pgfsetstrokecolor{currentstroke}%
\pgfsetdash{}{0pt}%
\pgfpathmoveto{\pgfqpoint{0.854393in}{1.700493in}}%
\pgfpathlineto{\pgfqpoint{0.854393in}{1.715700in}}%
\pgfusepath{stroke}%
\end{pgfscope}%
\begin{pgfscope}%
\pgfpathrectangle{\pgfqpoint{0.592283in}{0.475000in}}{\pgfqpoint{1.317500in}{1.309000in}}%
\pgfusepath{clip}%
\pgfsetbuttcap%
\pgfsetroundjoin%
\pgfsetlinewidth{0.501875pt}%
\definecolor{currentstroke}{rgb}{1.000000,0.000000,0.000000}%
\pgfsetstrokecolor{currentstroke}%
\pgfsetdash{}{0pt}%
\pgfpathmoveto{\pgfqpoint{0.823390in}{1.693093in}}%
\pgfpathlineto{\pgfqpoint{0.823390in}{1.718875in}}%
\pgfusepath{stroke}%
\end{pgfscope}%
\begin{pgfscope}%
\pgfpathrectangle{\pgfqpoint{0.592283in}{0.475000in}}{\pgfqpoint{1.317500in}{1.309000in}}%
\pgfusepath{clip}%
\pgfsetbuttcap%
\pgfsetroundjoin%
\pgfsetlinewidth{0.501875pt}%
\definecolor{currentstroke}{rgb}{1.000000,0.000000,0.000000}%
\pgfsetstrokecolor{currentstroke}%
\pgfsetdash{}{0pt}%
\pgfpathmoveto{\pgfqpoint{0.795899in}{1.701539in}}%
\pgfpathlineto{\pgfqpoint{0.795899in}{1.720370in}}%
\pgfusepath{stroke}%
\end{pgfscope}%
\begin{pgfscope}%
\pgfpathrectangle{\pgfqpoint{0.592283in}{0.475000in}}{\pgfqpoint{1.317500in}{1.309000in}}%
\pgfusepath{clip}%
\pgfsetbuttcap%
\pgfsetroundjoin%
\pgfsetlinewidth{0.501875pt}%
\definecolor{currentstroke}{rgb}{1.000000,0.000000,0.000000}%
\pgfsetstrokecolor{currentstroke}%
\pgfsetdash{}{0pt}%
\pgfpathmoveto{\pgfqpoint{0.771335in}{1.701913in}}%
\pgfpathlineto{\pgfqpoint{0.771335in}{1.716268in}}%
\pgfusepath{stroke}%
\end{pgfscope}%
\begin{pgfscope}%
\pgfpathrectangle{\pgfqpoint{0.592283in}{0.475000in}}{\pgfqpoint{1.317500in}{1.309000in}}%
\pgfusepath{clip}%
\pgfsetbuttcap%
\pgfsetroundjoin%
\pgfsetlinewidth{0.501875pt}%
\definecolor{currentstroke}{rgb}{1.000000,0.000000,0.000000}%
\pgfsetstrokecolor{currentstroke}%
\pgfsetdash{}{0pt}%
\pgfpathmoveto{\pgfqpoint{0.749242in}{1.698984in}}%
\pgfpathlineto{\pgfqpoint{0.749242in}{1.713730in}}%
\pgfusepath{stroke}%
\end{pgfscope}%
\begin{pgfscope}%
\pgfpathrectangle{\pgfqpoint{0.592283in}{0.475000in}}{\pgfqpoint{1.317500in}{1.309000in}}%
\pgfusepath{clip}%
\pgfsetbuttcap%
\pgfsetroundjoin%
\pgfsetlinewidth{0.501875pt}%
\definecolor{currentstroke}{rgb}{1.000000,0.000000,0.000000}%
\pgfsetstrokecolor{currentstroke}%
\pgfsetdash{}{0pt}%
\pgfpathmoveto{\pgfqpoint{0.729254in}{1.703623in}}%
\pgfpathlineto{\pgfqpoint{0.729254in}{1.724500in}}%
\pgfusepath{stroke}%
\end{pgfscope}%
\begin{pgfscope}%
\pgfpathrectangle{\pgfqpoint{0.592283in}{0.475000in}}{\pgfqpoint{1.317500in}{1.309000in}}%
\pgfusepath{clip}%
\pgfsetbuttcap%
\pgfsetroundjoin%
\pgfsetlinewidth{0.501875pt}%
\definecolor{currentstroke}{rgb}{1.000000,0.000000,0.000000}%
\pgfsetstrokecolor{currentstroke}%
\pgfsetdash{}{0pt}%
\pgfpathmoveto{\pgfqpoint{0.711074in}{1.711158in}}%
\pgfpathlineto{\pgfqpoint{0.711074in}{1.721190in}}%
\pgfusepath{stroke}%
\end{pgfscope}%
\begin{pgfscope}%
\pgfpathrectangle{\pgfqpoint{0.592283in}{0.475000in}}{\pgfqpoint{1.317500in}{1.309000in}}%
\pgfusepath{clip}%
\pgfsetbuttcap%
\pgfsetroundjoin%
\pgfsetlinewidth{0.501875pt}%
\definecolor{currentstroke}{rgb}{1.000000,0.000000,0.000000}%
\pgfsetstrokecolor{currentstroke}%
\pgfsetdash{}{0pt}%
\pgfpathmoveto{\pgfqpoint{0.694459in}{1.708153in}}%
\pgfpathlineto{\pgfqpoint{0.694459in}{1.718975in}}%
\pgfusepath{stroke}%
\end{pgfscope}%
\begin{pgfscope}%
\pgfpathrectangle{\pgfqpoint{0.592283in}{0.475000in}}{\pgfqpoint{1.317500in}{1.309000in}}%
\pgfusepath{clip}%
\pgfsetbuttcap%
\pgfsetroundjoin%
\pgfsetlinewidth{0.501875pt}%
\definecolor{currentstroke}{rgb}{1.000000,0.000000,0.000000}%
\pgfsetstrokecolor{currentstroke}%
\pgfsetdash{}{0pt}%
\pgfpathmoveto{\pgfqpoint{0.679211in}{1.699484in}}%
\pgfpathlineto{\pgfqpoint{0.679211in}{1.717206in}}%
\pgfusepath{stroke}%
\end{pgfscope}%
\begin{pgfscope}%
\pgfpathrectangle{\pgfqpoint{0.592283in}{0.475000in}}{\pgfqpoint{1.317500in}{1.309000in}}%
\pgfusepath{clip}%
\pgfsetbuttcap%
\pgfsetroundjoin%
\pgfsetlinewidth{0.501875pt}%
\definecolor{currentstroke}{rgb}{1.000000,0.000000,0.000000}%
\pgfsetstrokecolor{currentstroke}%
\pgfsetdash{}{0pt}%
\pgfpathmoveto{\pgfqpoint{0.665161in}{1.706059in}}%
\pgfpathlineto{\pgfqpoint{0.665161in}{1.723555in}}%
\pgfusepath{stroke}%
\end{pgfscope}%
\begin{pgfscope}%
\pgfpathrectangle{\pgfqpoint{0.592283in}{0.475000in}}{\pgfqpoint{1.317500in}{1.309000in}}%
\pgfusepath{clip}%
\pgfsetbuttcap%
\pgfsetroundjoin%
\pgfsetlinewidth{0.501875pt}%
\definecolor{currentstroke}{rgb}{1.000000,0.000000,0.000000}%
\pgfsetstrokecolor{currentstroke}%
\pgfsetdash{}{0pt}%
\pgfpathmoveto{\pgfqpoint{0.652169in}{1.705330in}}%
\pgfpathlineto{\pgfqpoint{0.652169in}{1.724284in}}%
\pgfusepath{stroke}%
\end{pgfscope}%
\begin{pgfscope}%
\pgfpathrectangle{\pgfqpoint{0.592283in}{0.475000in}}{\pgfqpoint{1.317500in}{1.309000in}}%
\pgfusepath{clip}%
\pgfsetbuttcap%
\pgfsetroundjoin%
\pgfsetlinewidth{0.501875pt}%
\definecolor{currentstroke}{rgb}{0.000000,0.000000,0.000000}%
\pgfsetstrokecolor{currentstroke}%
\pgfsetdash{{1.850000pt}{0.800000pt}}{0.000000pt}%
\pgfpathmoveto{\pgfqpoint{0.592283in}{1.715056in}}%
\pgfpathlineto{\pgfqpoint{1.909783in}{1.715056in}}%
\pgfusepath{stroke}%
\end{pgfscope}%
\begin{pgfscope}%
\pgfpathrectangle{\pgfqpoint{0.592283in}{0.475000in}}{\pgfqpoint{1.317500in}{1.309000in}}%
\pgfusepath{clip}%
\pgfsetrectcap%
\pgfsetroundjoin%
\pgfsetlinewidth{0.501875pt}%
\definecolor{currentstroke}{rgb}{0.000000,0.000000,1.000000}%
\pgfsetstrokecolor{currentstroke}%
\pgfsetdash{}{0pt}%
\pgfpathmoveto{\pgfqpoint{1.849896in}{0.534500in}}%
\pgfpathlineto{\pgfqpoint{1.592755in}{0.534500in}}%
\pgfpathlineto{\pgfqpoint{1.414189in}{0.541211in}}%
\pgfpathlineto{\pgfqpoint{1.282259in}{0.792358in}}%
\pgfpathlineto{\pgfqpoint{1.180415in}{1.335289in}}%
\pgfpathlineto{\pgfqpoint{1.099181in}{1.606817in}}%
\pgfpathlineto{\pgfqpoint{1.032726in}{1.683740in}}%
\pgfpathlineto{\pgfqpoint{0.977249in}{1.700143in}}%
\pgfpathlineto{\pgfqpoint{0.930167in}{1.700889in}}%
\pgfpathlineto{\pgfqpoint{0.889656in}{1.686598in}}%
\pgfpathlineto{\pgfqpoint{0.854393in}{1.680633in}}%
\pgfpathlineto{\pgfqpoint{0.823390in}{1.674544in}}%
\pgfpathlineto{\pgfqpoint{0.795899in}{1.673923in}}%
\pgfpathlineto{\pgfqpoint{0.771335in}{1.680509in}}%
\pgfpathlineto{\pgfqpoint{0.749242in}{1.686349in}}%
\pgfpathlineto{\pgfqpoint{0.729254in}{1.691444in}}%
\pgfpathlineto{\pgfqpoint{0.711074in}{1.702629in}}%
\pgfpathlineto{\pgfqpoint{0.694459in}{1.706108in}}%
\pgfpathlineto{\pgfqpoint{0.679211in}{1.702753in}}%
\pgfpathlineto{\pgfqpoint{0.665161in}{1.703871in}}%
\pgfpathlineto{\pgfqpoint{0.652169in}{1.704120in}}%
\pgfusepath{stroke}%
\end{pgfscope}%
\begin{pgfscope}%
\pgfpathrectangle{\pgfqpoint{0.592283in}{0.475000in}}{\pgfqpoint{1.317500in}{1.309000in}}%
\pgfusepath{clip}%
\pgfsetrectcap%
\pgfsetroundjoin%
\pgfsetlinewidth{0.501875pt}%
\definecolor{currentstroke}{rgb}{1.000000,0.000000,0.000000}%
\pgfsetstrokecolor{currentstroke}%
\pgfsetdash{}{0pt}%
\pgfpathmoveto{\pgfqpoint{1.849896in}{1.189646in}}%
\pgfpathlineto{\pgfqpoint{1.592755in}{1.482300in}}%
\pgfpathlineto{\pgfqpoint{1.414189in}{1.627943in}}%
\pgfpathlineto{\pgfqpoint{1.282259in}{1.679018in}}%
\pgfpathlineto{\pgfqpoint{1.180415in}{1.700765in}}%
\pgfpathlineto{\pgfqpoint{1.099181in}{1.704865in}}%
\pgfpathlineto{\pgfqpoint{1.032726in}{1.710830in}}%
\pgfpathlineto{\pgfqpoint{0.977249in}{1.714807in}}%
\pgfpathlineto{\pgfqpoint{0.930167in}{1.712819in}}%
\pgfpathlineto{\pgfqpoint{0.889656in}{1.710085in}}%
\pgfpathlineto{\pgfqpoint{0.854393in}{1.708096in}}%
\pgfpathlineto{\pgfqpoint{0.823390in}{1.705984in}}%
\pgfpathlineto{\pgfqpoint{0.795899in}{1.710955in}}%
\pgfpathlineto{\pgfqpoint{0.771335in}{1.709091in}}%
\pgfpathlineto{\pgfqpoint{0.749242in}{1.706357in}}%
\pgfpathlineto{\pgfqpoint{0.729254in}{1.714061in}}%
\pgfpathlineto{\pgfqpoint{0.711074in}{1.716174in}}%
\pgfpathlineto{\pgfqpoint{0.694459in}{1.713564in}}%
\pgfpathlineto{\pgfqpoint{0.679211in}{1.708345in}}%
\pgfpathlineto{\pgfqpoint{0.665161in}{1.714807in}}%
\pgfpathlineto{\pgfqpoint{0.652169in}{1.714807in}}%
\pgfusepath{stroke}%
\end{pgfscope}%
\begin{pgfscope}%
\pgfsetrectcap%
\pgfsetmiterjoin%
\pgfsetlinewidth{0.803000pt}%
\definecolor{currentstroke}{rgb}{0.000000,0.000000,0.000000}%
\pgfsetstrokecolor{currentstroke}%
\pgfsetdash{}{0pt}%
\pgfpathmoveto{\pgfqpoint{0.592283in}{0.475000in}}%
\pgfpathlineto{\pgfqpoint{0.592283in}{1.784000in}}%
\pgfusepath{stroke}%
\end{pgfscope}%
\begin{pgfscope}%
\pgfsetrectcap%
\pgfsetmiterjoin%
\pgfsetlinewidth{0.803000pt}%
\definecolor{currentstroke}{rgb}{0.000000,0.000000,0.000000}%
\pgfsetstrokecolor{currentstroke}%
\pgfsetdash{}{0pt}%
\pgfpathmoveto{\pgfqpoint{1.909783in}{0.475000in}}%
\pgfpathlineto{\pgfqpoint{1.909783in}{1.784000in}}%
\pgfusepath{stroke}%
\end{pgfscope}%
\begin{pgfscope}%
\pgfsetrectcap%
\pgfsetmiterjoin%
\pgfsetlinewidth{0.803000pt}%
\definecolor{currentstroke}{rgb}{0.000000,0.000000,0.000000}%
\pgfsetstrokecolor{currentstroke}%
\pgfsetdash{}{0pt}%
\pgfpathmoveto{\pgfqpoint{0.592283in}{0.475000in}}%
\pgfpathlineto{\pgfqpoint{1.909783in}{0.475000in}}%
\pgfusepath{stroke}%
\end{pgfscope}%
\begin{pgfscope}%
\pgfsetrectcap%
\pgfsetmiterjoin%
\pgfsetlinewidth{0.803000pt}%
\definecolor{currentstroke}{rgb}{0.000000,0.000000,0.000000}%
\pgfsetstrokecolor{currentstroke}%
\pgfsetdash{}{0pt}%
\pgfpathmoveto{\pgfqpoint{0.592283in}{1.784000in}}%
\pgfpathlineto{\pgfqpoint{1.909783in}{1.784000in}}%
\pgfusepath{stroke}%
\end{pgfscope}%
\begin{pgfscope}%
\pgfsetbuttcap%
\pgfsetmiterjoin%
\definecolor{currentfill}{rgb}{1.000000,1.000000,1.000000}%
\pgfsetfillcolor{currentfill}%
\pgfsetfillopacity{0.500000}%
\pgfsetlinewidth{1.003750pt}%
\definecolor{currentstroke}{rgb}{0.800000,0.800000,0.800000}%
\pgfsetstrokecolor{currentstroke}%
\pgfsetstrokeopacity{0.500000}%
\pgfsetdash{}{0pt}%
\pgfpathmoveto{\pgfqpoint{0.679783in}{0.537500in}}%
\pgfpathlineto{\pgfqpoint{1.565119in}{0.537500in}}%
\pgfpathquadraticcurveto{\pgfqpoint{1.590119in}{0.537500in}}{\pgfqpoint{1.590119in}{0.562500in}}%
\pgfpathlineto{\pgfqpoint{1.590119in}{0.898611in}}%
\pgfpathquadraticcurveto{\pgfqpoint{1.590119in}{0.923611in}}{\pgfqpoint{1.565119in}{0.923611in}}%
\pgfpathlineto{\pgfqpoint{0.679783in}{0.923611in}}%
\pgfpathquadraticcurveto{\pgfqpoint{0.654783in}{0.923611in}}{\pgfqpoint{0.654783in}{0.898611in}}%
\pgfpathlineto{\pgfqpoint{0.654783in}{0.562500in}}%
\pgfpathquadraticcurveto{\pgfqpoint{0.654783in}{0.537500in}}{\pgfqpoint{0.679783in}{0.537500in}}%
\pgfpathclose%
\pgfusepath{stroke,fill}%
\end{pgfscope}%
\begin{pgfscope}%
\pgfsetbuttcap%
\pgfsetroundjoin%
\pgfsetlinewidth{0.501875pt}%
\definecolor{currentstroke}{rgb}{0.000000,0.000000,1.000000}%
\pgfsetstrokecolor{currentstroke}%
\pgfsetdash{}{0pt}%
\pgfpathmoveto{\pgfqpoint{0.829783in}{0.767361in}}%
\pgfpathlineto{\pgfqpoint{0.829783in}{0.892361in}}%
\pgfusepath{stroke}%
\end{pgfscope}%
\begin{pgfscope}%
\pgfsetrectcap%
\pgfsetroundjoin%
\pgfsetlinewidth{0.501875pt}%
\definecolor{currentstroke}{rgb}{0.000000,0.000000,1.000000}%
\pgfsetstrokecolor{currentstroke}%
\pgfsetdash{}{0pt}%
\pgfpathmoveto{\pgfqpoint{0.704783in}{0.829861in}}%
\pgfpathlineto{\pgfqpoint{0.954783in}{0.829861in}}%
\pgfusepath{stroke}%
\end{pgfscope}%
\begin{pgfscope}%
\definecolor{textcolor}{rgb}{0.000000,0.000000,0.000000}%
\pgfsetstrokecolor{textcolor}%
\pgfsetfillcolor{textcolor}%
\pgftext[x=1.054782in,y=0.786111in,left,base]{\color{textcolor}\rmfamily\fontsize{9.000000}{10.800000}\selectfont Plug-in}%
\end{pgfscope}%
\begin{pgfscope}%
\pgfsetbuttcap%
\pgfsetroundjoin%
\pgfsetlinewidth{0.501875pt}%
\definecolor{currentstroke}{rgb}{1.000000,0.000000,0.000000}%
\pgfsetstrokecolor{currentstroke}%
\pgfsetdash{}{0pt}%
\pgfpathmoveto{\pgfqpoint{0.829783in}{0.593056in}}%
\pgfpathlineto{\pgfqpoint{0.829783in}{0.718056in}}%
\pgfusepath{stroke}%
\end{pgfscope}%
\begin{pgfscope}%
\pgfsetrectcap%
\pgfsetroundjoin%
\pgfsetlinewidth{0.501875pt}%
\definecolor{currentstroke}{rgb}{1.000000,0.000000,0.000000}%
\pgfsetstrokecolor{currentstroke}%
\pgfsetdash{}{0pt}%
\pgfpathmoveto{\pgfqpoint{0.704783in}{0.655556in}}%
\pgfpathlineto{\pgfqpoint{0.954783in}{0.655556in}}%
\pgfusepath{stroke}%
\end{pgfscope}%
\begin{pgfscope}%
\definecolor{textcolor}{rgb}{0.000000,0.000000,0.000000}%
\pgfsetstrokecolor{textcolor}%
\pgfsetfillcolor{textcolor}%
\pgftext[x=1.054782in,y=0.611806in,left,base]{\color{textcolor}\rmfamily\fontsize{9.000000}{10.800000}\selectfont Splitting}%
\end{pgfscope}%
\end{pgfpicture}%
\makeatother%
\endgroup%